\documentclass{article}
\usepackage[utf8]{inputenc}
\usepackage{times}
\usepackage[sort&compress,numbers]{natbib}
\usepackage[colorlinks=true,citecolor=blue,breaklinks]{hyperref}
\usepackage[hyphenbreaks]{breakurl} 
\usepackage[tracking]{microtype}

\usepackage{graphicx}
\usepackage{subfigure}
\usepackage{amsmath,amssymb,amsfonts,amsthm,mathrsfs,bm}
\usepackage{algpseudocode}
\usepackage{algorithm}

\usepackage{multicol}
\usepackage{multirow}
\usepackage{color}

\usepackage{xspace}
\usepackage{url}
\usepackage{enumerate}
\usepackage{float}
\usepackage{listings}

\usepackage{hyperref}
\usepackage{tabularx}

\hypersetup{
  colorlinks=true,
  citecolor=blue,
  linkcolor=red,
  urlcolor=magenta,
}

\makeatletter
\newlength\aftertitskip     \newlength\beforetitskip
\newlength\interauthorskip  \newlength\aftermaketitskip

\def\maketitle{\par
 \begingroup
   \def\thefootnote{\fnsymbol{footnote}}
   \def\@makefnmark{\hbox to 4pt{$^{\@thefnmark}$\hss}}
   \@maketitle \@thanks
 \endgroup
\setcounter{footnote}{0}
 \let\maketitle\relax \let\@maketitle\relax
 \gdef\@thanks{}\gdef\@author{}\gdef\@title{}\let\thanks\relax}

\def\@startauthor{\noindent \normalsize\bf}
\def\@endauthor{}
\def\@starteditor{\noindent \small {\bf Editor:~}}
\def\@endeditor{\normalsize}
\def\@maketitle{\vbox{\hsize\textwidth
 \linewidth\hsize \vskip \beforetitskip
 {\begin{center} \LARGE\@title \par \end{center}} \vskip \aftertitskip
 {\def\and{\unskip\enspace{\rm and}\enspace}%
  \def\addr{\small\it}%
  \def\email{\hfill\small\tt}%
  \def\name{\normalsize\bf}%
  \def\AND{\@endauthor\rm\hss \vskip \interauthorskip \@startauthor}
  \@startauthor \@author \@endauthor}
}}

\makeatother

\newcommand{\va}{\bm{a}}                          
\newcommand{\vb}{\bm{b}}                          
                          
\newcommand{\vd}{\bm{d}}

\newcommand{\vi}{\bm{i}}

\newcommand{\vp}{\bm{p}}                          
\newcommand{\vq}{\bm{q}}                          
\newcommand{\vr}{\bm{r}}                          
\newcommand{\vs}{\bm{s}}                          
                        
\newcommand{\vu}{\bm{u}}                          
\newcommand{\vv}{\bm{v}}                          
\newcommand{\vw}{\bm{w}}                          
\newcommand{\vx}{\bm{x}}                          
\newcommand{\vy}{\bm{y}}                          
\newcommand{\vz}{\bm{z}}                          

\newcommand{\valpha}{\bm{\alpha}}
\newcommand{\vbeta}{\bm{\beta}}

\newcommand{\Xc}{\mathcal{X}}

\newcommand{\Zc}{\mathcal{Z}}

\newcommand{\ma}{\bm{A}}
\newcommand{\mb}{\bm{B}}
\newcommand{\md}{\bm{D}}
\newcommand{\mH}{\bm{H}}

\newcommand{\mi}{\bm{I}}
\newcommand{\mk}{\bm{K}}
\newcommand{\ml}{\bm{L}}
\newcommand{\mn}{\bm{N}}

\newcommand{\mr}{\bm{R}}

\newcommand{\mx}{\bm{X}}
\newcommand{\my}{\bm{Y}}

\newcommand{\tu}{\mathsf{U}}
\newcommand{\tx}{\mathsf{X}}
\newcommand{\ty}{\mathsf{Y}}

\newcommand{\tw}{\mathsf{W}}

\newcommand{\mynorm}[2]{\| {#1} \|_{#2}}
\newcommand{\norm}[2]{\mynorm{#1}{#2}}

\newcommand{\infnorm}[1]{\mynorm{#1}{\infty}}

\newcommand{\enorm}[1]{\mynorm{#1}{2}}

\newcommand{\frob}[1]{\|{#1}\|_{\text{F}}}

\DeclareMathOperator*{\argmin}{argmin}

\DeclareMathOperator{\dom}{dom}

\DeclareMathOperator{\sgn}{sgn}

\DeclareMathOperator{\Diag}{Diag}

\DeclareMathOperator{\vect}{vec}
\DeclareMathOperator{\prox}{prox}

\def\abovestrut#1{\rule[0in]{0in}{#1}\ignorespaces}
\def\belowstrut#1{\rule[-#1]{0in}{#1}\ignorespaces}
\def\abovespace{\abovestrut{0.20in}}

\def\belowspace{\belowstrut{0.10in}}

\newcommand{\infconv}{\:\square\:}
\newcommand{\set}[1]{\left\{ {#1}\right\}}
\newcommand{\nlsum}{\sum\nolimits}
\newcommand{\nlmin}{\min\nolimits}

\newcommand{\defeq}{\ \stackrel{\text{def}}{=}\ }
\newcommand{\ip}[2]{\left\langle#1, #2\right\rangle}
\newcommand{\reals}{\mathbb{R}}
\newcommand{\half}{\tfrac{1}{2}}
\newcommand{\sml}[1]{{\small #1}}
\newcommand{\fromto}[3]{\sml{$#1 \le #2 \le #3$}}
\newcommand{\tvell}{\text{Tv}}

\newcommand{\ind}{\delta}

\newcommand{\trip}{\textsc{Trip}\xspace}

\newcommand{\oned}{\text{1D}}
\newcommand{\twod}{\text{2D}}
\newcommand{\threed}{\text{3D}}

\newtheorem{theorem}{Theorem}

\newtheorem{proposition}[theorem]{Proposition}

\newtheorem{defn}[theorem]{Definition}

\usepackage{accents}
\newcommand{\ubar}[1]{\underaccent{\bar}{#1}}
\graphicspath{{./}}

\numberwithin{equation}{section}
\numberwithin{theorem}{section}

\newcommand{\MIT}{Laboratory for Information and Decision Systems\\ Massachusetts Institute of Technology (MIT), Cambridge, MA}
\newcommand{\mpi}{An initial version of this work was performed during 2013-14, when the author was with the Max Planck Institute for Intelligent Systems, Tübingen, Germany, and with Carnegie Mellon University, Pittsburgh.}
\newcommand{\iisc}{Instituto de Ingenier\'ia del Conocimiento and Universidad Aut\'onoma de Madrid\\Francisco Tom\'as y Valiente 11, Madrid, Spain}

\begin{document}

\title{Modular proximal optimization for multidimensional total-variation regularization}
\author{\name \'Alvaro Barbero
  \email alvaro.barbero@inv.uam.es\\
  \addr \iisc
  \AND
  \name Suvrit Sra\thanks{\mpi} \email suvrit@mit.edu\\
  \addr \MIT
}

\maketitle

\begin{abstract}
  We study \emph{TV regularization}, a widely used technique for eliciting structured sparsity. In particular, we propose efficient algorithms for computing prox-operators for $\ell_p$-norm TV. The most important among these is $\ell_1$-norm TV, for whose prox-operator we present a new geometric analysis which unveils a hitherto unknown connection to taut-string methods. This connection turns out to be remarkably useful as it shows how our geometry guided implementation results in efficient weighted and unweighted 1D-TV solvers, surpassing   state-of-the-art methods. Our 1D-TV solvers provide the backbone for building more complex (two or higher-dimensional) TV solvers within a modular proximal optimization approach. We review the literature for an array of methods exploiting this strategy, and illustrate the benefits of our modular design through extensive suite of experiments on (i) image denoising, (ii) image deconvolution, (iii) four variants of fused-lasso, and (iv) video denoising. To underscore our claims and permit easy reproducibility, we provide all the reviewed and our new TV solvers in an easy to use multi-threaded C++, Matlab and Python library.
\end{abstract}

\section{Introduction}
Sparsity impacts the entire data analysis pipeline, touching algorithmic, modeling, as well as practical aspects. Most commonly, sparsity is elicited via $\ell_1$-norm regularization~\citep{LASSO,candesTao04}. However, numerous applications rely on  more refined ``structured'' notions of sparsity, e.g., groupwise-sparsity~\citep{Meier08,liu09b,yuan06,bach11}, hierarchical  sparsity~\citep{bach10,Mairal10}, gradient sparsity~\citep{RudinTV92,vogel,fl}, or sparsity over structured `atoms'~\citep{atomic}. 

Such regularizers typically arise in optimization problems of the form
\begin{equation}
  \label{eq.1}
  \nlmin_{\vx \in \reals^n}\quad \Phi(\vx) := \ell(\vx) + r(\vx),
\end{equation}
where $\ell: \reals^n \to \reals$ is a smooth loss function (often convex), while  $r: \reals^n \to \reals \cup \set{+\infty}$ is a lower semicontinuous, convex, and nonsmooth regularizer that induces sparsity. 

We focus on instances of~\eqref{eq.1} where $r$ is a weighted \emph{anisotropic Total-Variation} (TV) regularizer:\footnote{We use the term ``anisotropic'' to refer to the specific TV penalties considered in this paper.}, which, for a vector $\vx \in \reals^n$ and fixed weights $\vw \ge 0$ is defined as
\begin{equation}
  \label{eq.oned}
  r(\vx) \defeq \tvell_p^{1}(\vw;\vx) \defeq \Bigl(\nlsum_{j=1}^{n-1}w_j |x_{j+1}-x_j|^p\Bigr)^{1/p} 
  \quad p \ge 1.
\end{equation}
More generally, if $\tx$ is an order-$m$ tensor in $\reals^{\prod_{j=1}^m n_j}$ with entries $\tx_{i_1,i_2,\ldots,i_m}$ (\fromto{1}{i_j}{n_j} for \fromto{1}{j}{m}); we define the weighted \emph{$m$-dimensional anisotropic TV} regularizer as
\begin{equation}
  \label{eq.multid}
  \tvell^m_{\vp}(\tw; \tx) \defeq \sum_{k=1}^m \sum_{I_{k} = \set{i_1,\ldots,i_m} \setminus i_k}\biggl(\sum_{j=1}^{n_k-1}w_{I_k,j}|\tx^{[k]}_{j+1}-\tx^{[k]}_{j}|^{p_k}\biggr)^{1/p_k},
\end{equation}
where $\tx^{[k]}_j \equiv \tx_{i_1,\ldots,i_{k-1},j,i_{k+1},\ldots,i_m}$, $w_{I_k,j} \ge 0$ are weights, and $\vp \equiv [p_k \ge 1]$ for \fromto{1}{k}{m}. If $\tx$ is a matrix, expression \eqref{eq.multid} reduces to (note, $p, q \ge 1$)
\begin{equation}
  \label{eq.twod}
  \begin{split}
    \tvell_{p,q}^{2}(\tw; \tx) &=
    \sum_{i=1}^{n_1}\Bigl(\sum_{j=1}^{n_2-1}w_{1,j}|x_{i,j+1}-x_{i,j}|^p\Bigr)^{1/p}
    + \sum_{j=1}^{n_2}\Bigl(\sum_{i=1}^{n_1-1}w_{2,i}|x_{i+1,j}-x_{i,j}|^q\Bigr)^{1/q},    
  \end{split}
\end{equation}

These definitions look formidable; already  2D-TV~\eqref{eq.twod} or even the simplest 1D-TV~\eqref{eq.oned} are fairly complex, which further complicates the overall optimization problem~\eqref{eq.1}. Fortunately, this complexity can be ``localized'' by invoking \emph{prox-operators}~\citep{moreau62}, which are now widely used across machine learning~\citep{sraBook,parikh2014proximal}. 

The main idea of using prox-operators while solving~\eqref{eq.1} is as follows. Suppose $\Phi$ is a convex lsc function on a set $\Xc \subset \reals^n$. The \emph{prox-operator} of $\Phi$ is defined as the map
\begin{equation}
  \label{eq.2}
  \prox_\Phi \defeq \vy \mapsto  \argmin_{\vx \in \Xc}\quad \half\enorm{\vx-\vy}^2 + \Phi(\vx)\quad\text{for}\quad \vy \in \reals^n.
\end{equation}
A popular method based on prox-operators is the 
\emph{proximal gradient method} (also known as `forward backward splitting'), which performs a gradient (forward) step followed by a proximal (backward) step to iterate
\begin{equation}
  \label{eq.fbs}
  \vx_{k+1} = \prox_{\eta_kr}(\vx_k - \eta_k\nabla\ell(\vx_k)),\quad k=0,1,\ldots.
\end{equation}
Numerous other proximal methods exist---see e.g.,~\citep{fista,nest07,Combettes09,KimICML10,schmidt11}. 

To implement the proximal-gradient iteration~\eqref{eq.fbs} efficiently, we require a subroutine that computes the prox-operator $\prox_r$. An additional concern is whether the overall algorithm requires an \emph{exact} computation of $\prox_r$, or merely a moderately \emph{inexact} computation. This concern is justified: rarely does $r$ admit an exact algorithm for computing $\prox_r$. Fortunately, proximal methods easily admit inexactness, e.g.,~\citep{schmidt11,salzo,nocops}, which allows approximate prox-operators (as long as the approximation is sufficiently accurate).

We study both exact and inexact prox-operators in this paper, contingent upon the $\ell_p$-norm used and on the data dimensionality $m$.

\subsection{Contributions}
In particular, we review, analyze, implement, and experiment with a variety of fast algorithms. The ensuing contributions of this paper are summarized below.

\begin{itemize}
  \setlength{\itemsep}{-1pt}
\item Geometric analysis that leads to a new, efficient version of the classic Taut String Method~\citep{daviesTautString}, whose origins can be traced back to~\citep{Barlow} -- this version turns out to perform better than most of the recently developed TV proximity methods.
\item A previously unknown connection between (a variation of) this classic algorithm and Condat's \emph{unweighted} TV method~\citep{fastTV}. This connection provides a geometric, more intuitive interpretation and helps us define a hybrid taut-string algorithm that combines the strengths of both methods, while also providing a new efficient algorithm for \emph{weighted} $\ell_1$-norm 1D-TV proximity.
\item Efficient prox-operators for general $\ell_p$-norm ($p \ge 1$) 1D-TV. In particular, 
  \begin{itemize}
    \vspace*{-4pt}
    \setlength{\itemsep}{-1pt}
  \item For $p=2$, we present a specialized Newton method based on the root-finding strategy of \cite{More83},
  \item For the general $p \geq 1$ case we describe both ``projection-free'' and projection based first-order methods.
  \end{itemize}
\item Scalable proximal-splitting algorithms for computing 2D~\eqref{eq.twod} and higher-D TV~\eqref{eq.multid} prox-operators. We review an array of methods in the literature that use prox-splitting, and through extensive experiments show that a splitting strategy based on alternating reflections is the most effective in practice. Furthermore, this modular construction of 2D and higher-D TV solvers allows reuse of our fast 1D-TV routines and exploitation of the massive parallelization inherent in matrix and tensor TV.
\item The final most important contribution of our paper is a well-tuned, multi-threaded open-source C++, Matlab and Python implementation of all the reviewed and developed methods.\footnote{See \href{https://github.com/albarji/proxTV}{{https://github.com/albarji/proxTV}}}
\end{itemize}

\noindent To complement our algorithms, we illustrate several applications of TV prox-operators to: (i) image and video denoising; (ii) image deconvolution; and (iii) four variants of fused-lasso. 

\vspace*{6pt}
\noindent\textbf{Note:} We have invested great efforts to ensure reproducibility of our results. In particular, given the vast attention that TV problems have received in the literature, we believe it is valuable to both users of TV and other researchers to have access to our code, datasets, and scripts, to independently verify our claims, if desired.\footnote{This material shall be made available at: \href{http://suvrit.de/work/soft/tv.html}{http://suvrit.de/work/soft/tv.html}}

\subsection{Related work}
The literature on TV is too large to permit a comprehensive review here. Instead, we mention the most directly related work to help place our contributions in perspective.

We focus on \emph{anisotropic}-TV (in the sense of~\citep{TwIST}), in contrast to  \emph{isotropic}-TV~\citep{RudinTV92}. 
Isotropic TV regularization arises frequently in image denoising and signal processing, and quite a few TV-based denoising algorithms exist~\citep[see e.g.]{ZhuTVAlg08}.

The anisotropic TV regularizers $\tvell_1^{\oned}$ and $\tvell_{1,1}^{\twod}$ arise in image denoising and deconvolution~\citep{DahlTV10}, in the fused-lasso~\citep{fl}, in logistic fused-lasso~\citep{kolar}, in change-point detection~\citep{harLev10}, in graph-cut based image segmentation~\citep{chaDar09}, in submodular optimization~\citep{jegBac13}; see also the related work in~\citep{vert}. This broad applicability and importance of anisotropic TV is the key motivation towards developing carefully tuned proximity operators.

There is a rich literature of methods tailored to anisotropic TV, e.g., those developed in the context of fused-lasso~\citep{frHaHoTi07,liuYe10}, graph-cuts~\citep{chaDar09}, ADMM-style approaches~\citep{Combettes09,BoydTV}, or fast methods based on dynamic programming~\citep{dpTV} or KKT conditions analysis~\citep{fastTV}. However, it seems that anisotropic TV norms other than $\ell_1$ has not been studied much in the literature, although recognized as a form of Sobolev semi-norms~\citep{pontow09}.

For 1D-TV and for the particular $\ell_1$ norm, there exist several direct methods that are exceptionally fast. We treat this problem in detail in Section~\ref{sec:tvl1}, and hence refer the reader to that section for discussion of closely related work on fast solvers. We note here, however, that in contrast to many of the previous fast solvers, our solvers allow weights, a capability that can be very important in applications~\citep{jegBac13}.

Regarding 2D-TV, \citet{splitbreg} presented a so-called ``Split-Bregman'' (SB). It turns out that this method is essentially a variant of the well-known ADMM method. In contrast to the 2D approach presented here, the SB strategy followed by \citet{splitbreg} is to rely on $\ell_1$-soft thresholding substeps instead of 1D-TV substeps. From an implementation viewpoint, the SB approach is somewhat simpler, but not necessarily more accurate. Incidentally, sometimes such direct ADMM approaches turn out to be less effective than ADMM methods that rely on more complex 1D-TV prox-operators~\citep{ramdas}.

It is worth highlighting that it is not just proximal solvers such as FISTA \citep{fista}, SpaRSA~\citep{sparsa}, SALSA \citep{SALSA}, TwIST \citep{TwIST},  \trip \citep{KimICML10}, that can benefit from our fast prox-operators. All other 2D and higher-D TV solvers, e.g., \citep{YangTV}, as well as the recent ADMM based trend-filtering solvers of~\citet{trendf} immediately benefit, not only in speed but also by gaining the ability to solve weighted problems. 

\subsection{Summary of the paper}
The remainder of the paper is organized as follows. In Section~\ref{sec:tvl1} we consider prox operators for 1D-TV problems when using the most common $\ell_1$ norm. The highlight of this section is our analysis on taut-string TV solvers, which lead to the development a new hybrid method and a weighted TV solver (Sections~\ref{sec:tauthybrid},~\ref{sec:tautweighted}). Thereafter, we discuss  variants of 1D-TV (Section~\ref{sec:tvoneothers}), including a specialized $\tvell_2^{\oned}$ solver, and a more general $\tvell_p^{\oned}$ method based on a projected-Newton strategy. Subsequently, we describe multi-dimensional TV problems  and study their prox-operators in Section~\ref{sec:proxMulti}, paying special attention to 2D-TV; for both 2D and multi-D, prox-splitting methods are used. After these theoretical sections, we describe experiments and applications in Section~\ref{sec:exps}. In particular, extensive experiments for 1D-TV are presented in Section~\ref{sec:TVproxExp}; 2D-TV experiments are in Section~\ref{sec:2d.expt}, while an application of multi-D TV is the subject of Section~\ref{sec:appl.multi}. The appendices to the paper include further technical details and additional information about the experimental setup.

\section{TV-L1: Fast prox-operators for $\tvell_1^{\oned}$}
\label{sec:tvl1}
We begin with the 1D-TV problem~\eqref{eq.oned} for an $\ell_1$ norm choice, for which we review several carefully tuned algorithms. Using such well--tuned algorithms pays off: we can find fast, robust, and low-memory (in fact, in place) algorithms, which are not only of independent value, but also ideal building blocks for scalably solving 2D- and higher-D TV problems.

Computation of the $\ell_1$-norm TV prox-operator can be compactly written as the problem
\begin{equation}
  \label{eq:l1tv}
  \min_{\vx \in \reals^n}\quad\half\enorm{\vx-\vy}^2 + \lambda \norm{\md\vx}{1},
\end{equation}
where $\md$ is the \emph{differencing matrix}, all zeros except $d_{ii}=-1$ and $d_{i,i+1}=1$ (\fromto{1}{i}{n-1}). 

To solve~\eqref{eq:l1tv} we will analyze an approach based on the line of ``taut-string'' methods. We first introduce these methods for the unweighted \emph{TV-L1} problem~\eqref{eq:l1tv}, before discussing the elementwise weighted TV problem~\eqref{eq.14}. Most of the previous fastest methods handle only unweighted-TV.  It is often nontrivial to extend them to handle weighted-TV, a problem that is crucial to several applications, e.g., segmentation~\citep{chaDar09} and certain submodular optimization problems~\citet{jegBac13}. 

A remarkably efficient approach to TV-L1 was presented in~\citep{fastTV}. We will show Condat's fast algorithm can be interpreted as a ``linearized'' version of the taut-string approach, a view that paves the way to obtain an equally fast solver for weighted TV-L1.

Before proceeding we note that other than~\citep{fastTV}, other efficient methods to address unweighted $\tvell_1^{\oned}$ proximity have been proposed. \citet{dpTV}  shows how solving $\tvell_p^{\oned}$ proximity is equivalent to computing the data likelihood of an specific Hidden Markov Model (HMM), which suggests a dynamic programming approach based on the well-known Viterbi algorithm for HMMs. The resulting algorithm is very competitive, and guarantees an overall $O(n)$  performance while requiring approximately $8n$ storage. Another similarly performing algorithm was presented by~\citet{Kolmogorov16} in the form of a message passing method. We will also consider these algorithms in our experimental comparison in \S\ref{sec:TVproxExp}.

Yet another family of methods is based on projected-Netwon (PN) techniques: we also present in Appendix \ref{app:projNewton} a PN approach for its instructive value, and also because it provides key subroutines for solving TV problems with $p > 1$. Our derivation may also be helpful to readers seeking to implement efficient prox-operators for problems that have structure similar to TV, for instance $\ell_1$-trend filtering~\citep{boyd.kim,trendf}. Indeed, the PN approach proves to be foundational for the fast ``group fused-lasso'' algorithms of~\citep{wytock}. 

\subsection{The taut-string method for $\tvell_1^{\oned}$}
While taut-string methods seem to be largely unknown in machine learning, they have been widely applied in statistics---see e.g.,~\citep{grasmairTV07,daviesTautString,Barlow}. 

We start by transforming the problem as follows. For TV-L1, elementary manipulations, e.g., using Proposition~\ref{prop.pd}, yield the dual (re-written as a minimization problem)
\begin{equation}
  \label{eq:31}
  \min_{\vu}\ \half\enorm{\md^T\vu}^2 - \vu^T\md\vy,\ \text{s.t. } \infnorm{\vu} \le \lambda.
\end{equation}
Without changing the minimizer, the objective~\eqref{eq:31} can be replaced by
$\enorm{\md^T\vu - \vy}^2$, which then unfolds into
\begin{align*}
   \left(u_1 - y_1 \right)^2 + \nlsum_{i=2}^{n-1} \left( -u_{i-1} + u_i - y_i \right)^2 + \left( -u_{n-1} - y_{n} \right)^2  .
\end{align*}
Introducing the fixed extreme points $u_0 = u_n = 0$, we can replace the problem~\eqref{eq:31} by
\begin{equation}
  \label{eq:32}
 \min_{\vu}\ \nlsum_{i=1}^n \left( - u_{i-1} + u_i - y_i \right)^2,\ \text{s.t. } \infnorm{\vu} \le \lambda,\quad u_0 = u_n = 0.
\end{equation}
Now we perform a change of variables by defining the new set of variables $\vs = \vr - \vu$, where $r_i := \sum_{k=1}^i y_k$ is the cumulative sum of input signal values. Thus, \eqref{eq:32}  becomes
\begin{equation*}
 \min_{\vs}\ \nlsum_{i=1}^n \left( - r_{i-1} + s_{i-1} + r_i - s_i - y_i \right)^2,\ \text{s.t. } \infnorm{\vs - \vr} \le \lambda,\quad r_0 - s_0 = r_n - s_n = 0 ,
\end{equation*}
which upon simplification becomes
\begin{equation}
  \label{eq:33}
 \min_{\vs}\ \nlsum_{i=1}^n \left( s_{i-1} - s_i \right)^2,\quad\text{s.t.}\quad \infnorm{\vs - \vr} \le \lambda, s_0 = 0,\  s_n = r_n .
\end{equation}
Now the key trick: problem~\eqref{eq:33} can be shown to share the same optimum as 
\begin{equation}
 \label{eq:tautString}
 \min_{\vs}\ \sum_{i=1}^n \sqrt{1 + \left( s_{i-1} - s_i \right)^2},\ \text{s.t. } \infnorm{\vs - \vr} \le \lambda,\quad s_0 = 0,\  s_n = r_n.
\end{equation}
A proof of this relationship may be found in~\citep{SteidlTautString}; for completeness, and also because it will help us generalize to the weighted $\tvell_1^{\oned}$ variant, we include an alternative proof in Appendix~\ref{app:tautStringProof}.

The name ``taut-string'' is explained as follows. The objective in~\eqref{eq:tautString} can be interpreted as the Euclidean length of a polyline through the points $(i, \vs_i)$. Thus, \eqref{eq:tautString} seeks the minimum length polyline (the \emph{taut-string}) crossing a tube of height $\lambda$ with center the cumulative sum $\vr$ and having the fixed endpoints ($s_0, s_n$). An example illustrating this description is shown in Figure \ref{fig:tautStringExample}. 

Once the taut string is found, the solution for the original TV problem~\eqref{eq:l1tv} can be recovered by observing that
\begin{align*}
 s_i - s_{i-1}\quad=\quad r_i - u_i - (r_{i-1} - u_{i-1})\quad=\quad y_i - u_i + u_{i-1}\quad=\quad x_i,
\end{align*}
where we used the primal-dual relation $\vx = \vy - \md^T \vu$. Intuitively, the above argument shows that the solution to the TV-L1 proximity problem is obtained as the discrete gradient of the taut string, or as the slope of its segments.

\begin{figure}[htbp]
  \centering
  \includegraphics[width = \textwidth]{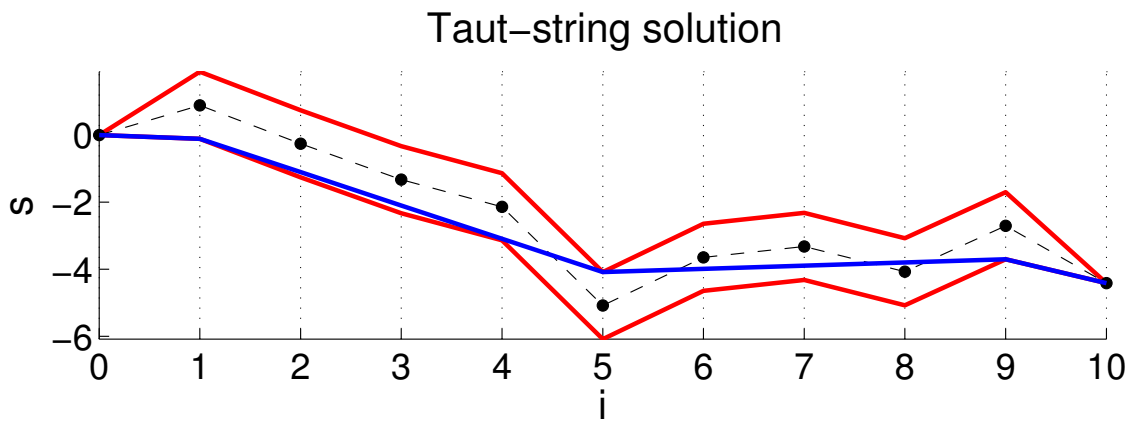}
  \caption{\small Example of the taut string method. The cumulative sum $\vr$ of the input signal values $\vy$ is shown as the dashed line; the black dots mark the points $(i,r_i)$. The bottom and top of the $\lambda$-width tube are shown in red. The taut string solution $\vs$ is shown as a blue line.
   }   \label{fig:tautStringExample}
\end{figure}

It remains to describe how to find the taut string. The most widely used approach seems to be the one due to~\citet{daviesTautString}. This approach starts from the fixed point $s_0 = 0$, and incrementally computes the \emph{greatest convex minorant} of the upper bounds on the $\lambda$ tube, as well as the \emph{smallest concave majorant} of the lower bounds on the $\lambda$ tube. When both curves intersect, the \emph{left-most} point where either the majorant or the minorant touched the tube is used to fix a first segment of the taut string. The procedure is then resumed at the end of the identified segment, and iterated until all taut string segments have been obtained. Pseudocode of this method is presented as Algorithm~\ref{algTV1classicTautString}, while an example of this procedure is shown in Figure~\ref{fig:classictautstringrun}.

\begin{algorithm}[tp]\small
 \caption{\small Taut string algorithm for TV-L1-proximity}
 \label{algTV1classicTautString}
 \begin{algorithmic}[1]
   \State {\bf Inputs}: input signal $\vy$ of length $n$, regularizer $\lambda$.
   \State {\bf Initialize} $i=0$, $concmajorant = \varnothing$, $convminorant = \varnothing$, $\vr_i = \sum_{k=1}^i \vy_k$.
   \While{$i < n$}
      \State Add new segment: $concmajorant = concmajorant \cup \left( (i-1, \vr_{i-1}-\lambda) \rightarrow (i, \vr_i-\lambda) \right)$.
      \While{$concmajorant$ is not concave}
	\State Merge the last two segments of $concmajorant$
      \EndWhile
      \State Add new segment: $convminorant = convminorant \cup \left( (i-1, \vr_{i-1}+\lambda) \rightarrow (i, \vr_i+\lambda) \right)$.
      \While{$convminorant$ is not convex}
	\State Merge the last two segments of $convminorant$
      \EndWhile
      \If{slope(left-most segment in $concmajorant$) $>$ slope(lest-most segment in $convminorant$)}
        \State $break$ = left-most point where either the majorant or the minorant touched the tube
        \If{$break \in convminorant$}
          \State Remove left-most segment of the minorant, add it to the taut-string solution $\vx$.
          \State Majorant is recalculated as a straight line from $break$ to its last point.
        \EndIf
        \If{$break \in concmajorant$}
          \State Remove left-most segment of the majorant, add it to the taut-string solution $\vx$.
          \State Minorant is recalculated as a straight line from $break$ to its last point.
        \EndIf
      \EndIf
      \State $i++$
  \EndWhile
  \State Add last segment from either the majorant or minorant to the solution $\vx$.
 \end{algorithmic}
\end{algorithm}

\begin{figure}[htbp]
  \centering
  \includegraphics[width = \textwidth]{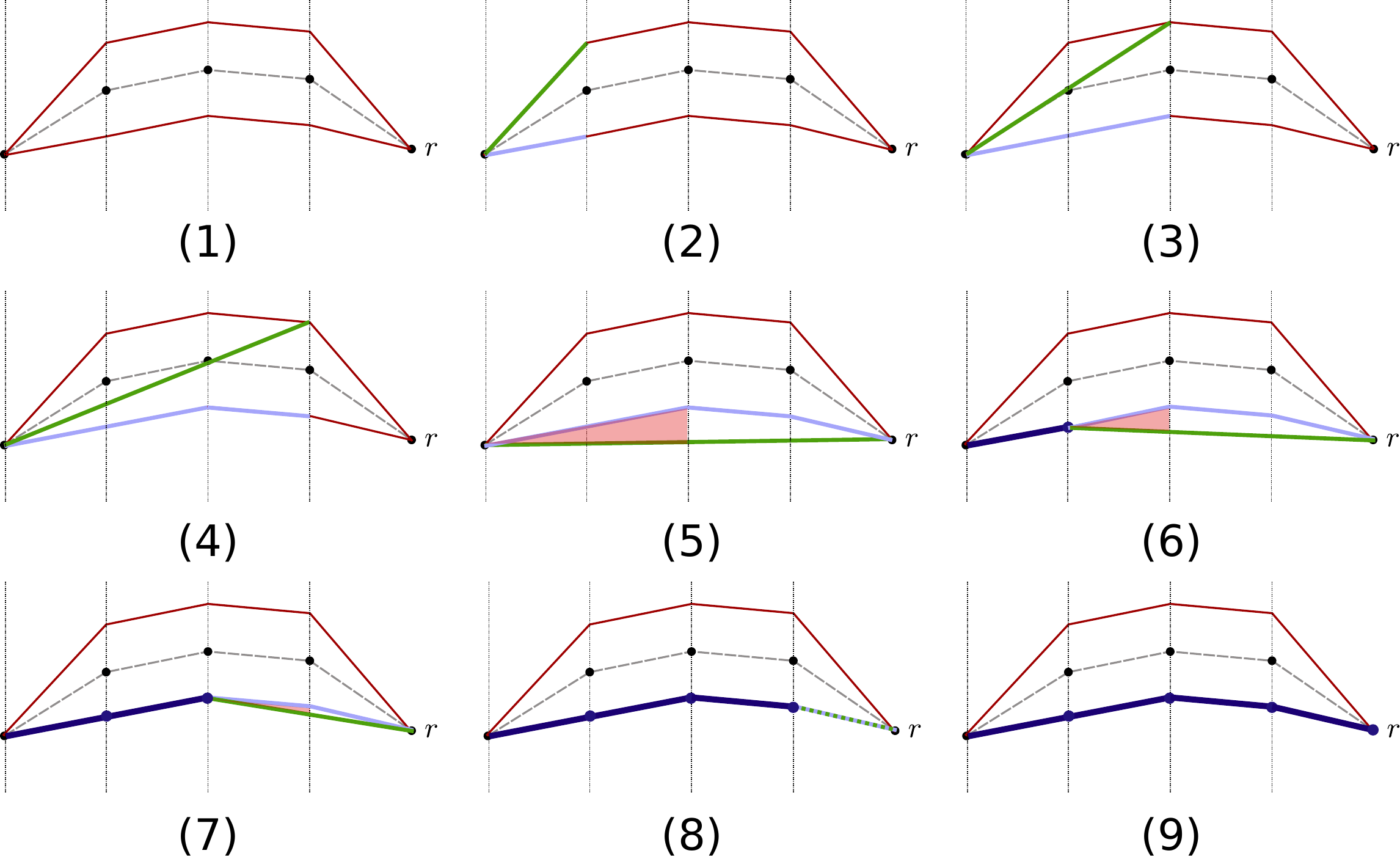}
  \caption{\small Example of the evolution of the taut string method. The smallest concave majorant (blue) and largest convex minorant (green) are updated are every step. At step (1) the algorithm is initialized. Steps (2) to (4) successfully manage to update majorant and minorant without producing crossings between them. Note how while the concave majorant keeps adding segments without issue, the convex minorant must remove and merge existing segments with new ones to mantain a convex function from the origin to the new points. At step (5) the end of the tube is reached, but the minorant and majorant slopes overlap, and so it is necessary to break the segment at the left-most point where the majorant/minorant touched the tube. Since the left-most touching point is in the concave majorant it's leftmost segment is removed and placed in the solution, while the convex minorant is updated as a straight line from the detected breakpoint to the last explored point, resulting in (6). The algorithm would then continue adding segments, but since the majorant/minorant slopes are still crossing, the procedure of fixing segments to the solution is repeated through steps (6), (7) and (8). Finally at step (9) the slopes are no longer crossing and the method would continue adding tube segments, but since the end of the tube has already been reached the algorithm stops.
   }   \label{fig:classictautstringrun}
\end{figure}

It is important to note that since we have a discrete number of points in the tube, the greatest convex minorant can be expressed as a piecewise linear function with segments of monotonically increasing slope, while the smallest concave majorant is another piecewise linear function with segments of monotonically decreasing slope. Another relevant fact is that each segment in the tube upper/lower bound enters the minorant/majorant exactly once in the algorithm, and is also removed exactly once. This limits the extent of the inner loops in the algorithm, and in fact an analysis of the computational complexity of this behavior leads to an overall $O(n)$ performance~\citep{daviesTautString}.

In spite of this, \citet{fastTV} notes that maintaining the minorant and majorant functions in memory is inefficient, and views a taut-string approach as potentially inferior to his proposed method. To this observation we make two claims: \citeauthor{fastTV}'s method can be interpreted as a linearized version of the taut-string method (see Section~\ref{sec:taut}); and that a careful implementation of the taut-string method can be highly competitive in practice. 

\subsubsection{Efficient implementation of taut-strings}
We propose now an efficient implementation of the taut-string method. The main idea is to carefully use double-ended queues~\citep{Knuth97} to store the majorant and minorant information. Therewith, all majorant/minorant operations such as appending a segment or removing segments from either the beginning or the end of the majorant can be performend in constant time. Note however that usual double-ended queue implementations use doubly linked lists, dynamic arrays or circular buffers: these approaches require dynamically reallocating memory chunks at some of the insert or remove operations. But in the taut-string algorithm, the maximum number of segments of the majorant/minorant is just the size of the input signal ($n$), and also the number of segments to be inserted in the queue throughout the algorithm will be $n$. Making use of these facts we implement a specialized queue based on a contiguous array of fixed length $n$. New segments are added from the start of the array on, and a couple of pointers are maintained to keep track of the first and last valid segments in the array, much in the way of a circular buffer. This implementation, however, does not require of the usual circular logic. Overall, this double-ended queue requires a single memory allocation at the beginning of the algorithm, keeping the rest of queue operations free from memory management and all but the simplest pointer or index algebra.

We also store for each segment the following values: $x$ length of the segment, $y$ length and slope. Slopes might seem as redundant given the other two factors, but given the number of times the algorithm requires comparing slopes between segments (e.g., to preserve convexity/concavity) it pays off to precompute these values. This fact together with other calculation and code optimization details produces our implementation; these can be reviewed in the code itself at \url{https://github.com/albarji/proxTV}.

\subsection{Linearized taut-string method for $\tvell_1^{\oned}$} 
\label{sec:taut}
We now present a variant, linearized version of the taut-string method. Surprisingly, the resulting algorithm turns out to be equivalent to the fast algorithm of~\citet{fastTV}, though now with a clearer interpretation based on taut-strings.

The key idea is to build linear approximations to the greatest convex minorant and smallest concave majorant, producing exactly the same results but significantly reducing the bookkeeping of the method to a handful of simple variables. We therefore replace the greatest convex minorant and smallest convex majorant by a \emph{greatest affine minorant} and \emph{smallest affine majorant}.

An example of the method is presented in Figure \ref{fig:tautStringAlgEvolution}. A proof showing that this linearization does not change the resultant taut-string is given in Appendix \ref{app:tautStringAlgEq}. In what follows, we describe the linearized method in depth.

\vskip5pt
\noindent\textbf{Details.} Linearized taut-string requires only the following bookkeeping variables:
\begin{enumerate}
  \vspace*{-2pt}
  \setlength{\itemsep}{-2pt}
 \item $i_0$: index of the current segment start
 \item $\bar \delta$: slope of the line joining segment start with majorant at the current point
 \item $\ubar \delta$: slope of the line joining segment start with minorant at the current point
 \item $\bar h$: height of majorant w.r.t.\ the $\lambda$-tube center
 \item $\ubar h$: height of minorant w.r.t.\ $\lambda$-tube center
 \item $\bar i$: index of last point where $\bar \delta$ was updated---potential majorant break point
 \item $\ubar i$: index of last point where $\ubar \delta$ was updated---potential minorant break point.
\end{enumerate}
Figure~\ref{fig:tautStringAlgVariables} gives a geometric interpretation of these variables; we use these variables to detect  minorant-majorant intersections, without the need to compute or store them explicitly.

\begin{figure}[htbp]
  \centering
  \includegraphics[width = 0.8\textwidth]{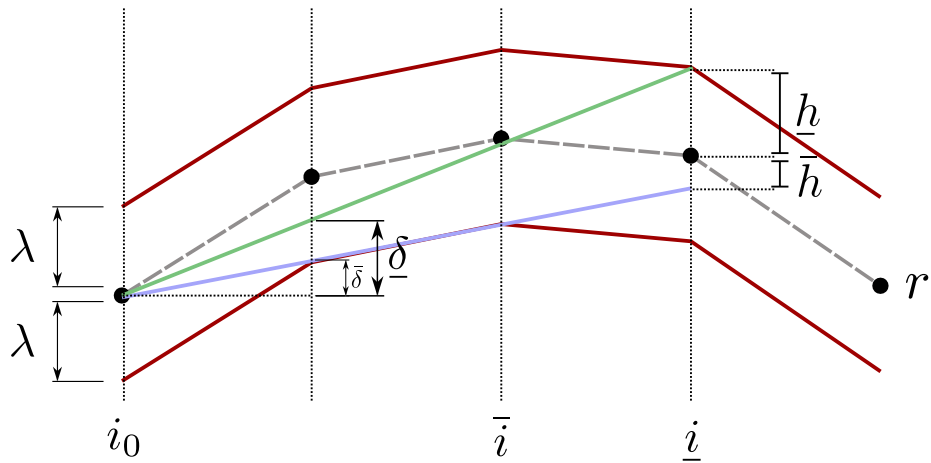}
  \caption{\small Illustration of the geometric concepts involved in the linearized taut string method. The greatest linear minorant (of the tube ceiling) is depicted in green, while the smallest linear majorant (of the tube bottom) is shown in blue. The $\delta$ slopes and $h$ heights are presented updated up to the index shown as $\underbar i$.
   }   \label{fig:tautStringAlgVariables}
\end{figure}

Algorithm \ref{algTV1tautString} presents full pseudocode of the linearized taut-string method. Broadly, the algorithm proceeds in the same fashion as the classic taut-string method, updating the affine approximations to the majorant and minorant at each step, and introducing a breakpoint whenever the slopes of these two functions cross.

More precisely, at each each iteration the method steps one point further through the tube, updating the minorant/majorant slopes  ($\ubar \delta$, $\bar \delta$) as well their heights at the current point ($\ubar h$, $\bar h$). To check for minorant/majorant crossings it suffices to compare the slopes ($\ubar \delta$, $\bar \delta$), or equivalently, to check whether the height of the minorant $\ubar h$ falls below the tube bottom (since the minorant follows the tube ceiling) or the height of the majorant $\bar h$ grows above the tube ceiling (since the majorant follows the tube bottom). We make use of this last variant, since updating heights turns out to be slightly cheaper than updating slopes, and so it is faster to ensure no crossing will take place before performing such updates.

When a crossing is detected, we perform similar steps as in the classic taut-string method but with one significant difference: the algorithm is completely restarted at the newly introduced breakpoint. This restart idea is in contrast with the classic method, where we simply re-use the previously computed information about the minorant and majorant to update their estimates and continue working with them. In the linearized version we do not keep enough information to perform such an operation, so all data about minorant and majorant is discarded and the algorithm begins anew. Because of this choice the same tube segment might be reprocessed up to $O(n)$ times in the method, and therefore the overall worst case performance is $O(n^2)$. This fact was already observed in \citep{fastTV}.

In what follows we describe the rationale behind the height update formulae.

\begin{algorithm}[t]\small
 \caption{\small Linearized taut string algorithm for TV-L1-proximity}
 \label{algTV1tautString}
 \begin{algorithmic}[1]
   \State Initialize $i = \bar i = \ubar i = \bar h = \ubar h = 0$, $\ubar \delta = y_0 + \lambda$, $\bar \delta = y_0 - \lambda$
   \While{$i < n$}
   \State Find tube height: $\tilde \lambda = \lambda$ if $i < n-1$, else $\tilde \lambda = 0$
   \State Update majorant height following current slope: $\bar h = \bar h + \bar \delta - y_i$.
   \State /* Check for ceiling violation: majorant is above tube ceiling */
   \If { $\bar h > \tilde \lambda$}
     \State Build valid segment up to last majorant breaking point: $\vx_{i_0+1:\bar i} = \bar \delta$.
     \State Start new segment after break: $(i_0, \ubar i) = \bar i$, $\ubar \delta = y_i + 2\lambda$, $\bar \delta = y_i$, $\ubar h = \lambda$, $\bar h = -\lambda$, $i = \bar i+1$
     \State \textbf{continue} 
   \EndIf
   \State Update minorant height following current slope: $\ubar h = \ubar h + \ubar \delta - y_i$.
   \State /* Check for bottom violation: minorant is below tube bottom */
   \If { $\ubar h < -\tilde \lambda$}
     \State Build valid segment up to last minorant breaking point: $\vx_{i_0+1:\ubar i} = \ubar \delta$.
     \State Start new segment after break: $(i_0, \bar i) = \ubar i$, $\ubar \delta = y_i$, $\bar \delta = -2\lambda + y_i$, $\ubar h = \lambda$, $\bar h = -\lambda$, $i = \ubar i + 1$
     \State \textbf{continue}
   \EndIf
   \State /* Check if majorant height is below the floor */
   \If { $\bar h \leq -\tilde \lambda$ }
     \State Correct slope: $\bar \delta = \bar \delta + \frac{\tilde \lambda - \bar h}{i - i_0}$
     \State The majorant now touches the floor: $\bar h = -\tilde \lambda$
     \State This is a possible majorant breaking point: $\bar i = i$
   \EndIf
   \State /* Check if minorant height is above the ceiling */
   \If { $\ubar h \geq \tilde \lambda$ }
     \State Correct slope: $\ubar \delta = \ubar \delta + \frac{-\tilde\lambda - \ubar h}{i - i_0}$
     \State The minorant now touches the ceiling: $\ubar h = \tilde \lambda$
     \State This is a possible minorant breaking point: $\ubar i = i$
   \EndIf
   \State Continue building current segment: $i = i+1$
  \EndWhile
  \State Build last valid segment: $\vx_{i_0+1:n} = \bar \delta$.
 \end{algorithmic}
\end{algorithm}

\paragraph{Height variables.} To implement the method described above, the height variables $h$ are not strictly necessary as they can be obtained from the slopes $\delta$. However, explicitly including them leads to efficient updating rules at each iteration, as we show below. 

Suppose we are updating the heights and slopes from their estimates at step $i-1$ to step $i$. Updating the heights is immediate given the slopes, since
\begin{equation*}
 h_{i} = h_{i-1} + \delta - y_i.
\end{equation*}
In other words, since we are following a line with slope $\delta$, the change in height from one step to the next is given by precisely such a slope. Note, however, that in this algorithm we do not compute absolute heights but instead relative heights with respect to the $\lambda$--tube center. Therefore we need to account for the change in the tube center between steps $i-1$ and $i$, which is given by $r_{i} - r_{i-1} = y_{i}$. This completes the update, which is shown in Algorithm \ref{algTV1tautString} as lines 4 and 11.

However, it is possible that the new height $h$ runs over or under the tube. This would mean that we cannot continue using the current slope in the majorant or minorant, and a recalculation is needed, which again can be done efficiently by using the height information. Assume without loss of generality that the starting index of the current segment is $0$ and the absolute height of the starting point of the segment is given by $\alpha$. Then, for adjusting the minorant slope $\bar \delta_i$ so that it touches the tube ceiling at the current point, we note that
\begin{equation*}
 \bar \delta_i = \frac{\lambda + r_i - \alpha}{i} = \frac{\lambda + (\bar h_i - \bar h_i) + r_i - \alpha}{i},
\end{equation*}
where we have also added and subtracted the current value of $\bar h_i$. Observe that this value was computed using the estimate $\delta_{i-1}$ of the slope so far, so we can rewrite it as the projection of the initial point in the segment following such a slope, that is, as $\bar h_{i} = i \bar \delta_{i} - r_{i} + \alpha$. Doing so for one of the added heights $\bar h_{i}$ produces
\begin{equation*}
 \bar \delta_i = \frac{\lambda + (i \bar \delta_{i-1} - r_{i} + \alpha) - \bar h_i + r_{i} - \alpha}{i} = \bar \delta_{i-1} + \frac{\lambda - \bar h_i}{i},
\end{equation*}
which generates a simple updating rule. A similar derivation holds for the minorant. The resulting updates are included in the algorithm in lines 20 and 26. After recomputing this slope we need to adjust the corresponding height back to the tube: since the heights are relative to the tube center we can just set $\bar h = \lambda$,  $\ubar h = -\lambda$; this is done in lines 21 and 27.

\begin{figure}[htbp]
  \centering
  \includegraphics[width = \textwidth]{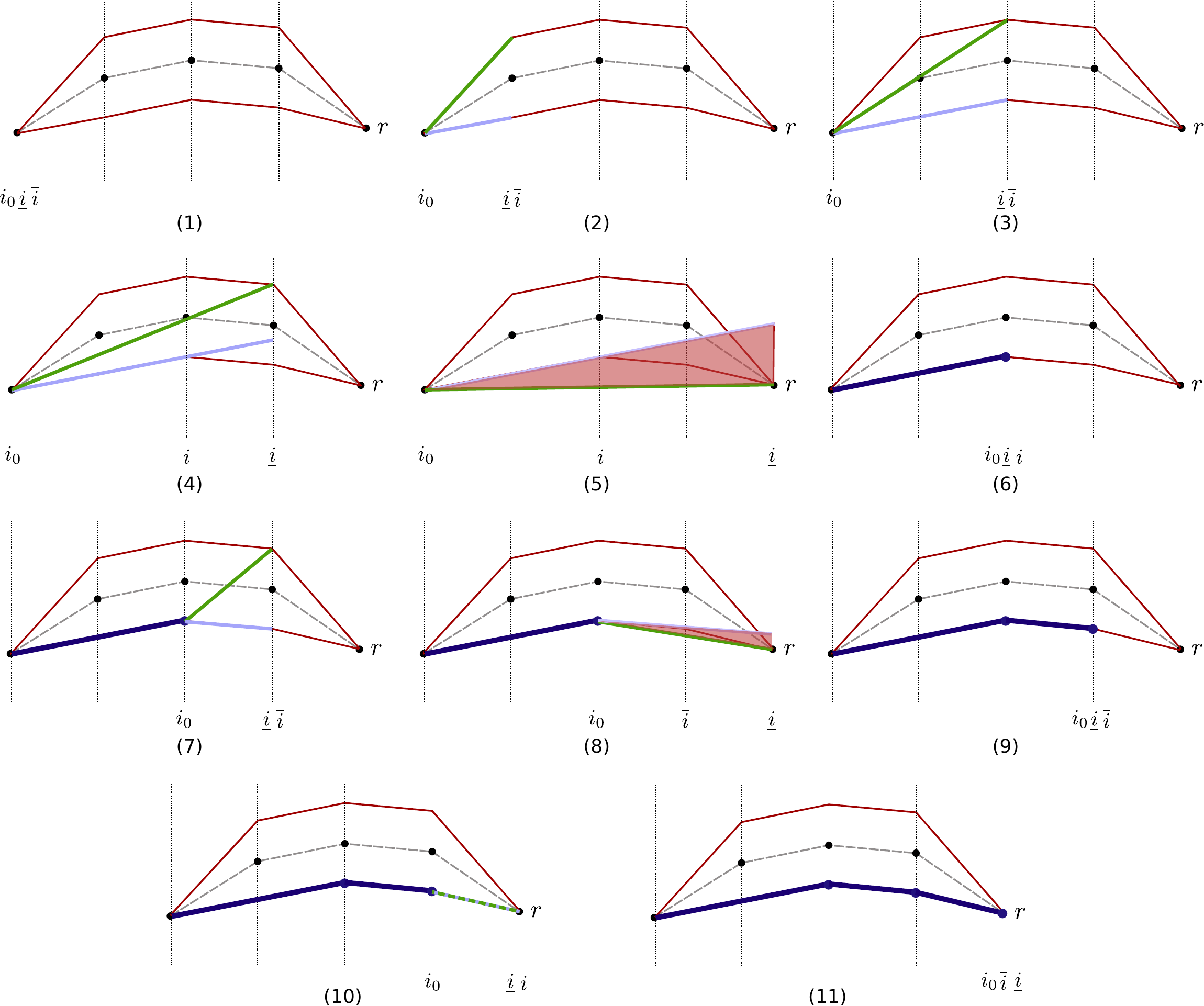}
  \caption{\small Example of the evolution of the linearized taut string method. The smallest affine majorant of the tube bottom (blue) and greatest affine minorant of the tube ceiling (green) are updated at every step. At step (1) the algorithm is initialized. Steps (2) to (4) successfully manage to update majorant/minorant without crossings. At step (5), however, the slopes cross, and so it is necessary to break the segment. Since the left-most tube touching point is the one in the majorant, the majorant is broken down at that point and its left-hand side is added to the solution, resulting in (6). The method is then restarted at the break point, with majorant/minorant being updated at step (7), though at step (8) once again a crossing is detected. Hence, at step (9) a breaking point is introduced again and the algorithm is restarted once more. Following this, step (10) manages to update majorant/minorant slopes up to the end of the tube, and so at step (11) the final segment is built using the (now equal) slopes.
   }   \label{fig:tautStringAlgEvolution}
\end{figure}

Notice also that the special case of the last point in the tube where the taut-string must meet $s_n = r_n$ is handled by line 3, where $\tilde \lambda$ is set to $0$ at such a point to enforce this constraint. Overall, one iteration of the method is very efficient, as mostly just additions and subtractions are involved with the sole exception of the division required for the slope updates, which are not performed at every iteration. Moreover, no additional memory is required beyond the constant number of bookkeeping variables, and in-place updates are also possible because $y_i$ values for already fixed sections of the taut-string are not required again, so the output $\vx$ and the input $\vy$ can both refer to the same memory locations.

The resulting algorithm turns out to be equivalent, almost line by line, to the method of~\citet{fastTV}, even though its theoretical grounds are radically different: while the approach presented here has a strong geometric basis due to its taut-string relationship, \citep{fastTV} is based solely on analysis of KKT conditions. Therefore, we have shown that Condat's fast TV method is, in fact, a linearized taut-string algorithm. 

\subsection{Comparison of taut-string methods and a hybrid strategy}
\label{sec:tauthybrid}

\begin{table}
  \begin{tabularx}{\textwidth}{l|X|X}
    & \bf Classic & \bf Linearized (Condat's) \\
  \hline
  Worst-case performance & $O(n)$ & $O(n^2)$ \\
  \hline
  In--memory & No & Yes \\
  \hline
  Other considerations & Fast bookkeeping through double-ended queues & Very fast iteration, cache friendly
  \end{tabularx}
  \caption{Comparison of the main features of reviewed taut-string algorithms.}
  \label{tab:tautstringComparison}
\end{table}

Table \ref{tab:tautstringComparison} summarizes the main features of the classic and linearized taut-string methods reviewed so far. Although the classic taut-string method has been largely neglected in the machine learning literature, its guarantee in linear performance makes it an attractive choice. Furthermore, although we could not find any references on implementation details of this method, we have empirically seen that a very efficient solver can be produced by making use of a double-ended queue to bookkeep the majorant/minorant information.

In contrast to this, the linearized taut-string method (equivalent to \citep{fastTV}) features a much better performance per step in the tube traversal, mainly due to not requiring additional memory and making use of only a small constant number of variables, making the method friendly for CPU cache or registers calculation. As a tradeoff of keeping such scarce information in memory, the method does not guarantee linear performance, falling to a quadratic theoretical runtime in the worst case. This fact was already observed in \citep{fastTV}, though such worst case was deemed as pathological, claiming a $O(n)$ performance in all practical situations. We shall review these claims in the experimental sections in this manuscript.

The key points of Table \ref{tab:tautstringComparison} show that no taut-string variant is clearly superior. While the classic method provides a safe linear time solution to the problem, the linearized method is potentially faster but riskier in terms of worst case performance. Following these observations we propose here a simple hybrid method combining both approaches: run the linearized algorithm up to a prefixed number of steps $n^S$, $S\in(1,2)$, and if the solution has not yet been found, we switch to the classic method. We therefore limit the worst-case scenario to $O(n^S) + O(n) \simeq O(n^S)$, because once the classic method kicks, it will ensure an $O(n)$ performance guarantee.

Implementation of this hybrid method is easy upon realizing the similarities between algorithms: a switch--check is added to the linearized method every time a segment of the taut-string has been identified (Algorithm \ref{algTV1tautString}, lines 7, 14). If it is confirmed that the method has already run for $n^S$ steps without reaching the solution, the remaining part of the signal for which the taut-string has not yet been found is passed on to the classic method, whose solution is concatenated to the part the linearized method managed to find so far. We also report the empirical performance of this method in the experimental section.

\subsection{Taut-string methods for weighted $\tvell_1^{\oned}$}
\label{sec:tautweighted}
Several applications TV require penalizing the discrete gradients individually, which can be done by solving the \emph{weighted TV-L1} problem
\begin{equation}
  \label{eq.14}
  \nlmin_{\vx}\quad\half\enorm{\vx-\vy}^2 + \nlsum_{i=1}^{n-1} w_i|x_{i+1}-x_i|,
\end{equation}
where the weights $\{w_i\}_{i=1}^{n-1}$ are all positive. To solve~\eqref{eq.14} using a taut-string approach, we again begin with its dual (written as a minimization problem)
\begin{equation}
  \label{eq.15}
  \nlmin_{\vu}\quad\half\enorm{\md^T\vu}^2 - \vu^T\bm{Dy}\quad\text{s.t.}\quad |u_i| \le w_i,\quad 1 \le i < n.
\end{equation}
Then, we repeat the derivation of the unweighted taut-string method but with a few key modifications. More precisely, we transform~\eqref{eq.15} by introducing $u_0 = u_n = 0$ to obtain
\begin{equation*}
 \min_{\vu}\ \nlsum_{i=1}^n \left(y_i - u_i + u_{i-1} \right)^2 \quad \text{s.t.}\  \, |u_i| \le w_i,\quad 1 \le i < n.
\end{equation*}
Then, we perform the change of variables $\vs = \vr - \vu$, where $r_i := \sum_{k=1}^i y_k$, and consider
\begin{equation*}
 \min_{\vs} \nlsum_{i=1}^n \left(s_i - s_{i-1} \right)^2 \quad \text{s.t.}\  |s_i - r_i| \le w_i,\quad 1 \le i < n\ s_0 = 0,\ s_n = r_n.
\end{equation*}
Finally, applying Theorem~\ref{the:tautStringEq} we obtain the equivalent \emph{weighted taut-string} problem
\begin{equation}
  \label{eq:34}
 \min_{\vs} \nlsum_{i=1}^n \sqrt{1 + \left(s_i - s_{i-1} \right)^2} \quad \text{s.t.}\ |s_i - r_i| \le w_i,\ 1\le i < n,\ s_0 = 0,\ s_n = r_n.
\end{equation}

Problem~\eqref{eq:34} differs from its unweighted counterpart~\eqref{eq:tautString} in the constraints $|s_i - r_i| \le w_i$ ($1\le i < n$), which allow different weights for each component instead of using the same value $\lambda$. Our geometric intuition also carries over to the weighted problem, albeit with a slight modification: the tube we are trying to traverse now has varying widths at each step instead of the previous fixed $\lambda$ width---Figure~\ref{fig:tautStringWeightedExample} illustrates this idea.

\begin{figure}[htbp]
  \centering
  \includegraphics[width = \textwidth]{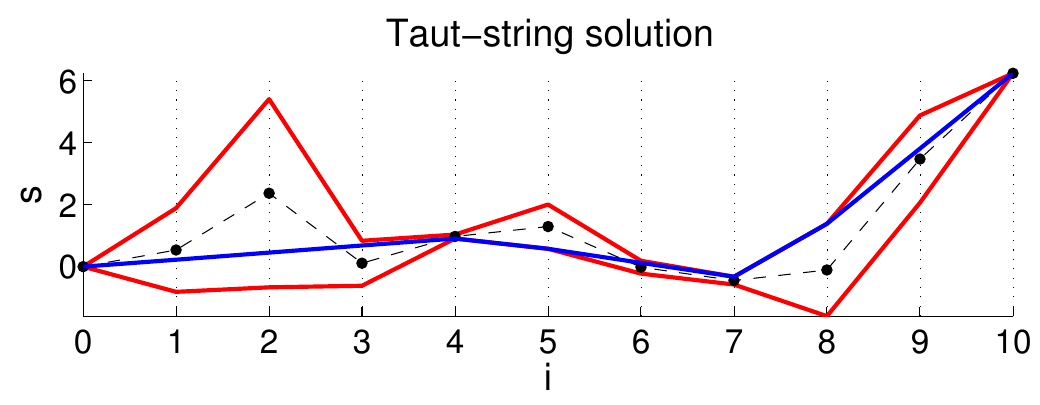}
  \caption{\small Example of the weighted taut string method with $\vw=(1.35$, $3.03$, $0.73$, $0.06$, $0.71$, $0.20$, $0.12$, $1.49$, $1.41)$. The cumulative sum $\vr$ of the input signal values $\vy$ is shown as the dashed line, with the black dots marking the points $(i,\vr_i)$. The bottom and ceiling of the tube are shown in red, which vary in width at each step following the weights $\vw_i$. The weighted taut string solution $\vs$ is shown as a blue line.
   }   \label{fig:tautStringWeightedExample}
\end{figure}

As a consequence of the above derivation and intuition, taut-string methods can be produced to solve the weighted $\tvell_1^{\oned}$ problem. The original formulation of the classic taut-string method in~\citep{daviesTautString} defines the limits of the tube through possibly varying bottom and ceiling values $(l_i, u_i)$ $\forall i$, and so this method easily extends to solve the weighted TV problem by assigning $l_i = r_i - w_i$, $u_i = r_i + w_i$. In our pseudocode in Algorithm~\ref{algTV1classicTautString} we just need to replace $\lambda$ by the appropriate $w_i$ values.

Similar considerations apply for the linearized version (Algorithm~\ref{algTV1tautString}), in particular, when checking ceiling/floor violations as well as when checking slope recomputations and restarts, we must account for varying tube heights. Algorithm \ref{algTV1tautStringWeighted} presents the precise modifications that we must make to Algorithm~\ref{algTV1tautString} to handle weights. Regarding the convergence of this method, the proof of equivalence with the classic taut-string method still holds in the weighted case (see Appendix \ref{app:tautStringAlgEq}).

\begin{algorithm}[htp]\small
 \caption{\small Modified lines for weighted version of Algorithm~\ref{algTV1tautString}}
 \label{algTV1tautStringWeighted}
 \begin{algorithmic}[1]
   \makeatletter
   \setcounter{ALG@line}{2}
   \makeatother
   \State\hskip12pt Find tube height: $\tilde \lambda = w_{i+1}$ if $i < n-1$, else $\tilde \lambda = 0$
   \makeatletter
   \setcounter{ALG@line}{7}
   \makeatother
     \State\hskip12pt Start new segment after break: $(i_0, \ubar i) = \bar i$, $\ubar \delta = y_i + w_{i-1} + w_i$, $\bar \delta = y_i + w_{i-1} - w_i$, $\ubar h = w_i$, $\bar h = -w_i$, $i = \bar i+1$
   \makeatletter
   \setcounter{ALG@line}{14}
   \makeatother
     \State\hskip12pt Start new segment after break: $(i_0, \bar i) = \ubar i$, $\ubar \delta = y_i + w_{i-1} - w_i$, $\bar \delta = y_i + w_{i-1} + w_i$, $\ubar h = w_i$, $\bar h = -w_i$, $i = \ubar i + 1$
 \end{algorithmic}
\end{algorithm}

The very same analysis as portrayed in Table \ref{tab:tautstringComparison} applies here: both the benefits and problems of the two taut-string solvers carry on to the weighted variant of the problem.

\section{Other one-dimensional TV variants}
\label{sec:tvoneothers}
While more infrequent, replacing the $\ell_1$ norm of the standard TV regularizer by an $\ell_p$-norm version can also be useful. In this section we focus first on a specialized solver for $p=2$, before discussing a less efficient but more general solver for any $\ell_p$ with $p \geq 1$. We also briefly cover the $p = \infty$ case.

\subsection{TV-L2: Proximity for $\tvell_2^{\oned}$}
\label{sec:tvl2}
For TV-L2 proximity ($p=2$) the dual to the prox-operator for~(\ref{eq.oned}) reduces to
\begin{equation}
  \label{eqTV2dual}
  \nlmin_{\vu}\ \phi(\vu) := \half\enorm{\md^T\vu}^2 - \vu^T\md\vy,\ \ \text{s.t.}\ \norm{\vu}{2} \le \lambda.
\end{equation}
Problem~\eqref{eqTV2dual} is nothing but a version of the well-known trust-region subproblem (TRS), for which a variety of numerical approaches are known~\citep{Conn}. 

We derive a specialized algorithm based on the classic Mor\'e-Sorensen Newton (MSN) method of~\citep{More83}. This method in general can be quite expensive, but for~\eqref{eqTV2dual} the Hessian is tridiagonal which can be well-exploited (see Appendix~\ref{app:projNewton}). Curiously, experiments show that for a limited range of $\lambda$ values, even ordinary gradient-projection (GP) can be competitive. But for overall best performance, a hybrid MSN-GP approach is preferable. 




Towards solving~\eqref{eqTV2dual}, consider its KKT conditions:
\begin{equation}
 \label{eqCSPTV2}
 \begin{split}
   &(\md\md^T + \alpha \mi) \vu = \md\vy,\\
   &\alpha (\enorm{\vu} - \lambda) = 0,\quad \alpha \ge 0,
 \end{split}
\end{equation}
where $\alpha$ is a Lagrange multiplier. There are two possible cases: either $\enorm{\vu} < \lambda$ or $\enorm{\vu} = \lambda$.

If $\enorm{\vu} < \lambda$, then the KKT condition $\alpha (\enorm{\vu} - \lambda) = 0$, implies that $\alpha=0$ must hold and $\vu$ can be obtained immediately by solving the linear system $\md\md^T\vu = \md\vy$. This can be done in $O(n)$ time owing to the bidiagonal structure of $\md$. Conversely, if the solution to $\md\md^T\vu = \md\vy$ lies in the interior of the ball $\enorm{\vu} \le \lambda$, then it solves (\ref{eqCSPTV2}). Therefore, this case is trivial, and we need to consider only the harder case $\enorm{\vu}=\lambda$.


For any given $\alpha$ one can obtain the corresponding vector $\vu$ as $\vu_{\alpha} = (\md\md^T+\alpha\mi)^{-1}\md\vy$. Therefore, optimizing for $\vu$ reduces to the problem of finding the ``true'' value of $\alpha$. 

An obvious approach is to solve  $\enorm{\vu_{\alpha}}^2=\lambda^2$. Less obvious is the \emph{MSN equation}
\begin{equation}
  \label{eq:25}
  h_{\alpha} := \lambda^{-1} - \norm{\vu_{\alpha}}{2}^{-1} = 0,
\end{equation}
which has the benefit of being almost linear in the search interval, which results in fast convergence~\citep{More83}. Thus, the task is to find the root of the function $h_{\alpha}$, for which we use Newton's method, which in this case leads to the iteration
\begin{equation}
  \label{eq:2}
  \alpha  \gets \alpha - h_{\alpha}/h'_{\alpha}.
\end{equation}
Some calculation shows that the derivative $h'$ can be computed as
\begin{equation}
  \label{eq:1}
  \frac{1}{h'_{\alpha}} = \frac{\enorm{\vu_{\alpha}}^3}{\vu_{\alpha}^T(\md\md^T+\alpha\mi)^{-1}\vu_{\alpha}}.
\end{equation}

The key idea in MSN is to eliminate the matrix inverse in~\eqref{eq:1} by using the Cholesky decomposition $\md \md^T + \alpha \mi = \mr^T_{\alpha} \mr_{\alpha}$ and defining a vector $\vq_{\alpha} = (\mr^T_{\alpha})^{-1} \vu$, so that $\enorm{\vq_{\alpha}}^2 = \vu_{\alpha}^T(\md \md^T+\alpha \mi)^{-1} \vu_{\alpha}$. As a result, the Newton iteration~\eqref{eq:2} becomes

\begin{eqnarray*}
        \alpha - \frac{h_{\alpha}}{h'_{\alpha}} & = & \alpha - (\enorm{\vu_{\alpha}}^{-1} - \lambda^{-1}) \cdot \frac{\enorm{\vu_{\alpha}}^3}{\vu_{\alpha}^T(\md \md^T+\alpha \mi)^{-1} \vu_{\alpha}}, \\
         & = & \alpha - \frac{\enorm{\vu_{\alpha}}^2 - \lambda^{-1} \enorm{\vu_{\alpha}}^3}{\enorm{\vq_{\alpha}}^2},  \\
         & = & \alpha - \frac{\enorm{\vu_{\alpha}}^2}{\enorm{\vq_{\alpha}}^2} \left( 1 - \frac{\enorm{\vu_{\alpha}}}{\lambda} \right), \\
\end{eqnarray*}
and therefore
\begin{equation}
  \label{eq:MSNupdate}
  \alpha \quad \gets \quad \alpha - \frac{\enorm{\vu_{\alpha}}^2}{\enorm{\vq_{\alpha}}^2} \left( 1 - \frac{\enorm{\vu_{\alpha}}}{\lambda} \right).
\end{equation}

As shown for TV-L$_1$ (Appendix~\ref{app:projNewton}), the tridiagonal structure of $(\md \md^T + \alpha \mi)$ allows one to compute both $\mr_{\alpha}$ and $\vq_{\alpha}$ in linear time, so the overall iteration runs in $O(n)$ time.

The above ideas are presented as pseudocode in Algorithm \ref{algMS}. As a stopping criterion two conditions are checked: whether the duality gap is small enough, and whether $\vu$ is close enough to the boundary. This latter check is useful because intermediate solutions could be dual-infeasible, thus making the duality gap an inadequate optimality measure on its own. In practice we use tolerance values $\epsilon_{\lambda} = 10^{-6}$ and $\epsilon_{\text{gap}} = 10^{-5}$.

\begin{algorithm}[t]
 \caption{MSN based TV-L2 proximity}
 \label{algMS}
 \begin{algorithmic}
  \State {\it Initialize:} $\alpha = 0$, $\vu_{\alpha} = 0$.
  \While{$\left| \enorm{\vu_{\alpha}}^2 - \lambda \right|$ $>$ $\epsilon_{\lambda}$ \textbf{or} $\text{gap}(\vu_{\alpha})$ $>$ $\epsilon_{\text{gap}}$}
   \State Compute Cholesky decomp. $\md \md^T + \alpha \mi = \mr^T_{\alpha} \mr_{\alpha}$.
   \State Obtain $\vu_{\alpha}$ by solving $\mr^T_{\alpha} \mr_{\alpha} \vu_{\alpha} = \md \vy$.
   \State Obtain $\vq_{\alpha}$ by solving $\mr^T_{\alpha} \vq_{\alpha} = \vu_{\alpha}$.
   \State $\alpha = \alpha - \frac{\enorm{\vu_{\alpha}}^2}{\enorm{\vq_{\alpha}}^2} \left( 1 - \frac{\enorm{\vu_{\alpha}}}{\lambda} \right)$.
  \EndWhile
  \State\Return $\vu_{\alpha}$
 \end{algorithmic}
\end{algorithm}

Even though Algorithm~\ref{algMS} requires only linear time per iteration, it is fairly sophisticated, and in fact a much simpler method can be devised. This is illustrated here by a gradient-projection method with a \emph{fixed} stepsize $\alpha_0$, whose iteration is
\begin{equation}
  \label{eq:gp}
  \vu^{t+1} = P_{\enorm{\cdot} \le \lambda}(\vu^t - \alpha_0\nabla\phi(\vu^t)).
\end{equation}

The theoretically ideal choice for the stepsize $\alpha_0$ is given by the inverse of the Lipschitz constant $L$ of the gradient $\nabla \phi(\vu)$~\citep{nest07,fista}. Since $\phi(\vu)$ is a convex quadratic, $L$ is simply the largest eigenvalue of the Hessian $\md \md^T$. Owing to its special structure, the eigenvalues of the Hessian have closed-form expressions, namely $\lambda_i = 2 - 2 \cos \left(\frac{i \pi}{n+1}\right)$ (for $1\le i \le n$). The largest one is $\lambda_{n} = 2 - 2 \cos \left(\frac{(n-1) \pi}{n}\right)$, which tends to $4$ as $n \to \infty$; thus the choice $\alpha_0 = 1/4$ is a good and cheap approximation. Pseudocode showing the whole procedure is presented in Algorithm \ref{alg:GP}. Combining this with the fact that the projection $P_{\enorm{\cdot} \le \lambda}$ is also trivial to compute, the GP iteration~\eqref{eq:gp} turns out to be very attractive. Indeed, sometimes it can even outperform the more sophisticated MSN method, though only for a very limited range of $\lambda$ values. 
Therefore, in practice we recommend a hybrid of GP and MSN, as suggested by our experiments (see~\S\ref{sec:exptvl2}).

\begin{algorithm}[tbp]
 \caption{GP algorithm for TV-L$_2$ proximity}
 \label{alg:GP}
 \begin{algorithmic}
  \State {\it Initialize} $\vu^0 \in {\mathbb R}^N$, $t=0$.
  \While{($\neg$ converged)}
   \State Gradient update: $\vv^t = \vu^t - \frac{1}{4} \nabla f(\vu^t)$.
   \State Projection: $\vu^{t+1} = \max(1-\lambda/\enorm{\vv^t}, 0)\cdot \vv^t$.
   \State $t \gets t + 1$.
  \EndWhile
   \State\Return $\vu^t$.
 \end{algorithmic}
\end{algorithm}

\subsection{TV-Lp: Proximity for $\tvell_p^{1D}$}
\label{sec:tvp}
For TV-$L_p$ proximity (for $1 < p < \infty$) the dual problem becomes
\begin{equation}
  \label{eqTVpdual}
  \min_{\vu}\ \phi(\vu) := \half\enorm{\md^T\vu}^2 - \vu^T\md\vy,\ \ \text{s.t.}\ \norm{\vu}{q} \le \lambda,
\end{equation}
where $q = 1 / (1 - 1/p)$.  Problem~\eqref{eqTVpdual} is not particularly amenable to Newton-type approaches, as neither PN (Appendix \ref{app:projNewton}), nor MSN-type methods (\S\ref{sec:tvl2}) can be applied easily. It is partially amenable to gradient-projection (GP), for which the same update rule as in~\eqref{eq:gp} applies, but unlike the $q=2$ case, the projection step here is much more involved. Thus, to complement GP, we may favor the projection-free Frank-Wolfe (FW) method. As expected, the overall best performing approach is actually a hybrid of GP and FW. We summarize both choices below.

\subsubsection{Efficient projection onto the $\ell_q$-ball}


The problem of projecting onto the $\ell_q$-norm ball is
\begin{equation}
  \label{eq:11}
  \min\nolimits_{\vw}\quad d(\vw) := \half \enorm{\vw - \vu}^2,\quad \text{s.t.}\ \ \norm{\vw}{q} \leq \lambda.
\end{equation}
For this problem, it turns out to be more convenient to address its Fenchel dual
\begin{equation}
  \label{eq:8}
  \nlmin_{\vw}\quad d^*(\vw) := \half \enorm{\vw - \vu}^2  +  \lambda \norm{\vw}{p},
\end{equation}
which is actually nothing but $\prox_{\lambda\norm{\cdot}{p}}(\vu)$. The optimal solution, say $\vw^*$, to~\eqref{eq:11} can be obtained by solving~\eqref{eq:8}, by using the  Moreau-decomposition~\eqref{eq.10} which yields
\begin{equation*}
  \vw^* = \vu - \prox_{\lambda\norm{\cdot}{p}}(\vu).
\end{equation*}
Projection~\eqref{eq:11} is computed many times within GP, so it is crucial to solve it rapidly and accurately. To this end, we first turn~\eqref{eq:8} into a differentiable problem and then derive a projected-Newton method following our approach presented in Appendix \ref{app:projNewton}. 

Assume therefore, without loss of generality that $\vu \ge 0$, so that $\vw \ge 0$ also holds (the signs can be restored after solving this problem). Thus, instead of~\eqref{eq:8}, we solve
\begin{equation}
  \label{eq:10}
  \nlmin_{\vw}\quad d^*(\vw) := \half\enorm{\vw-\vu}^2 + \lambda\bigl(\nlsum_i w_i^p\bigr)^{1/p}\quad \text{s.t.}\ \vw \ge 0.
\end{equation}
The gradient of $d^*$ 
may be compactly written as
\begin{equation}
  \label{eq:lpgrad}
  \nabla d^*(\vw) = \vw - \vu + \lambda \norm{\vw}{p}^{1-p}\vw^{p-1},
\end{equation}
where $\vw^{p-1}$ denotes elementwise exponentiation of $\vw$. Elementary calculation yields
\begin{align*}
 \tfrac{\partial^2}{\partial w_i \partial w_j} d^*(\vw)  &= 
  \delta_{ij} \bigl( 1 + \lambda (p-1) \bigl( \tfrac{w_i}{\norm{w}{p}} \bigr)^{p-2} \norm{\vw}{p}^{-1} \bigr) + \lambda (1-p) \bigl( \tfrac{w_i}{\norm{w}{p}} \bigr)^{p-1} \bigl( \tfrac{w_j}{\norm{w}{p}} \bigr)^{p-1} \norm{\vw}{p}^{-1} \\
  &= \delta_{ij} \bigl( 1 - c \hat \vw_i^{p-2} \bigr) + c \bar \vw_i \bar \vw_j,
\end{align*}
where $c := \lambda(1-p) \norm{\vw}{p}^{-1}$, $\hat{\bm{w} }:= \bm{w}/\norm{\vw}{p}$, $\bar{\vw} := (\vw/\norm{\vw}{p})^{p-1}$, and $\delta_{ij}$ is the Dirac delta. In matrix notation, this Hessian's diagonal plus rank-1 structure becomes apparent
\begin{equation}
 \label{eq:hessian}
 \mH (\vw) = \Diag \bigl( 1 - c \hat \vw^{p-2} \bigr) + c \bar \vw \cdot \bar \vw^T
\end{equation}

To develop an efficient Newton method it is imperative to exploit this structure. It is not hard to see that for a set of non-active variables $\bar I$ the reduced Hessian takes the form
\begin{equation}
  \label{eq:lpredhess}
  \mH_{\bar I}(\vw) = \Diag \bigl( \bm{1} - c \hat \vw^{p-2}_{\bar I} \bigr) + c \bar \vw_{\bar I} \bar \vw_{\bar I}^T.
\end{equation}
With the shorthand $\Delta = \Diag \bigl( \bm{1} - c \hat \vw^{p-2}_{\bar I} \bigr)$, the matrix-inversion lemma yields
\begin{equation}
  \label{eq:smlpredhess}
  \mH^{-1}_{\bar I}(\vw) = \bigl( \Delta + c \bar \vw_{\bar I} \bar \vw_{\bar I}^T \bigr)^{-1} = \Delta^{-1} - \frac{\Delta^{-1} c \bar\vw_{\bar I} \bar\vw_{\bar I}^T \Delta^{-1}}{1 + c \bar\vw_{\bar I}^T \Delta^{-1} \bar\vw_{\bar I}}.
\end{equation}
Furthermore, since in PN the inverse of the reduced Hessian always operates on the reduced gradient, we can rearrange the terms in this operation for further efficiency; that is, 
\begin{equation}
   \mH_{\bar I}(\vw)^{-1} \nabla_{\bar I} f(\vw) = \vv \odot \nabla_{\bar I} f(\vw) - \frac{\bigl( \vv \odot \bar\vw_{\bar I} \bigr) \bigl( \vv \odot \bar\vw_{\bar I} \bigr)^T \nabla_{\bar I} f(\vw)} {1/c + \bar\vw_{\bar I} \bigl( \vv \odot \bar\vw_{\bar I} \bigr)} ,
\end{equation}
where $\vv := \bigl(\bm{1} - c \hat\vw^{p-2}_{\bar I}\bigr)^{-1}$, and $\odot$ denotes componentwise product.

The relevant point of the above derivations is that the Newton direction, and thus the overall PN iteration can be computed in $O(n)$ time, which results in a highly effective solver. 

\subsubsection{Frank-Wolfe algorithm for TV-$L_p$ proximity}
\label{sec:FW}
The Frank-Wolfe (FW) algorithm (see e.g., \citep{JaggiFW} for a recent overview), also known as the conditional gradient method~\citep{Bertsekas} solves differentiable optimization problems over compact convex sets, and can be quite effective if we have access to a subroutine to solve linear problems over the constraint set.

The generic FW iteration is illustrated in Algorithm \ref{alg:FW}. FW offers an attractive strategy for TV-$L_p$ because both the descent-direction as well as stepsizes can be computed easily. Specifically, to find the descent direction we need to solve
\begin{equation}
  \label{eq:5}
 \nlmin_{\vs}\quad\vs^T \left( \md \md^T \vu - \md \vy \right),\quad\text{s.t.}\quad\norm{\vs}{q} \le \lambda.
\end{equation}
This problem can be solved by observing that $\max_{\norm{\vs}{q} \le 1} \vs^T \vz$ is attained by some vector $\vs$  proportional to $\vz$, of the form $|\vs^*| \propto |\vz|^{p-1}$. Therefore, $\vs^*$ in~\eqref{eq:5} is found by taking $\vz = \md \md^T \vu - \md \vy$, computing $\vs = - \sgn(\vz) \odot \left| \vz \right|^{p-1}$ and then rescaling $\vs$ to meet $\norm{\vs}{q} = \lambda$.

\begin{algorithm}[htbp]
 \caption{
    \label{alg:FW}
    Frank-Wolfe (FW)
 }
 \begin{algorithmic}
  \State {\bf Inputs:} $f$, compact convex set ${\cal D}$.
  \State Initialize $\vx_0 \in {\cal D}$, $t=0$.
  \While{stopping criteria not met}
   \State Find descent direction: $\min_{\vs} \vs \cdot \nabla f (\vx_t) \; \text{s.t.} \; \vs \in {\cal D}$.
   \State Determine stepsize: $\min_{\gamma} f(\vx_t + \gamma (\vs - \vx_t)) \; \text{s.t.} \; \gamma \in [0, 1]$.
   \State Update: $\vx_{t+1} = \vx_t + \gamma (\vs - \vx_t)$
   \State $t \gets t+1$.
  \EndWhile
  \State\Return{$\vx_t$.}
 \end{algorithmic}
\end{algorithm}

The stepsize can also be computed in closed form owing to the objective function being quadratic. Note the update in FW  takes the form $\vu + \gamma (\vs - \vu)$, which can be rewritten as $\vu + \gamma \vd$ with $\vd = \vs - \vu$. Using this notation the optimal stepsize is obtained by solving
\begin{equation*}
 \nlmin_{\gamma \in [0, 1]} \half \enorm{\md^T (\vu + \gamma \vd)}^2 - \left( \vu + \gamma \vd\right)^T \md \vy.
\end{equation*}
A brief calculation on the above problem yields 
\begin{equation*}
 \gamma^* = \min \left\lbrace \max \left\lbrace \hat\gamma, 1 \right\rbrace, 0 \right\rbrace,
\end{equation*}
where $\hat{\gamma} = - (\vd^T \md \md^T \vu + \vd^T \md \vy)/(\vd^T \md \md^T \vd)$ is the unconstrained optimal stepsize. We note that following \citep{JaggiFW} we 
also check a ``surrogate duality-gap''
\begin{equation*}
 g(\vx)  = \vx^T \nabla f(\vx) - \min_{\vs \in {\cal D}} \vs^T \nabla f(\vx)
         = \left( \vx - \vs^* \right)^T \nabla f(\vx),
\end{equation*}
at the end of each iteration. If this gap is smaller than the desired tolerance, the real duality gap is computed and checked; if it also meets the tolerance, the algorithm stops.

\subsection{Prox operator for TV-L$_\infty$}
The final case is $\tvell_{\infty}^{\oned}$ proximity. We mention this case only for completeness. The dual to the prox-operator here is
\begin{equation}
  \label{eq:20}
  \nlmin_{\vu}\quad\half\enorm{\md^T\vu}^2 - \vu^T\md\vy,\ \ \text{s.t.}\ \norm{\vu}{1} \le \lambda.
\end{equation}
This problem can be again easily solved by invoking GP, where the only non-trivial step is projection onto the $\ell_1$-ball. But the latter is an extremely well-studied operation (see e.g.,~\citep{liu09,kiwiel}), and so $O(n)$ time routines for this purpose are readily available. By integrating them in our GP framework an efficient prox solver is obtained.

\section{Prox operators for multidimensional TV}
\label{sec:proxMulti}
We now move onto discussing how use the efficient 1D-TV prox operators derived above within a prox-splitting framework to handle multidimensional TV~(\ref{eq.multid}) proximity. 

\subsection{Proximity stacking}
The basic composite objective~(\ref{eq.1}) is a special case of the more general class of models where one may have several regularizers, so that we now solve
\begin{equation}
  \label{eq:sumobj}
 \nlmin_{\vx} \quad f(\vx) + \nlsum_{i=1}^m r_i(\vx),
\end{equation}
where each $r_i$ (for $1\le i \le m$) is lsc and convex. 

Just like the basic problem~\eqref{eq.1}, the more complex problem~\eqref{eq:sumobj} can also be tackled via proximal methods. The key to doing so is to use \emph{inexact proximal methods} along with a technique we should call \textbf{proximity stacking}. Inexact proximal methods allow one to use approximately computed prox operators without impeding overall convergence, while proximity stacking allows one to compute the prox operator for the entire sum  $r(\vx) = \sum_{i=1}^m r_i(\vx)$ by ``stacking'' the individual $r_i$ prox operators. This stacking leads to a highly modular design; see Figure \ref{fig:proximalModulesStacking} for a visualization.  In other words, proximity stacking involves computing the prox operator
\begin{equation}
  \label{eq:proxstacking}
  \prox_r(\vy) := \argmin_{\vx}\quad \half\enorm{\vx - \vy}^2 + \nlsum_{i=1}^m r_i(\vx),
\end{equation}
by iteratively invoking the individual prox operators $\prox_{r_i}$ and then combining their outputs. This mixing is done by means of a combiner method, which guarantees convergence to the solution of the overall $\prox_r(\vy)$.

\begin{figure}[htbp]
  \centering
  \includegraphics[scale=0.25]{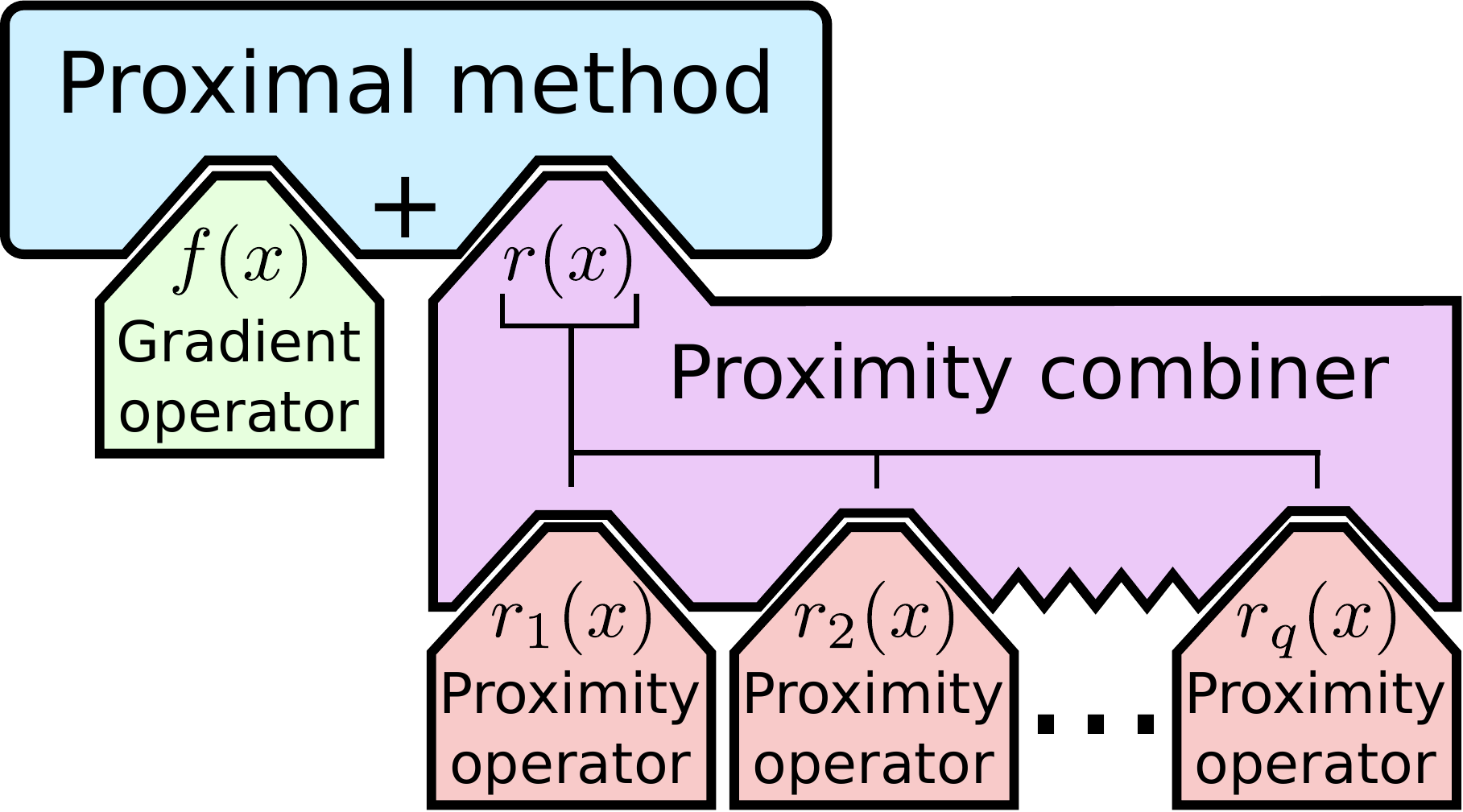}
  \caption{\small Design schema in proximal optimization for minimizing the function $f(\vx) + \sum_{i=1}^m r_i(\vx)$. Proximal stacking makes the sum of regularizers appear as a single one to the proximal method, while retaining modularity in the design of each proximity step through the use of a combiner method. For non-smooth $f$ the same schema applies by just replacing the $f$ gradient operator by its corresponding proximity operator.
   \label{fig:proximalModulesStacking}
   }
\end{figure}

Different proximal combiners can used for computing $\prox_r$~\eqref{eq:proxstacking}. In what follows we briefly describe some of the possibilities. The crux of all of them is that their key steps will be proximity steps over the individual $r_i$ terms. Thus, using proximal stacking and combination, any convex machine learning problem with multiple regularizers can be solved in a highly modular proximal framework. After this section we exemplify these ideas by applying them to two- and higher-dimensional TV proximity, which we then use within proximal solvers for addressing a wide array of applications.

\subsubsection{Proximal Dykstra (PD)}
\label{sec:PD}
The Proximal Dykstra method~\citep{Combettes09} solves problems of the form
\begin{equation*}
 \min_{\vx} \quad \half \enorm{\vx - \vy}^2 + r_1(\vx) + r_2(\vx),
\end{equation*}
which is a particular case of~\eqref{eq:proxstacking} for $m=2$. The method follows the procedure detailed in Algorithm~\ref{alg:PD}, which is guaranteed to converge to the desired solution. Using PD for proximal stacking for 2D Total-Variation was previously proposed in~\citep{Barbero11}.

It has also been shown that the application of this method is equivalent to performing alternating projections onto certain dual polytopes~\citep{jegBac13}, a procedure whose effectiveness varies depending on the relative orientation of such polytopes. A more efficient method based on reflections instead of projections is possible, as we will see below.

\begin{algorithm}
 \caption{
    \label{alg:PD}
    Proximal Dykstra
 }
 \begin{algorithmic}
  \State {\bf Inputs:} $r_1,r_2$, input signal $\vy \in \reals^n$.
  \State Initialize $\vx_0 = \vy$, $\vp_0 = \vq_0 = 0$, $t=0$.
  \While{stopping criteria not met}
   \State $r_2$ proximity operator: $\vz_t = \prox_{r_2} (\vx_t + \vp_t)$.
   \State $r_2$ step: $\vp_{t+1} = \vx_t + \vp_t - \vz_t$.
   \State $r_1$ proximity operator: $\vx_{t+1} = \prox_{r_1} (\vz_t + \vq_t)$.
   \State $r_1$ step: $\vq_{t+1} = \vz_t + \vq_t - \vx_{t+1}$.
   \State $t \gets t+1$.
  \EndWhile
  \State {\bf Return} $\vx_t$.
 \end{algorithmic}
\end{algorithm}

More generally, if more than two regularizers are present (i.e., $m > 2$), then it is more fitting to use \emph{Parallel-Proximal Dykstra} (PPD)~\citep{ParallelDykstra} (see Alg.~\ref{alg:PPXA}), a generalization obtained via the ``product-space trick'' of \citet{pierra84}. This parallel proximal method is attractive because it not only combines an arbitrary number of regularizers, but also allows parallelizing the calls to the individual prox operators. This feature allows us to develop a highly parallel implementation for multidimensional TV proximity (\S\ref{sec:multid}).

\begin{algorithm}[t]
  \caption{\small Parallel-Proximal Dykstra}
  \label{alg:PPXA}
  \begin{algorithmic}
    \State {\bf Inputs:} $r_1, \ldots, r_m$, input signal $\vy \in \reals^n$.
    \State Initialize $\vx_0 = \vy$, $\vz_{0}^i=0$, for $i = 1, \ldots, m$; $t=0$
    \While{stopping criterion not met}
    \For{$i=1$ to $m$ in \emph{parallel}}
    \State $\vp^i_t = \prox_{r_i} (\vz^i_t)$
    \EndFor
    \State $\vx_{t+1} = \frac{1}{m}\sum_i \vp^i_t$
    \For{$i=1$ to $m$ in \emph{parallel}}
    \State $\vz^i_{t+1} = \vx_{t+1} + \vz^i_t - \vp^i_t$
    \EndFor
    \State $t \gets t + 1$
    \EndWhile
    \State {\bf Return} $\vx_t$
  \end{algorithmic}
\end{algorithm}

\subsubsection{Alternating reflections -- Douglas-Rachford (DR)}
\label{sec:DR}
The Douglas-Rachford (DR) method was originally devised for minimizing the sum of two (nonsmooth) convex functions~\citep{Combettes09}, in the form:
\begin{equation}
 \label{eq:DRproblem}
 \min_{\vx} \quad f_1(\vx) + f_2(\vx),
\end{equation}
such that $(\text{ri}\, \dom f_1) \cap  (\text{ri} \, \dom f_2) \neq  \emptyset$. The method operates by iterating a series of reflections, and in its simplest form can be written as
\begin{equation}
 \label{eq:DRsequence}
 \vz_{k+1} = \half \left[ R_{f_1} R_{f_2} + I \right] \vz_k,
\end{equation}
where the \emph{reflection operator} $R_{\phi} := 2 \prox_{\phi} - I$. This method is not cleanly applicable to problem~\eqref{eq:proxstacking} because of the squared norm term. Nevertheless in \citep{jegBac13} a suitable transformation was proposed by making use of arguments from submodular optimization; a minimal background on this topic is given in Appendix \ref{app:mathbackground}. We summarize the key ideas from~\citep{jegBac13} below.

Assume $m=2$ and $r_1, r_2$ being Lovász extensions to some submodular functions (Total-Variation is the Lovász extension of a submodular graph-cut problem, see~\citep{Bach13}). Defining $\hat r_1(\vx) = r_1(\vx) - \vx^T \vy$, $\hat r_1$ is also a Lovász extension of some submodular function (see Appendix \ref{app:mathbackground}). Therefore, we may consider the problem
\begin{equation*}
 \prox_r(\vy) := \argmin_{\vx}\quad \half\enorm{\vx}^2 + \hat r_1(\vx) + r_2(\vx),
\end{equation*}
which can be rewritten (using Proposition \ref{pro:decompsubmodular}) as 
\begin{equation}
 \label{eq:tvdr}
 \min_{a, b} \enorm{a - b}, 
  \quad \text{s.t.} \quad a \in -B_{\hat r_1}, \, b \in B_{r_2},
\end{equation}
where $B_r$ denotes the base polytope of submodular function corresponding to $r$ (see Appendix \ref{app:mathbackground}). The original solution can be recovered through $\vx = \va - \vb$. Problem~\eqref{eq:tvdr} is still not in a form amenable to DR~\eqref{eq:DRproblem}---nevertheless, if we apply DR to the indicator functions of the sets $-B_{\hat r_1}, B_{r_2}$, that is, to the problem
\begin{equation*}
 \min_{\vx} \quad \delta_{-B_{\hat r_1}}(\vx) + \delta_{B_{r_2}}(\vx),
\end{equation*}
it can be shown~\citep{Bauschke04} that the sequence~\eqref{eq:DRsequence} generated by DR is divergent, but that after a correction through projection converges to the desired solution of~\eqref{eq:tvdr}. Such solution is given by the pair
\begin{equation}
 \label{eq:drprojection}
 \vb = \Pi_{B_{r_2}}(\vz_k), \quad 
 \va = \Pi_{-B_{\hat r_1}}(\vb).
\end{equation}
Although in this derivation many concepts have been introduced, suprisingly all the operations in the algorithm can be reduced to performing proximity steps. Note first that the projections onto a base polytope required to get a solution~\eqref{eq:drprojection} can be written in terms of proximity operators (Proposition~\ref{pro:projpolytopes}), which in this case implies
\begin{align*}
 \Pi_{B_{r_2}}(\vz) &= \vz - \prox_{r_2}(\vz), \\
 \Pi_{-B_{\hat r_1}}(\vz) &= \vz + \prox_{\hat r_2}(-\vz) = \vz + \prox_{r_2}(-\vz + \vy),
\end{align*}
where we use the fact that for $f(\vx) = \phi(\vx) + \vu^T \vx$, $\prox_f(\vx) = \prox_{\phi}(\vx - \vu)$. The reflection operations in which the DR iteration is based~\eqref{eq:DRsequence} can also be written in terms of proximity steps, as we are applying DR to the indicator functions $\delta_{-B_{\hat r_1}}, \delta_{B_{r_2}}$, and proximity for an indicator function equals projection. 

This alternating reflections variant of DR is presented in Algorithm~\ref{alg:DR}. Note that in contrast with the original DR method, this variant does not require tuning any hyperparameters, thus enhancing its practicality.

\begin{algorithm}
 \caption{
    \label{alg:DR}
    Alternating reflections -- Douglas Rachford (DR)
 }
 \begin{algorithmic}
  \State {\bf Inputs:} $r_1,r_2$ Lovász extensions of some submodular function, input signal $\vy \in \reals^n$.
  \State Initialize $\vz_0 \in \reals^n$, $t=0$.
  \State Define the following operations:
  \State $\quad \Pi_{-B_{\hat r_1}}(\vz) \defeq \vz + \prox_{r_1}(-\vz + \vy)$.
  \State $\quad \Pi_{B_{r_2}}(\vz) \defeq \vz - \prox_{r_2}(\vz)$.
  \State $\quad R_{-B_{\hat r_1}}(\vz) \defeq 2 \Pi_{-B_{\hat r_1}}(\vz) - \vz$.
  \State $\quad R_{B_{r_2}}(\vz) \defeq 2 \Pi_{B_{r_2}}(\vz) - \vz$.
  \While{stopping criteria not met}
   \State $\vz_{t+1} = \half \left[R_{-B_{\hat r_1}} R_{B_{r_2}} + I \right] \vz_k$
   \State $t \gets t+1$.
  \EndWhile
  \State $\vb = \Pi_{B_{r_2}}(z_t), \quad \va = \Pi_{-B_{\hat r_1}}(\vb)$.
  \State {\bf Return} $\vx^* = \va - \vb$.
 \end{algorithmic}
\end{algorithm}

\subsubsection{Alternating-Direction Method of Multipliers (ADMM)}
\label{sec:ADMM}

Although many times presented as a particular algorithm for solving problems involving the minimization of a certain objetive $f(x) + g(Lx)$ with $L$ a linear operator~\citep{Combettes09}, the Alternating-Direction Method of Multipliers can be thought as general splitting strategy for solving the unconstrained minimization of a sum of functions. This strategy boils down to transforming a problem in the form $\min_{\vx} \sum_{i=1}^m f_i(\vx)$ into a saddle-point problem by introducing consensus constraints and incorporating them into the objective through augmented Lagrange multipliers,
\begin{align*}
 \min_{\vx} \quad \sum_{i=1}^m f_i(\vx) &= \min_{\vx, \vz_1, \ldots, \vz_m} \quad \sum_{i=1}^m f_i(\vz_i) \quad \text{s.t.} \,\, \vz_1 = \vx, \ldots, \vz_m = \vx , \\
  &\equiv \min_{\vx, \vz_1, \ldots, \vz_m} \max_{\vu_1, \ldots, \vu_m} \quad \sum_{i=1}^m \left( f_i(\vz_i) + \vu_i^T (\vz_i - \vx) + \frac{\rho}{2} \enorm{\vz_i - \vx} \right).
\end{align*}
The method then proceeds to solve this problem by alternating steps of minimization on $\vx$, minimization on every $\vz_i$, and a gradient step on every $\vu_i$.

In~\citep{YangTV} a proposal using this method was presented to solve $m$--dimensional anisotropic TV~\eqref{eq.multid}. This approach applies equally to the more general proximal stacking framework under discussion here~\eqref{eq:proxstacking}, by the transformation
\begin{align*}
 \prox_r(\vy) &:= \argmin_{\vx}\quad \half\enorm{\vx - \vy}^2 + \nlsum_{i=1}^m r_i(\vx), \\
 &\equiv \min_{\vx, \vz_1, \ldots, \vz_m} \max_{\vu_1, \ldots, \vu_m} \quad \half\enorm{\vx - \vy}^2 +  \sum_{i=1}^m \left( f_i(\vz_i) + \vu_i^T (\vz_i - \vx) + \frac{\rho}{2} \enorm{\vz_i - \vx} \right).
\end{align*}
The steps for obtaining a solution then follow as Algorithm~\ref{alg:ADMM}. Similar to Parallel Proximal Dykstra, this approach allows computing the prox-operator of each function $r_i$ in parallel.

\begin{algorithm}[t]
  \caption{\small Alternating Direction Method of Multipliers (ADMM)}
  \label{alg:ADMM}
  \begin{algorithmic}
    \State {\bf Inputs:} $r_1, \ldots, r_m$, input signal $\vy \in \reals^n$.
    \State Initialize $\vx_0 = \vz^i_0 = \vy$ for $i = 1, \ldots, m$; $t=0$
    \While{stopping criterion not met}
    \State $\vx_{t+1} = \frac{\vy + \sum_{i=1}^m (\vu^i_t + \rho \vz^i_t)}{1 + m\rho}$.
    \For{$i=1$ to $m$ in \emph{parallel}}
    \State $\vz^i_t = \prox_{\frac{\lambda}{\rho} r_i} (-\frac{1}{\rho} \vu^i_t + \vx_{t+1})$
    \State $\vu^i_{t+1} = \vu_{t+1} + \rho(\vz^i_{t+1} - \vx_{t+1})$
    \EndFor
    \State $t \gets t + 1$
    \EndWhile
    \State {\bf Return} $\vx_t$
  \end{algorithmic}
\end{algorithm}

\subsubsection{Dual proximity methods}
\label{sec:dualproxmethods}
Another family of approaches to solve~\eqref{eq:proxstacking} is to compute the global proximity operator using the Fenchel duals $\prox_{r^*_i}$. This can be advantageous in settings where dual prox-operator is easier to compute than the primal operator; isotropic Total-Variation problems are an instance of such a setting, and thus investigating this approach for their anisotropic variants is worthwhile.

Indeed, in the context of image processing a popular splitting approach is given by~\citet{chambollePock}, which consider a problem in the form
\begin{align*}
 \min_{\vx} \quad F(\mk \vx) + G(\vx),
\end{align*}
for $\mk$ some linear operator, $F, G$ convex lower-semicontinuous functions. Through a strategy similar to ADMM an equivalent saddle point problem can be obtained,
\begin{align*}
 \min_{\vx} \max_{\vy} \quad (\mk \vx)^T \vy + G(\vx) - F^*(\vy),
\end{align*}
with $F^*$ convex conjugate of $F$. This problem is then solved by alternating maximization on $\vy$ and minimization on $\vx$ through proximity steps, as
\begin{align*}
 \vy_{t+1} &= \prox_{\sigma F^*}(\vy_t + \sigma \mk \bar \vx_t) \\
 \vx_{t+1} &= \prox_{\tau G}(\vx_t - \tau \mk^* \vy_{t+1}) \\
 \bar \vx_{t+1} &= \vx_{t+1} + \theta (\vx_{t+1} - \vx_t) ,
\end{align*}
where $\mk^*$ is the conjugate transpose of $\mk$. $\sigma$, $\tau$ and $\theta$ are algorithm parameters that should be either selected under some bounds \citep[Algorithm 1]{chambollePock} or readjusted every iteration making use of Lipschitz convexity of $G$ \citep[Algorithm 2]{chambollePock}, resulting in an accelerating scheme much in the style of FISTA~\citep{fista}. The overall procedure can also be shown to be an instance of preconditioned ADMM, where the preconditioning is given by the application of a proximity step for the maximization of $\vy$ (instead of the usal dual gradient step of ADMM) and the auxiliary point $\bar \vx$. Note also how proximity is computed over the dual $F^*$ instead of the primal $\prox_F$.

Now, this decomposition strategy can be applied for some instances of proximal stacking~\eqref{eq:proxstacking} when the $r_i$ terms allow the particular composition
\begin{align*}
 \sum_{i=1}^m r_i(\vx) = F \left( \left[ \begin{array}{c}
                                     \mk_1 \\
                                     \vdots \\
                                     \mk_m
                                    \end{array} \right] \vx \right)
		 = F(\mk \vx) ,
\end{align*}
which does not hold in general but holds for 2D TV~\eqref{eq.twod} when taking the identities
\begin{align*}
 F(\vx) &= \norm{\vx}{1} , G(\vx) = \half \enorm{\vx - \vy}^2, \\
 \mk &=  \left[ \begin{array}{c}
	      \mi \otimes \md \\
	      \md \otimes \mi
	    \end{array} \right],
\end{align*}
with $\md$ the differencing matrix as before, $\otimes$ denotes Kronecker product, and $\vx$ a vectorization of the 2D input. The iterates above can then be applied easily: proximity over $G$ is trivial and proximity over $F^*$ is also easy upon realizing that $\prox_{\norm{\cdot}{1}^*} = \prox_{\delta_{\norm{\cdot}{\infty}  \leq 1}} = \Pi_{\norm{\cdot}{\infty}  \leq 1}$, which is solved through thresholding.

A generalization of this approach is presented by~\citet{condatGenProx}, who considers
\begin{equation*}
 \min_{\vx} \quad f(\vx) + g(\vx) + \sum_{i=1}^m r_i(\ml_i \vx),
\end{equation*}
a problem that cleanly fits into~\eqref{eq:proxstacking} with $f(\vx) = \half \enorm{\vx - \vy}^2$, $g(\vx) = 0$, $\ml = \mi$. The procedure to find a solution is proposed as
\begin{align*}
 \bar \vx^{t+1} &= \prox_{\tau g^*} \left( \vx^t - \tau \nabla f(\vx^t) - \tau \sum_{i=1}^m \ml_i^* \vu^t_i \right) \\
 \vx_{n+1} &= \rho \bar \vx^{t+1} + (1 - \rho) \vx^t \\
 \bar \vu^{t+1}_i &= \prox_{\sigma h^*_i} (\vu^t_i + \sigma \ml_i (2 \bar \vx_{t+1} - \vx_t)) \quad \forall i=1,\ldots,m \,, \\
 \vu^{t+1}_i &= \rho \bar \vu^{t+1}_i + (1 - \rho) \vu^t_i \quad \forall i=1,\ldots,m \,, 
\end{align*}
for $\tau, \rho$ parameters of the algorithm. When applying this procedure to 2D TV ($m=2$, $r_1(\vx) =$ proximity over rows, $r_2(\vx) =$ proximity over columns) an algorithm almost equivalent to~\citet{chambollePock} is obtained, the only difference being that here the gradient of $f$ is used, instead of the $\prox_G$ operation.

Finally, another related method is the splitting approach of~\citet{Kolmogorov16}, which for $m=2$ performs the following splitting:
\begin{align*}
 & \min_{\vx} \quad \half \enorm{\vx -\vy}^2 + r_1(\vx) + r_2(\vx), \\
 \equiv & \min_{\vx, \vx'} \quad \enorm{\vx -\vy}^2 + r_1(\vx) + r_2(\vx') \quad \text{s.t.} \; \vx = \vx', \\
 \equiv & \min_{\vx, \vx'} \max_{\vz} \quad \enorm{\vx -\vy}^2 + r_1(\vx) + r_2(\vx') + \vz^T(\vx - \vx'), \\
 \equiv & \min_{\vx} \max_{\vz} \quad \enorm{\vx -\vy}^2 + r_1(\vx) - r^*_2(\vz) + \vx^T \vz.
\end{align*}
where we have made use of the Fenchel dual $r^*_2(\vz) = \max_{\vx'} \vz^T \vx' - r_2(\vx')$. 
This problem can be solved through a primal-dual minimization:
\begin{align*}
 \vz^{t+1} &= \prox_{\sigma^t r^*_2} \left( \vz^t + \sigma^t (\vx^t + \theta^t (\vx^t - \vx^{t-1})) \right), \\
 \vx^{t+1} &= \prox_{\tau^t (\enorm{\cdot -\vy}^2 + r_1)} \left( \vx^t - \tau^t \vz^{t+1} \right).
\end{align*}
The primal proximity operator over the squared norm term plus $r_1$ can be rewritten in terms of $\prox_{r_1}$ as
\begin{align*}
\prox_{\tau (r_1 + \half \enorm{\cdot - \vy}^2)}(\vw) &= \argmin_{\vx} r_1(\vx) + \frac{1 + \tau^{-1}}{2} \enorm{\vx - (1+\tau^{-1})^{-1} (\vy + \tau^{-1} \vw)}^2, \\
  &= \prox_{(1+\tau^{-1})^{-1} r_1} \left( (1+\tau^{-1})^{-1} (\vy + \tau^{-1} \vw) \right).
\end{align*}
Regarding the dual step, in the previously presented methods the decompositions allowed to disentangle the effect of a linear operator $L_i$ from each $r_i$. The present decomposition, however, does not take into account this possibility, thus increasing the complexity of computing $r_2^*$. To address this difficulty the Moreau decomposition \eqref{prop.decomp} is helpful, as
\begin{align*}
 \prox_{\sigma r^*_2} \left( \vw \right) &= \vw - \sigma \left( \argmin_{\vx} r_2(\vx) + \frac{\sigma}{2} \enorm{\vx - \sigma^{-1} \vw}^2 \right), \\
  &= \vw - \sigma \prox_{\sigma^{-1} r_2}(\sigma^{-1} \vw),
\end{align*}
thus solving the dual proximity operator in terms of the primal $\prox_{r_2}$.
Regarding the algorithm parameters $\theta$, $\tau$ and $\sigma$, they can be adjusted at every iteration for greater performance making use of Lipschitz convexity~\citep{ChambolleErgodic14}.

\subsection{Two-dimensional TV}
\label{sec:twod}
Recall that for a matrix $\mx \in \reals^{n_1 \times n_2}$, the anisotropic 2D-TV regularizer takes the form
\begin{equation}
  \label{eq.twod.again}
    \tvell_{p,q}^{2}(\mx) :=   
    \nlsum_{i=1}^{n_1}\Bigl(\nlsum_{j=1}^{n_2-1}|x_{i,j+1}-x_{i,j}|^p\Bigr)^{1/p}
    + \nlsum_{j=1}^{n_2}\Bigl(\nlsum_{i=1}^{n_1-1}|x_{i+1,j}-x_{i,j}|^q\Bigr)^{1/q}.
\end{equation}
This regularizer applies a $\tvell_p^{\oned}$ regularization over each row of $\mx$, and a $\tvell_q^{\oned}$ regularization over each column. Introducing differencing matrices $\md_n$ and $\md_m$ for the row and column dimensions, the regularizer~\eqref{eq.twod.again} can be rewritten as
\begin{equation}
  \label{eq:13}
  \tvell_{p,q}^{\twod}(\mx) = \nlsum_{i=1}^n \norm{\md_n \vx_{i,:}}{p} + \nlsum_{j=1}^m \norm{\md_m \vx_{:,j}}{q},
\end{equation}
where $\vx_{i,:}$ denotes the $i$-th row of $\mx$, and $\vx_{:,j}$ its $j$-th column.  The corresponding $\tvell_{p,q}^{\twod}$-proximity problem is
\begin{equation}
  \label{eq:14}
  \nlmin_{\mx}\quad \half\frob{\mx-\my}^2 + \lambda\,\tvell^{\twod}_{p,q}(\mx),
\end{equation}
where we use the Frobenius norm $\frob{\mx} = \sqrt{\sum_{ij}x_{i,j}^2} = \enorm{\vect(\mx)}$,  where $\vect(\mx)$ is the vectorization of $\mx$. Using~\eqref{eq:13}, problem~\eqref{eq:14} becomes
\begin{equation}
  \label{eq:16}
  \nlmin_{\mx}\quad \half\frob{\mx-\my}^2 + \lambda \left( \nlsum_i \norm{\md_n x_{i,:}}{p} \right) + \lambda \left( \nlsum_j \norm{\md_m x_{:,j}}{q} \right),
\end{equation}
where the parentheses make explicit that $\tvell^{\twod}_{p,q}$ is a combination of two regularizers: one acting over the rows and the other over the columns. Formulation~\eqref{eq:16} fits the model solvable by the strategies presented above, though with an important difference: each of the two regularizers that make up $\tvell^{\twod}_{p,q}$ is itself composed of a sum of several ($n$ or $m$) 1D-TV regularizers. Moreover, each of the 1D row (column) regularizers operates on a different row (columns), and can thus be solved independently.



\subsection{Higher-dimensional TV}
\label{sec:multid}
Going even beyond $\tvell_{p,q}^{\twod}$ is the general multidimensional TV~\eqref{eq.multid}, which we recall below. 

Let $\tx$ be an order-$m$ tensor in $\reals^{\prod_{j=1}^m n_j}$, whose components are indexed as $\tx_{i_1,i_2,\ldots,i_m}$ (\fromto{1}{i_j}{n_j} for \fromto{1}{j}{m}); we define TV for $\tx$ as
\begin{equation}
  \label{eq.multid.again}
  \tvell^m_{\vp}(\tx) \defeq \sum_{k=1}^m \sum_{\set{i_1,\ldots,i_m} \setminus i_k}\Bigl(\sum_{j=1}^{n_k-1}
  |\tx_{i_1,\ldots,i_{k-1},j+1,i_{k+1},\ldots,i_m}-\tx_{i_1,\ldots,i_{k-1},j,i_{k+1},\ldots,i_m}|^{p_k}\Bigr)^{1/p_k},
\end{equation}
where $\vp = [p_1,\ldots,p_m]$ is a vector of scalars $p_k \ge 1$. This corresponds to applying a 1D-TV to each of the 1D fibers of $\tx$ along each of the dimensions. 

Introducing the \emph{multi-index} $\vi(k) = (i_1,\ldots,i_{k-1},i_{k+1},\ldots,i_m)$, which iterates over every 1-dimensional fiber of $\tx$ along the $k$-th dimension, the regularizer~\eqref{eq.multid.again} can be written more compactly as
\begin{equation}
  \label{eq:17}
  \tvell^m_{\vp}(\tx) = \nlsum_{k=1}^m \nlsum_{\vi(k)}\norm{\md_{n_k} \vx_{\vi(k)}}{p_k},
\end{equation}
where $\vx_{\vi(k)}$ denotes a row of $\tx$ along the $k$-th dimension, and $\md_{n_k}$ is a differencing matrix of appropriate size for the 1D-fibers along dimension $k$ (of size $n_k$).
The corresponding $m$-dimensional-TV proximity problem is 
\begin{equation}
  \label{eq:18}
  \nlmin_{\tx}\ \half\frob{\tx-\ty}^2 + \lambda\,\tvell^m_{\vp}(\tx),
\end{equation}
where $\lambda > 0$ is a penalty parameter, and the Frobenius norm for a tensor just denotes the ordinary sum-of-squares norm over the vectorization of such tensor. 


Problem~\eqref{eq:18} looks very challenging, but it enjoys decomposability as suggested by ~\eqref{eq:17} and made more explicit by writing it as a sum of $\tvell^{\oned}$ terms
\begin{equation}
  \label{eq:19}
  \nlmin_{\tx}\ \half\frob{\tx-\ty}^2 + \nlsum_{k=1}^m\nlsum_{\vi(k)} \tvell^{\oned}_{p_k}\bigl(\vx_{\vi(k)}\bigr).
\end{equation}
The proximity task~\eqref{eq:19} can be regarded as the sum of $m$ proximity terms, each of which further decomposes into a number of inner $\tvell^{\oned}$ terms. These inner terms are trivial to address since, as in the 2D-TV case, each of the $\tvell^{\oned}$ terms operates on different entries of $\tx$. Regarding the $m$ major terms, we can handle them by applying any of the combiner strategies presented above for $m>2$, which ultimately yield the prox operator for $\tvell^m_{\vp}$ by just repeatedly calling $\tvell^{\oned}$ prox operators. Most importantly, both proximal stacking and the natural decomposition of the problem provide a vast potential for parallel multithreaded computing, which is valuable when dealing with such complex and high-dimensional data.

\section{Experiments and Applications}
\label{sec:exps}
We will now demostrate the effectiveness of the various solvers covered in a wide array of experiments, as well as showing many of their practical applications. We will start by focusing on the $\tvell_1^{\oned}$ methods, moving then to other 1D-TV variants, and then to multidimensional TV.

All the solvers implemented for this paper were coded in C++ for efficiency. Our publicy available library {\bf proxTV} includes all these implementations, plus bindings for easy usage in Matlab or Python: \href{https://github.com/albarji/proxTV}{https://github.com/albarji/proxTV}. Matrix operations have been implented by exploiting the LAPACK (\textsc{Fortran}) library~\citep{LAPACK}.

\subsection{$\tvell_1^{\oned}$ experiments and Applications}
\label{sec:TVproxExp}
Since the most important components of the presented modular framework are the efficient $\tvell_1^{\oned}$ prox operators, let us begin by highlighting their empirical performance. We will do so both on synthetic and natural images data.

\subsubsection{Running time results for synthetic data}

We test the solvers under two scenarios of synthetic signals:\vspace*{-5pt}
\begin{enumerate}[I)]
  \setlength{\itemsep}{1pt}
 \item Increasing input size ranging from $n=10^1$ to $n=10^7$. A penalty $\lambda \in [0,50]$ is chosen at random for each run, and the data vector $\vy$ with uniformly random entries $y_i \in [-2 \lambda, 2 \lambda]$ (proportionally scaled to $\lambda$).
 \item Varying penalty parameter $\lambda$ ranging from $10^{-3}$ (negligible regularization) to $10^3$ (the TV term dominates); here $n$ is set to $1000$ and $y_i$ is randomly generated in the range $[-2, 2]$ (uniformly).
\end{enumerate}

\begin{figure}[htbp]
  \centering
  \subfigure[]{
       \includegraphics[width = 0.45\textwidth]{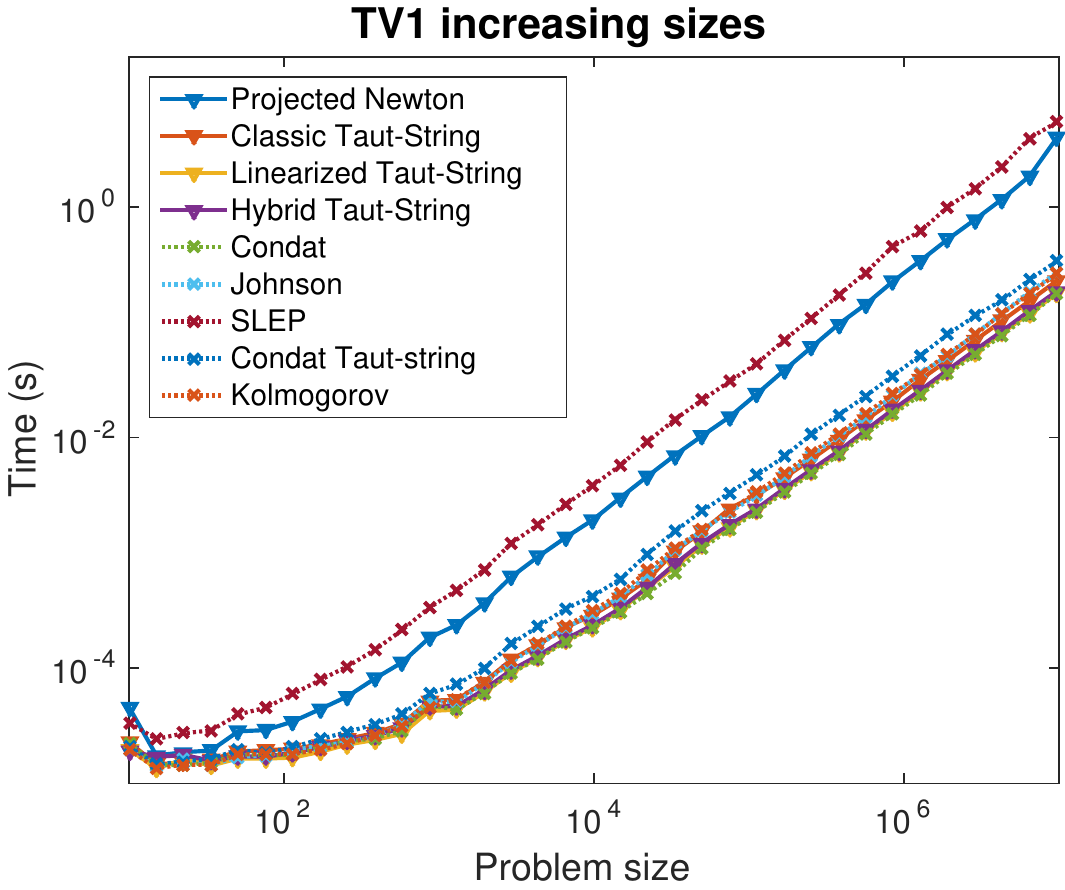}
  }
  \subfigure[]{
       \includegraphics[width = 0.45\textwidth]{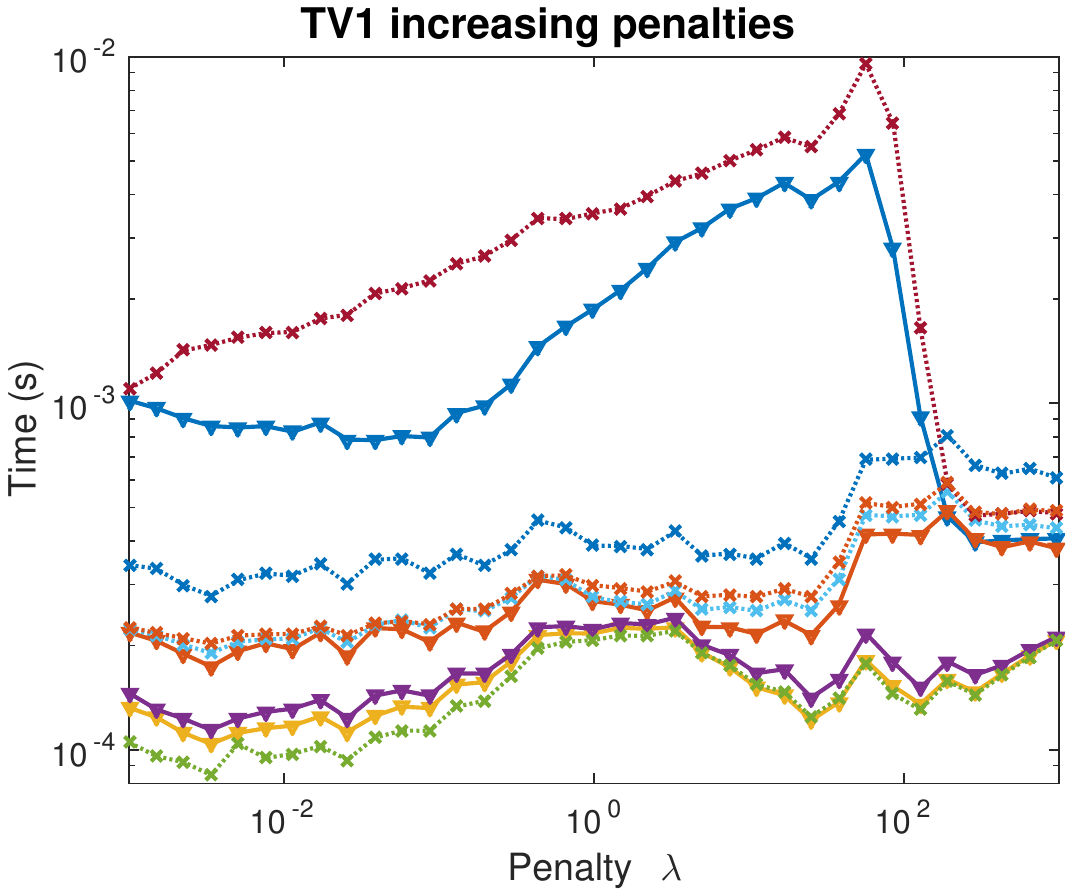}
       \label{fig:tv1runlambda}
  }
  \caption{Running times (in secs) for proposed and state of the art solvers for $\tvell_1^{\oned}$-proximity with increasing a)  input sizes, b) penalties. Both axes are on a log-scale.
   }   \label{fig:tv1run}
\end{figure}

We benchmark the performance of the following methods, including both our proposals and state of the art methods found in the literature:
\begin{itemize}
\setlength{\itemsep}{0pt}
 \item Our proposed Projected Newton method (Appendix \ref{app:projNewton}).
 \item Our efficient implementation of the classic taut string method.
 \item Another implementation of the classic taut string method by~\citet{fastTV}.
 \item An implementation of the linearized taut string method.
 \item Our proposed hybrid taut string approach.
 \item The {\bf FLSA} function (C implementation) of the SLEP library of \citet{slep} for $\tvell_1^{\oned}$-proximity~\citep{liuYe10}.
 \item The state-of-the-art method of~\citet{fastTV}, which we have seen to be equivalent to a linearized taut-string method.
 \item The dynamic programming method of~\citet{dpTV}, which guarantees linear running time.
 \item The message passing method of~\citet{Kolmogorov16}, which allows generalization for computing a Total Variation regularizer on a tree.
\end{itemize}

Another implementation of the classic taut string method, found in the literature, has been added to the benchmark to test whether the implementation we have proposed is on par with the state of the art. We would like to note the surprising lack of widely available implementations of this method: the only working and efficient code we could find was part of the same paper where Condat's method was proposed.

For Projected Newton and SLEP a duality gap of $10^{-5}$ is used as the stopping criterion. For the hybrid taut-string method the switch parameter is set as $S=1.05$. The rest of algorithms do not have parameters.

Timing results are presented in Figure~\ref{fig:tv1run} for both experimental scenarios. The following interesting facts are drawn from these results
\begin{itemize}
 \item Direct methods (Taut string methods, Condat, Johnson, Kolmogorov) prove to be much faster than iterative methods (Projected Newton, SLEP).
 \item Although Condat's (and hence linearized taut string) method, has a theoretical worst-case performance of $O(n^2)$, the practical performance seems to follow an $O(n)$ behavior, at least for these synthetic signals.
 \item Even if Johnson and Kolmogorov methods have a guaranteed running time of $O(n)$, they turn out to be slower than the linearized taut string and Condat's methods. This is in line with our previous observations of the cache-friendly properties of in-memory methods; in contrast Johnson's method requires an extra $\sim 8n$ memory storage. Kolmogorov's method has less memory requirements but nevetheless shows similar behavior.
 \item The same performance observation applies to the classic taut string method. It is also noticeable that our implementation of this method turns out to be faster than previously available implementations (Condat's Taut-string), even becoming slightly faster than the state of the art Johnson and Kolmogorov methods. This result is surprising, and shows that the full potential of the classic taut-string method has been largely unexploited by the research community, or at least that proper efficient implementations of this method have not been made readily available so far.
\end{itemize}

\subsubsection{Worst case scenario}

The point about comparing $O(n)$ and $O(n^2)$ algorithms deserves more attention. As an illustrative experiment we have generated a signal following the worst case description in ~\citet{fastTV}, and tested again the methods above on it, for increasing signal lengths. Figure \ref{fig:tv1worst} plots the results. Condat's method and consequently the linearized taut string method shows much worse performance than the rest of the direct methods. It is also remarkable how the hybrid method manages to avoid quadratic runtimes in this case.

\begin{figure}[htbp]
  \centering
  \includegraphics[width = \textwidth]{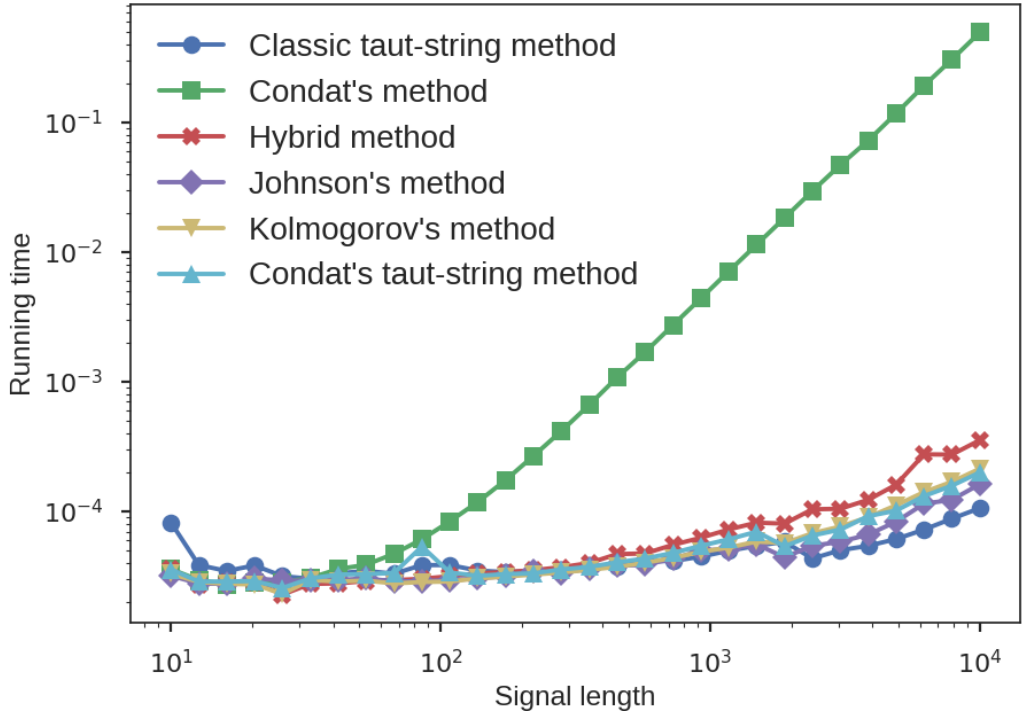} 
  \caption{Running times (in secs) for proposed and state of the art solvers for $\tvell_1^{\oned}$-proximity in the worst-case scenario for Condat's method, for increasing input sizes. Both axes are on a log-scale.
   }   \label{fig:tv1worst}
\end{figure}

\subsubsection{Running times on natural images}

In the light of the previous results the following question arises: in practical settings, are the problems to be solved closer to the worst or the average runtime scenario? This fact will determine whether the guaranteed linear time or the more risky quadratic methods are more apt for practical use. To test this we devise the following experiment: we take a large benchmark of natural images and run each solver over all the rows and columns of all the images in the set, counting total running times, for different regularization values. The benchmark is made from images obtained from the datasets detailed in Table \ref{tab:imagedatasetsTVL1}. We run this benchmark for the methods showing better performance in the experiments above: our implementation of the classic taut-string method, Condat's method ($\equiv$ linearized taut-string method), our proposed Hybrid taut-string method, Johnson's method and Kolmogorov et~al's method.

\begin{table}[t]
\caption{Detail of image datasets used for large-scale $\tvell_1^{\oned}$ experiments.
\label{tab:imagedatasetsTVL1}
}
\centering
\begin{tabular}{lcccc}
\hline
\abovespace\belowspace
 Dataset & Images & Average image size \\
\hline
\abovespace
INRIA holidays~\citep{INRIAholidays} & 812 & 1817 $\times$ 2233 $\times$ 3 px\\
\belowspace
LSVRC 2010 val set~\citep{Imagenet} & 50000 & 391 $\times$ 450 $\times$ 3 px\\
\hline
\end{tabular}
\end{table}

Figure~\ref{fig:tv1inria} shows runtime results for different penalty values over the whole INRIA holidays dataset~\citep{INRIAholidays}, while Figure~\ref{fig:tv1imagenet} shows similar results for the whole Large Scale Visual Recognition Challenge 2010 validation dataset~\citep{Imagenet}. The following facts of interest can be observed:

\begin{itemize}
 \item Condat's method (linearized taut-string) shows top performance for low penalty values, but bad scaling when moving to higher penalties. This can be explained using the geometric intuition developed above: for large penalty values the width of the tube is very large, and thus the taut-string will be composed of very long segments. This is troublesome for a linearized taut-string method, as each backtrack will require recomputing a large number of steps. On the contrary for smaller penalties the tube will be narrow, and the taut-string composed of many small segments, thus resulting in very cheap backtracking costs.
 \item The performance of Classic taut-string, Johnson and Kolmogorov becomes slightly worse for large penalties, but suffers significantly less than the linearized taut-string. Surprisingly, the best performing approach tends to be the classic taut-string method.
 \item The proposed hybrid strategy closely follows the performance of Condat's method for the low penalty regime, while adapting to a behaviour akin to Kolmogorov for large penalties, thus resulting in very good performances over the whole regularization spectrum.
\end{itemize}

\begin{figure}[htbp]
  \centering
  \includegraphics[width = \textwidth]{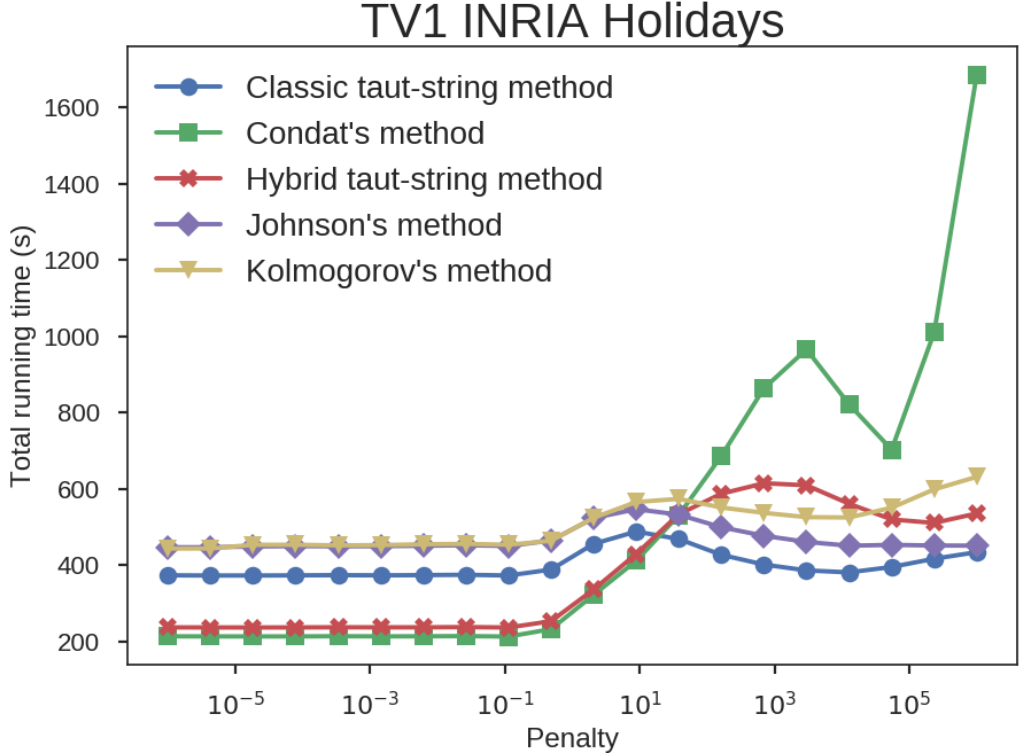} 
  \caption{Running times (in secs) for the top performing proposed and state of the art solvers for $\tvell_1^{\oned}$-proximity over the whole INRIA Holidays dataset, for increasing penalties.
   }   \label{fig:tv1inria}
\end{figure}

\begin{figure}[htbp]
  \centering
  \includegraphics[width = \textwidth]{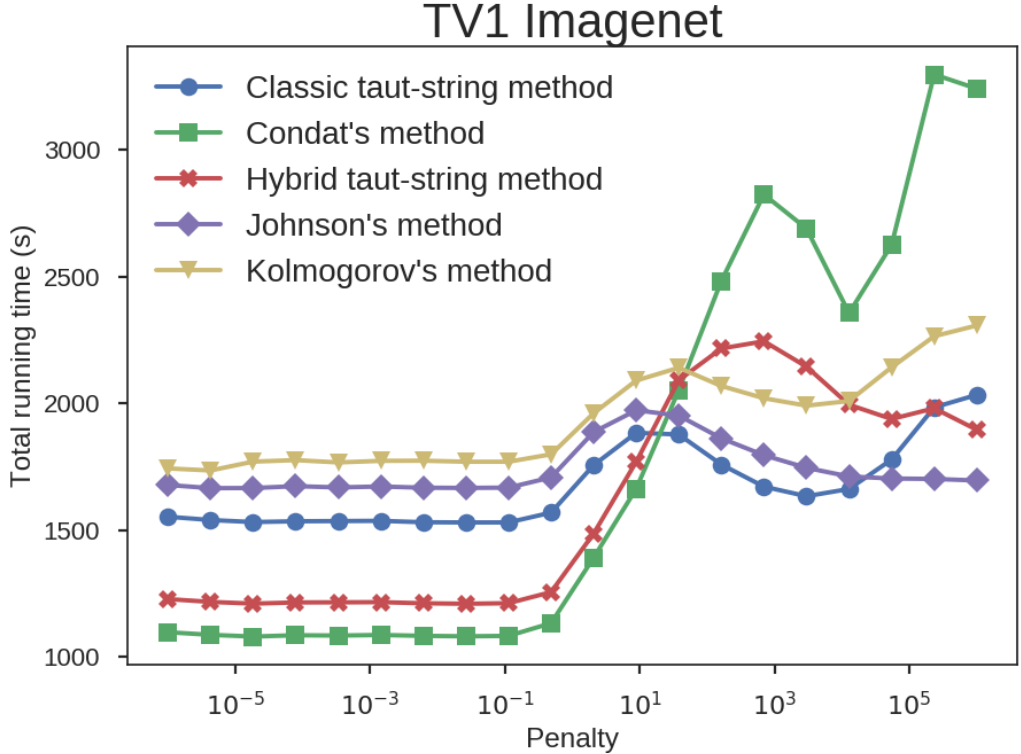} 
  \caption{Running times (in secs) for the top performing proposed and state of the art solvers for $\tvell_1^{\oned}$-proximity over the whole Large Scale Visual Recognition Challenge 2010 validation dataset, for increasing penalties.
   }   \label{fig:tv1imagenet}
\end{figure}

\subsubsection{Running time results for weighted TV-L1}

\begin{figure}[htbp]
  \centering
  \subfigure[]{
       \includegraphics[width = 0.45\textwidth]{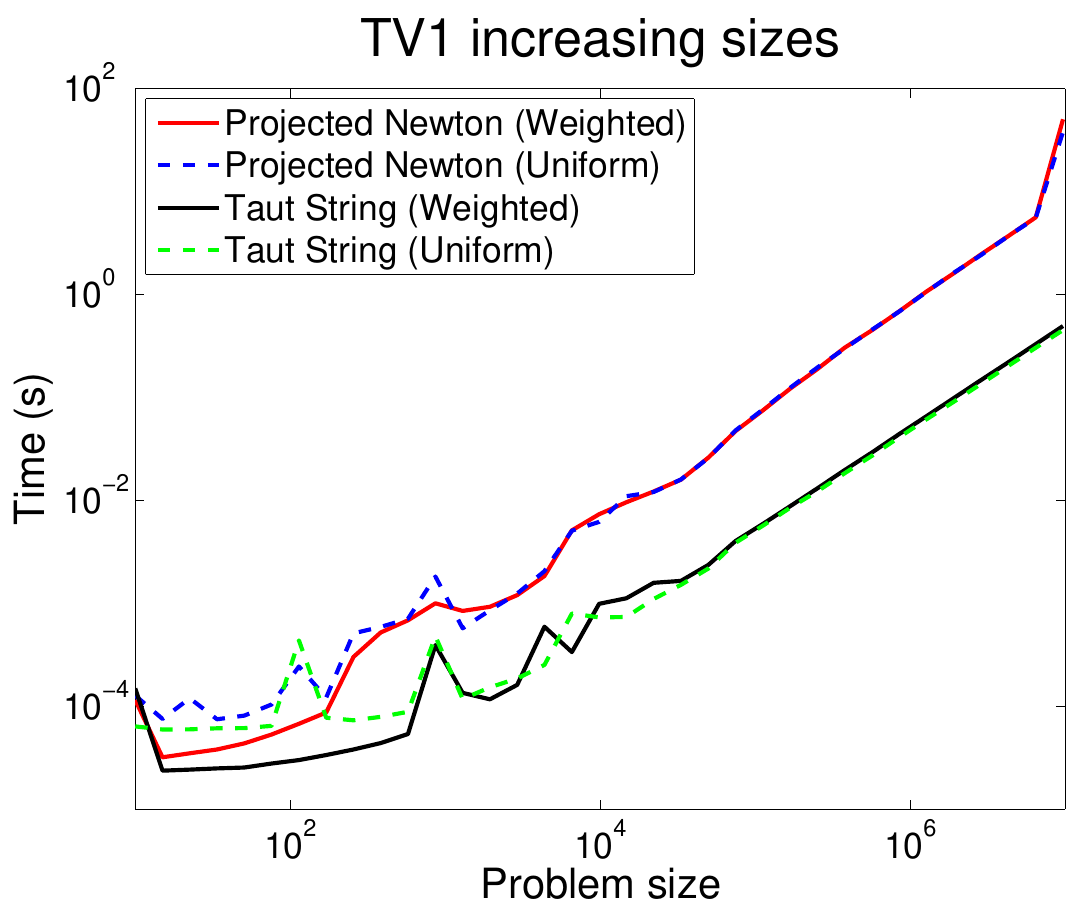}
  }
  \subfigure[]{
       \includegraphics[width = 0.45\textwidth]{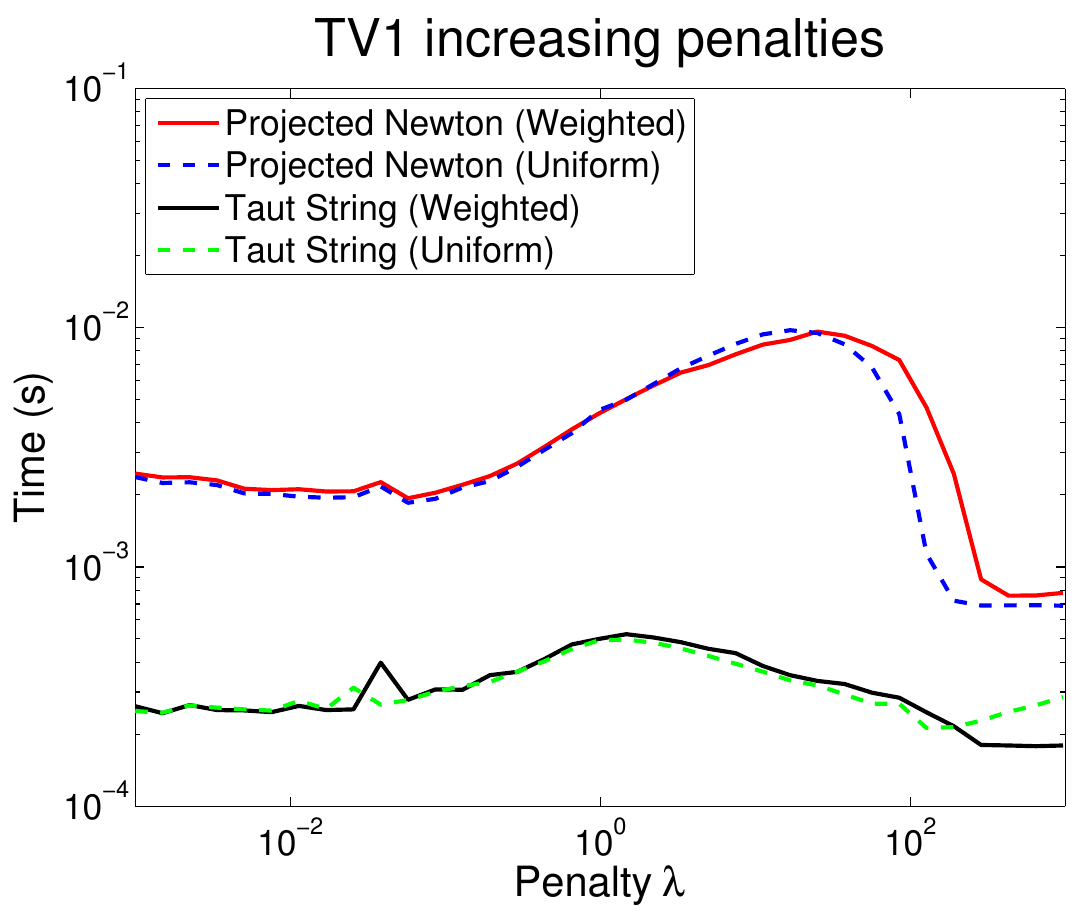}
  }
  \caption{Running times (in secs) for Projected Newton and Taut String solvers for weighted and uniform $\tvell_1^{\oned}$-proximity with increasing a)  input sizes, b) penalties. Both axes are on a log-scale.
   }   \label{fig:tv1wrun}
\end{figure}

An advantage of the solvers proposed in this paper is their flexibility to easily deal with the more difficult, weighted version of the TV-L1 proximity problem. To illustrate this, Figure~\ref{fig:tv1wrun} shows the running times of the Projected Newton and (linearized) Taut String methods when solving both the standard and weighted TV-L1 prox operators.

Since for this set of experiments a whole vector of weights $\vw$ is needed, we have adjusted the experimental scenarios as follows:\vspace*{-5pt}
\begin{enumerate}[I)]
  \setlength{\itemsep}{1pt}
 \item $n$ is generated as in the general setting, penalties $\vw \in [0,100]$ are chosen at random for each run, and the data vector $\vy$ with uniformly random entries $y_i \in [-2 \lambda, 2 \lambda]$, with $\lambda$ the mean of $\vw$, using also this $\lambda$ choice for the uniform (unweighted) case.
 \item $\lambda$ and $n$ are generated as in the general setting, and the weights vector $\vw$ is drawn randomly from the uniform distribution $\vw_i \in [0.5 \lambda, 1.5 \lambda]$.
\end{enumerate}

As can be readily observed, performance for both versions of the problem is almost identical, even if the weighted problem is conceptually harder. Conversely, adapting the other reviewed algorithms to address this problem while keeping up with performance is not a straightforward task.

\begin{figure}[htbp]
  \centering
  \subfigure[]{
       \includegraphics[width = 0.45\textwidth]{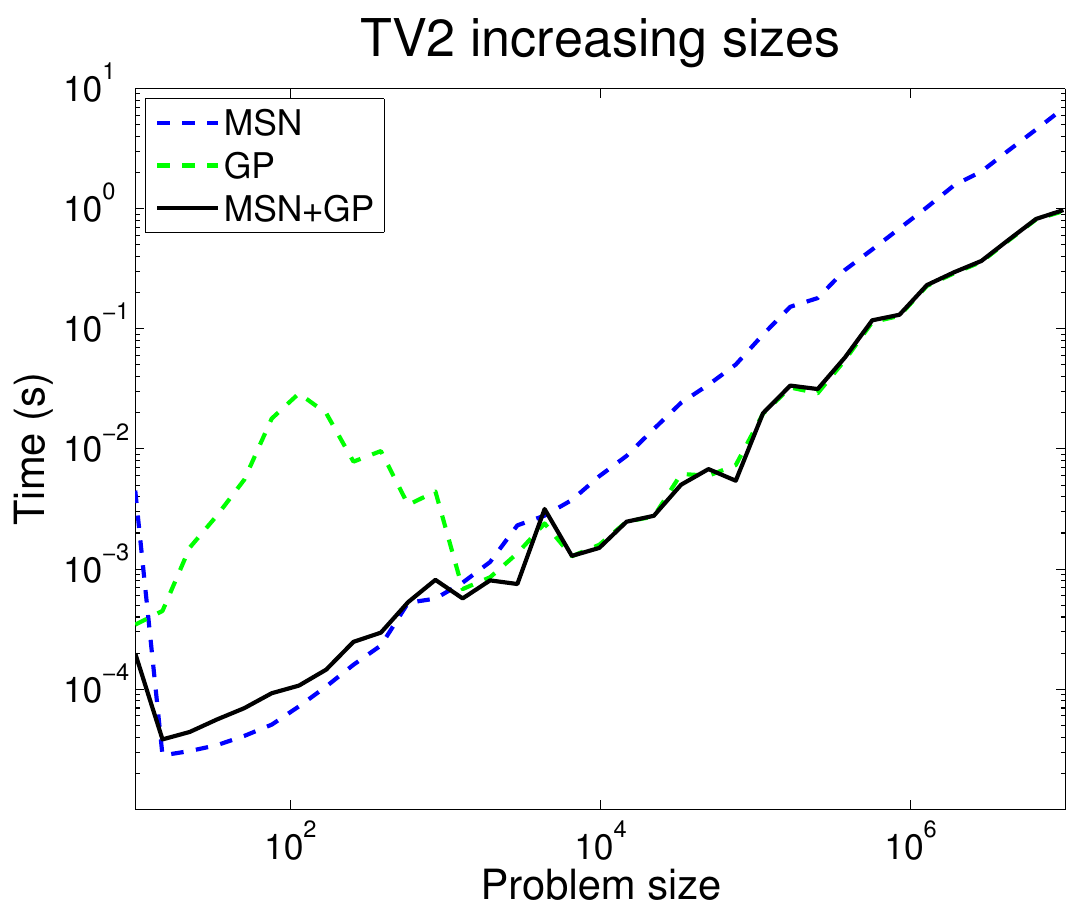}
       \label{fig:tv2runN}
  }
  \subfigure[]{
       \includegraphics[width = 0.45\textwidth]{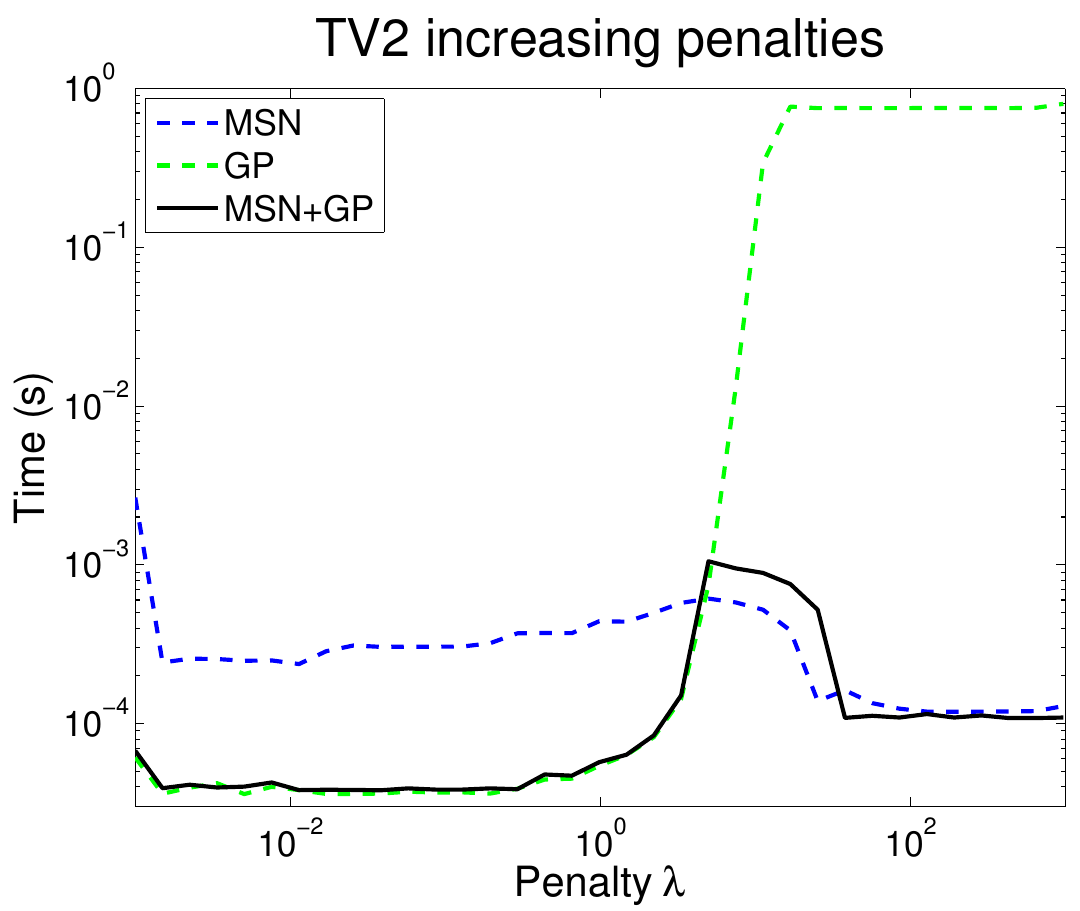}
       \label{fig:tv2runlambda}
  }
  \caption{Running times (in secs) for MSN, GP and a hybrid MSN+GP approach for $\tvell_2^{\oned}$-proximity with increasing a) input sizes, b) penalties. Both axes are on a log-scale.
   \label{fig:tv2run}
   }
\end{figure}

We would also like to point out that in the paper \cite{Kumar15} a practical application of this method for energy minimization in computer vision is presented, where exactly the code behind this paper has been put to use.

\subsection{Experiments for other 1D-TV variants}
\label{sec:expothertv1d}

\subsubsection{Running time results for TV-L2}
\label{sec:exptvl2}

Next we show results for $\tvell_2^{\oned}$ proximity. To our knowledge, this version of TV has not been explicitly treated before, so there do not exist highly-tuned solvers for it. Thus, we show running time results only for the MSN and GP methods. We use a duality gap of $10^{-5}$ as the stopping criterion; we also add an extra boundary check for MSN with tolerance $10^{-6}$ to avoid early stopping due to potentially infeasible intermediate iterates. Figure \ref{fig:tv2run} shows results for the two experimental scenarios under test. 

The results indicate that the performance of MSN and GP differs noticeably in the two experimental scenarios. While the results for the first scenario (Figure \ref{fig:tv2runN}) might suggest that GP converges faster than MSN for large inputs, it actually does so depending on the size of $\lambda$ relative to $\enorm{\vy}$. Indeed, the second scenario (Figure \ref{fig:tv2runlambda}) shows that although for small values of $\lambda$, GP runs faster than MSN, as $\lambda$ increases, GP's performance worsens dramatically, so much that for moderately large $\lambda$, it is unable to find an acceptable solution even after 10,000 iterations (an upper limit imposed in our implementation). Conversely, MSN finds a solution satisfying the stopping criterion under every situation, thus showing a more robust behavior.

These results suggest that it is preferable to employ a hybrid approach that combines the strengths of MSN and GP. Such a hybrid approach is guided using the following (empirically determined) rule of thumb: if $\lambda < \enorm{\vy}$ use GP, otherwise use MSN. Further, as a safeguard, if GP is invoked but fails to find a solution within 50 iterations, the hybrid should switch to MSN. This combination guarantees rapid convergence in practice. Results for this hybrid approach are also included in the plots in Figure \ref{fig:tv2run}, and show how it successfully mimics the behavior of the better algorithm amongst MSN and GP.

\subsubsection{Running time results for TV-Lp}
\label{sec:exptvlp}

Now we show results for $\tvell_p^{\oned}$ proximity. Again, to our knowledge efficient solvers for this version of TV are not available; still proposals for solving the $\ell_q$-ball projection problem do exist, such as the {\it epp} function in SLEP library \citep{slep}, based on a zero finding approach. Consequently, we present here a comparison between this reference projection subroutine and our PN--based projection when embedded in our proposed Gradient Projection solver of \S\ref{sec:tvp}. The alternative proposal given by the Frank--Wolfe algorithm of \S\ref{sec:FW} is also present in the comparison. We use a duality gap of $10^{-5}$ as stopping criterion both for GP and FW. Figure \ref{fig:tvpfixedrun} shows results for the two experimental scenarios under test, for $p$ values of $1.5$, $1.9$ and $3$.

\begin{figure}[htbp]
  \centering
  \subfigure[]{
       \includegraphics[width = 0.31\textwidth]{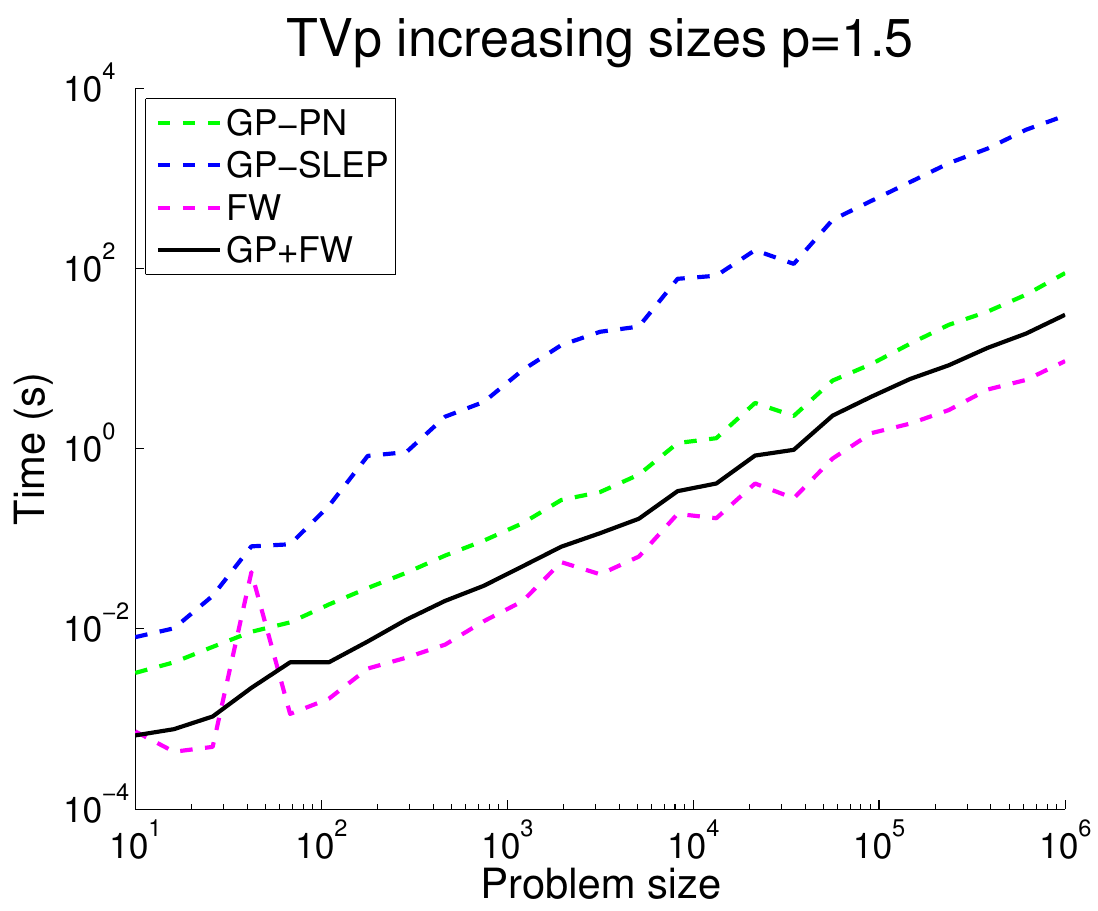}
  }
  \subfigure[]{
       \includegraphics[width = 0.31\textwidth]{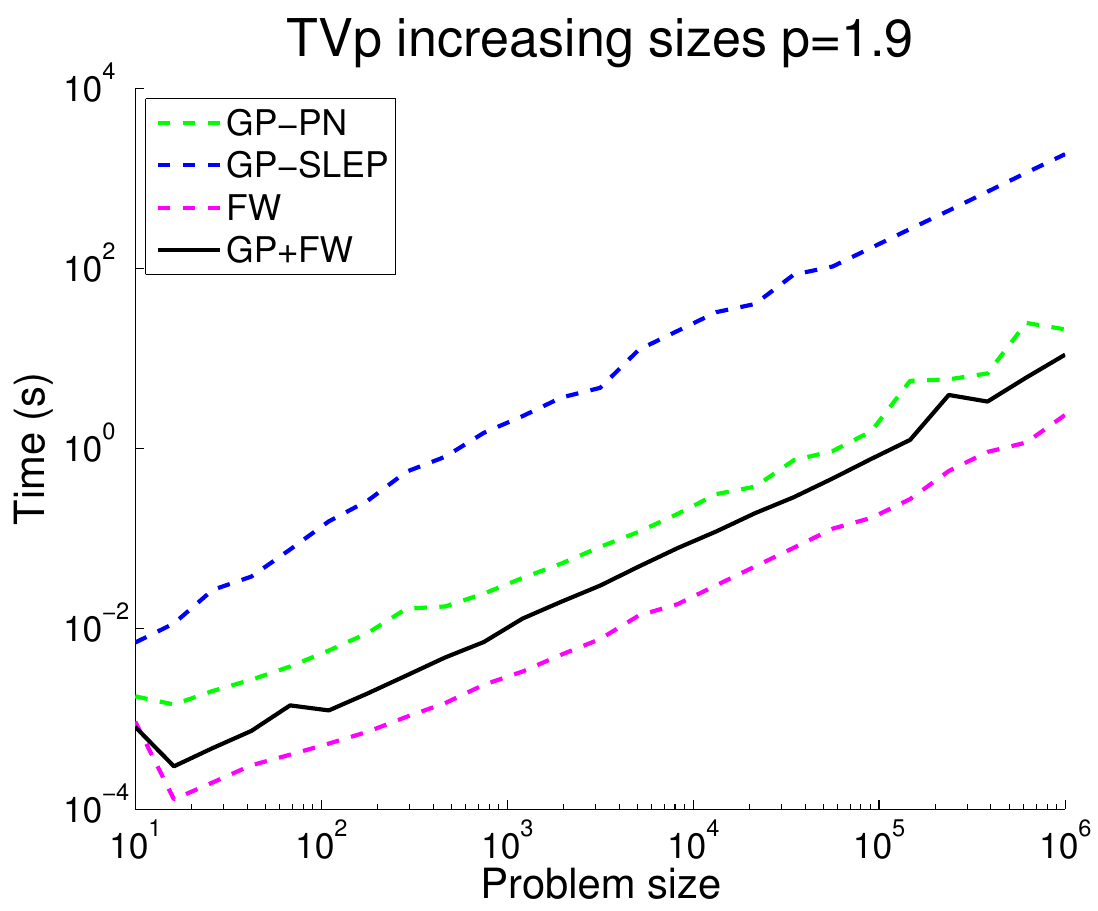}
  }
  \subfigure[]{
       \includegraphics[width = 0.31\textwidth]{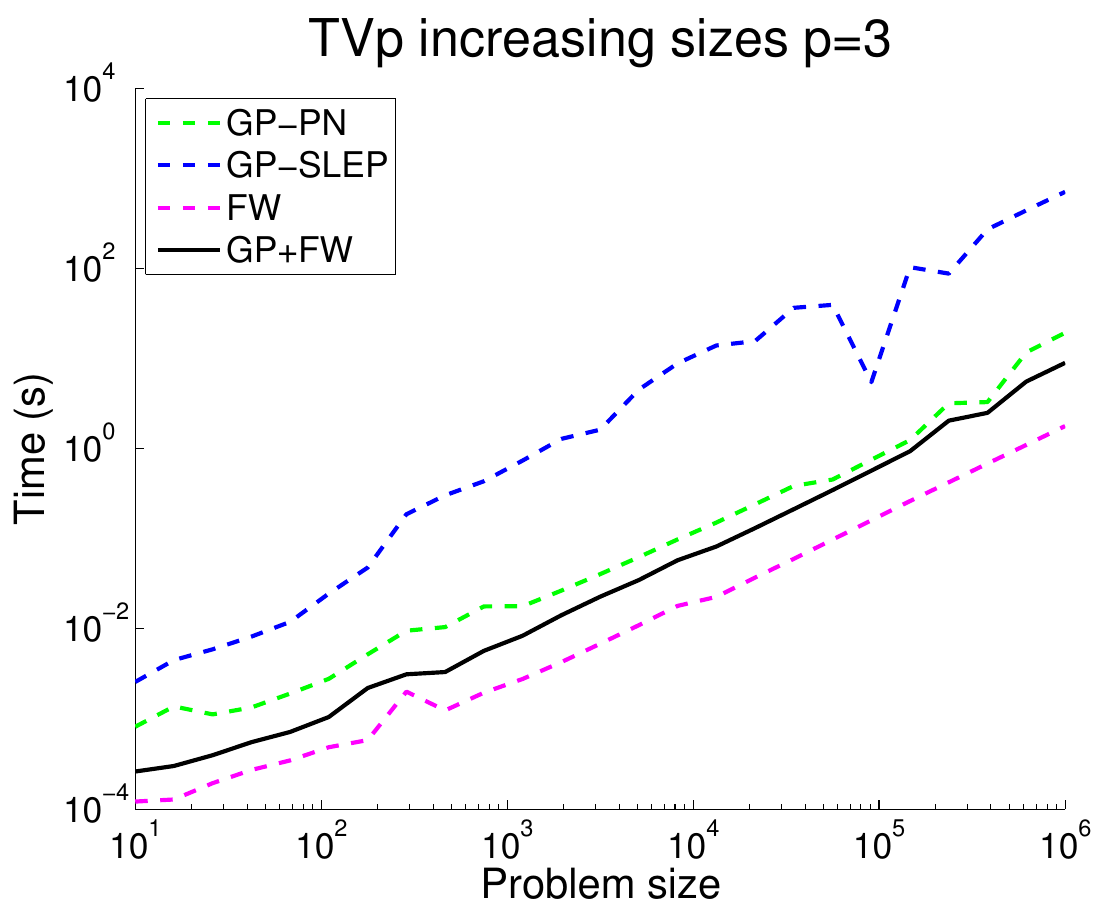}
  }
  \caption{Running times (in secs) for GP with PN projection, GP with SLEP's {\it epp} projection, FW and a hybrid GP+FW algorithm, for $\tvell_p^{\oned}$-proximity with increasing input sizes and three different choices of $p$. Both axes are on a log-scale.
   \label{fig:tvpfixedrun}
   }
\end{figure}

\begin{figure}[htbp]
  \centering
  \subfigure[]{
       \includegraphics[width = 0.31\textwidth]{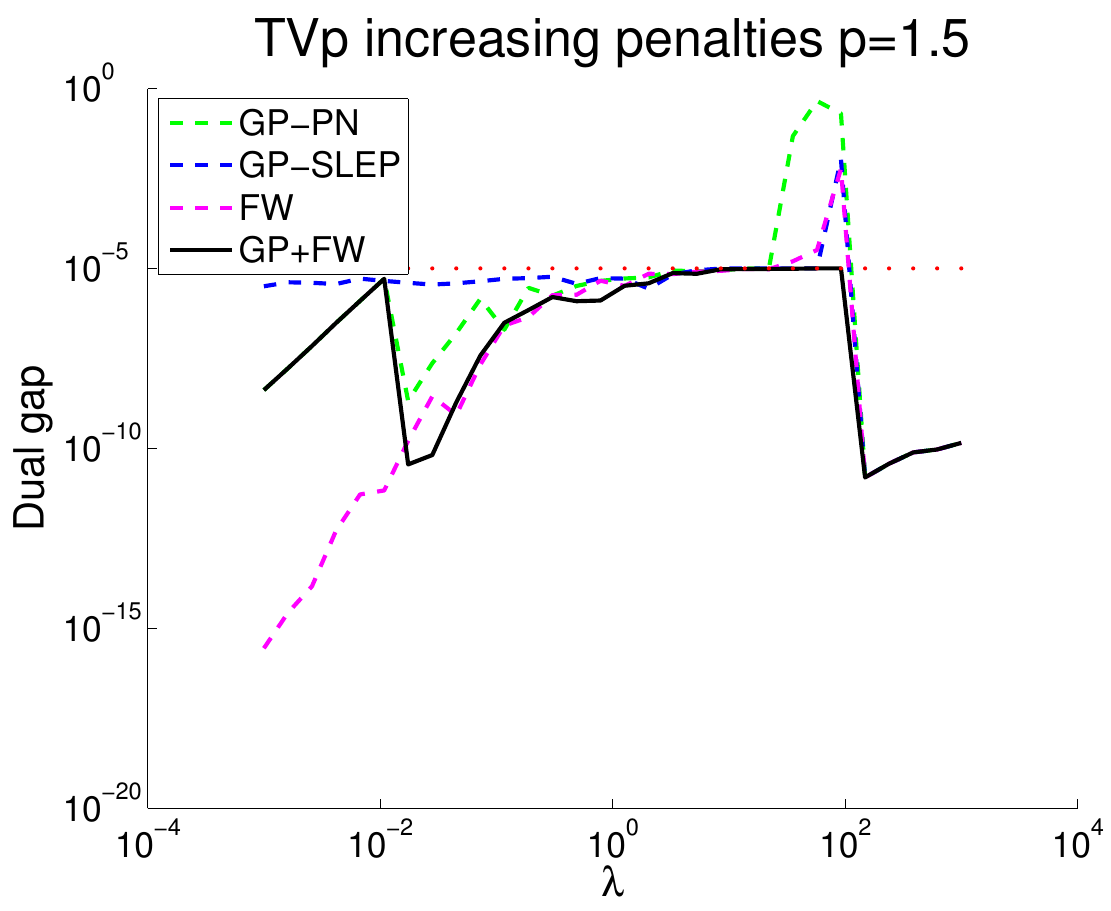}
  }
  \subfigure[]{
       \includegraphics[width = 0.31\textwidth]{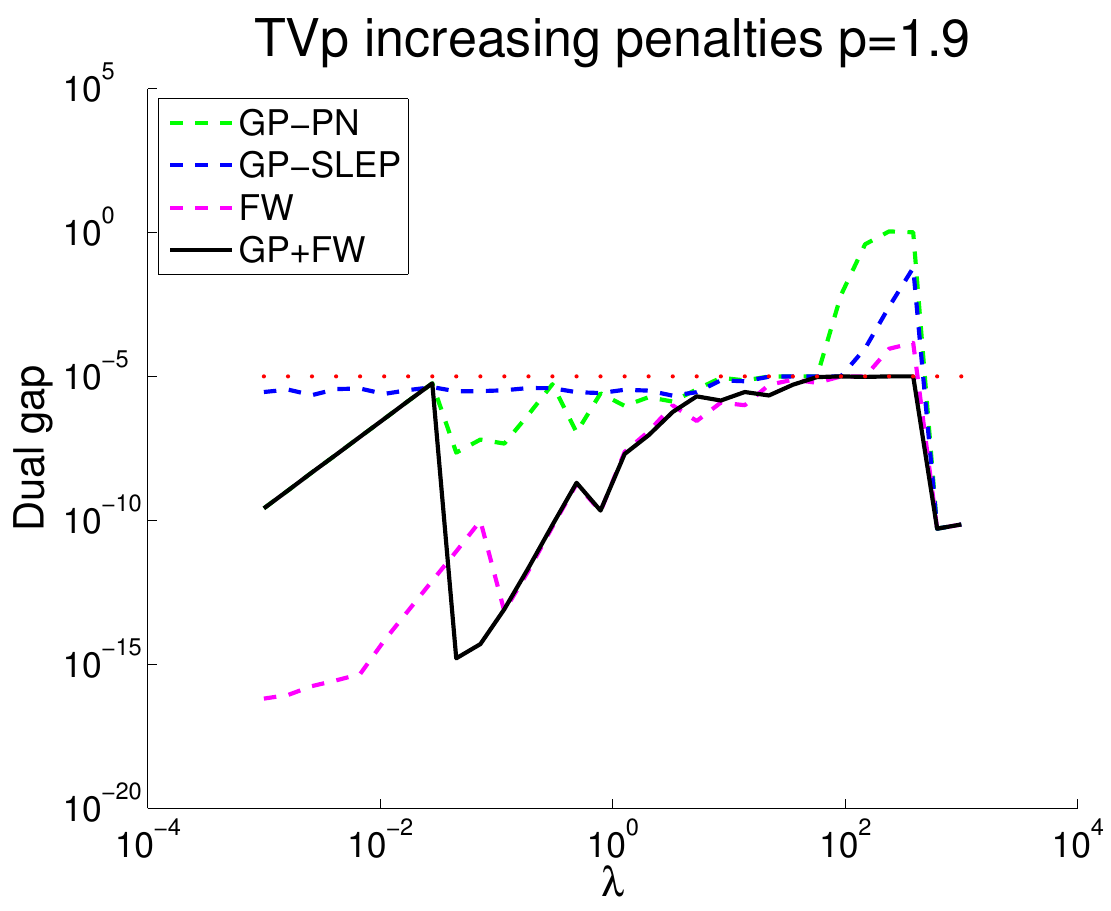}
  }
  \subfigure[]{
       \includegraphics[width = 0.31\textwidth]{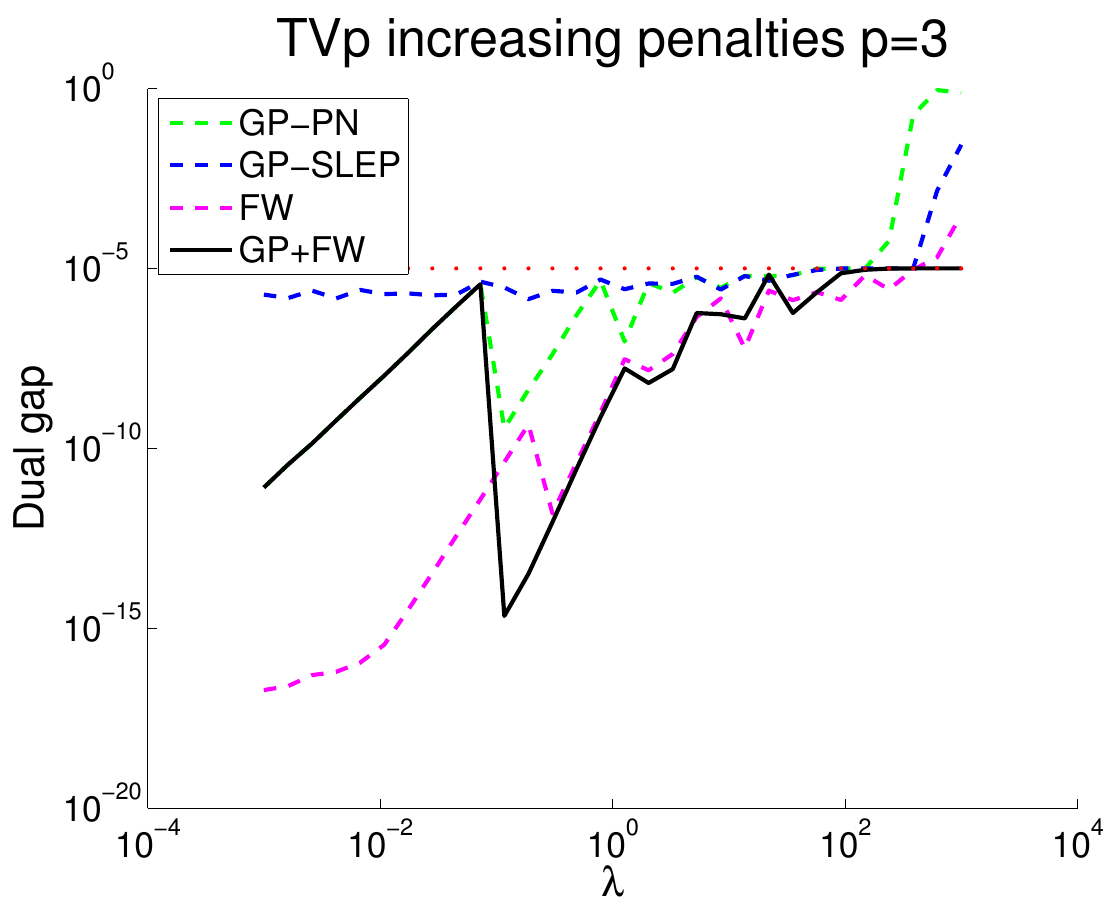}
  }
  \\
  \subfigure[]{
       \includegraphics[width = 0.31\textwidth]{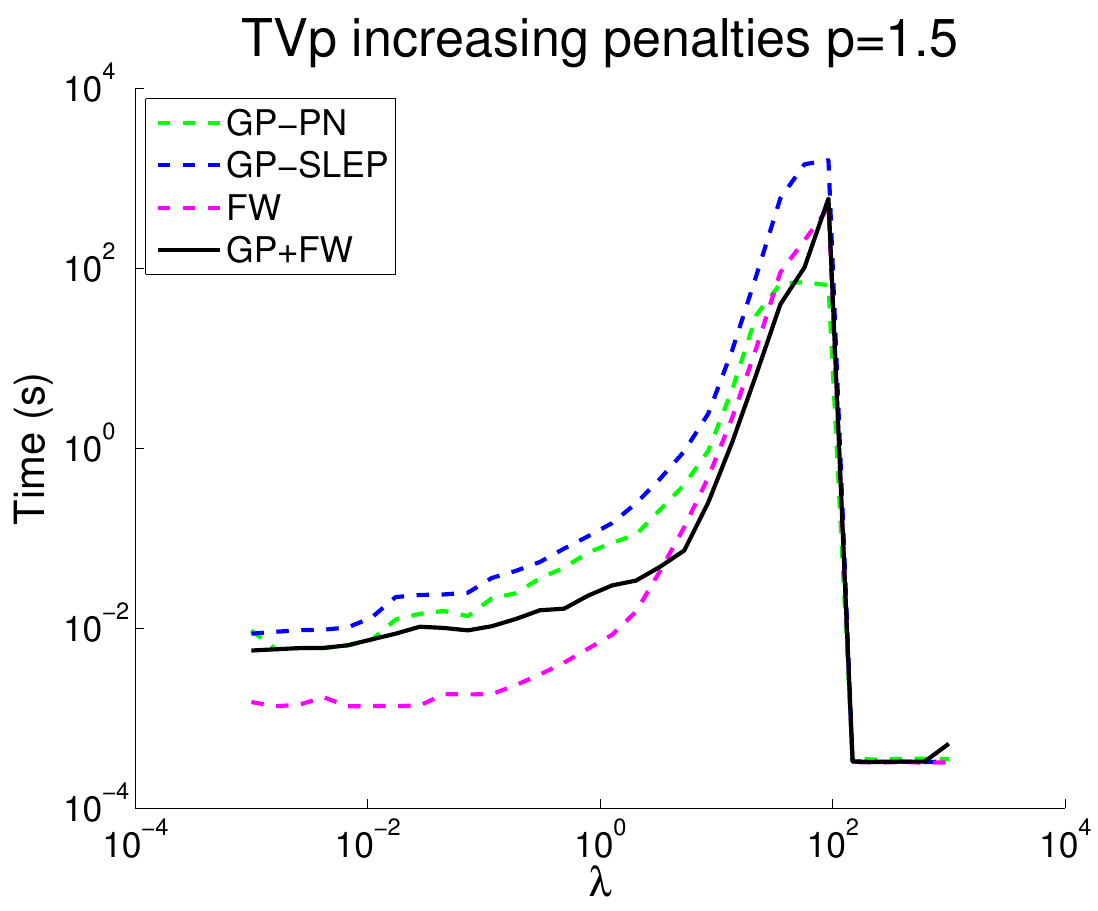}
  }
  \subfigure[]{
       \includegraphics[width = 0.31\textwidth]{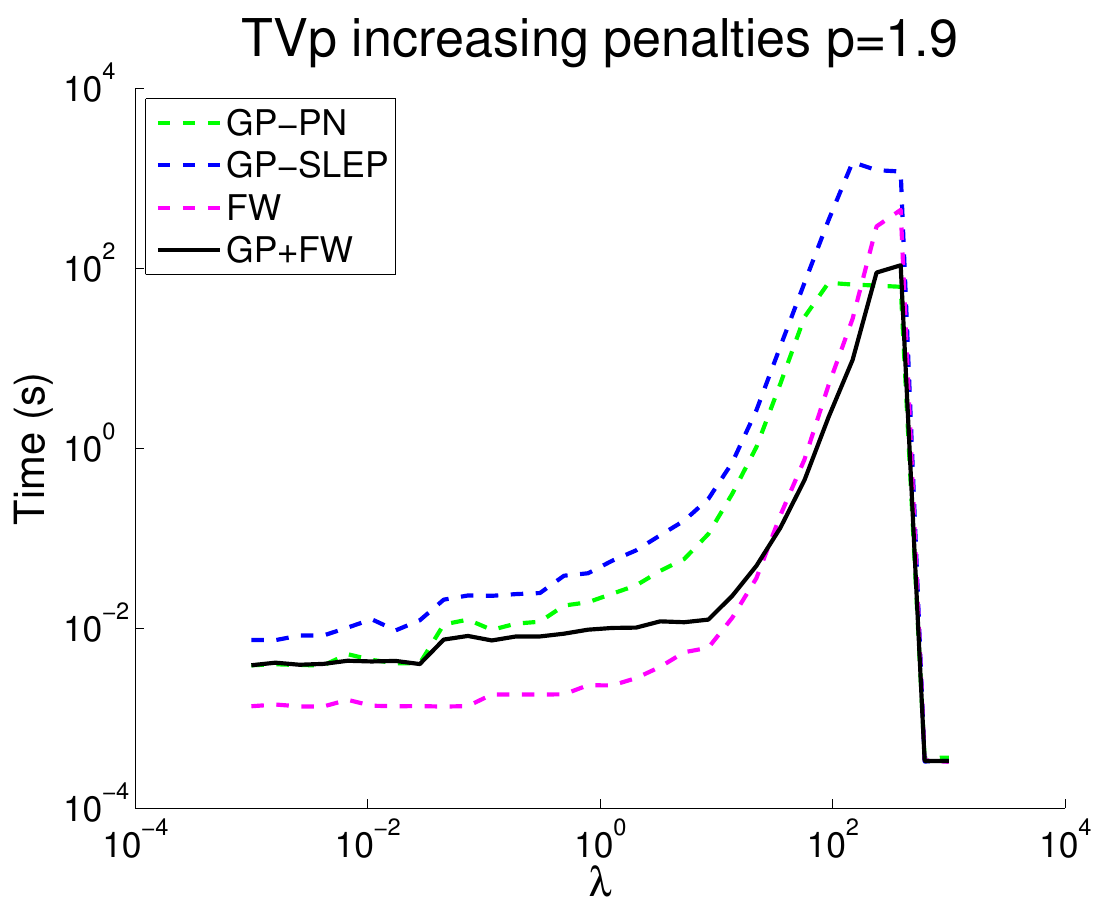}
  }
  \subfigure[]{
       \includegraphics[width = 0.31\textwidth]{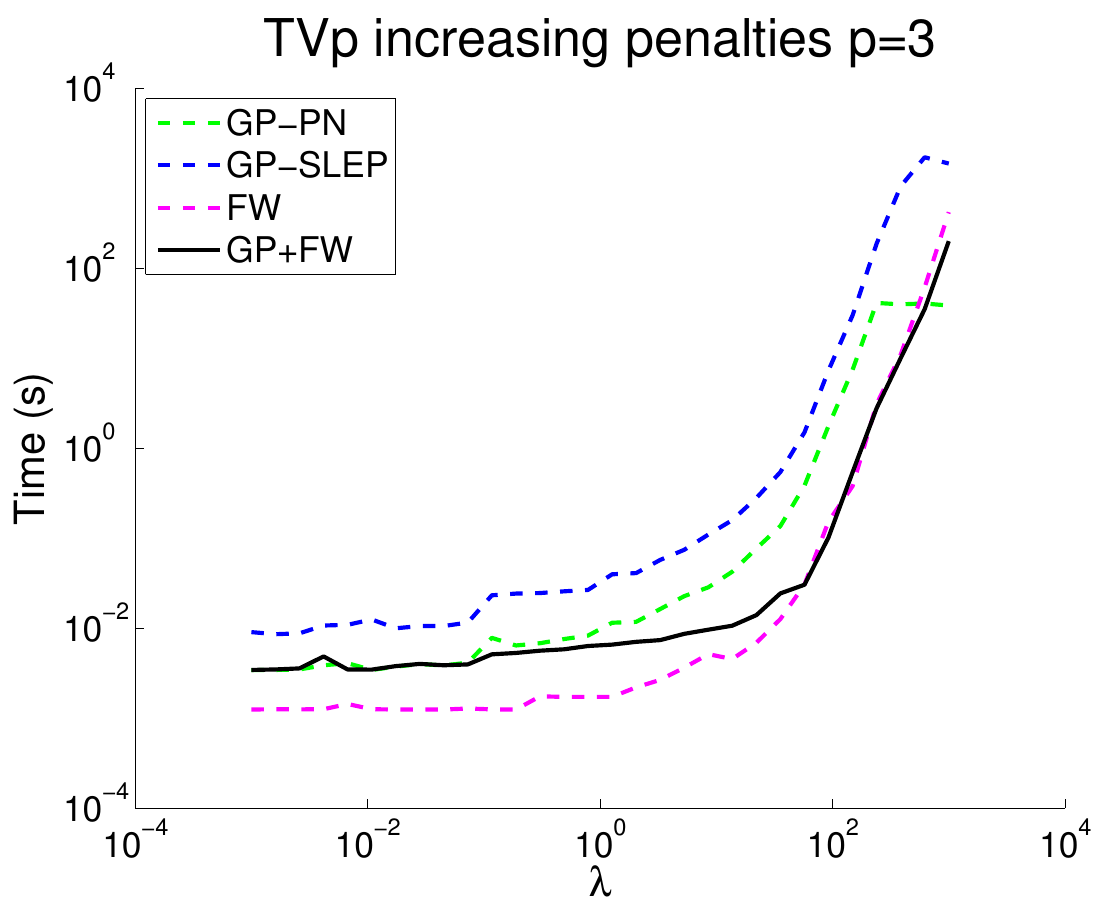}
  }
  \caption{Attained duality gaps (a-c) and running times (d-f, in secs) for GP with PN projection, GP with SLEP's {\it epp} projection, FW and a hybrid GP+FW algorithm, for $\tvell_p^{\oned}$-proximity with increasing penalties and three different choices of $p$. Both axes are on a log-scale.
   \label{fig:tvplambdafixedrun}
   }
\end{figure}

A number of interesting conclusions can be drawn from the results. First, our Projected Newton $\ell_q$-ball subroutine is far more efficient than {\it epp} when in the context of the GP solver. Two factors seem to be the cause of this: in the first place our Projected Newton approach proves to be faster than the zero finding method used by {\it epp}. Secondly, in order for the GP solver to find a solution within the desired duality gap, the projection subroutine must provide very accurate results (about $10^{-12}$ in terms of duality gap). Given its Newton nature, our $\ell_q$-ball subroutine scales better in term of running times as a factor of the desired accuracy, which explains he observed differences in performance.

It is also of relevance noting that Frank--Wolfe is significantly faster than Projected Newton. This should discourage the use of Projected Newton, but we find it to be extremely useful in the range of $\lambda$ penalties where $\lambda$ is large, but not enough to render the problem trivial ($\vw = 0$ solution). In this range the two variants of PN and also FW are unable to find a solution within the desired duality gap ($10^{-5}$), getting stuck at suboptimal solutions. We solve this issue by means of a hybrid GP+FW algorithm, in which updates from both methods are interleaved at a ratio of 10 FW updates per 1 GP update, as FW updates are faster. As both algorithms guarantee improvement in each iteration but follow different procedures for doing so, they complement each other nicely, resulting a superior method attaining the objective duality gap and performing faster than GP.

\subsubsection{Running time results for TV-L$\infty$}

For completeness we also include results for our $\tvell_{\infty}^{\oned}$ solver based on GP + a standard $\ell_1$-projection subroutine. Figure \ref{fig:tvinfrun} presents running times for the two experimental scenarios under test. Since $\ell_1$-projection is an easier problem than the general $\ell_q$-projection the resultant algorithm converges faster to the solution than the general GP $\tvell_{p}^{\oned}$ prox solver, as expected.

\begin{figure}[htbp]
  \centering
  \subfigure[]{
       \includegraphics[width = 0.4\textwidth]{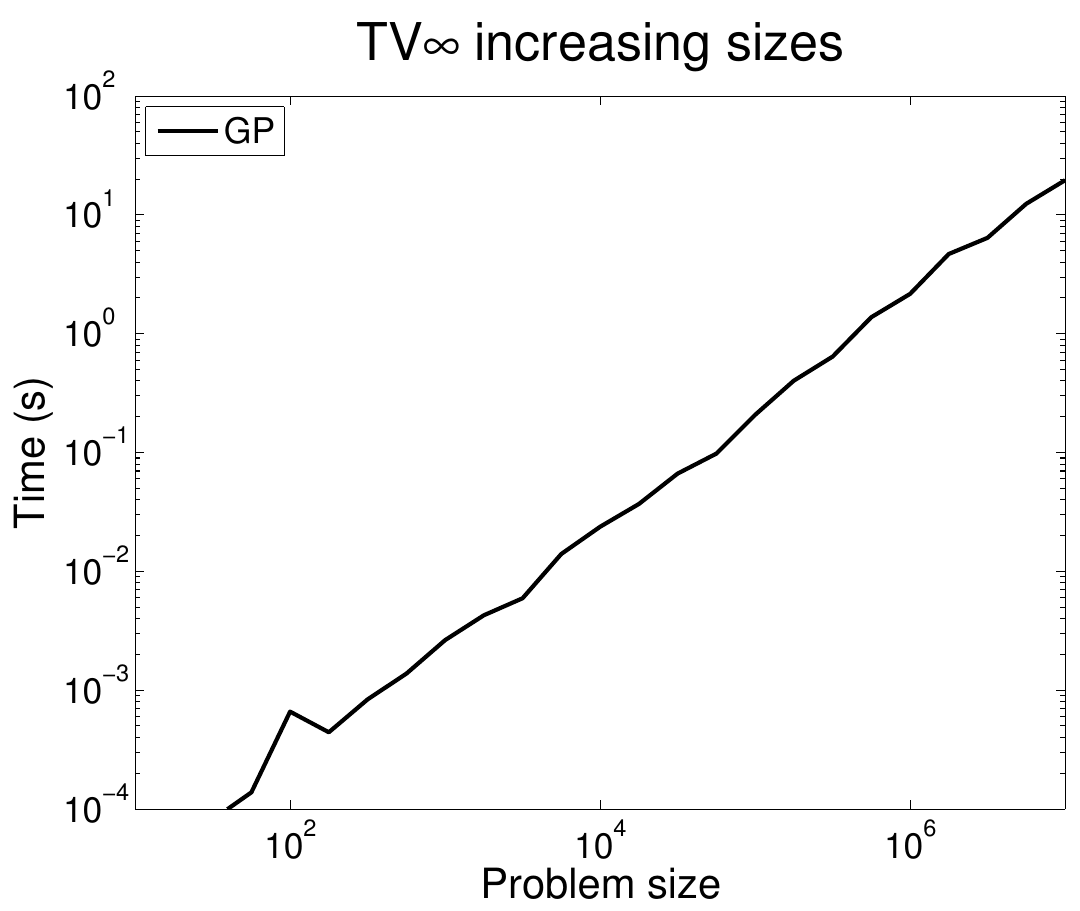}
       \label{fig:tvinfrunN}
  }
  \subfigure[]{
       \includegraphics[width = 0.4\textwidth]{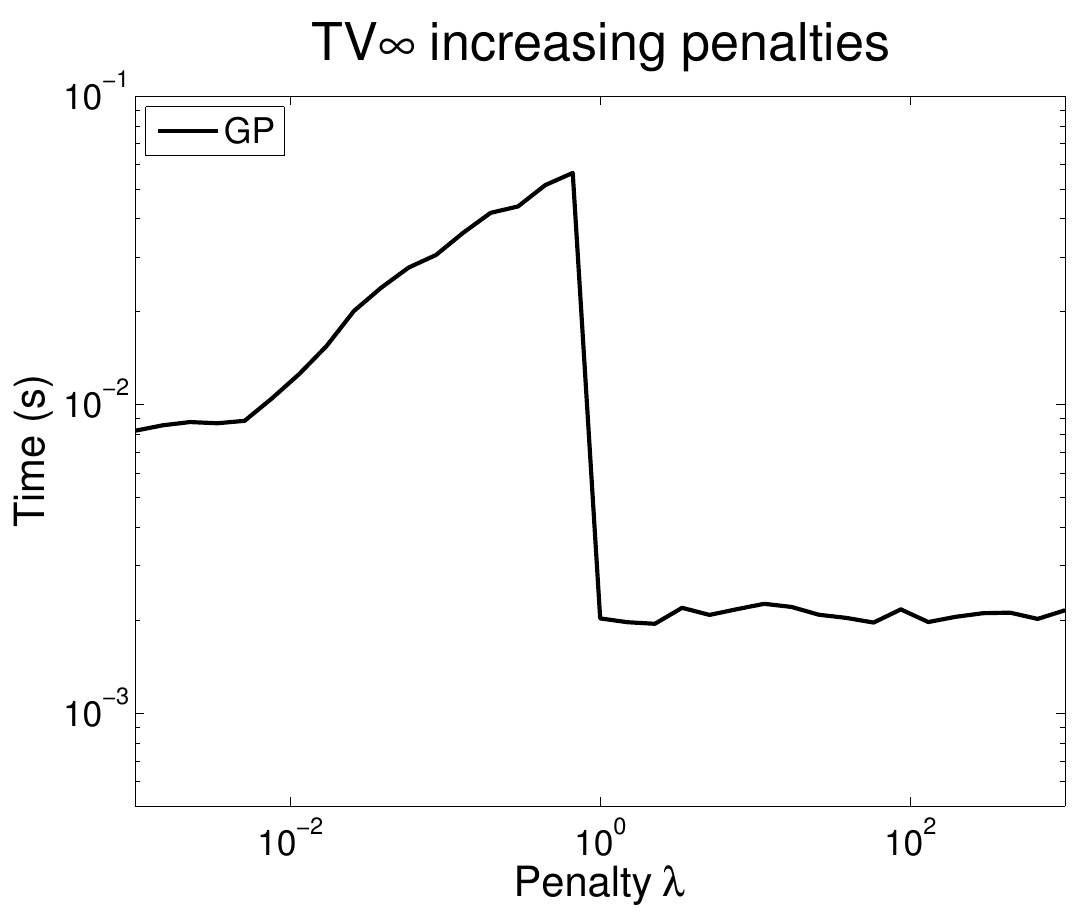}
       \label{fig:tvinfrunlambda}
  }
  \caption{Running times (in secs) for GP for $\tvell_{\infty}^{\oned}$-proximity with increasing a) input sizes, b) penalties. Both axes are on a log-scale.
   \label{fig:tvinfrun}
   }
\end{figure}

\begin{figure}[htbp]
 \centering
 \includegraphics[width = 0.8\textwidth]{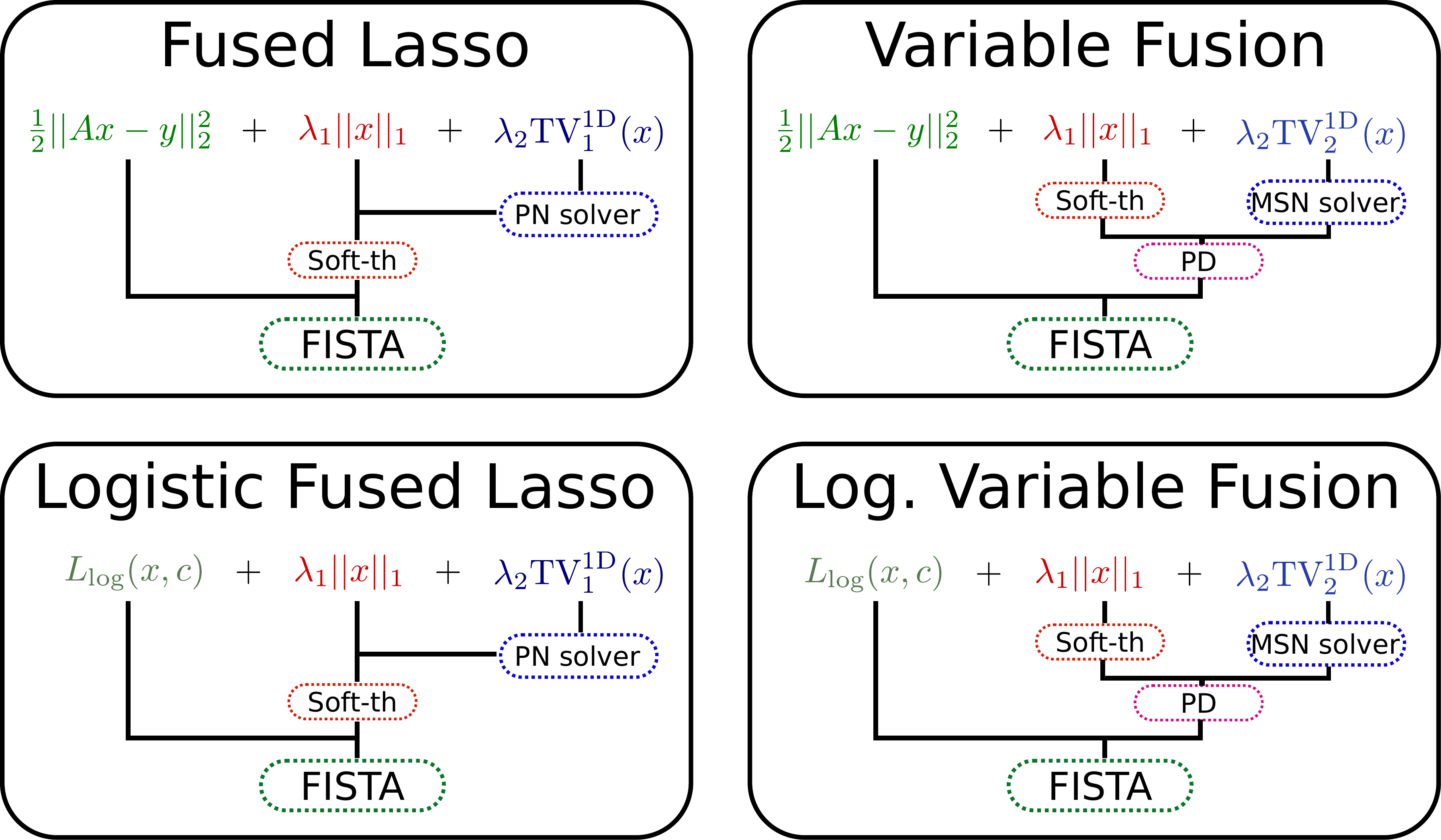}
 \caption{
  \label{fig:FLmodels}
  Fused-Lasso models addressed by proximal splitting.
 }
\end{figure}

\subsubsection{Application: Proximal optimization for Fused-Lasso}
\label{sec:FLasso}
We now present a key application that benefits from our TV prox operators:  \textbf{Fused-Lasso} (FL)~\citep{fl}, a model that takes the form
\begin{equation}
  \label{eq:24}
 \min_{\vx} \quad \half \enorm{\ma \vx - \vy}^2 + \lambda_1 \norm{\vx}{1} + \lambda_2 \tvell_1^{\oned}(\vx).
\end{equation}
 The $\ell_1$-norm in~\eqref{eq:24} forces many $x_i$ to be zero, while $\tvell_1^{\oned}$ favors nonzero components to appear in blocks of equal values $x_{i-1} = x_i = x_{i+1} = \ldots$. The FL model has been successfully applied in several bioinformatics applications \citep{cgh,cghsvm,frHaHoTi07}, as it encodes prior knowledge about consecutive elements in microarrays becoming active at once.

Following the ideas presented in Sec.~\ref{sec:proxMulti}, since the FL model uses two regularizers, we can use Proximal Dykstra as the combiner to handle the prox operator. To illustrate the benefits of this framework in terms of reusability, we apply it to several variants of FL. 
\begin{itemize}
  \setlength{\itemsep}{0pt}
\item \textbf{Fused-Lasso (FL):} Least-squares loss $ + \ell_1 + \tvell_1^{\oned}$ as in~\eqref{eq:24}

\item \textbf{$\ell_p$-Variable Fusion (VF):} Least-squares loss  $ + \ell_1 + \tvell_p^{\oned}$. Though Variable Fusion was already studied by \citet{land}, their approach proposed an $\ell_p^p$-like regularizer in the sense that $r(\vx) = \sum_{i=1}^{n-1} |x_{i+1}-x_i|^p$ is used instead of the TV regularizer $\tvell_p^{\oned}(x) = \left( \sum_{i=1}^{n-1} |x_{i+1}-x_i|^p \right)^{1/p}$. Using $\tvell_p$ leads to a more conservative penalty that does not oversmooth the estimates. This FL variant seems to be new.

\item \textbf{Logistic-fused lasso (LFL):} Logistic-loss $ + \ell_1 + \tvell_1^{\oned}$, where the loss takes the form  $\ell(\vx,c) = \sum_i \log \left( 1 + e^{-y_i (\va_i^T\vx + c)} \right)$, and can be used in a FL formulation to obtain models more appropriate for classification on a dataset $\left\lbrace (\va_i,y_i) \right\rbrace$ \citep{kolar}.

\item \textbf{Logistic + $\ell_p$-fusion (LVF):} Logistic loss $ + \ell_1 + \tvell_p^{\oned}$.
\end{itemize}

To solve these variants of FL, all that remains is to compute the gradients of the loss functions, but this task is trivial. 
Each of these four models can be then solved easily by invoking any proximal splitting method by appropriately plugging in gradient and prox operators. Incidentally, the {\bf SLEP} library~\citep{liuYe10} includes an implementation of FISTA~\citep{fista} carefully tuned for Fused Lasso, which we base our experiments on. Figure \ref{fig:FLmodels} shows a schematic of the algorithmic modules for solving each FL model.

\textbf{Remark:} A further algorithmic improvement can be obtained by realizing that for $r(\vx) = \lambda_1 \norm{\vx}{1} + \lambda_2 \tvell_1^{\oned}(\vx)$ the prox operator $\prox_{r} \equiv \prox_{\lambda_1 \norm{\cdot}{1}} \circ \prox_{\lambda_2 \tvell_1^{\oned}(\cdot)}$. Such a decomposition does not usually hold, but it can be shown to hold for this particular case~\citep{yaoLiang,rinaldo,fl}. Therefore, for FL and LFL we can compute the proximal operator for the combined regularizer $r$ directly, thus removing the need for a combiner algorithm. This is also shown in Figure~\ref{fig:FLmodels}.

\subsubsection{Fused-Lasso experiments: simulation}

The standard FL model has been well-studied in the literature, so a number of practical algorithms addressing it have already been proposed. The aforementioned Fused-Lasso algorithm in the {\bf SLEP} library can be regarded as the state of the art, making extensive use of an efficient proximity subroutine (FLSA). Our experiments on $\tvell_1^{\oned}$-proximity (\S\ref{sec:TVproxExp}) have already shown superiority of our prox solvers over FLSA; what remains to be checked is whether this benefit has a significant impact on the overall FL solver. To do so, we compare running times with synthetic data.

We generate random matrices $\ma \in \reals^{n \times m}$ with i.i.d.\ entries drawn from a zero mean, unit variance gaussian. We set the penalties to $\lambda_1 = \lambda_2 = 10$. We select the vector of responses $\vy$ using the formula $\vy=\sgn(\ma \vx_t + \vv)$, where $\vx_t$, and $\vv$ are random vectors whose entries have variances $1$ and $0.01$, respectively. The numerical results are summarized in Figure \ref{fig:FL-obj}, which compares out of the box SLEP (version 4.0) \citep{slep} against the very same algorithm employing our fast taut--string $\tvell_1^{\oned}$ solver instead of the default FLSA subroutine of SLEP. Comparison is done by showing the relative distance to the problem's optimum versus time. The optimal values in each setting were estimated by running both algorithms for a very large number of iterations.

\begin{figure}[htbp]
  \centering
  \begin{tabular}{llll}
    \hskip-25pt\includegraphics[width=.37\linewidth]{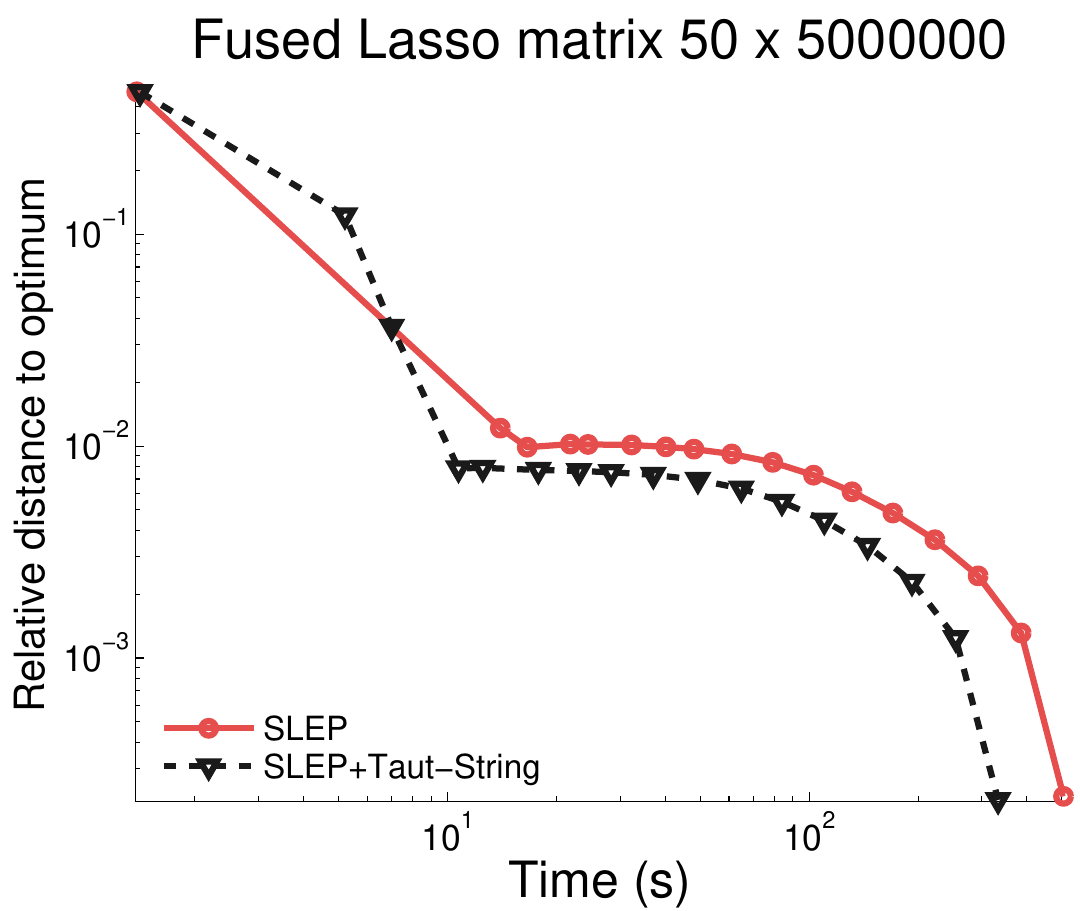}&\hskip-12pt
  \includegraphics[width=.37\linewidth]{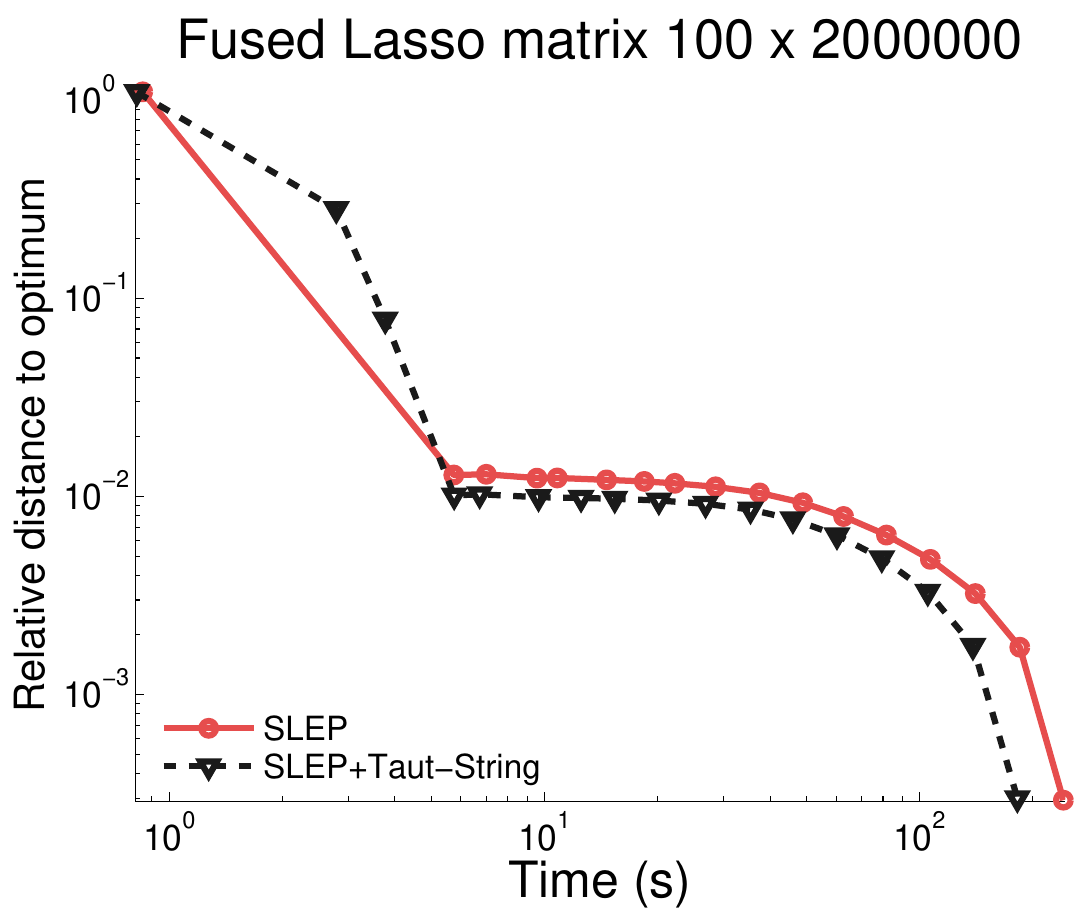}&\hskip-12pt
  \includegraphics[width=.37\linewidth]{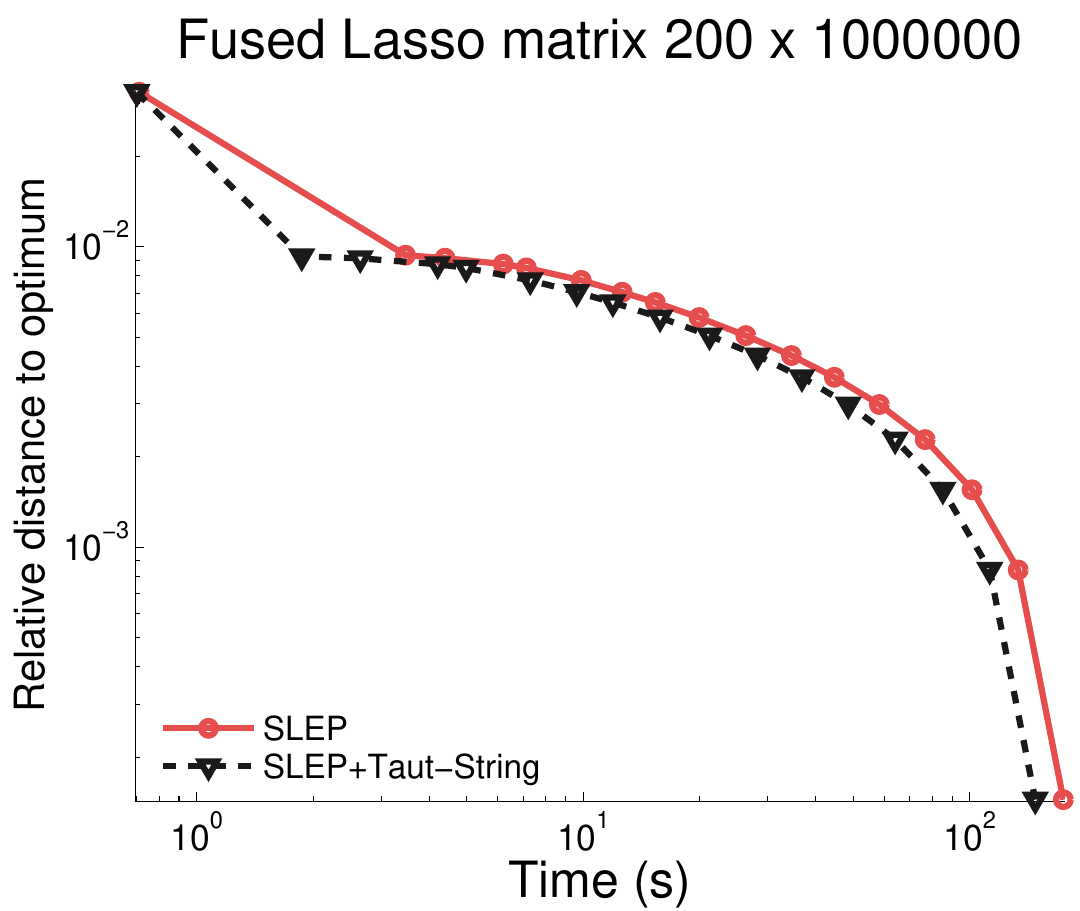}\\
  \hskip-25pt\includegraphics[width=.37\linewidth]{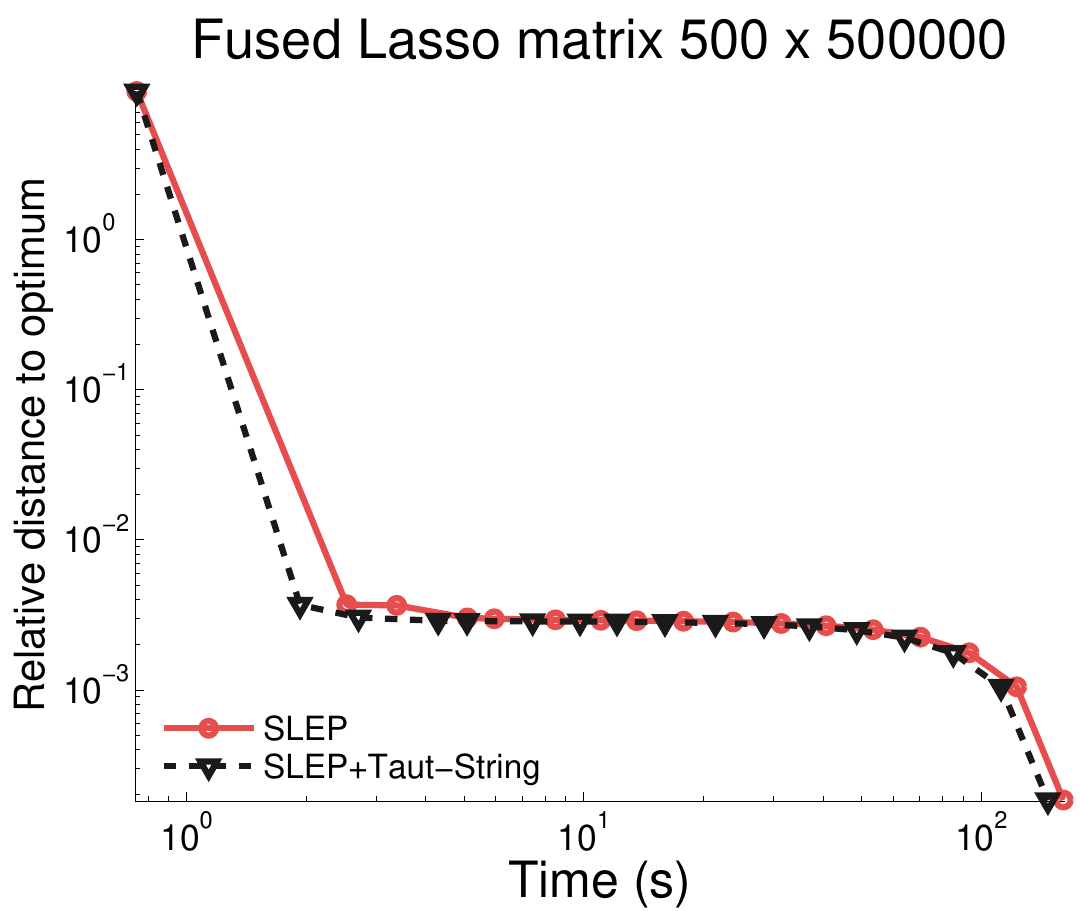}&\hskip-12pt
  \includegraphics[width=.37\linewidth]{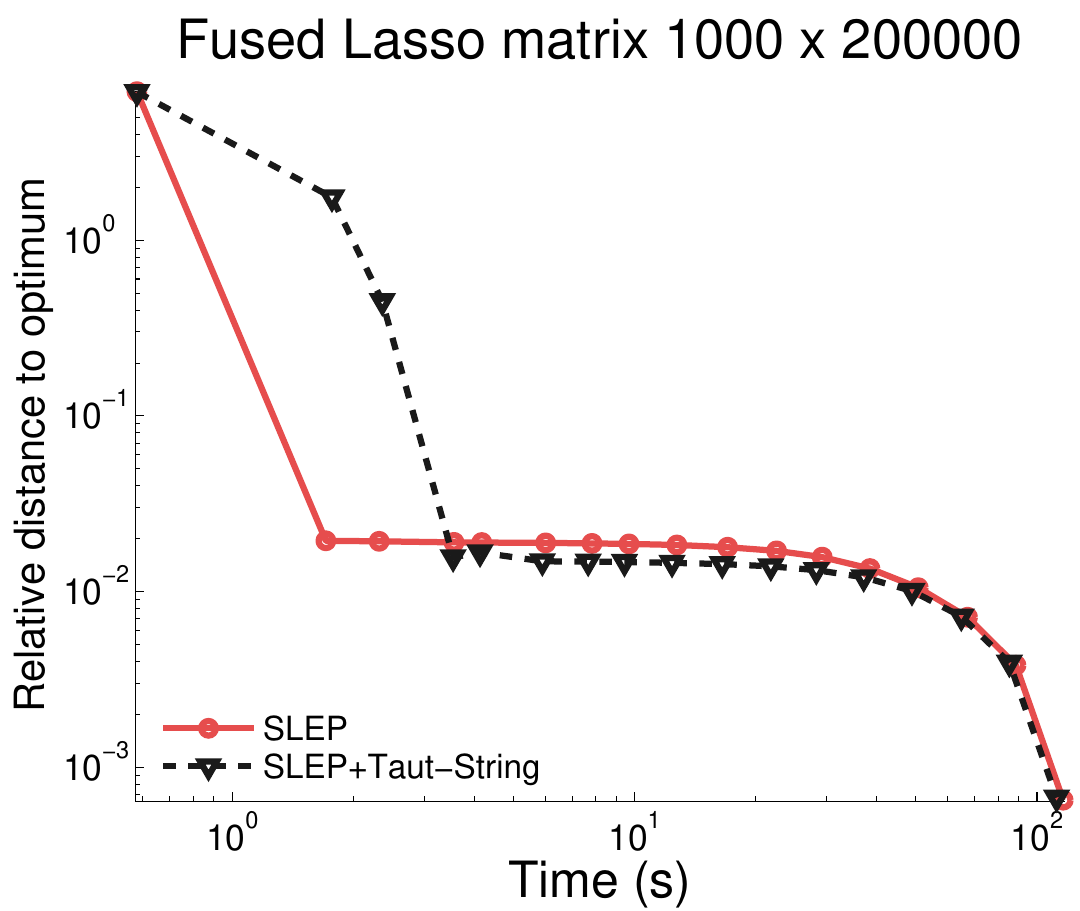}&\hskip-12pt
  \includegraphics[width=.37\linewidth]{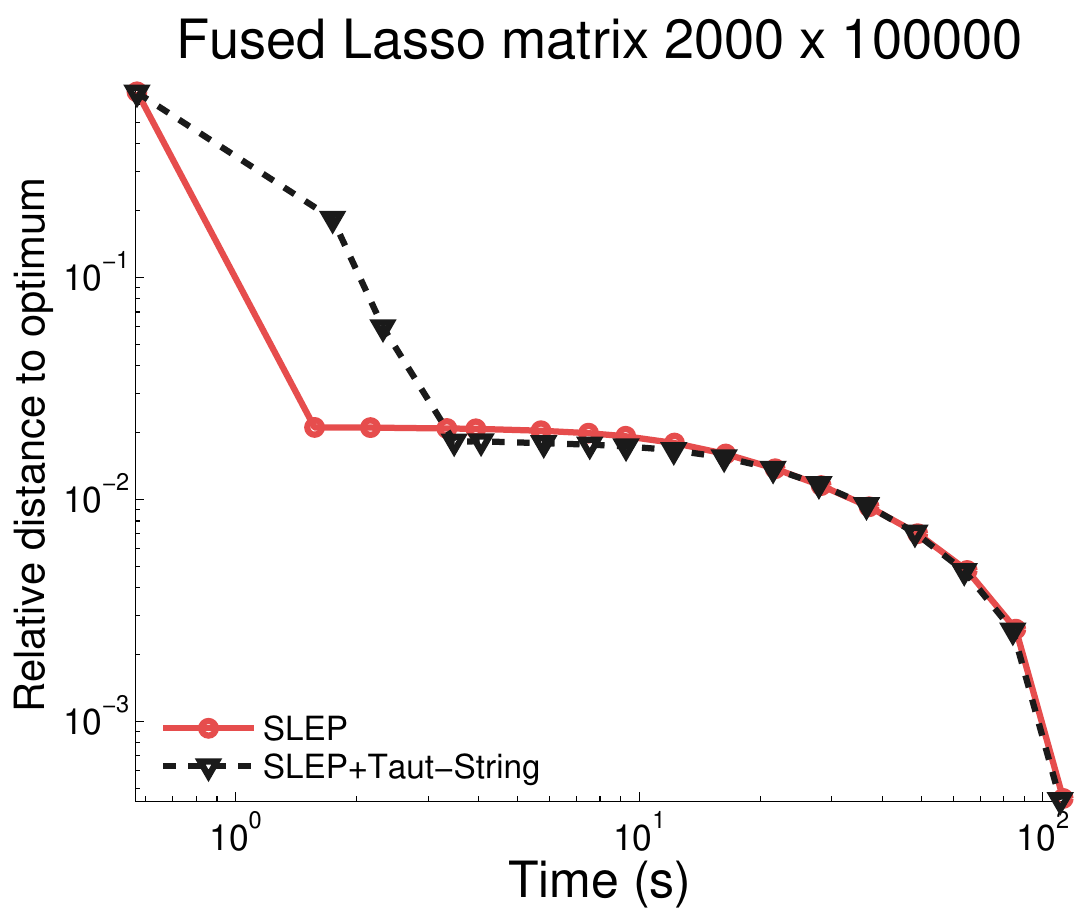}&\hskip-12pt
  \end{tabular}
  \caption{Relative distance to optimum vs time of the Fused Lasso optimizers under comparison, for the different layouts of synthetic matrices.}
  \label{fig:FL-obj}
\end{figure}

The plots show a clear trend: when the input matrices feature a very large column dimension the use of our taut-string $\tvell_1^{\oned}$ solver turns into speedups in optimization times, which however become negligible for matrices with a more balanced rows/columns ratio. This result is reasonable, as the vector $x$ under optimization has size equal to the number of columns of the data matrix $A$. If $A$ has a large number of columns the cost of solving $\tvell_1^{\oned}$ is significant, and thus any improvement in this step has a noticeable impact on the overall algorithm. Conversely, when the number of rows in $A$ is large the cost of computing the gradient of the loss function ($\nabla \half \enorm{\ma \vx - \vy}^2 = \ma^T \left( \ma \vx - \vy \right)$) dominates, getting limited benefits from such improvements in prox computations. Therefore, it is for data with a very large number of features where our proposed method can provide a useful speedup.

\subsubsection{Fused-Lasso Experiments: Microarray classification}
Now we report results of applying the four FL models on a series of problems from bioinformatics. We test the FL models on binary classification tasks for the following real microarray datasets: ArrayCGH \citep{Stransky06}, Leukemias \citep{Golub99}, Colon \citep{Alon99}, Ovarian \citep{RogersOvarian05} and Rat \citep{HuaFS09}. Each dataset was split into three equal parts (ensuring similar proportion of classes in every split) for training, validation and test. The penalty parameters were found by exhaustive grid search in the range $\lambda_1, \lambda_2 \in [10^{-4},10^2]$ to maximize classification accuracy on the validation splits.

\begin{table}[t]
\caption{Classification accuracies for the presented Fused--Lasso models on microarray data. For the Variable Fusion models an $\ell_2$ version of TV was employed.
\label{tab:FLreal}
}
\centering
\begin{tabular}{lcccc}
\hline
\abovespace\belowspace
 Dataset & FL & VF-$\ell_2$ & LFL & LVF-$\ell_2$ \\
\hline
\abovespace 
ArrayCGH &    73.6\%  &  73.6\% &  \bf{84.2\%} &  73.6\% \\
Leukemias &    \bf{92.0\%} &  88.0\% &  \bf{92.0\%} &  88.0\% \\
Colon &    \bf{77.2\%} &  \bf{77.2\%}  & \bf{77.2\%}  & \bf{77.2\%} \\
Ovarian &    \bf{88.8\%}  & 83.3\%  & 83.3\% &  83.3\% \\
\belowspace 
Rat &    68.8\%  & 65.5\%  & \bf{72.1\%}  & \bf{72.1\%}
\\ \hline
\end{tabular}
\end{table}

Table \ref{tab:FLreal} shows test accuracies. In general, as expected the logistic-loss based FL models yield better classification accuracies than those based on least-squares, as such loss function tends to be more appropriate for classification problems. However the Ovarian dataset proves to be an exception, showing better performance under a squared loss. Regarding the TV-regularizer, the classic $\tvell_1^{\oned}$-penalty seems to perform better in general, with the $\tvell_2^{\oned}$-penalty showing competitive results in some settings.

\subsection{2D-TV: Experiments and Applications}
\label{sec:2d.expt}
We address now several practical applications that benefit from two-dimensional TV regularization; our results show again how the presented $\tvell_{p,q}^{\twod}$ prox operators fits in seamlessly into our modular framework to produce efficient proximal splitting solvers.

\subsubsection{Image denoising through anisotropic filtering}
\label{sec:aniso2D}
Our first example is related to the classic problem of image denoising, but with the twist that we deal with noise of an anisotropic character. More specifically, suppose that the true image $\mu \in \reals^{n \times m}$ is contaminated by additive noise $\mn$, so that only $\mu_0 = \mu + \mn$ is observed. The denoising problem estimates $\mu$ given just the noisy version $\mu_0$. This problem is highly ill-posed and as such not approachable unless additional assumptions on the noise (or on the underlying image) are made. 

\textbf{Isotropic and anisotropic models}: an extremely common choice is to simply assume the noise to be gaussian, or some other zero-mean distribution. Under these conditions, a classic method to perform such denoising task is the {\bf Rudin-Osher-Fatemi} (ROF) model~\citep{RudinTV92}, which finds an approximation $\mx$ to the original image by solving
\begin{equation}
  \label{eq.rof}
 \min_{\mx} \quad \frob{\mx - \mu_0}^2 + \lambda \sum_{i=2}^{n} \sum_{j=2}^{m} \enorm{\partial x_{i,j}} ,
\end{equation}
where $\partial x_{i,j}$ is the \emph{discrete gradient} 
\begin{equation*}
 \partial x_{i,j} = \left[ \begin{array}{c} x_{i,j} - x_{i-1,j} \\ x_{i,j} - x_{i,j-1} \end{array} \right].
\end{equation*}
That is, it is the vector of differences of $\mx_{i,j}$ and its neighbors along both axes. 

The objective of the first term in the ROF model is to penalize any deviation of $\mx$ from the observed image $\mu_0$, while the second term can be readily recognized as a mixed $(2,1)$-norm over the discrete gradient of $\mx$. This regularizer models caters to some prior knowledge: in natural images sharp discontinuities in intensity between neighboring points only appear in borders of objects, while the rest of the pixels usually show smooth variations in intensity. It makes sense, therefore, to penalize large values of the gradient, as sharp changes have a higher probability of having being produced by noise. Conversely, as the mean of the noise is zero, it is also sensible to maintain the denoised image $\mx$ close to the observed $\mu_0$. Merging these two goals produces the ROF model~\eqref{eq.rof}.

A closer look at the ROF regularizer reveals that it follows the spirit of the reviewed 2D-TV regularizer which also penalizes sharp variations between neighboring pixels. Indeed, all such regularizers are broadly categorized as TV regularizers within the image processing community. It is clear, though, that the ROF regularizer~\eqref{eq.rof} does not coincide with the $\tvell_{p,q}^{\twod}$ regularizer used in this paper. Some authors~\citep{TwIST} differentiate between these regularizers by naming the ROF approach as {\bf isotropic TV} and the $\tvell_{p,q}^{\twod}$-style approach as {\bf anisotropic TV}. This naming comes from the fact that isotropic TV penalizes each component of the discrete gradient $\partial x_{i,j}$ following an $\ell_2$ norm, whereas the anisotropic $\tvell_{p,q}^{\twod}$-norm  and in particular $\tvell_{1,1}^{\twod}$-norm, penalize rows and columns independently.

While image filtering using isotropic TV is generally preferred for natural images denoising \citep{BioucasTV06}, in some settings anisotropic filtering can produce better results, and in fact has been favored by some authors in the past \citep{Choksi2DBarcode,LiL1TV}. This is specially true on those images that present a ``blocky'' structure, and thus are better suited to the structure modeled by the $\tvell_{p,q}^{\twod}$-norm. Therefore, efficient methods to perform anisotropic filtering are also important.

\textbf{Anisotropic denoising experiments}: denoising using the anisotropic $\tvell_{p,q}^{\twod}$-norm reduces to solving
\begin{equation}
  \label{eq:26}
  \min_{\mx} \quad \frob{\mx - \mu_0}^2 + \lambda \tvell_{p,q}^{\twod}(\mx).
\end{equation}
But~\eqref{eq:26} is nothing but the $\tvell_{p,q}^{\twod}$-proximity problem, and hence can be directly solved by applying the 2D-TV prox operators described above. We solve~\eqref{eq:26} below for the choice $p=q=1$ (which is common in practice), for the following selection of algorithms:

\begin{itemize}
  \setlength{\itemsep}{0pt}
 \item Proximal Dykstra (§~\ref{sec:PD})
 \item The Douglas-Rachford variant based on alternating projections (§~\ref{sec:DR})
 \item The Split Bregman method of~\citet{splitbreg}, which follows an ADMM--like approach to split the $\ell_1$ norm apart from the discrete gradient operator, thus not requiring the use of a 1D-TV prox operator.
 \item Chambolle-Pock's method applied to 2D TV (§~\ref{sec:dualproxmethods}).
 \item Condat's general splitting method (§~\ref{sec:dualproxmethods}).
 \item Kolmogorov et al primal-dual method (§~\ref{sec:dualproxmethods}).
 \item Yang's method (ADMM) (§~\ref{sec:ADMM})
 \item The maximum flow approach by~\citet{goldfarb2009parametric}, which shows the relationship between the 2D-TV proximity minimization and the maximum flow problem over a grid, and thus applies an efficient maximum flow method to solve a discrete-valued version of 2D-TV.
\end{itemize}
In Proximal Dykstra, Douglas-Rachford and ADMM we use the linearized taut--string strategy presented before as solver for the base proximity operators. All algorithm parameters were set as recommended in their corresponding papers or public implementations, except for Proximal Dykstra and Douglas-Rachford, which are parameter free. For Chambolle-Pock we tried both the scheme with fixed algorithm parameters~\citep[Algorithm 1]{chambollePock} and the scheme with acceleration~\citep[Algorithm 2]{chambollePock}; however the accelerated version did not converge to the desired solution within enough accuracy (relative difference of $10^{-5}$), therefore only the results for the fixed version are reported. For Kolmogorov we follow the recommendations in~\citet{ChambolleErgodic14}, taking into account the Lipschitz constants of the optimized functions and selecting the parameter updating strategy that produced faster performance in the experiments: $\theta^{t+1} = \frac{1}{\sqrt{1 + \tau^t}}, \tau^{t+1} = \theta^{t+1} \tau^t, \sigma^{t+1} = \frac{\sigma^t}{\theta^{t+1}}, \theta^0 = 1, \tau^0 = \half, \sigma^0 = 1$.

\begin{table}[htbp]
\caption{\small Types of noise and parameters for each test image. A $\varnothing$ indicates that such noise was not applied for the image. {\it Gaussian} and {\it Speckle} correspond to gaussian additive and multiplicative (respectively) noises with zero mean and the indicated variance. {\it Salt \& Pepper} noise turns into black or white the indicated fraction of image pixels. {\it Poisson} regenerates each pixel by drawing a random value from a Poisson distribution with mean equal to the original pixel value, thus producing a more realistic noise.
\label{tab:noises}
}
\centering
\begin{tabular}{lcccc}
\hline
\abovespace\belowspace
      Image & Gaussian & Speckle & Poisson & Salt \& Pepper \\
\hline
\abovespace 
randomQR & 0.2 & 0.3 & $\varnothing$ & $\varnothing$ \\
shape & 0.05 & $\varnothing$ & $\varnothing$ & $\varnothing$ \\
trollface & $\varnothing$ & 1 & $\varnothing$ & $\varnothing$ \\
diagram & $\varnothing$ & $\varnothing$ & $\checkmark$ & $\varnothing$  \\
text & $\varnothing$ & $\varnothing$ & $\varnothing$ & 0.1 \\
comic & 0.05 & $\varnothing$ & $\checkmark$ & $\varnothing$ \\
contour & $\varnothing$ & $\varnothing$ & $\checkmark$ & 0.4 \\
\belowspace
phantom & $\varnothing$ & 2 & $\checkmark$ & $\varnothing$
\\ \hline
\end{tabular}
\end{table}

The images used in the experiments are displayed in Appendix \ref{app:images} as Figure \ref{fig:imagesExp}. To test the filters under a variety of scenarios, different kinds of noise were introduced for each image. Table \ref{tab:noises} gives details on this, while the noisy images are shown in Figure \ref{fig:imagesNoised}. All QR barcode images used the same kind and parameters of noise. Noise was introduced using Matlab's {\it imnoise} function.

Values for the regularization parameter $\lambda$ were found by maximizing the quality of the reconstruction, measured using {\bf Improved Signal-to-Noise Ratio} (ISNR) \citep{SALSA}. ISNR is defined as
\begin{equation*}
 \text{ISNR}(\mx,\mu,\mu_0) = 10 \log_{10} \frac{\frob{\mu_0 - \mx}^2}{\frob{\mx - \mu}^2} ,
\end{equation*}
where $\mu$ is the original image, $\mu_0$ its noisy variant, and $\mx$ the reconstruction. 

\begin{figure}[htbp]
  \centering
  \begin{tabular}{llll}
    \hskip-25pt\includegraphics[width=.37\linewidth]{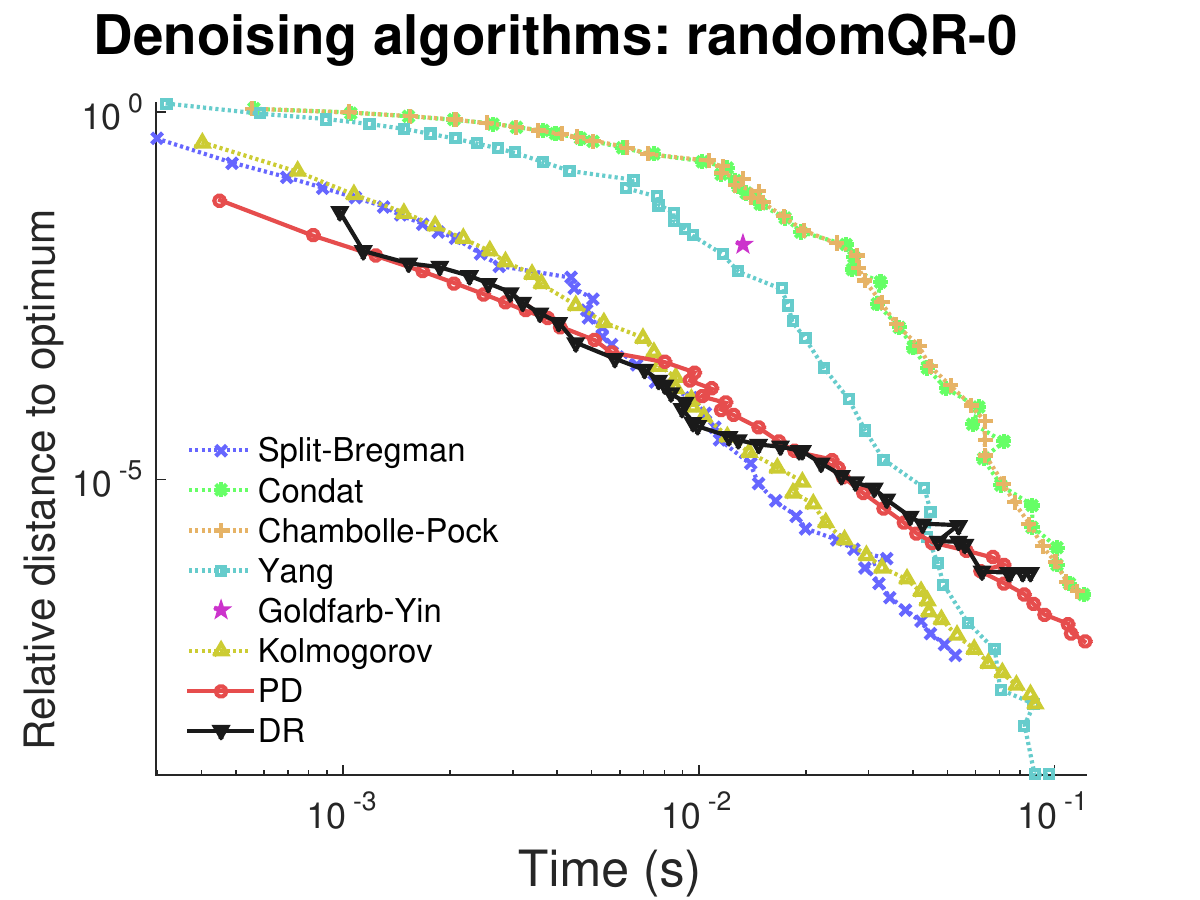}&\hskip-12pt
  \includegraphics[width=.37\linewidth]{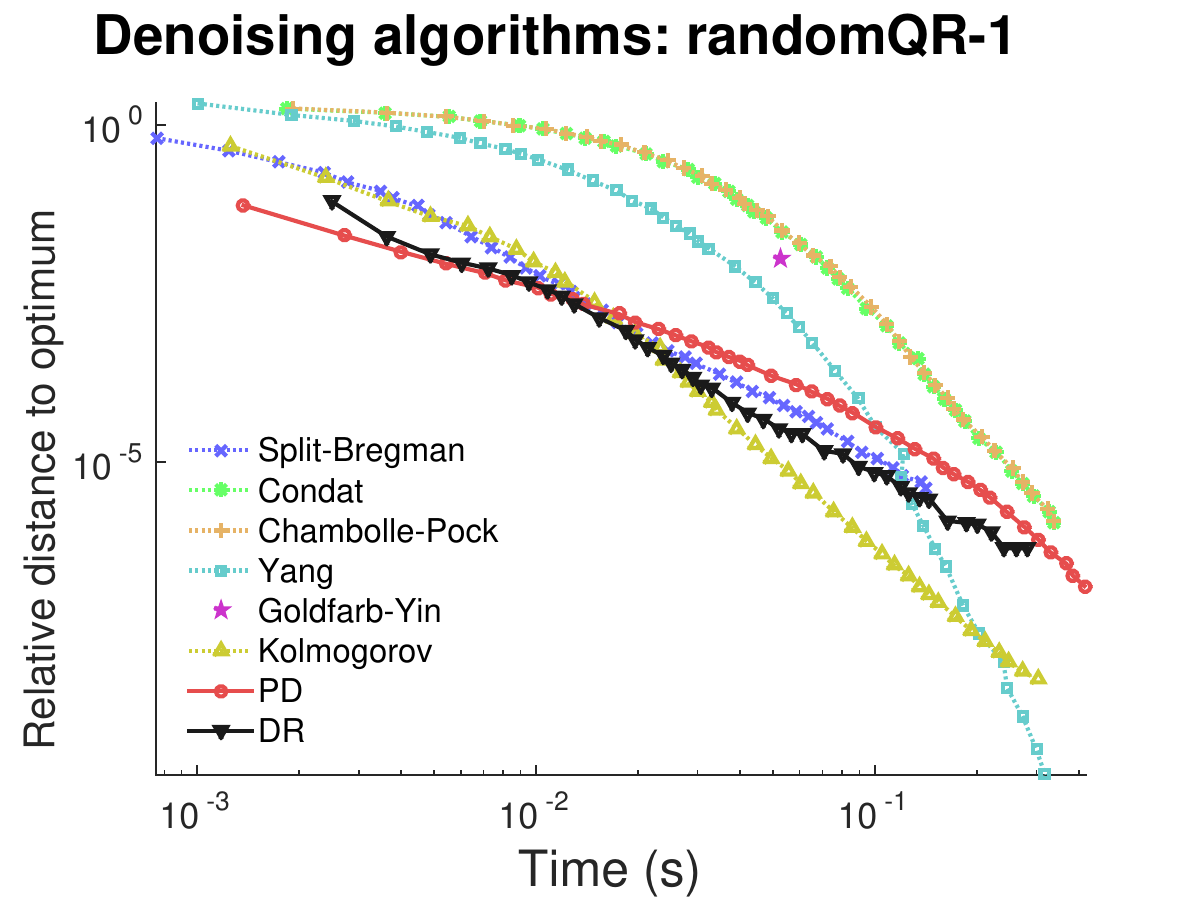}&\hskip-12pt
  \includegraphics[width=.37\linewidth]{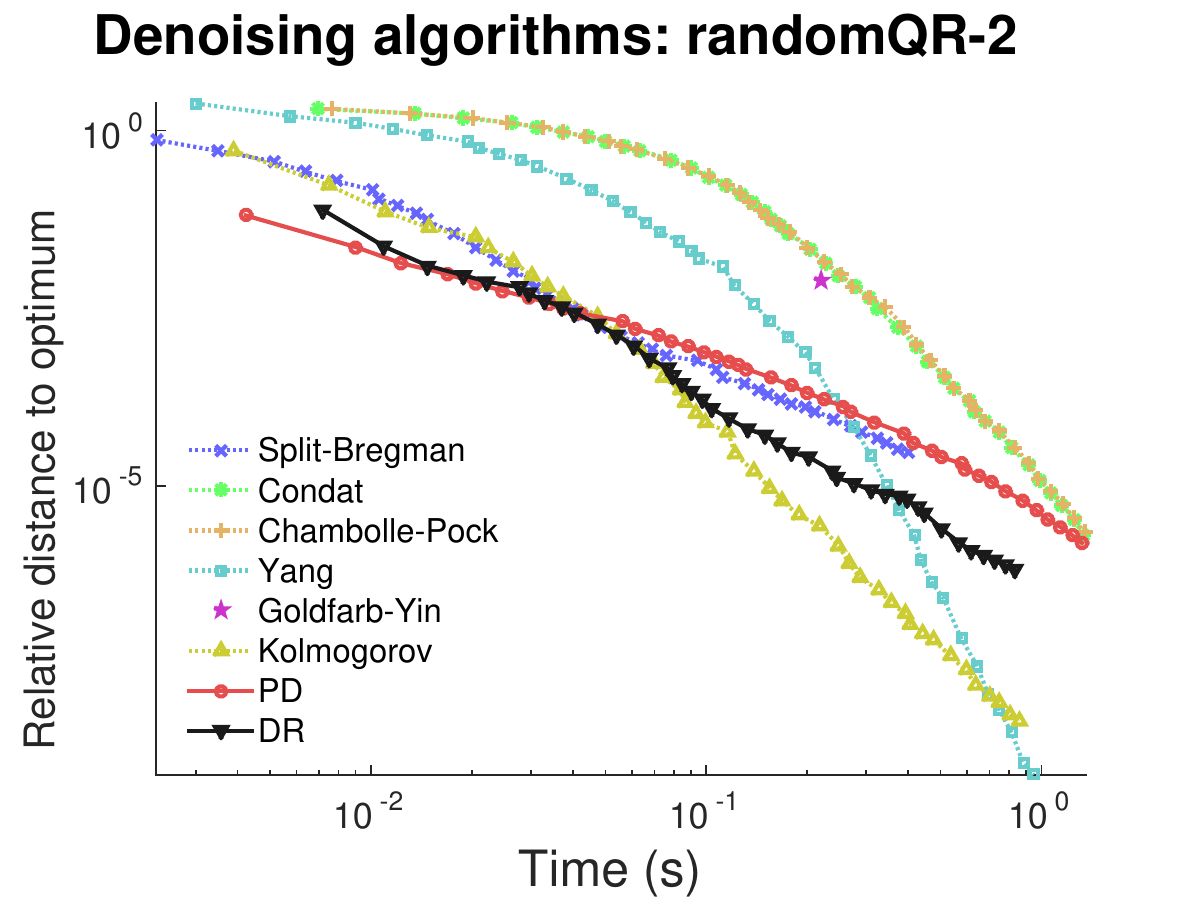}\\
  \hskip-25pt\includegraphics[width=.37\linewidth]{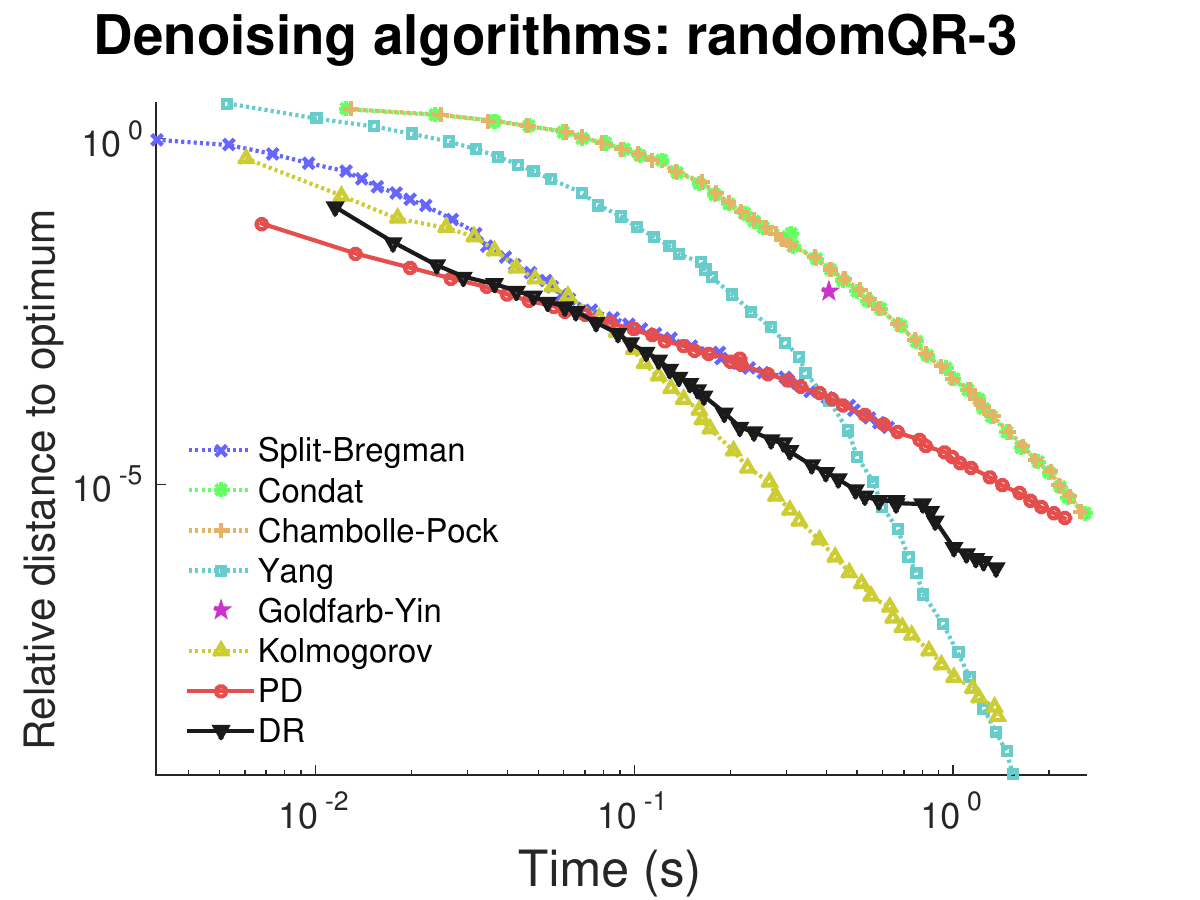}&\hskip-12pt
  \includegraphics[width=.37\linewidth]{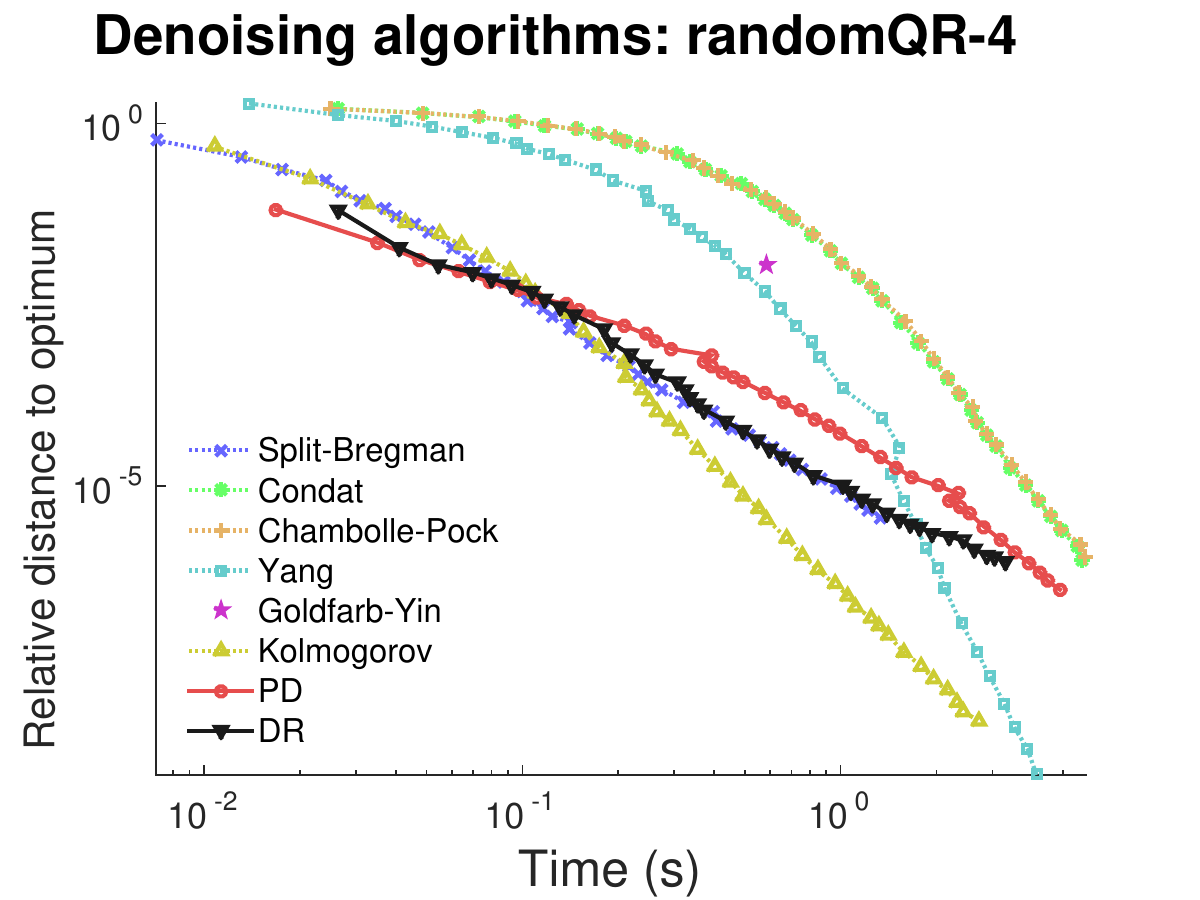}&\hskip-12pt
  \includegraphics[width=.37\linewidth]{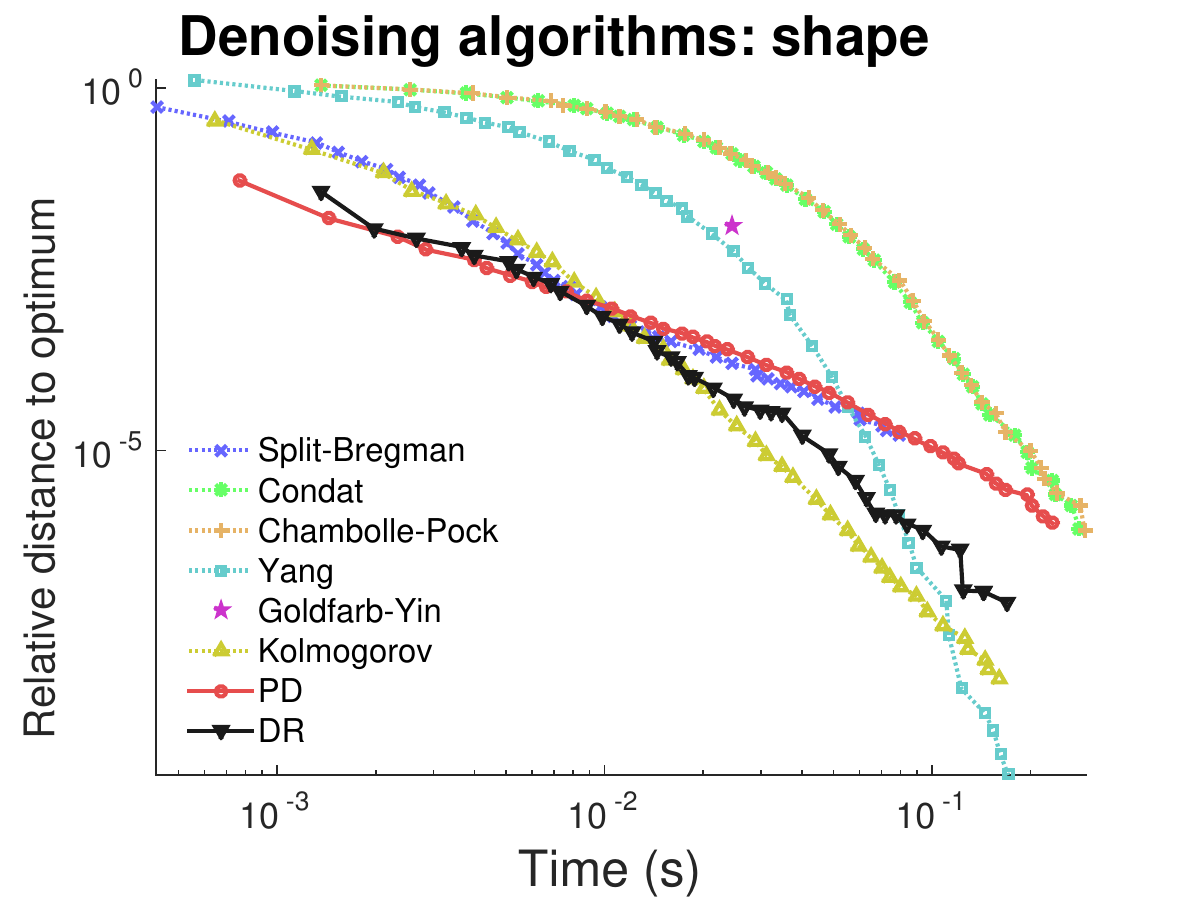}\\
  \hskip-25pt\includegraphics[width=.37\linewidth]{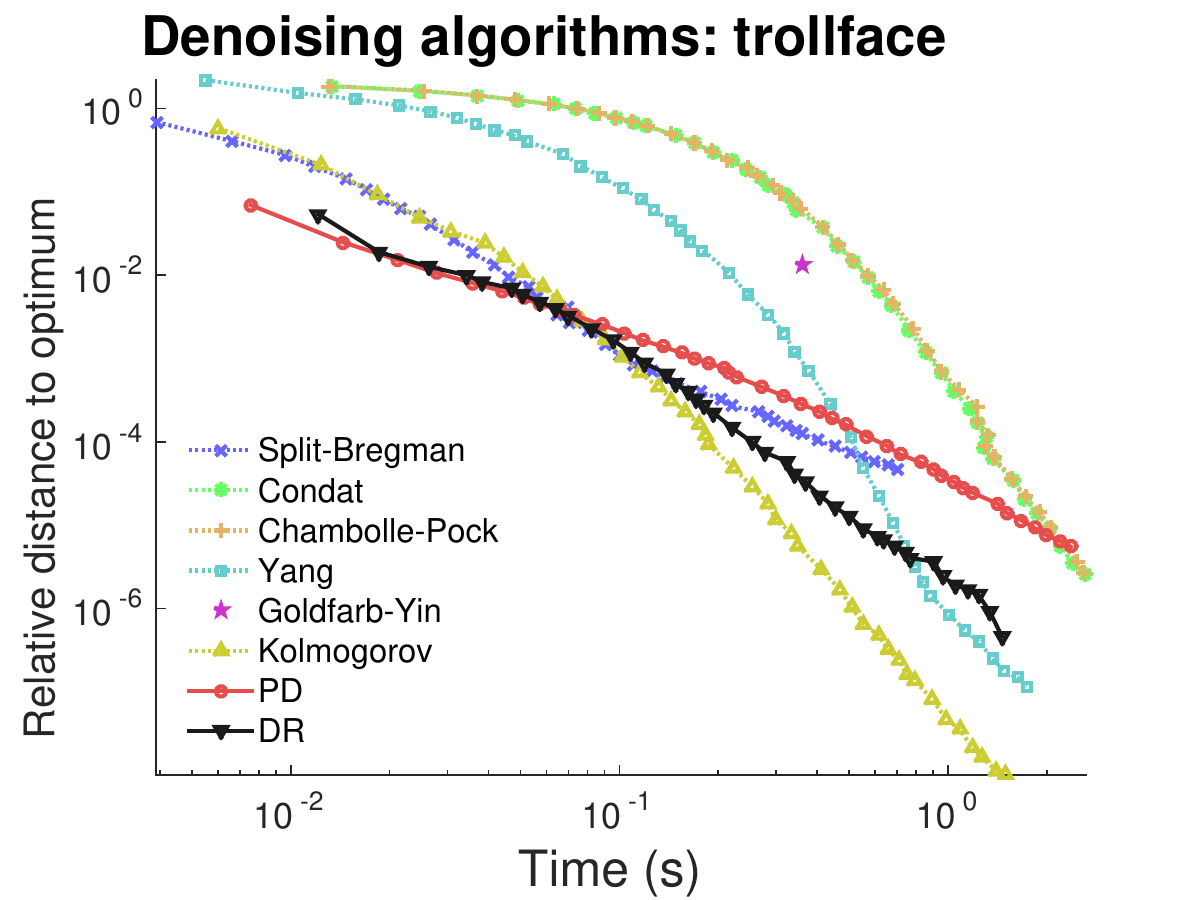}&\hskip-12pt
  \includegraphics[width=.37\linewidth]{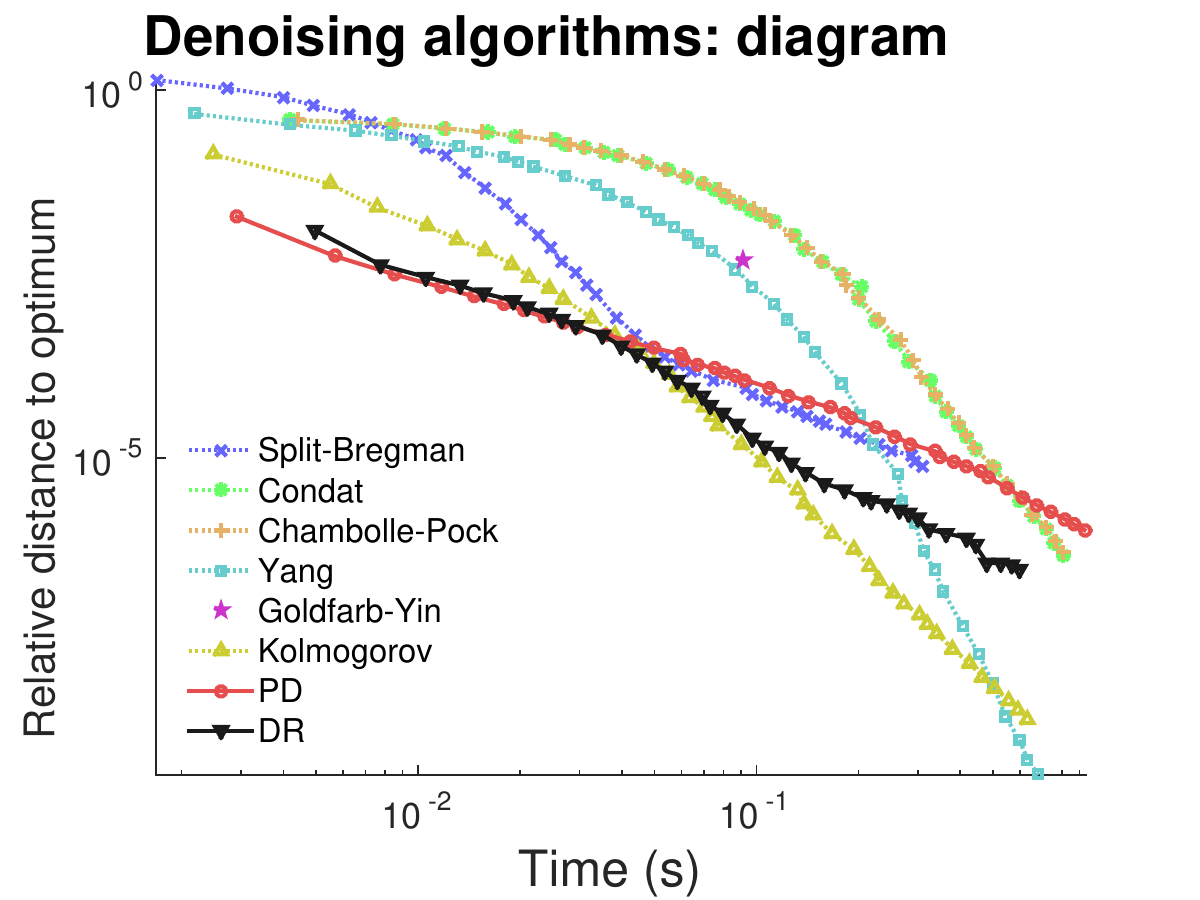}&\hskip-12pt
  \includegraphics[width=.37\linewidth]{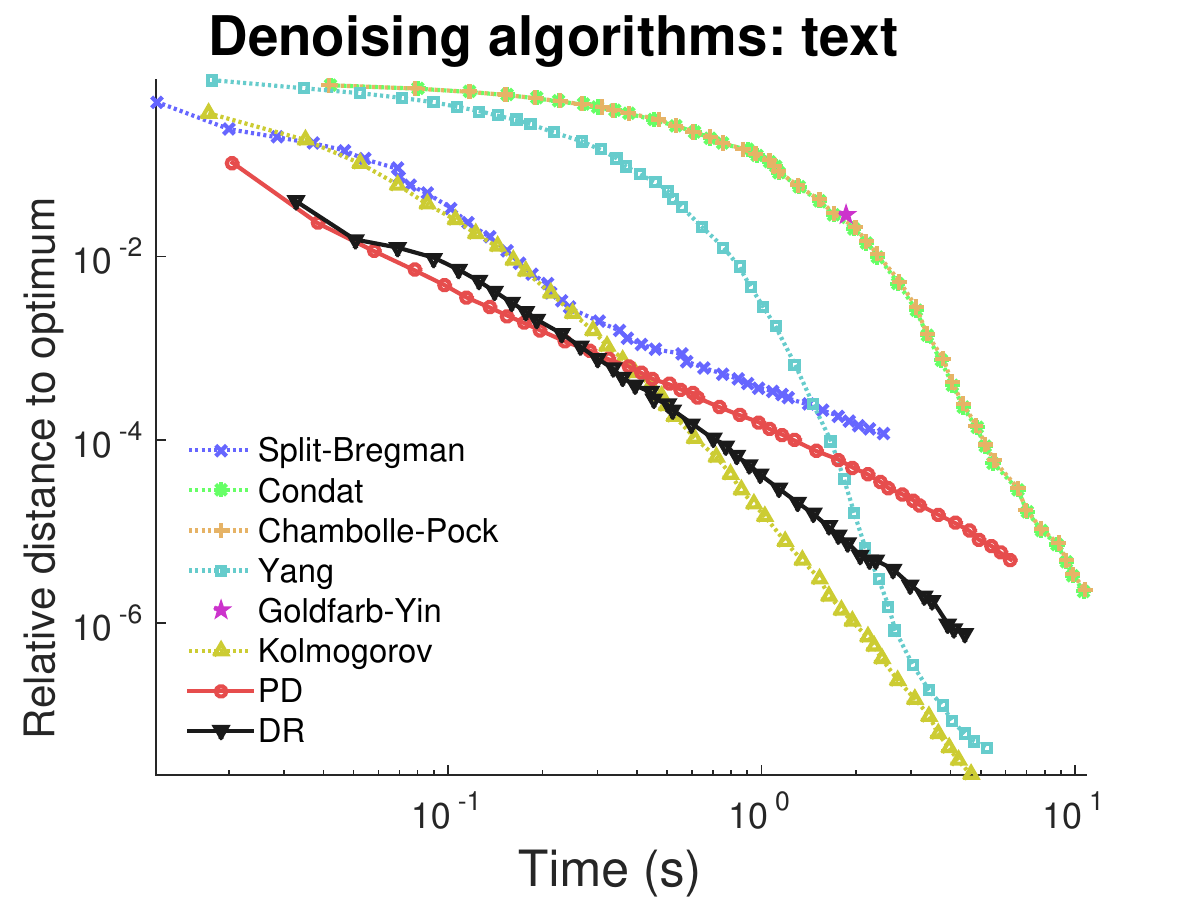}\\
  \hskip-25pt\includegraphics[width=.37\linewidth]{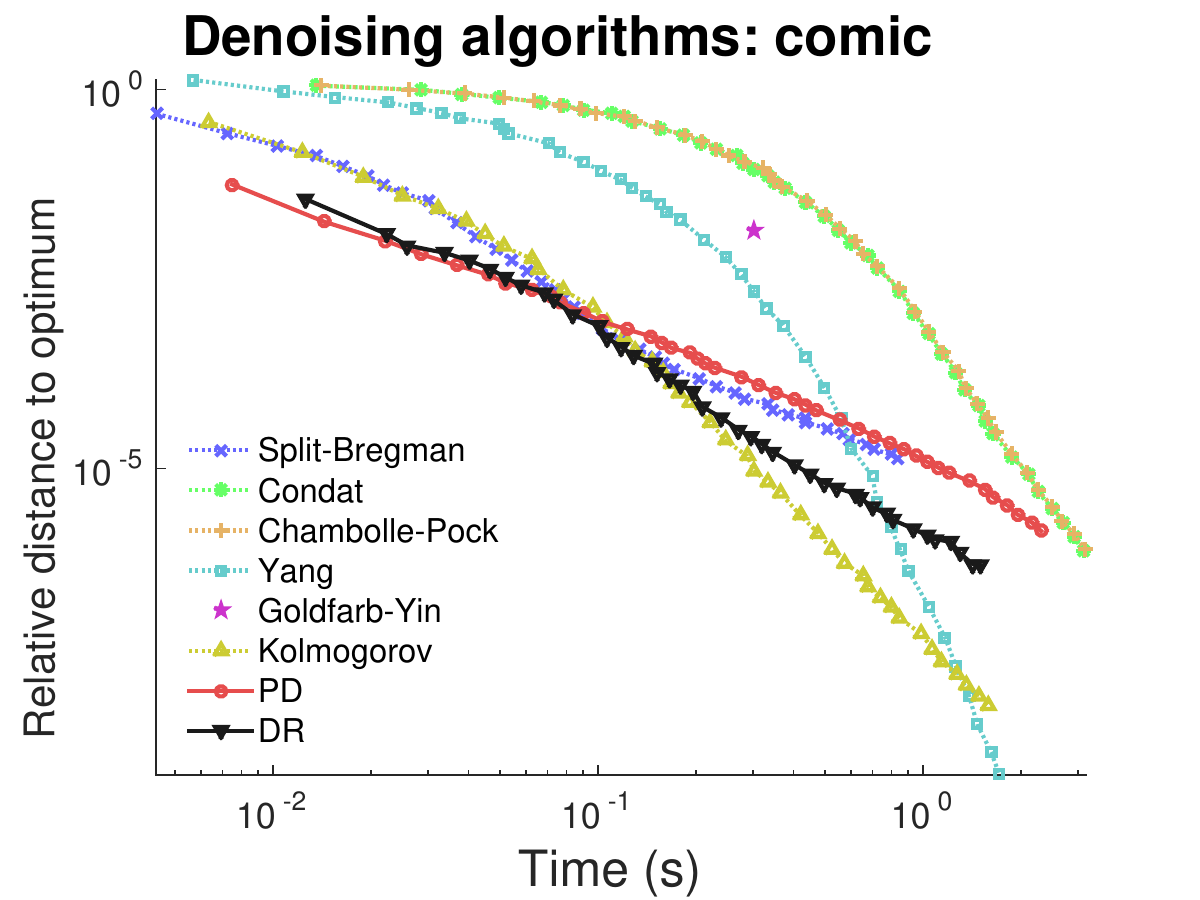}&\hskip-12pt
  \includegraphics[width=.37\linewidth]{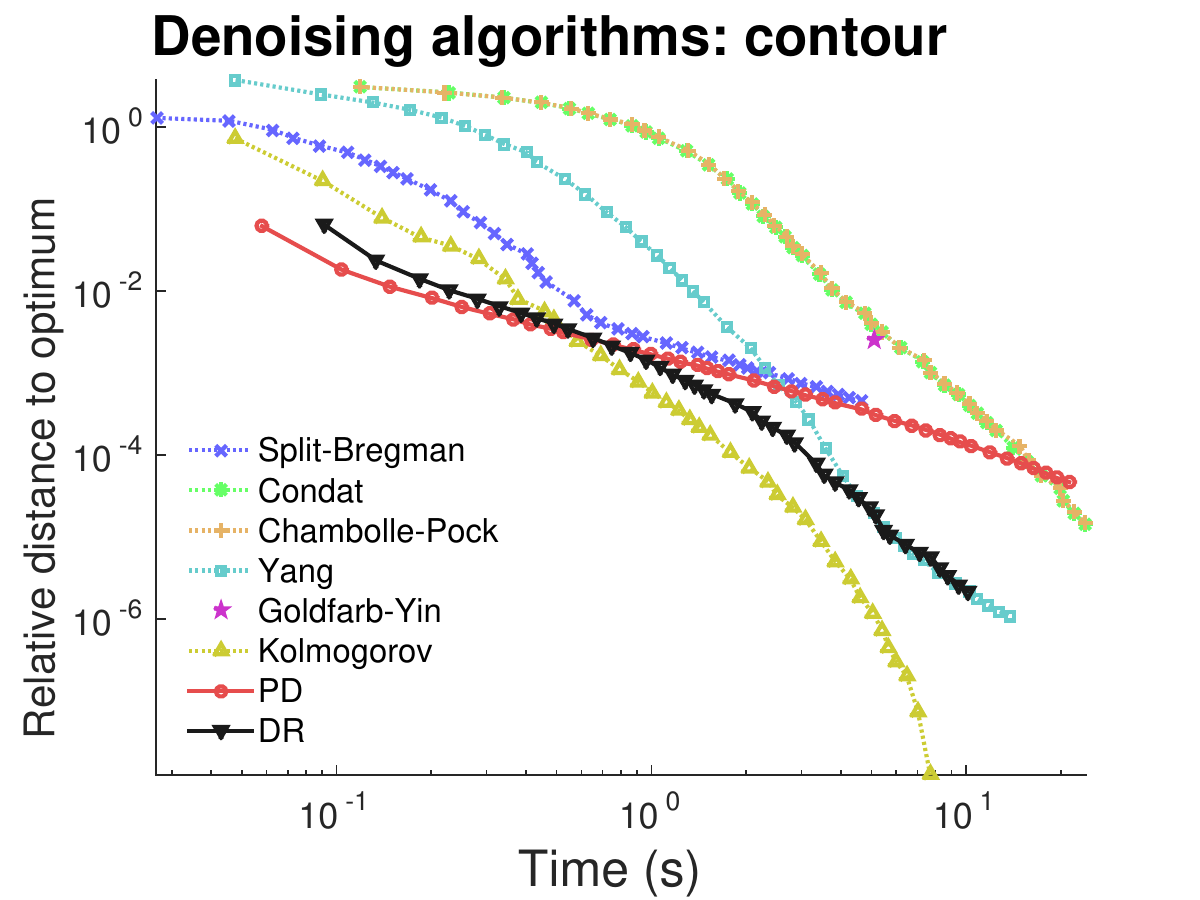}&\hskip-12pt
  \includegraphics[width=.37\linewidth]{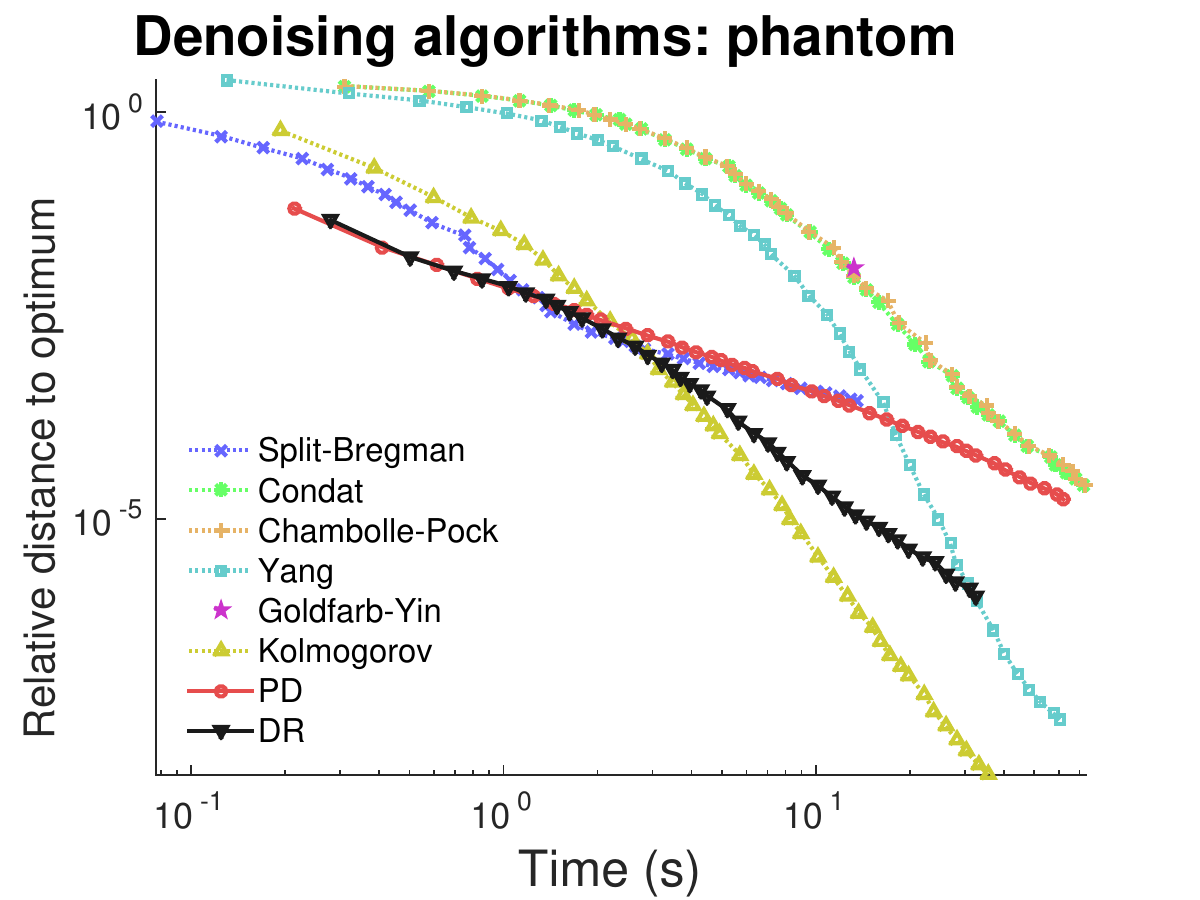}\\
  \end{tabular}
  \caption{Relative distance to optimum vs time of the denoising 2D-TV algorithms under comparison, for the different images considered in the experiments.}
  \label{fig:2DTV-obj}
\end{figure}

\begin{figure}[htbp]
  \centering
  \begin{tabular}{llll}
    \hskip-25pt\includegraphics[width=.37\linewidth]{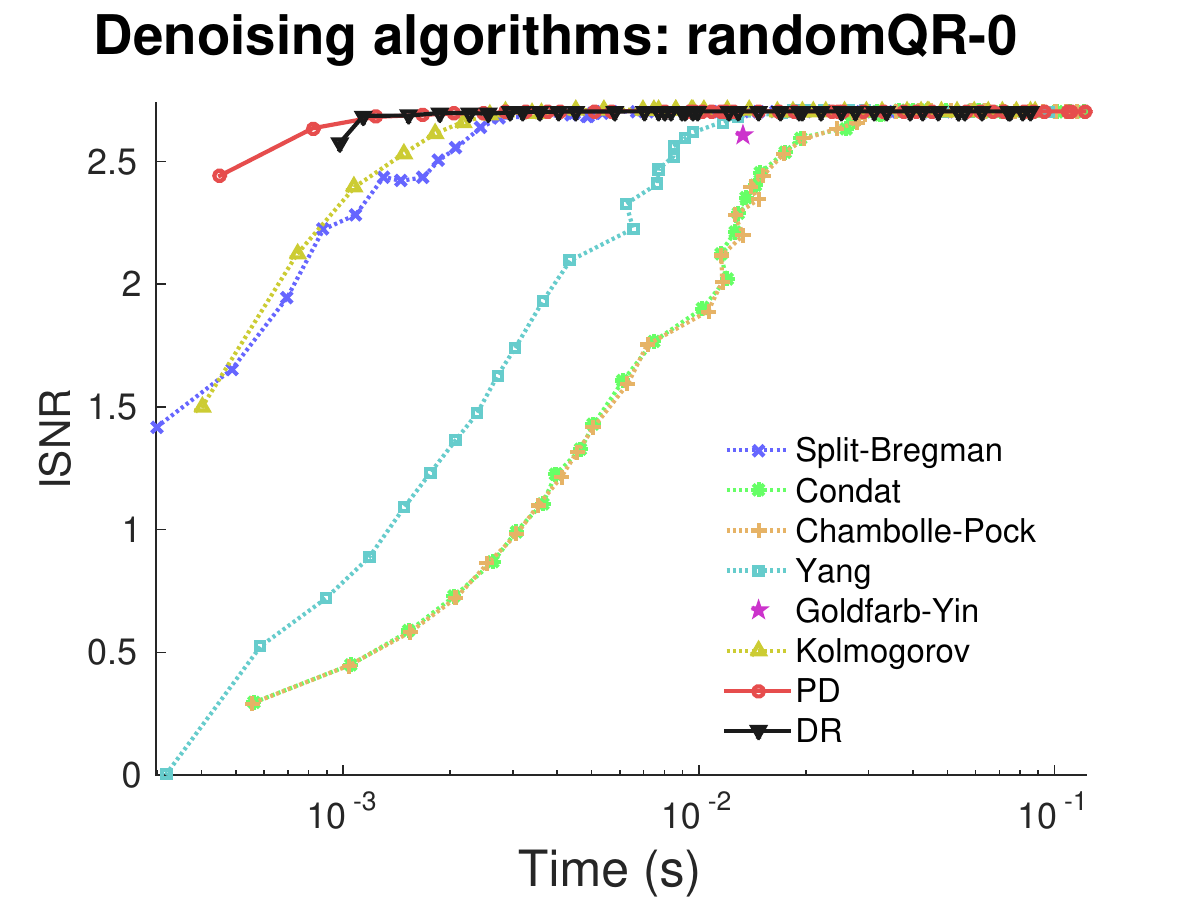}&\hskip-12pt
  \includegraphics[width=.37\linewidth]{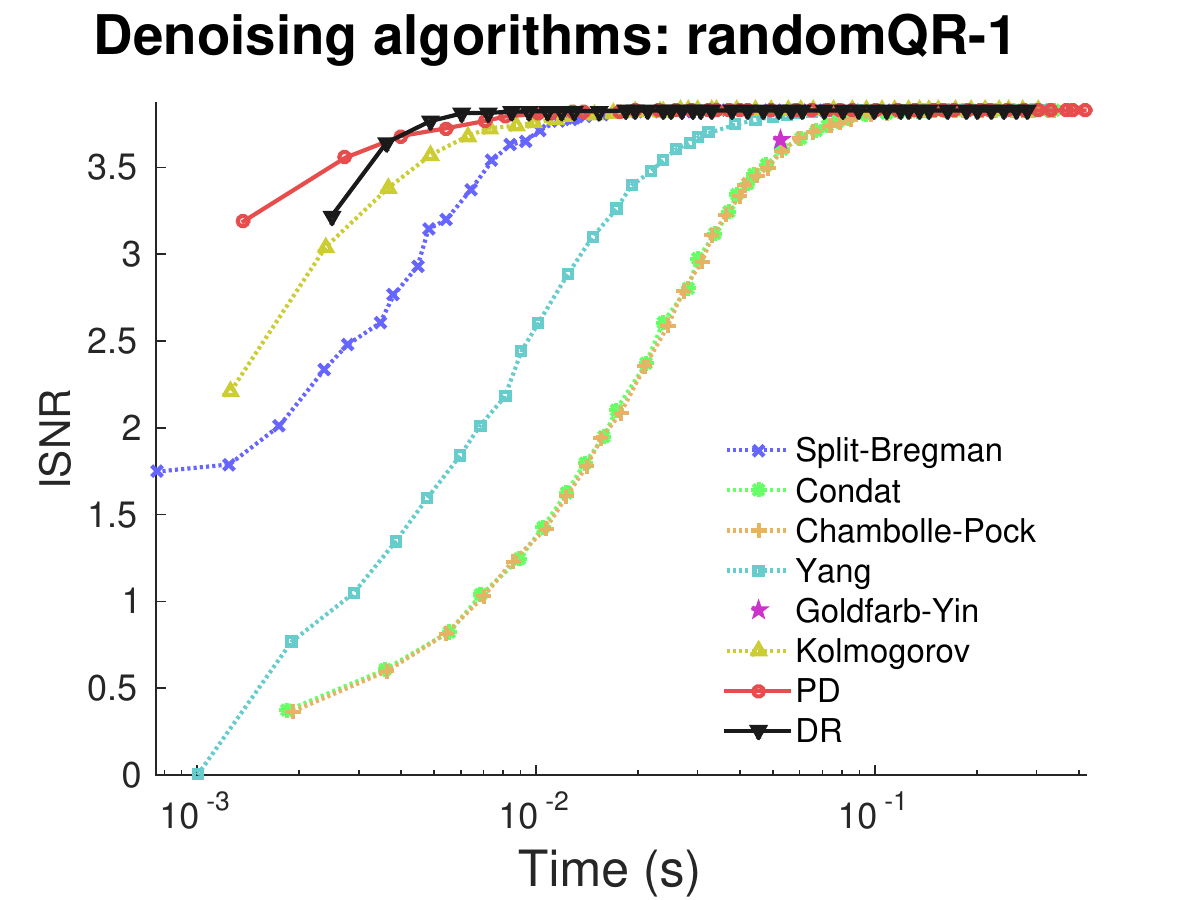}&\hskip-12pt
  \includegraphics[width=.37\linewidth]{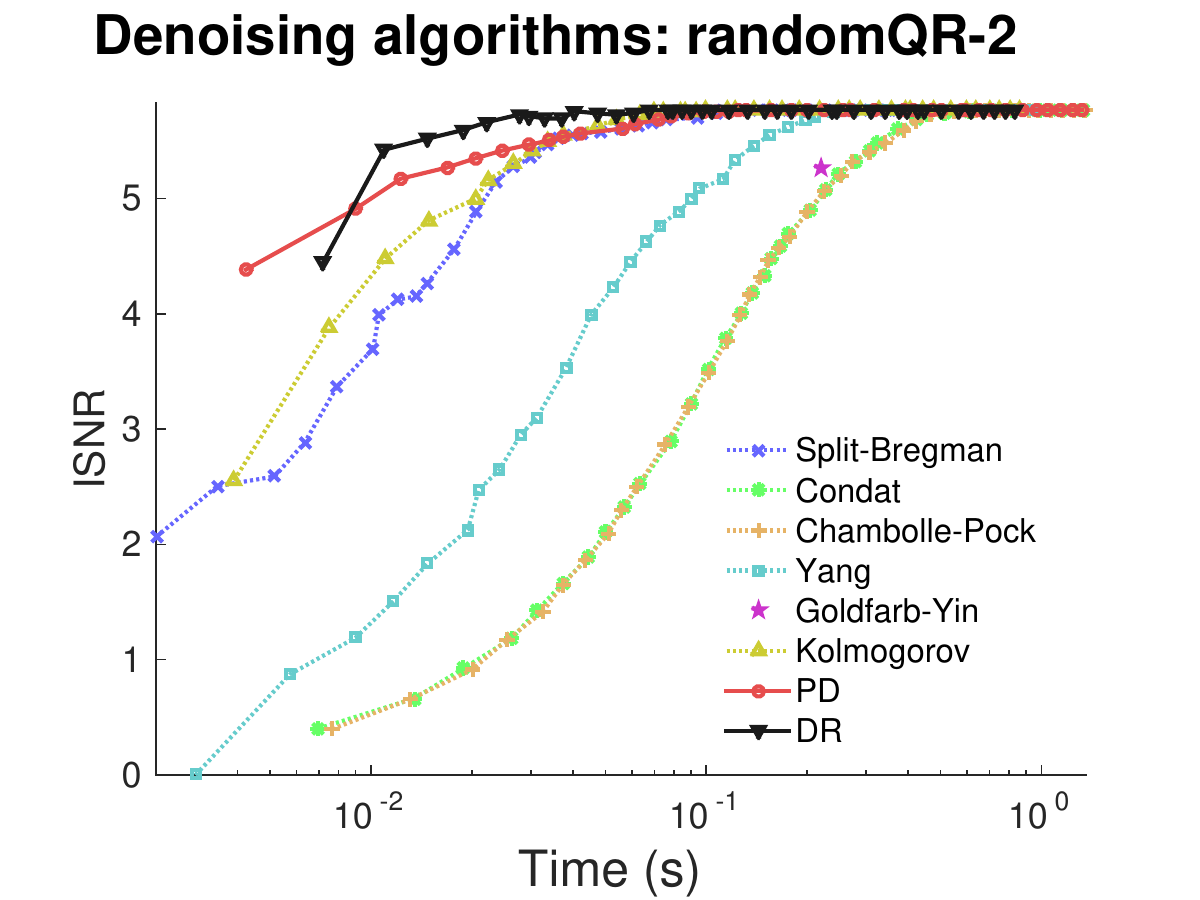}\\
  \hskip-25pt\includegraphics[width=.37\linewidth]{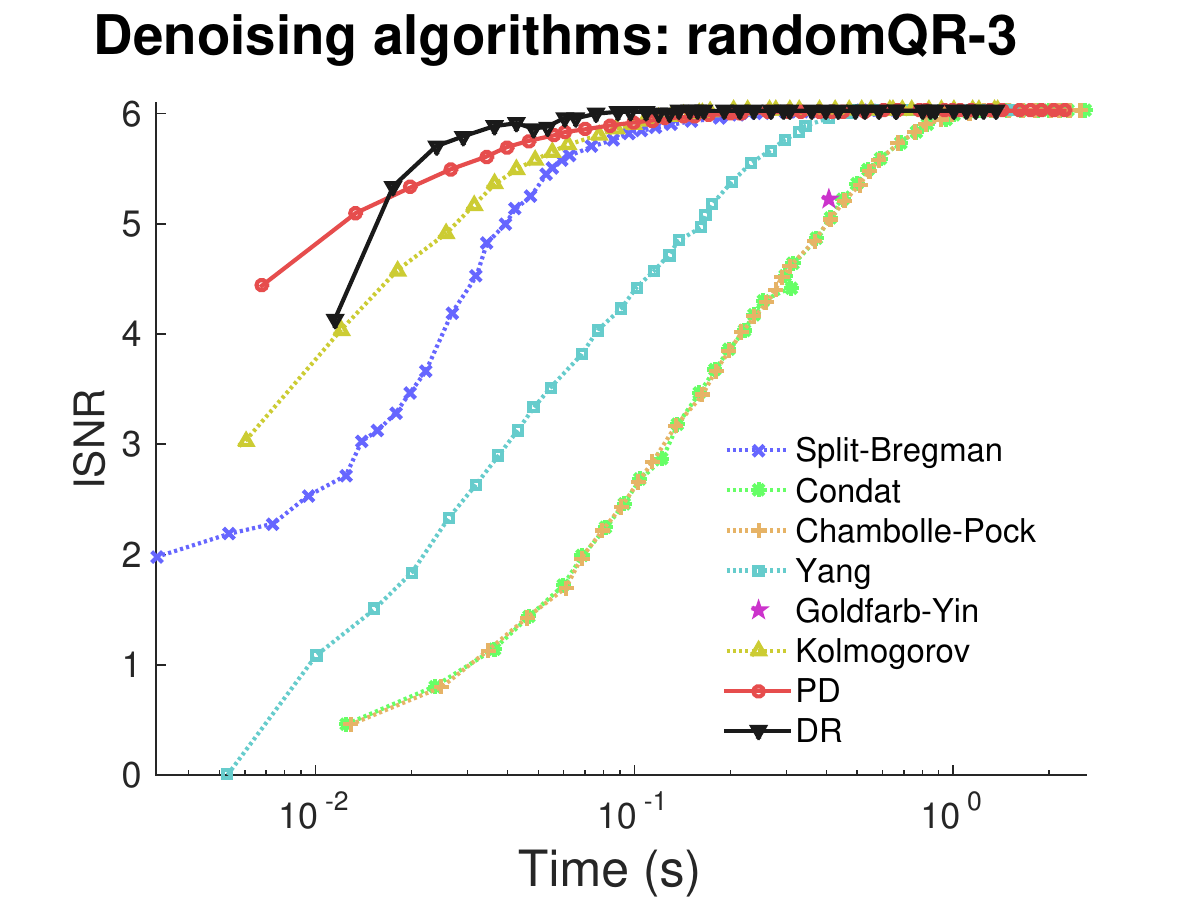}&\hskip-12pt
  \includegraphics[width=.37\linewidth]{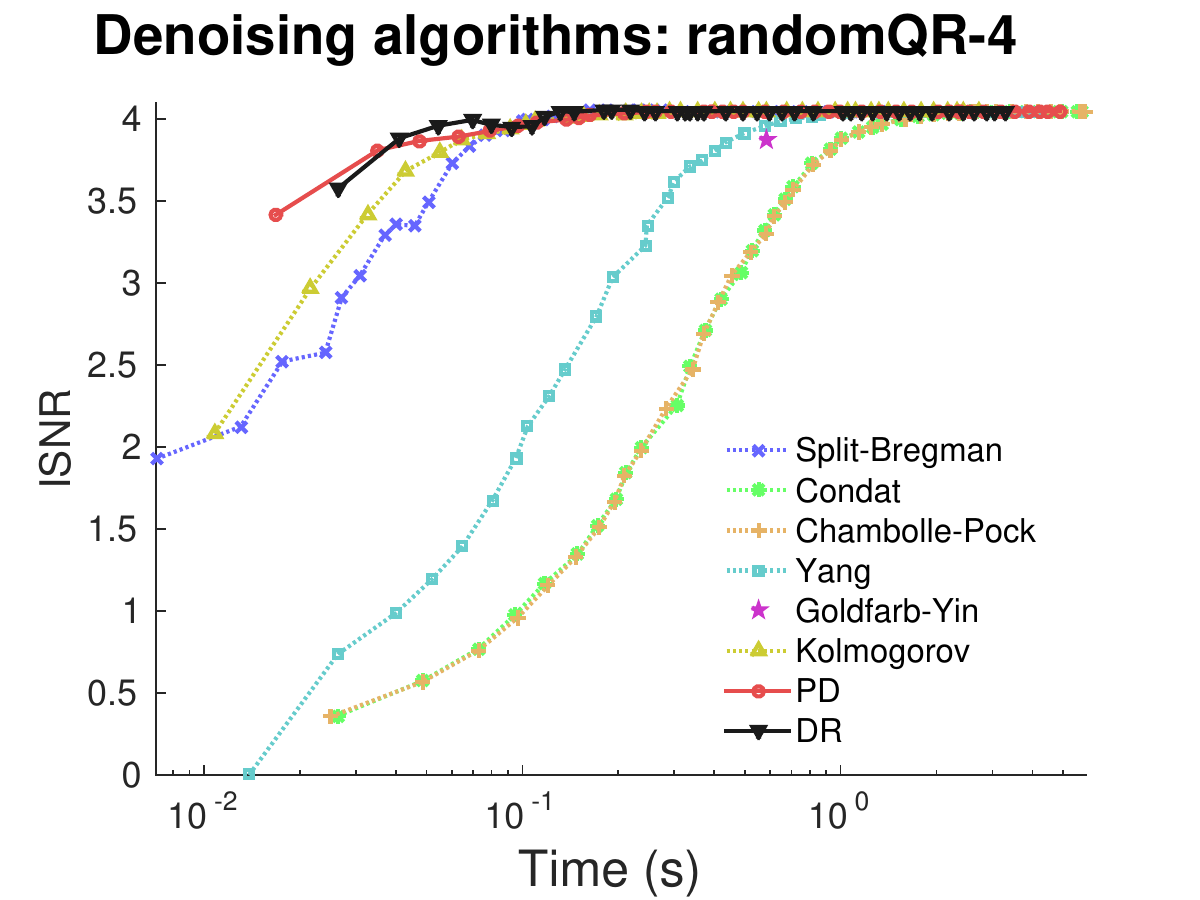}&\hskip-12pt
  \includegraphics[width=.37\linewidth]{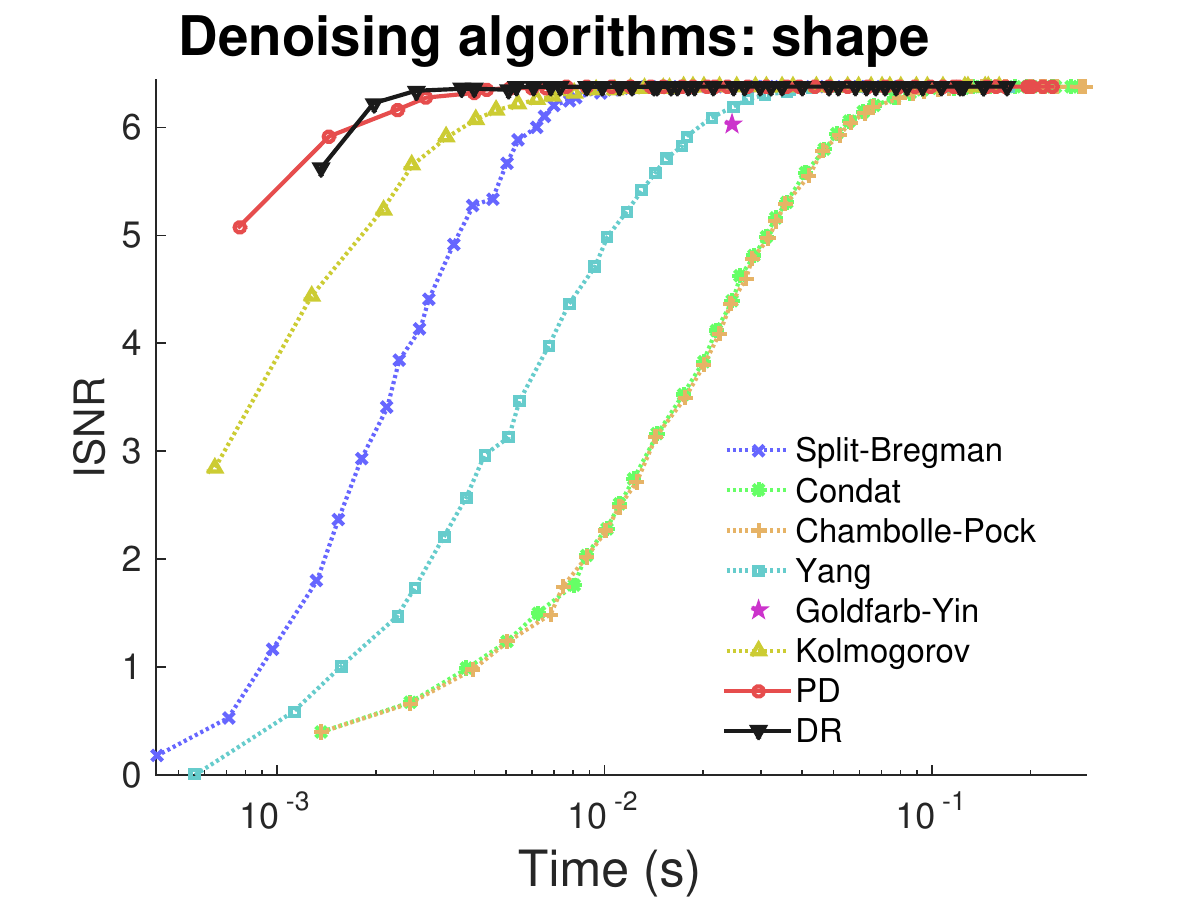}\\
  \hskip-25pt\includegraphics[width=.37\linewidth]{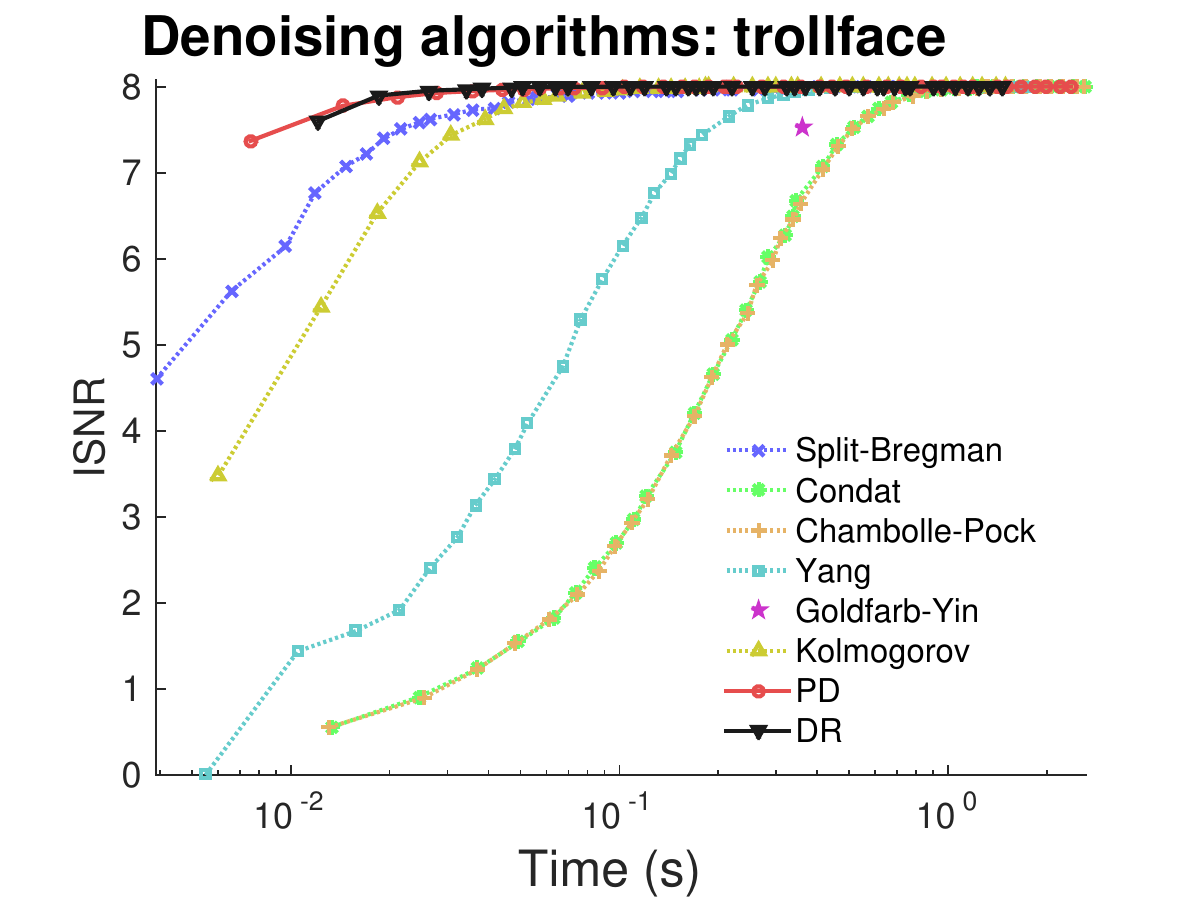}&\hskip-12pt
  \includegraphics[width=.37\linewidth]{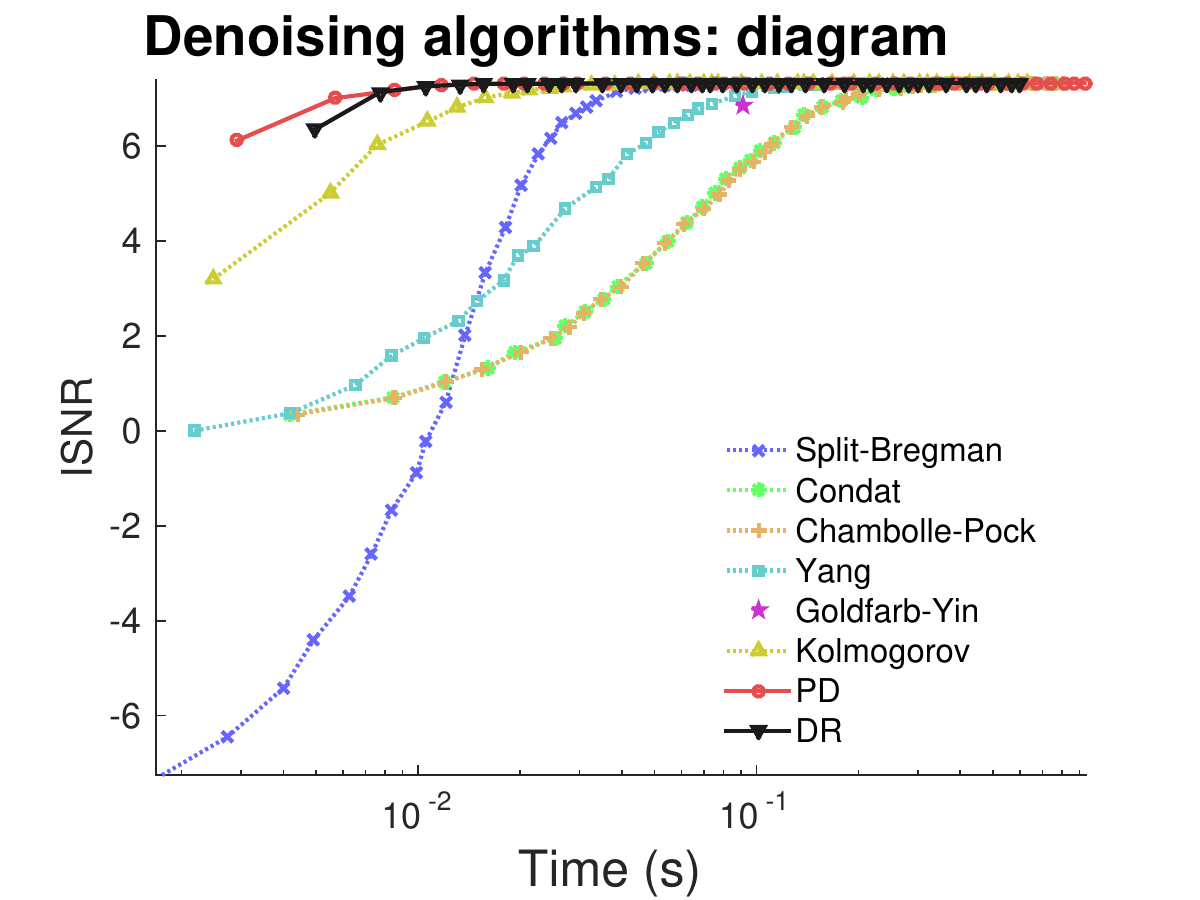}&\hskip-12pt
  \includegraphics[width=.37\linewidth]{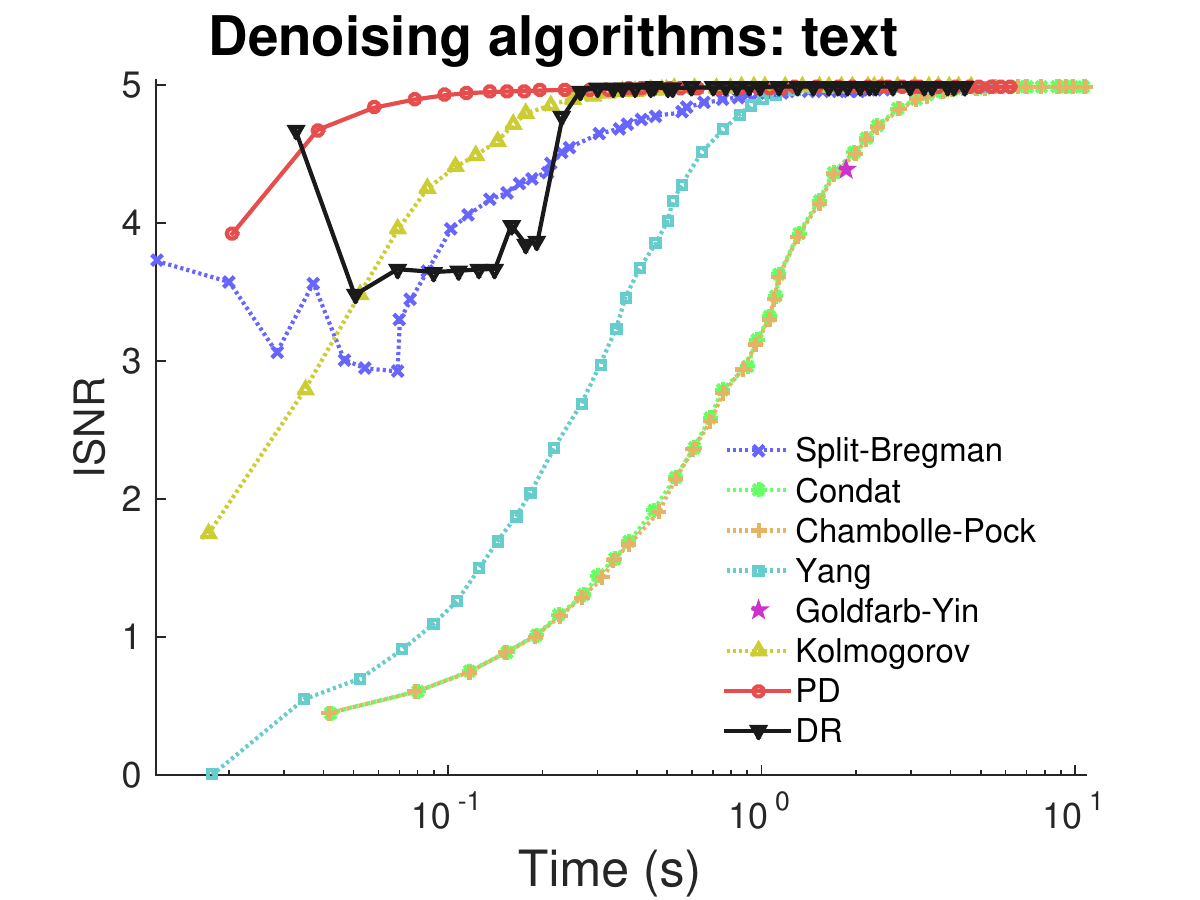}\\
  \hskip-25pt\includegraphics[width=.37\linewidth]{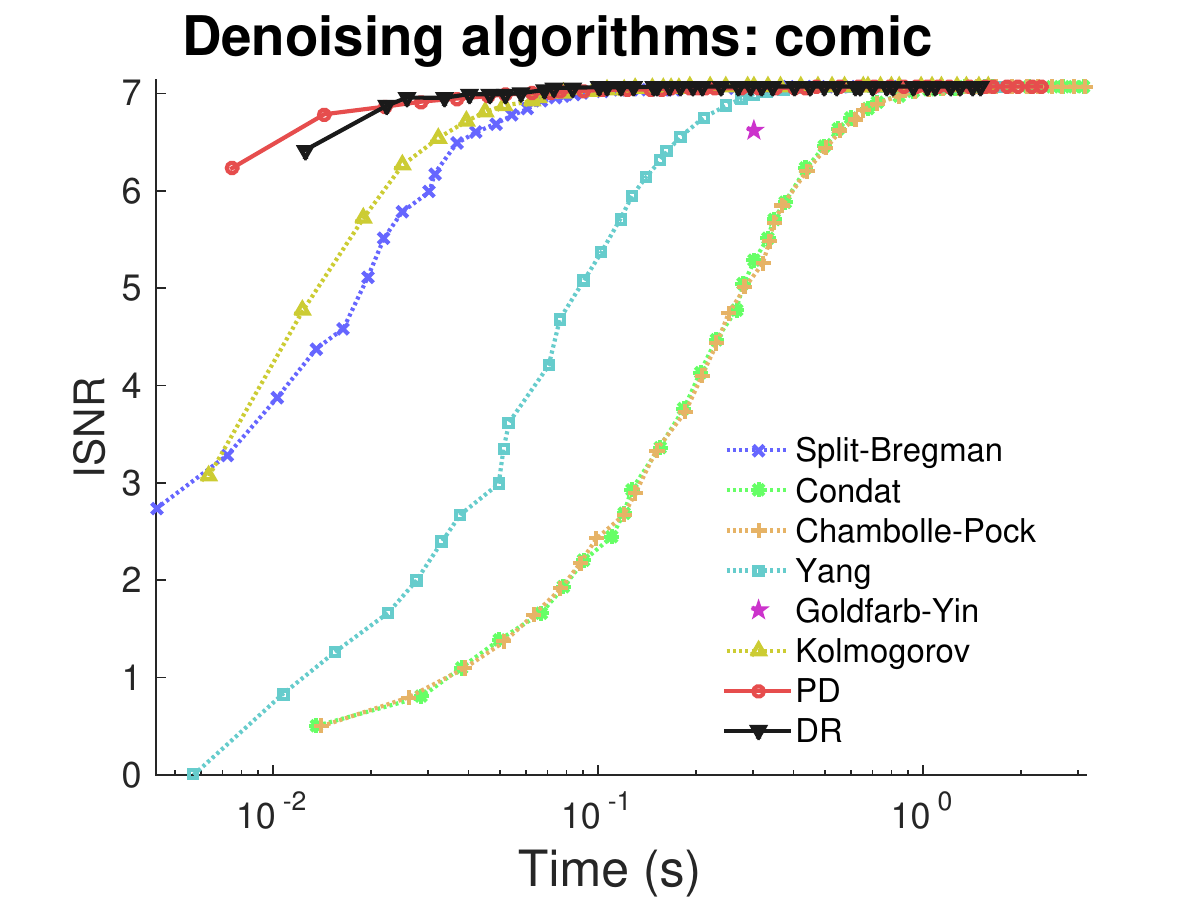}&\hskip-12pt
  \includegraphics[width=.37\linewidth]{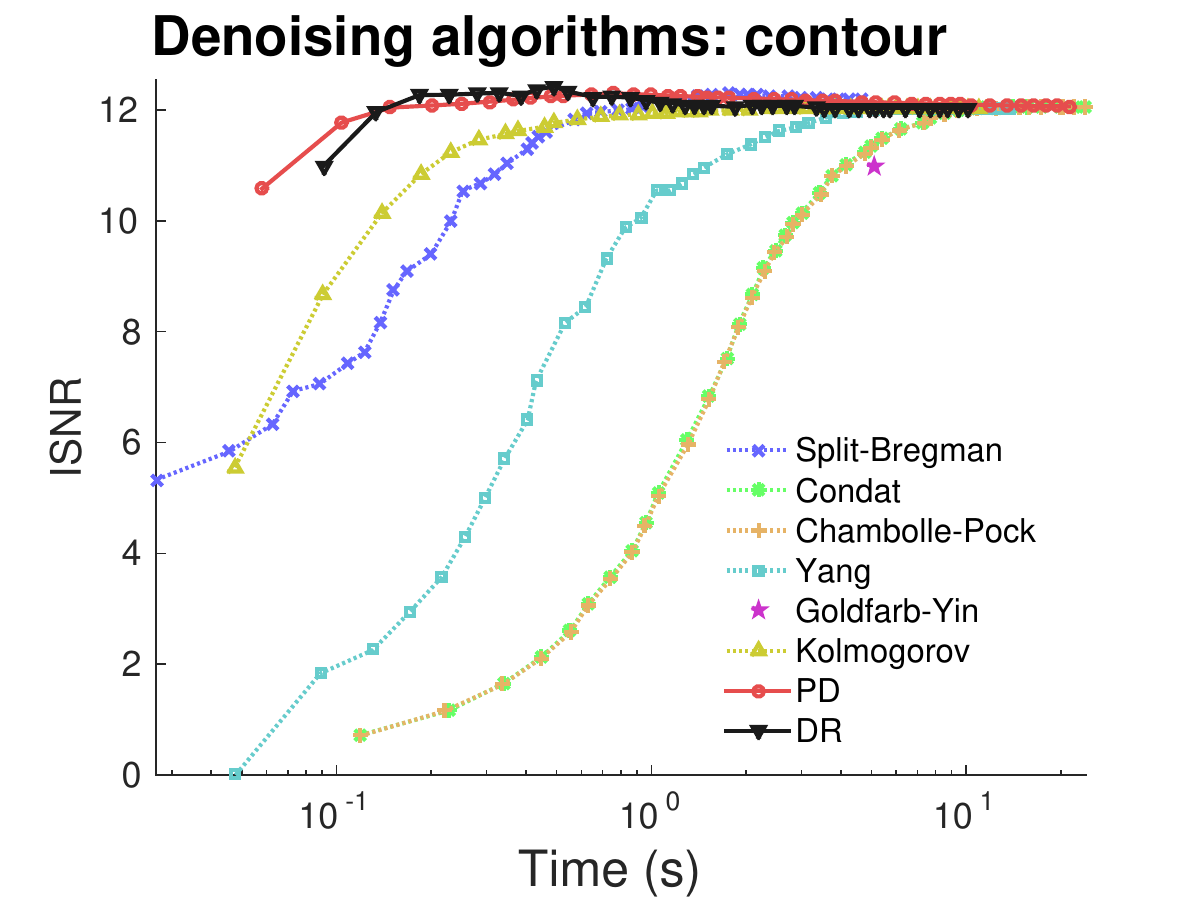}&\hskip-12pt
  \includegraphics[width=.37\linewidth]{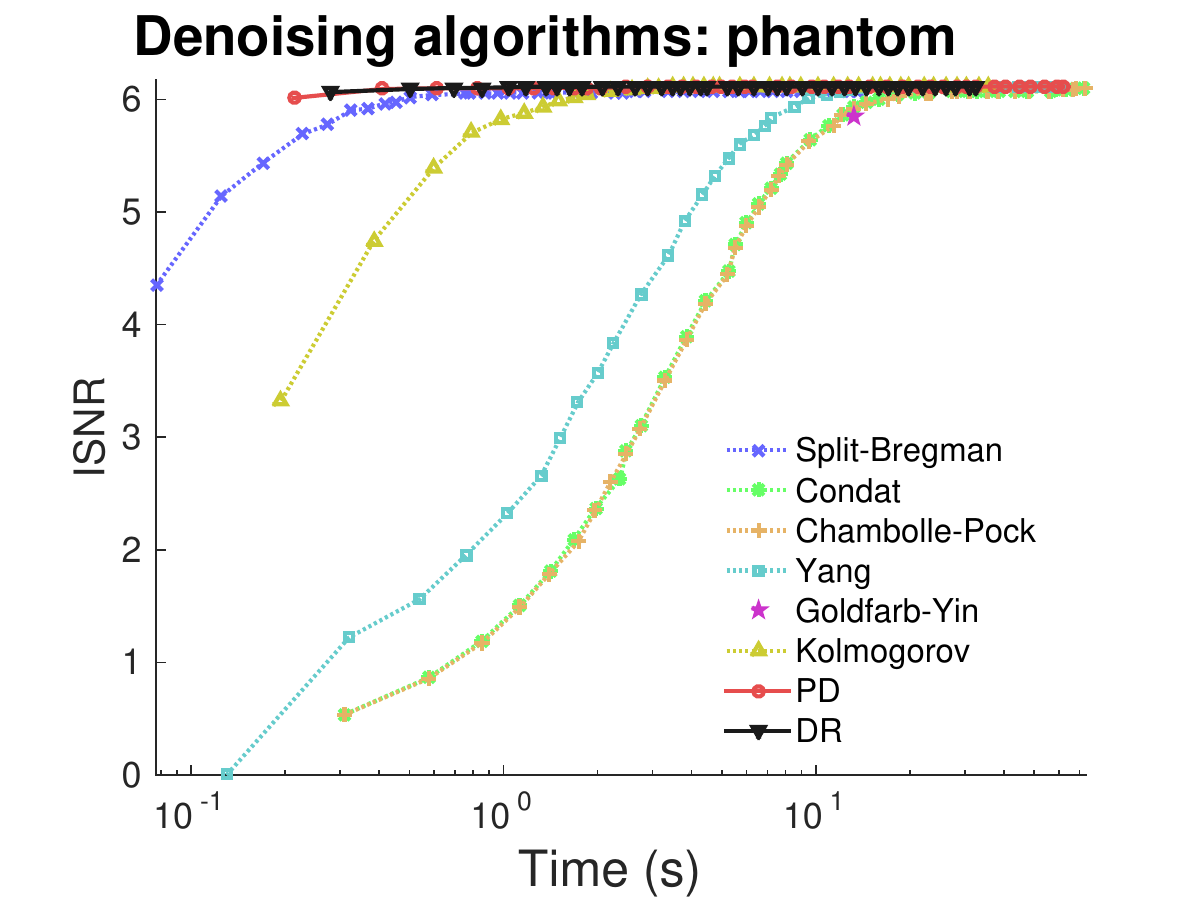}\\
  \end{tabular}
  \vspace{-0.5cm}
  \caption{Increased Signal to Noise Ratio (ISNR) vs time of the denoising 2D-TV algorithms under comparison, for the different images considered in the experiments.}
  \label{fig:2DTV-ISNR}
\end{figure}

To compare the algorithms we run all of them for each image and measured its ISNR and relative distance to the optimal objective value of the current solution at each iteration through their execution. The only exception to this procedure is the method of Goldfarb and Yin, which is non--iterative and thus always returns an exact solution, and so we just measure the time required to finish. The optimal objective value was estimated by running all methods for a very large number of iterations and taking the minimum value of them all. This produced the plots shown in Figures \ref{fig:2DTV-obj}--\ref{fig:2DTV-ISNR}. From them the following observations are of relevance:
\begin{itemize}
  \setlength{\itemsep}{1pt}
 \item Condat's method and Chambolle-Pock's method are reduced to essentially the same algorithm when applied to the particular case of anisotropic 2D TV denoising. Furthermore, they seem to perform slowly when compared to other methods.
 \item ADMM (Yang's method) exhibits slow performance at the beginning, but when run for sufficient time is able to achieve a good approximation to the optimum.
 \item The Split Bregman method, in spite of being an ADMM--like method much like Condat's or Chambolle-Pock, performs significantly better than those. We attribute this to the very efficient implementation provided by its authors, and to the fact that a fast approximate method is employed to compute the required matrix inversions throughout the method.
 \item The method by Goldfarb and Yin is slower than other approaches and seems to provide suboptimal solutions. We attribute this to the fact that this method solves a discrete (integer--rounded) approximation to the problem. We acknowledge that other methods exploiting the Total Variation - Minimum-cut relationship have been proposed with varying speed results, e.g.~\citep{DuanTV}, however the suboptimality issues still apply.
 \item The method by Kolmogorov et al, when properly accelerated by a suitable choice of adaptive stepsizes, seems to be the best choice for finding very accurate solutions, though it is very closely followed by ADMM.
 \item The parameter free methods PD and DR are the fastest to achieve a mid-quality solution, with Douglas-Rachford performing better than Proximal Dykstra.
\end{itemize}

Considering these facts, the method of choice among the ones considered depends on the desired accuracy. We argue, however, that for the purpose of image processing a mid-quality solution is sufficient. The ISNR plots of Figure \ref{fig:2DTV-ISNR} certainly seem to support this, as the perceived quality of the reconstruction, roughly approximated by the ISNR, saturates rapidly and no significant improvements are obtained through further optimization. Given this, the proposed methods seem to be the best suited for the considered task.

For quick reference, Table \ref{tab:2DmethodsComparison} presents a summary of key points of the compared methods, along with some recommendations about when to put them to use.

\begin{table}[tbp]
\caption{\small Summary of key points of the compared $\tvell_{1,1}^{\twod}$ proximity (denoising) methods.
\label{tab:2DmethodsComparison}
}
\centering
\begin{tabular}{lcp{0.7\textwidth}}
\hline
\abovespace\belowspace
      Method & & Key points \\
\hline
\abovespace 
{\bf Douglas Rachford} & \color{red}$+$ & Fast convergence to medium-quality \\
                       & \color{red}$+$ & Embarrassingly parallel \\
                       & \color{blue}$-$ & Slow for higher accuracies \\
                       & $\Rightarrow$ & Ideal for standard denoising tasks \belowspace\\
\hline
\abovespace 
Proximal Dykstra & \color{red}$+$ & Attainable accuracies similar to DR \\
                 & \color{blue}$-$ & But slower than DR \\
                 & $\Rightarrow$ & Use DR instead \belowspace\\
\hline
\abovespace 
Split Bregman & \color{red}$+$ & Eventually performs similarly to DR \\
              & \color{blue}$-$ & Slow convergence at first iterations \\
              & $\Rightarrow$ & Use DR instead \belowspace\\
\hline
\abovespace 
Chambolle--Pock & \color{blue}$-$ & Slow \\
                & $\Rightarrow$ & Use other method instead \belowspace\\
\hline
\abovespace 
Condat & \color{red}$+$ & Solves objectives involving a sum of smooth/non--smooth functions with linear operators \\
       & \color{blue}$-$ & Reduces to Chambolle--Pock when solving basic image denoising \\
       & $\Rightarrow$ & Use only when dealing with more complex functionals \belowspace\\
\hline
\abovespace 
ADMM (Yang) & \color{red}$+$ & More accurate \\
     & \color{blue}$-$ & Slightly slower than Kolmogorov \\
     & \color{blue}$-$ & Bad behavior for mid-quality solutions \\
     & $\Rightarrow$ & Use Kolmogorov instead \belowspace\\
\hline
\abovespace 
Kolmogorov & \color{red}$+$ & More accurate \\
     & \color{blue}$-$ & Slower than DR for low accuracies \\
     & $\Rightarrow$ & Useful when extremely accurate solutions are required \belowspace\\
\hline
\abovespace 
Goldfarb-Yin & \color{red}$+$ & Solves the discrete version of the problem \\
             & \color{blue}$-$ & Slow \\ 
             & \color{blue}$-$ & Poor accuracy for the continuous version \\
             & $\Rightarrow$ & Apply only when solving the discrete problem \belowspace\\
\hline
\end{tabular}
\end{table}

\subsubsection{Parallelization experiments}
In addition to the previous experiments and to illustrate the parallelization potential of the presented anisotropic filtering method, Figure~\ref{fig.multicore} plots running times for the PD algorithm as the number of processor core ranges from 1 through 16. We see that for the smaller images, the gains due to more processors essentially flatten out by 8 cores, where synchronization and memory contention offsets potential computational gains (first row). For the larger images, there is steadier speedup as the number of cores increase (in each plot there seems to be a ``bump'' at 14 processors; we attribute this to a quirk of the multicore machine that we used). From all the plots, however, the message is clear: our TV prox operators exploit parallelization well, and show substantial speedups as more processor cores become available. 

We should also note in passing that the Split Bregman method, which in the previous experiments showed a reasonable performance, turns out to be much harder to parallelize. This fact was already observed by~\cite{WangTV14} in the context of isotropic TV. Therefore when several processor cores are available the proposed modular strategy seems to be even more suitable to the task.

\begin{figure}[htbp]
  \centering
  \begin{tabular}{llll}
    \hskip-25pt\includegraphics[width=.28\linewidth]{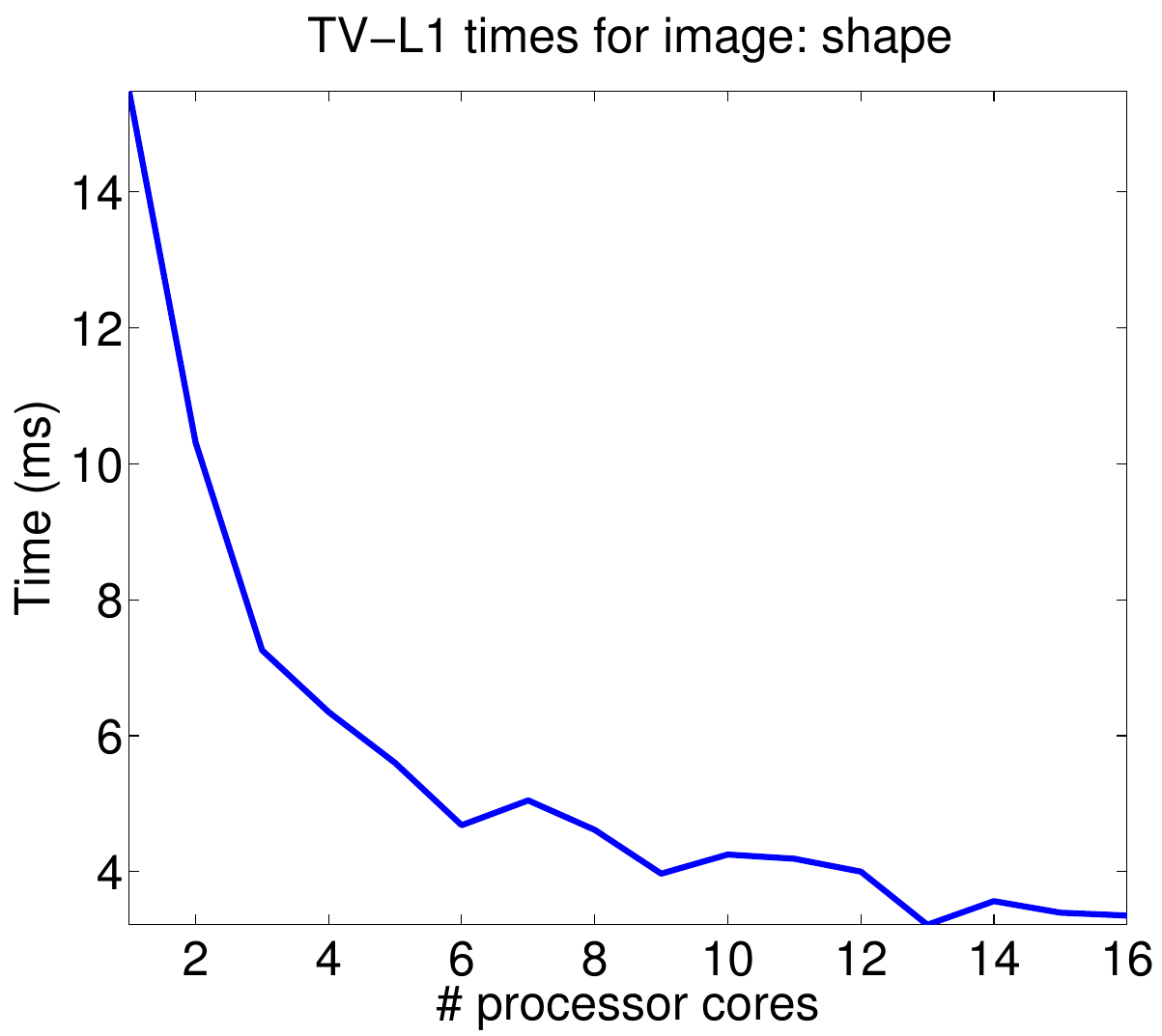}&\hskip-12pt
  \includegraphics[width=.28\linewidth]{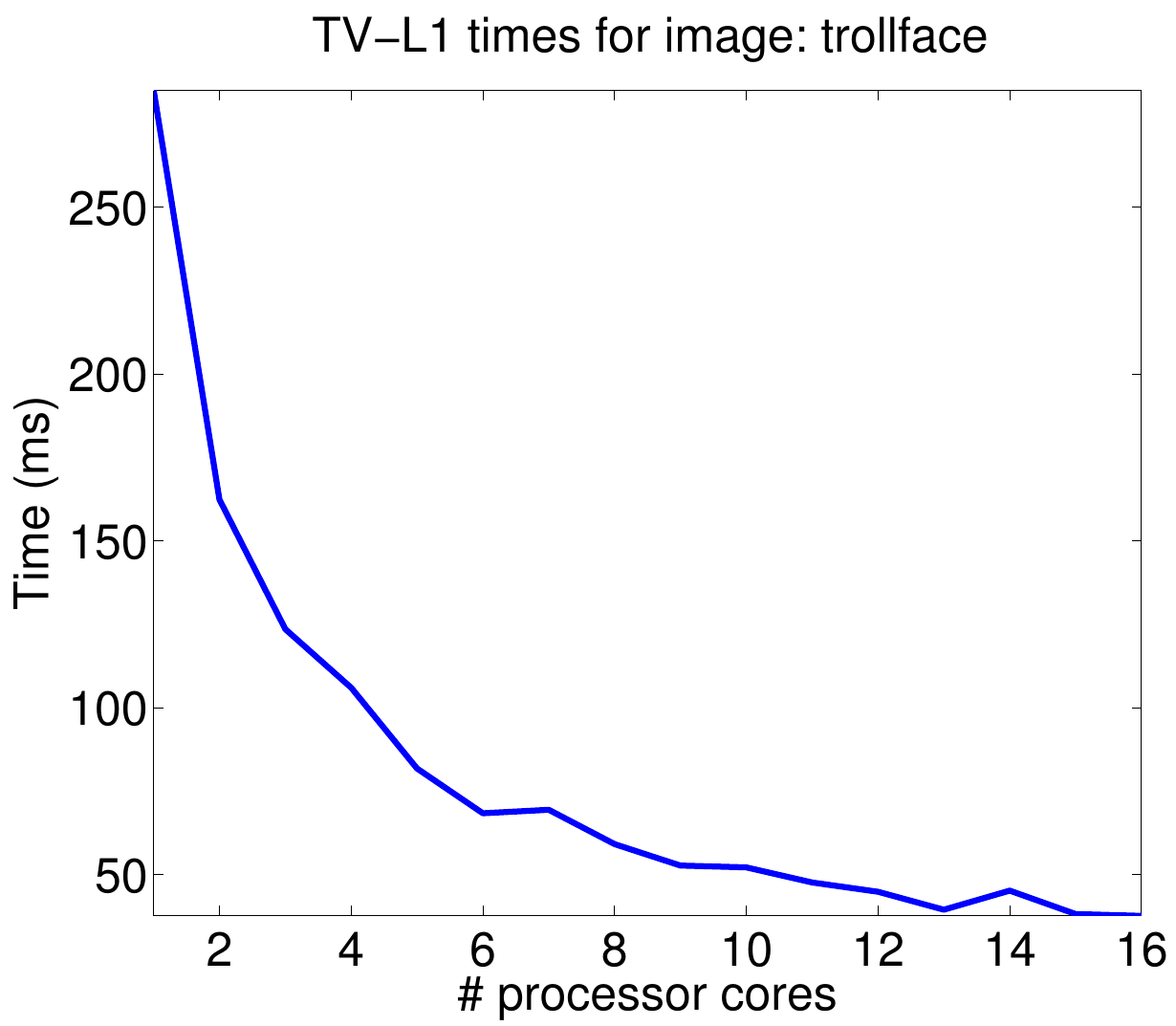}&\hskip-12pt
  \includegraphics[width=.28\linewidth]{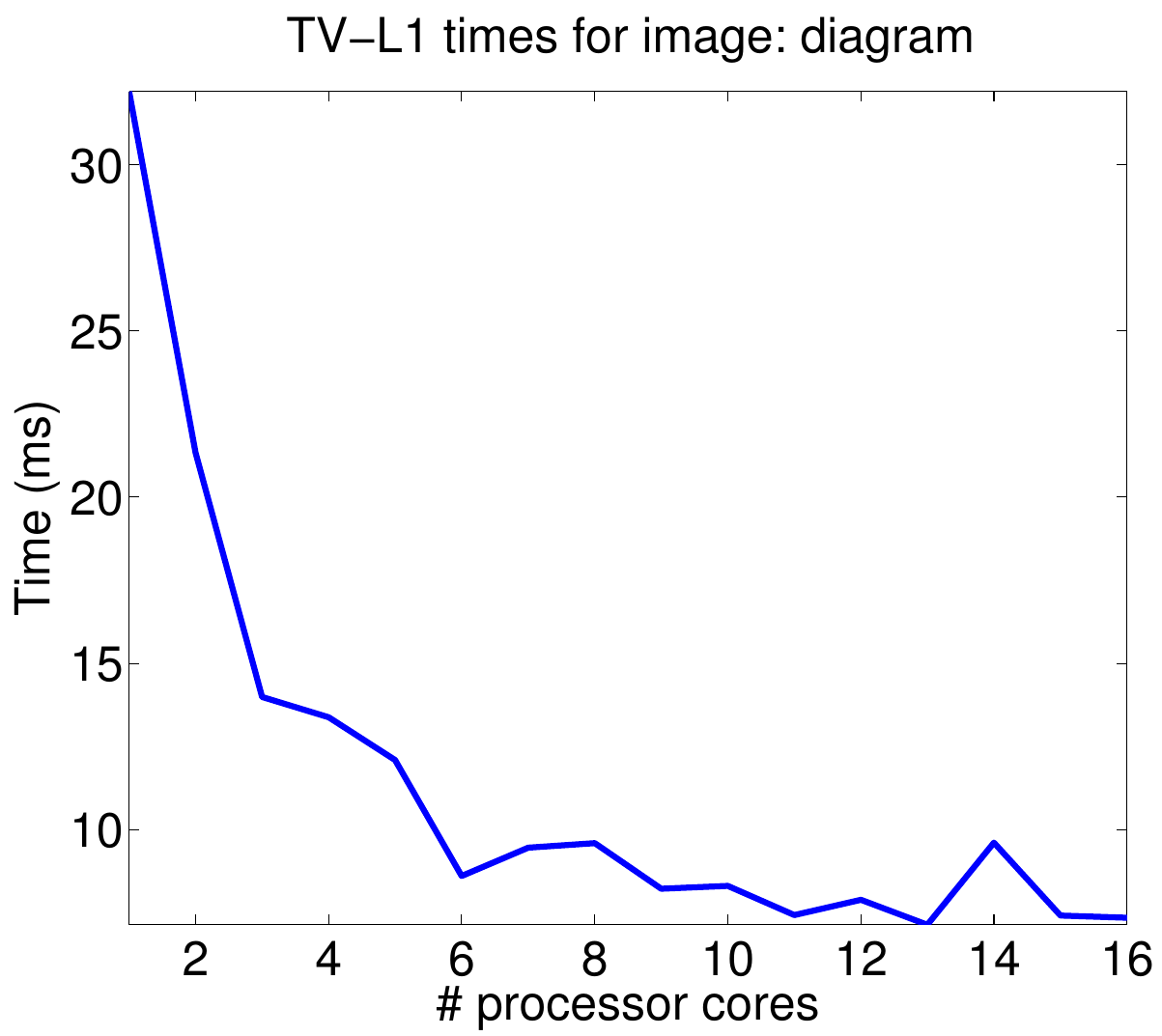}&\hskip-12pt
  \includegraphics[width=.28\linewidth]{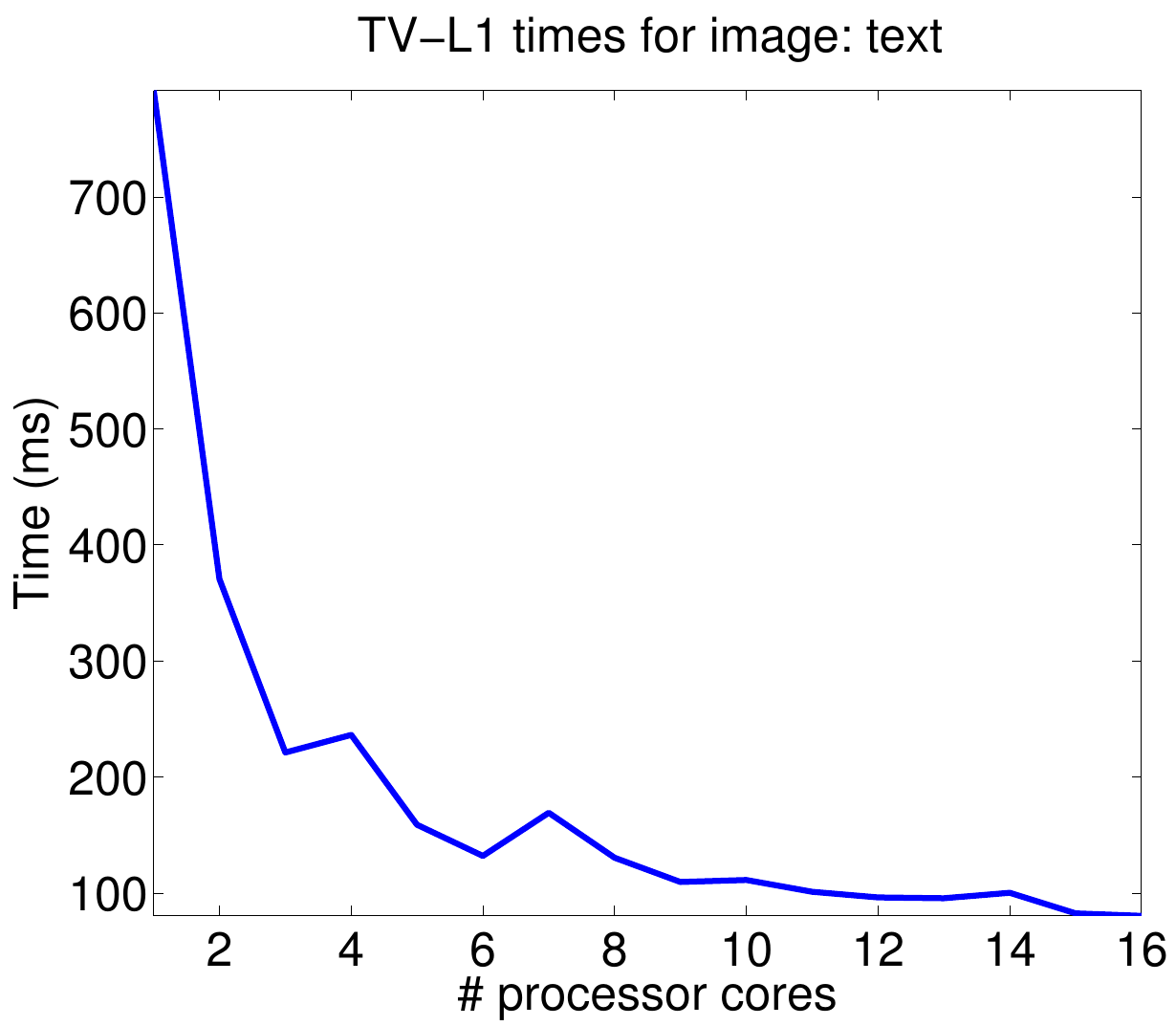}\\
  \hskip-25pt\includegraphics[width=.28\linewidth]{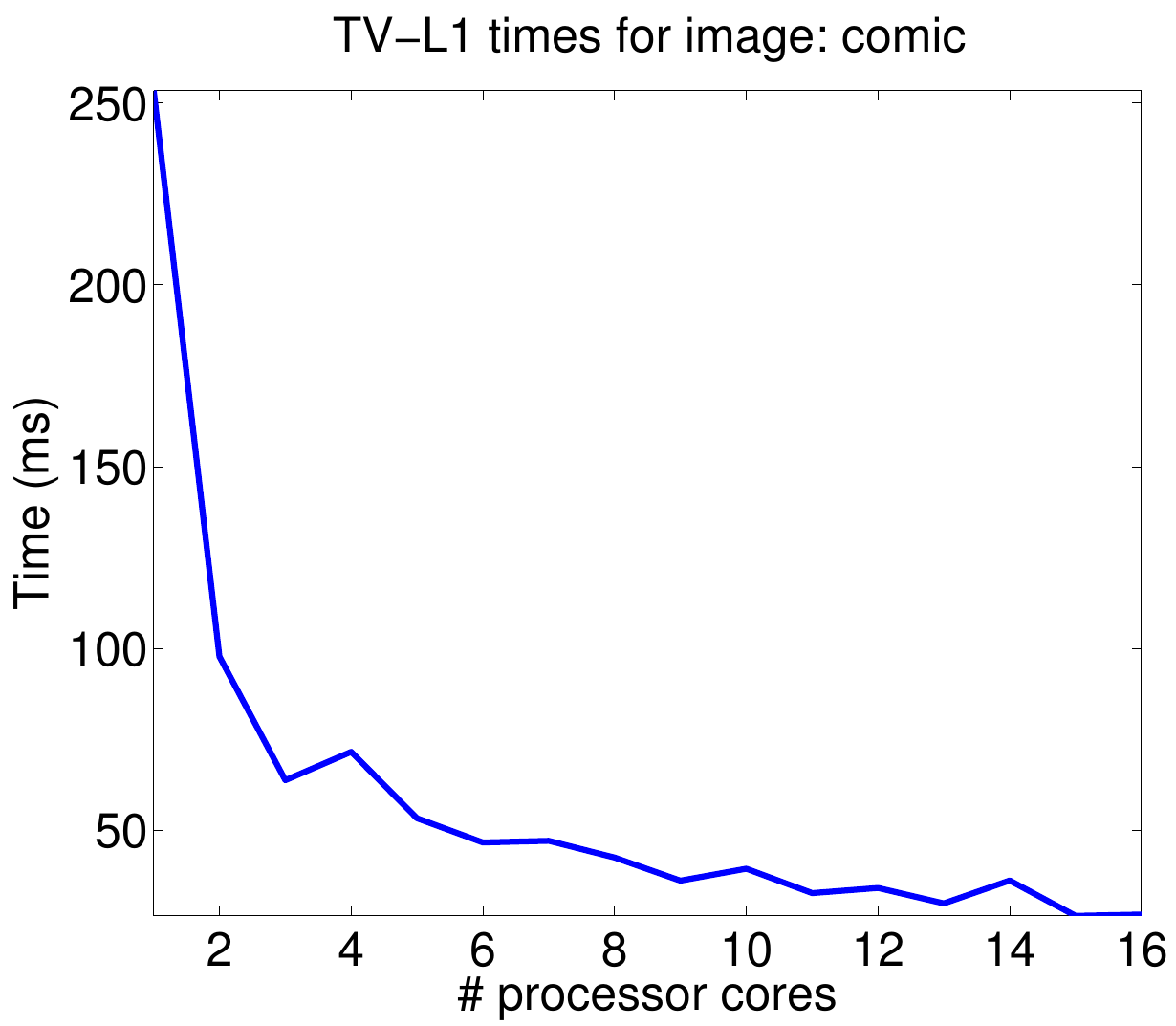}&\hskip-12pt
  \includegraphics[width=.28\linewidth]{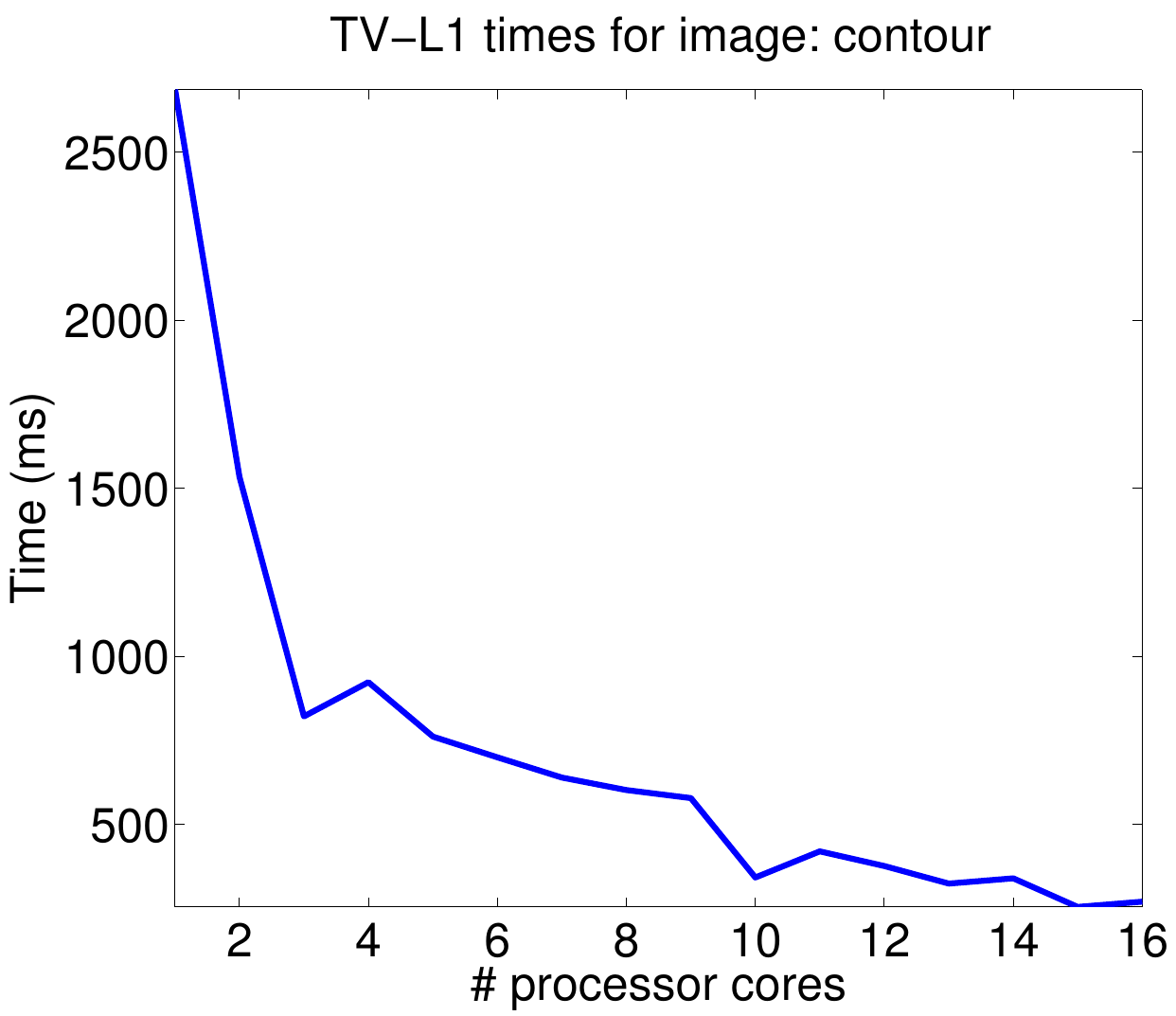}&\hskip-12pt
  \includegraphics[width=.28\linewidth]{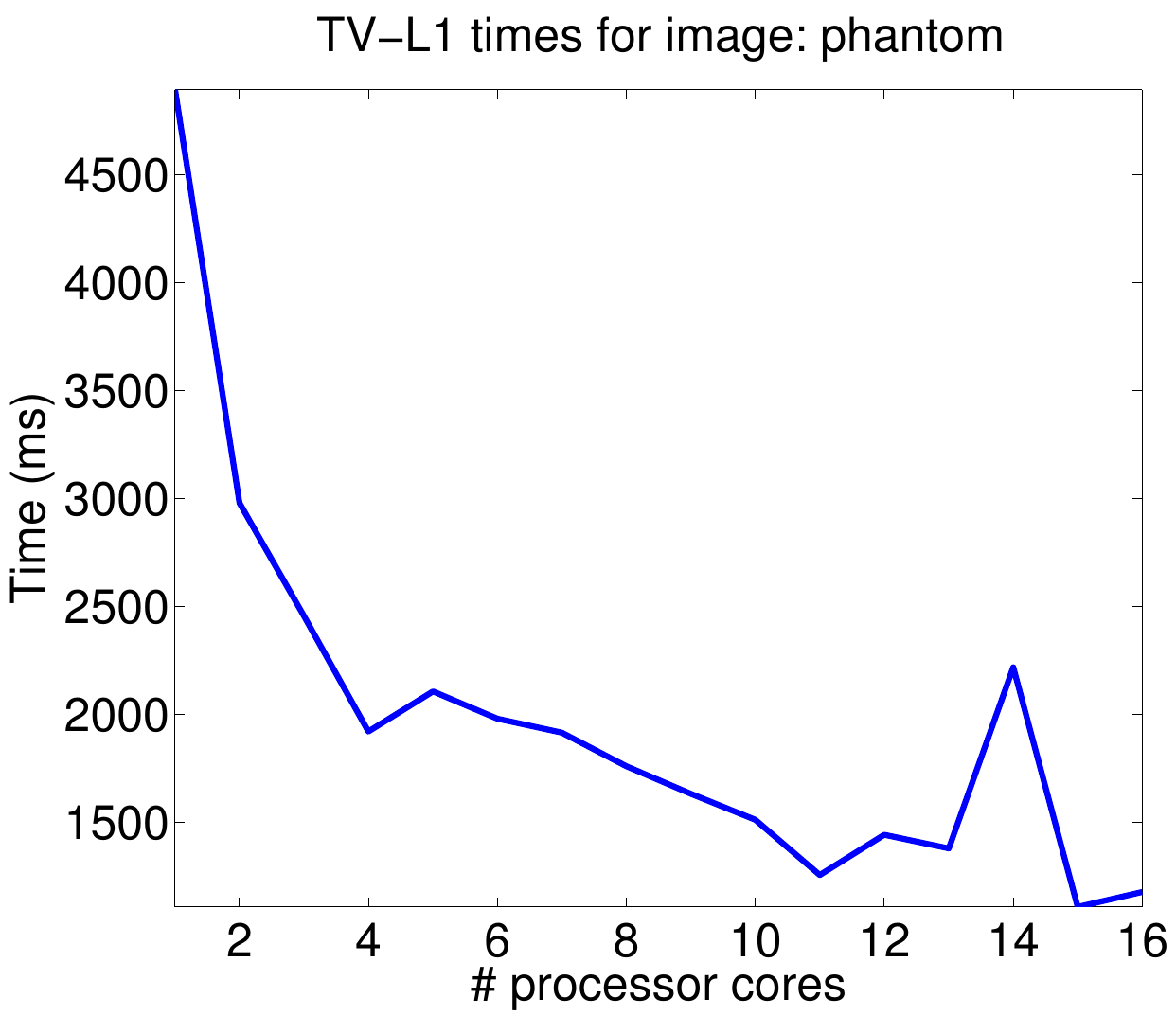}&\hskip-12pt
  \includegraphics[width=.28\linewidth]{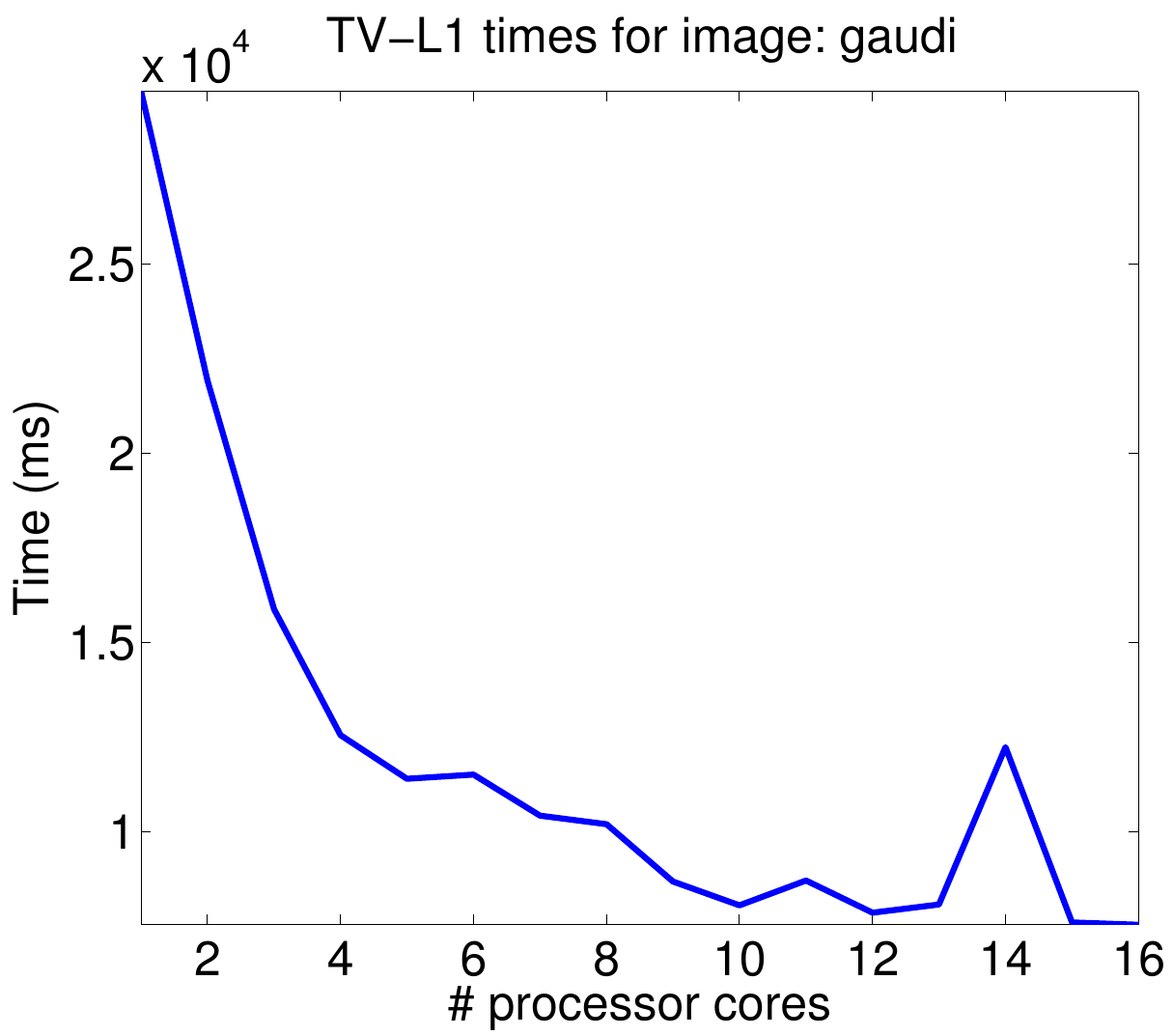}\\
  \end{tabular}
  \caption{Multicore speedups on different images (see also Table~\ref{tab:convolutions})}
  \label{fig.multicore}
\end{figure}

\subsubsection{Anisotropic image deconvolution}
Taking a step forward we now confront the problem of {\bf image deconvolution} (or image deblurring). This setting is more complex since the task of image recovery is made harder by the presence of a {\bf convolution kernel} $K$ that distorts the image as
\begin{equation*}
 \mu_0 = \mk \ast \mu + \mn,
\end{equation*}
where $\mn$ is noise as before and $\ast$ denotes convolution. To recover the original image $\mu$ from the observed $\mu_0$, it is common to solve the following deconvolution problem 
\begin{equation}
  \label{eq:27}
  \min_{\mx}\quad\half \frob{\mk \ast \mx- \mu}^2 + \lambda r(\mx).
\end{equation}
As before, the regularizer $r(\mx)$ can be isotropic or anisotropic TV, among others. Here we focus again on the anisotropic TV case to show how the presented solvers can also be used for this image task.

Problem~\eqref{eq:27} also fits the proximal splitting framework, and so we employ the popular FISTA~\citep{fista} method for image processing. The gradient of the loss can be dealt efficiently by exploiting $\mk$ being a convolution operator, which through the well--known convolution theorem is equivalent to a dot product in the frequencies space, and so the computation is done by means of fast Fourier transforms and products. Several other solvers that explicitly deal with convolution operators are also available~\citep{SALSA,TwIST}. A notable solver specific for the isotropic case is given by the work of \citet{Krishnan09}, that handles even nonconvex isotropic TV-norms ($0 < p < 1$). But this approach does not extend to the anisotropic case, so we focus on general proximal splitting.

\begin{table}[htbp]
\begin{center}
\begin{tabular}{lcccc}
\hline
\abovespace\belowspace
      Image & Convolution & Parameters \\
\hline
\abovespace 
randomQR & Motion & Length 5, Angle 35$^o$ \\
shape & Average & Size 3 $\times$ 3 \\
trollface & Disk & Radius 5 \\
diagram & Motion & Length 5, Angle 0$^o$ \\
text & Average & Size 1 $\times$ 10 \\
comic & Gaussian & Size 15, Deviation 2 \\
contour & Disk & Radius 5 \\
\belowspace
phantom & Motion & Length 100, Angle 240$^o$
\\ \hline
\end{tabular}
\end{center}
\caption{\small Convolution kernels used for each test image. {\it Average} substitutes each pixel with the average of its surrounding $n \times m$ neighbors. {\it Disk} performs the same operation within a disk-shaped neighborhood of the shown radius. {\it Gaussian} uses a $n \times n$ neighborhood and assigns different weights to each neighbor following the value of a gaussian distribution of the indicated deviation centered at the current pixel. {\it Motion} emulates the distortions produced when taking a picture in motion, defining a neighborhood following a vector of the indicated length and angle. 
\label{tab:convolutions}
}
\end{table}

We use the same test images as for our denoising experiments (Figure \ref{fig:imagesExp}), with identical noise patterns (Table \ref{tab:noises}) for the QR images, and gaussian noise with variance 0.05 for the rest. In addition, we convolve each image with a different type of kernel to assess the behavior for a variety of convolutions; Table \ref{tab:convolutions} shows the kernels applied. We constructed these kernels using Matlab's {\it fspecial} function; the convolved images are shown in Figure \ref{fig:imagesBlurred}.

The values for the regularizer $\lambda$ were determined by maximizing the reconstruction quality measured in ISNR. Since deconvolution is much more expensive than denoising, instead of performing an exhaustive search for the best $\lambda$, we used a Focused Grid Search strategy~\citep{HAIS07,NeurocomputingRejilla} to find the best performing values.

Any denoising subroutine can be plugged into the aforementioned deconvolution methods, however for comparison purposes we run our experiments with the best proposed method, Douglas Rachford (Alternating Reflections), and the best competing method among those reviewed from the literature, Kolmogorov et al. A key parameter in deconvolution performance is for how long should these methods be run at each FISTA iteration. To select this, we first run FISTA with 100 iterations of Douglas Rachford per step, for a large number of FISTA steps, and take the final objective value as an estimate of the optimum. Then we find the minimum number of Douglas Rachford and Kolmogorov iterations for which FISTA can achieve a relative distance to such optimum below $10^{-3}$. The reason for doing this is that for larger distances the attained ISNR values are still far from convergence. This turned to be 5 iterations for Douglas Rachford an 10 for Kolmogorov. We then run FISTA for such configurations of the inner solvers, and others with a larger number of inner iterations, for comparison 
purposes.

Figures \ref{fig:2DTVdeconv-obj}-\ref{fig:2DTVdeconv-ISNR} show the evolution of objective values and ISNR for all the tested configurations. In general, Douglas Rachford seems to be slightly better at finding more accurate solutions, and also faster at converging to the final ISNR value. We explain this by the fact that the major advantage of Douglas Rachford is its aforementioned ability to find medium--quality solutions in a very small number of iterations: this is why with a small number of inner DR iterates we can converge to good ISRN levels.

\begin{figure}[htbp]
  \centering
  \begin{tabular}{llll}
    \hskip-25pt\includegraphics[width=.37\linewidth]{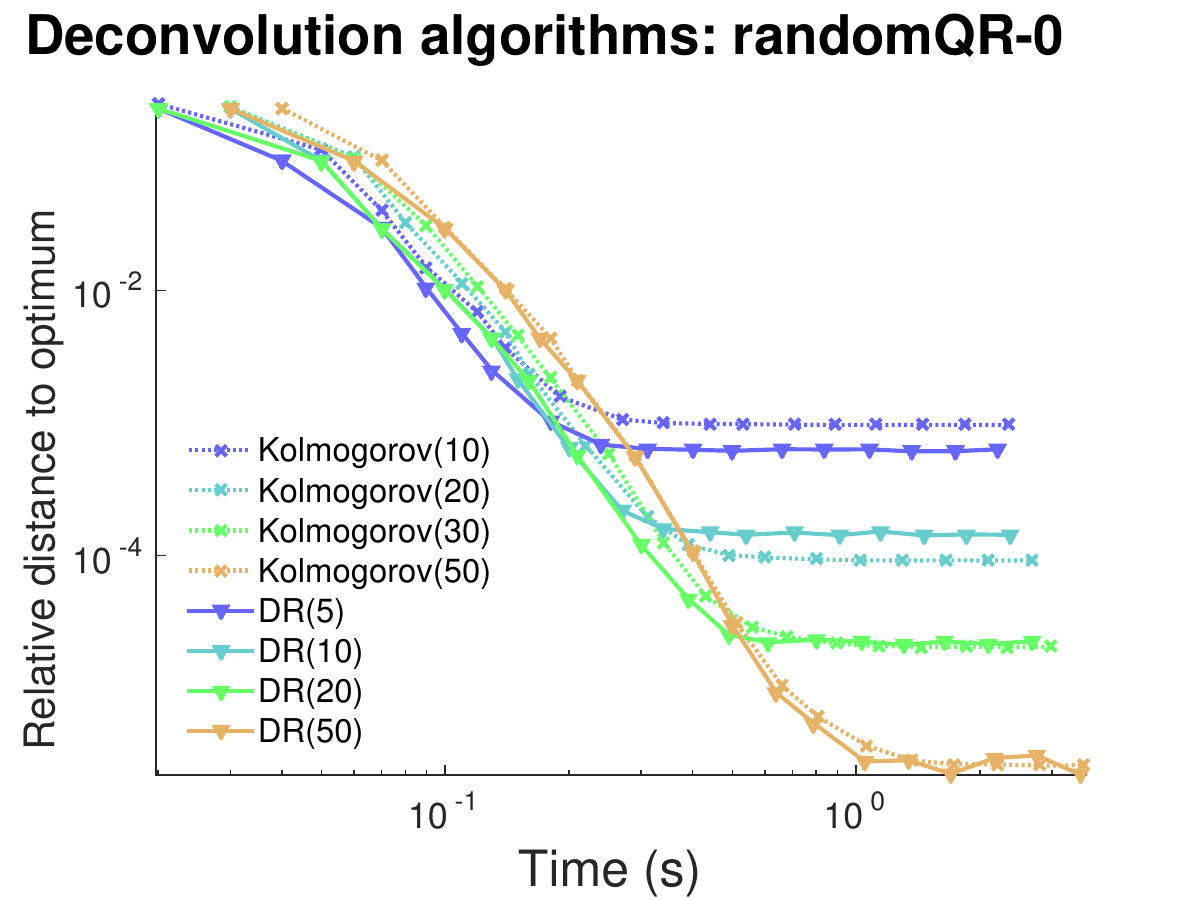}&\hskip-12pt
  \includegraphics[width=.37\linewidth]{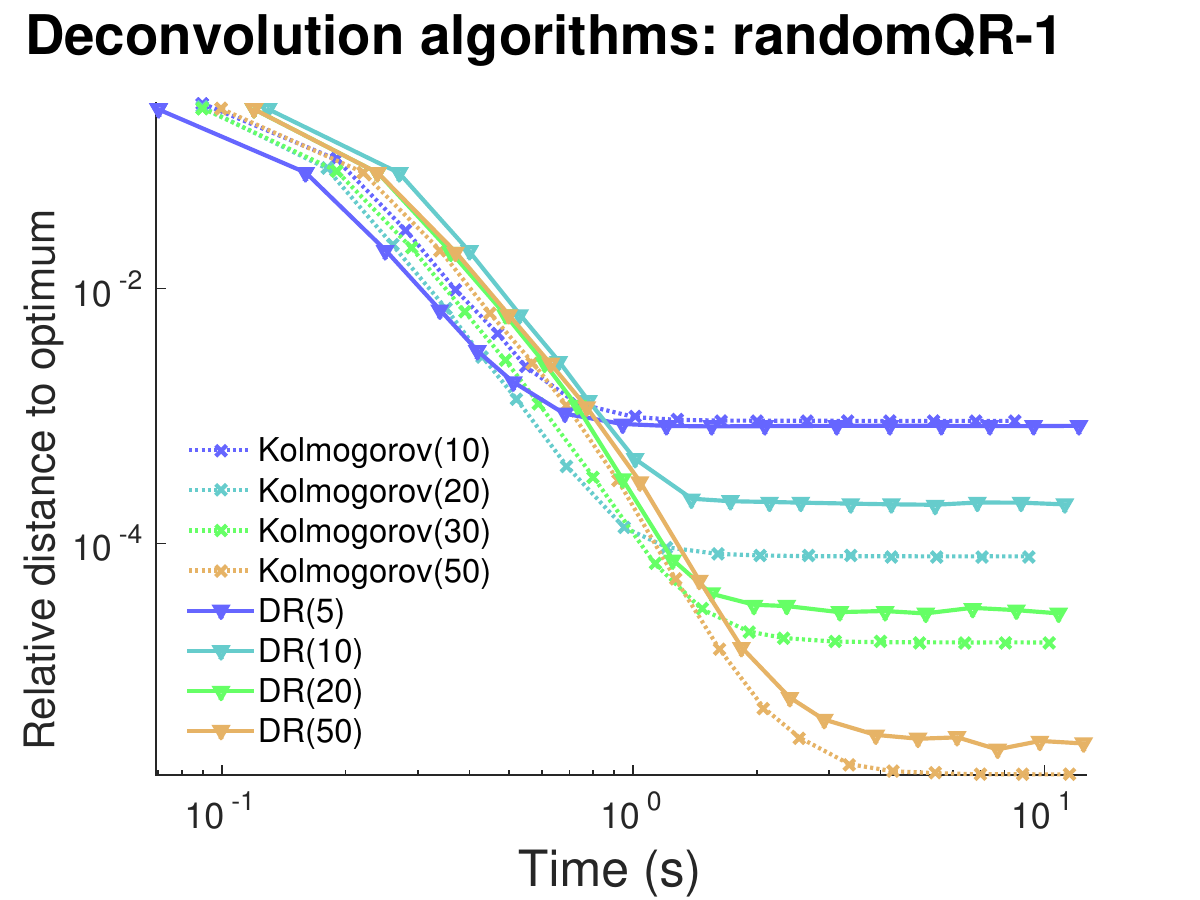}&\hskip-12pt
  \includegraphics[width=.37\linewidth]{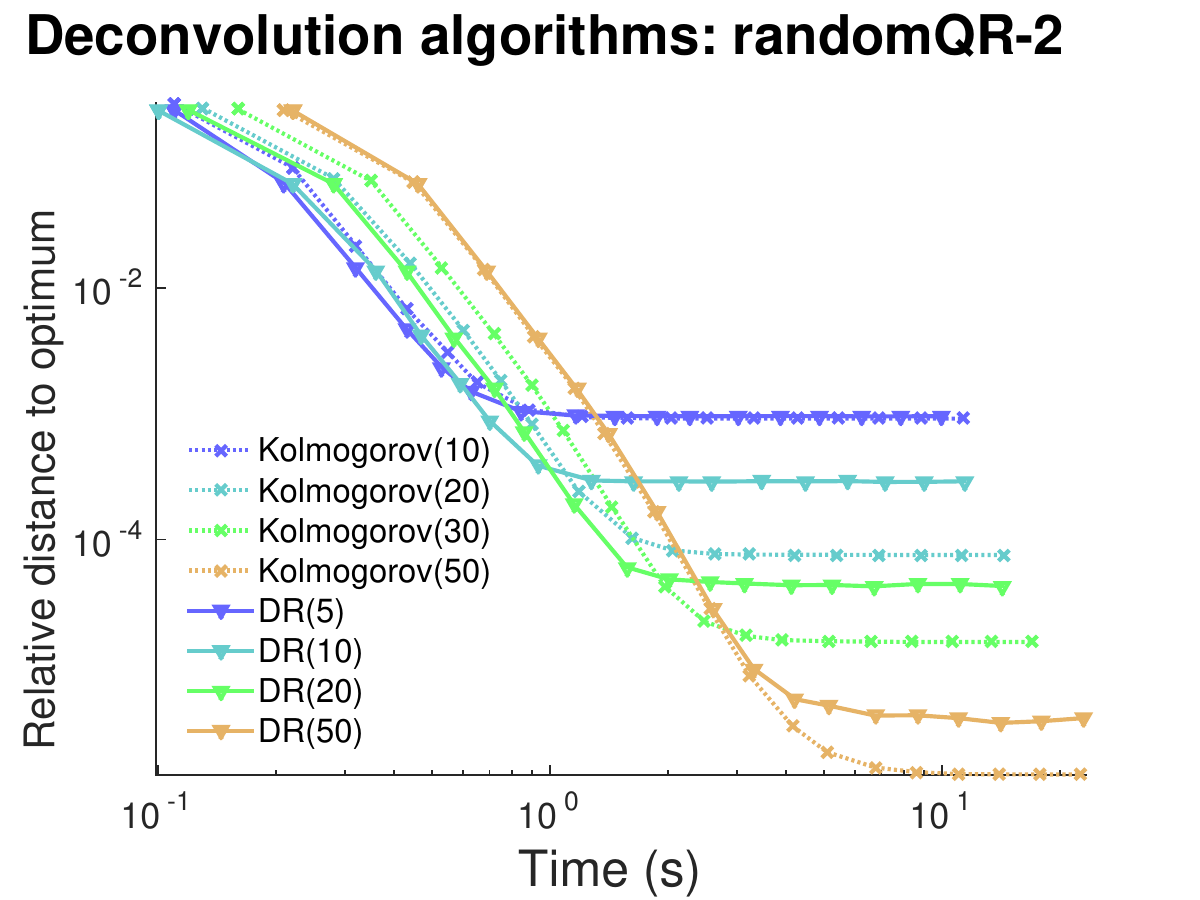}\\
  \hskip-25pt\includegraphics[width=.37\linewidth]{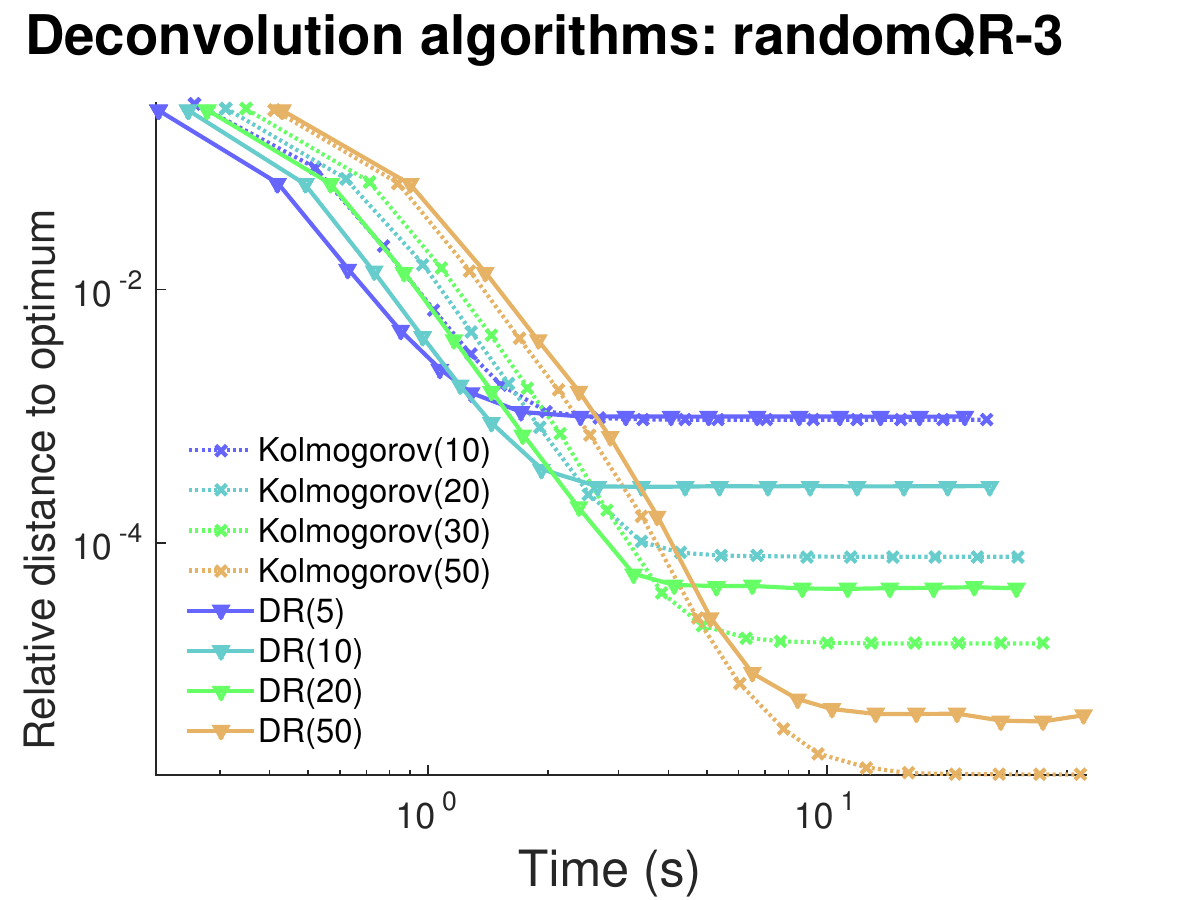}&\hskip-12pt
  \includegraphics[width=.37\linewidth]{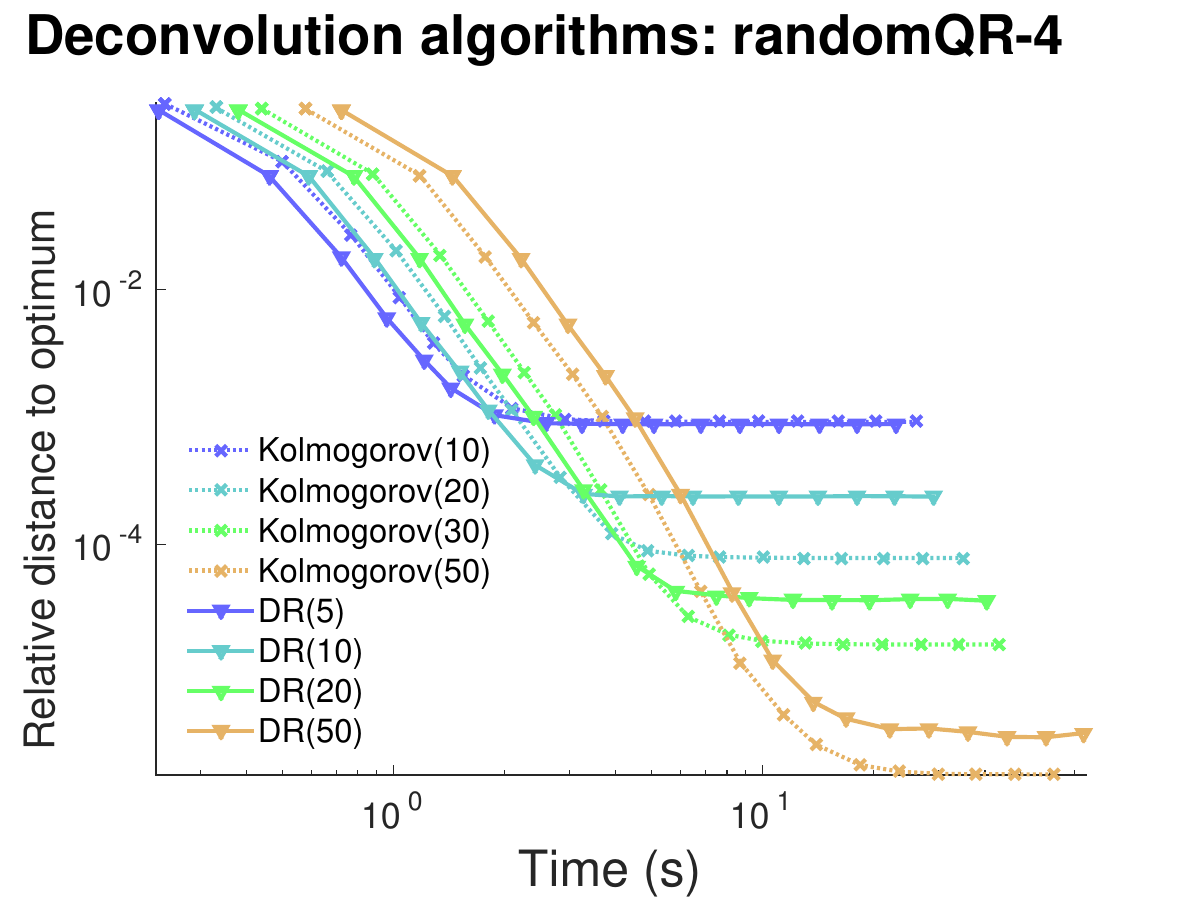}&\hskip-12pt
  \includegraphics[width=.37\linewidth]{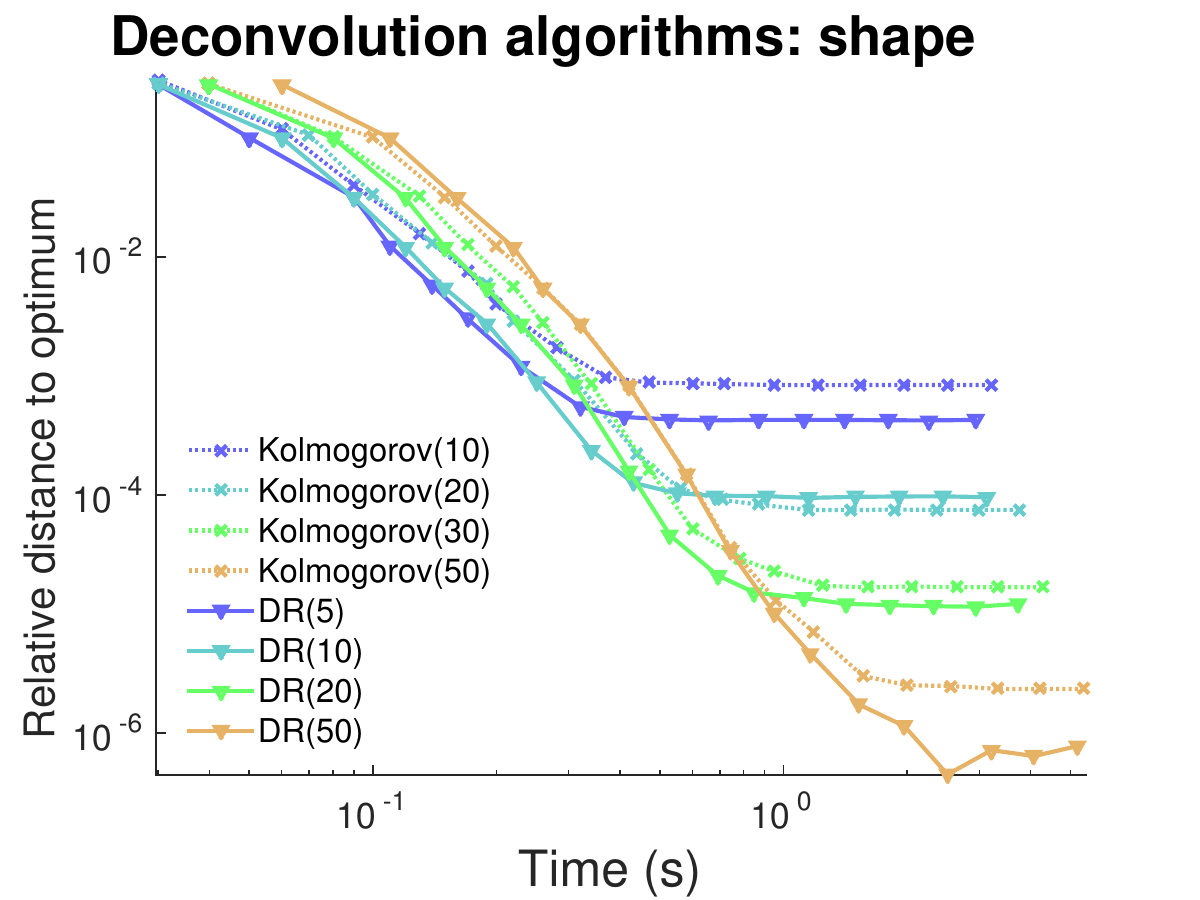}\\
  \hskip-25pt\includegraphics[width=.37\linewidth]{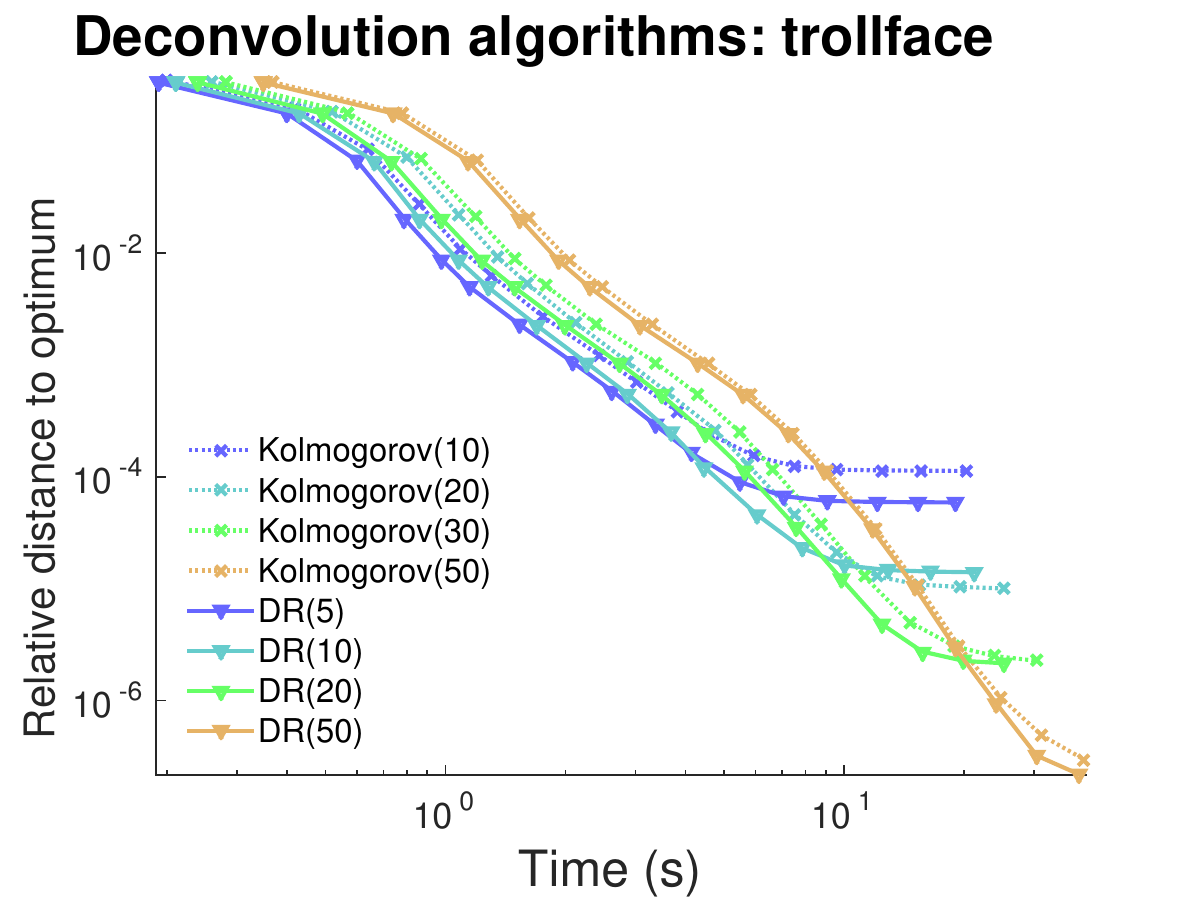}&\hskip-12pt
  \includegraphics[width=.37\linewidth]{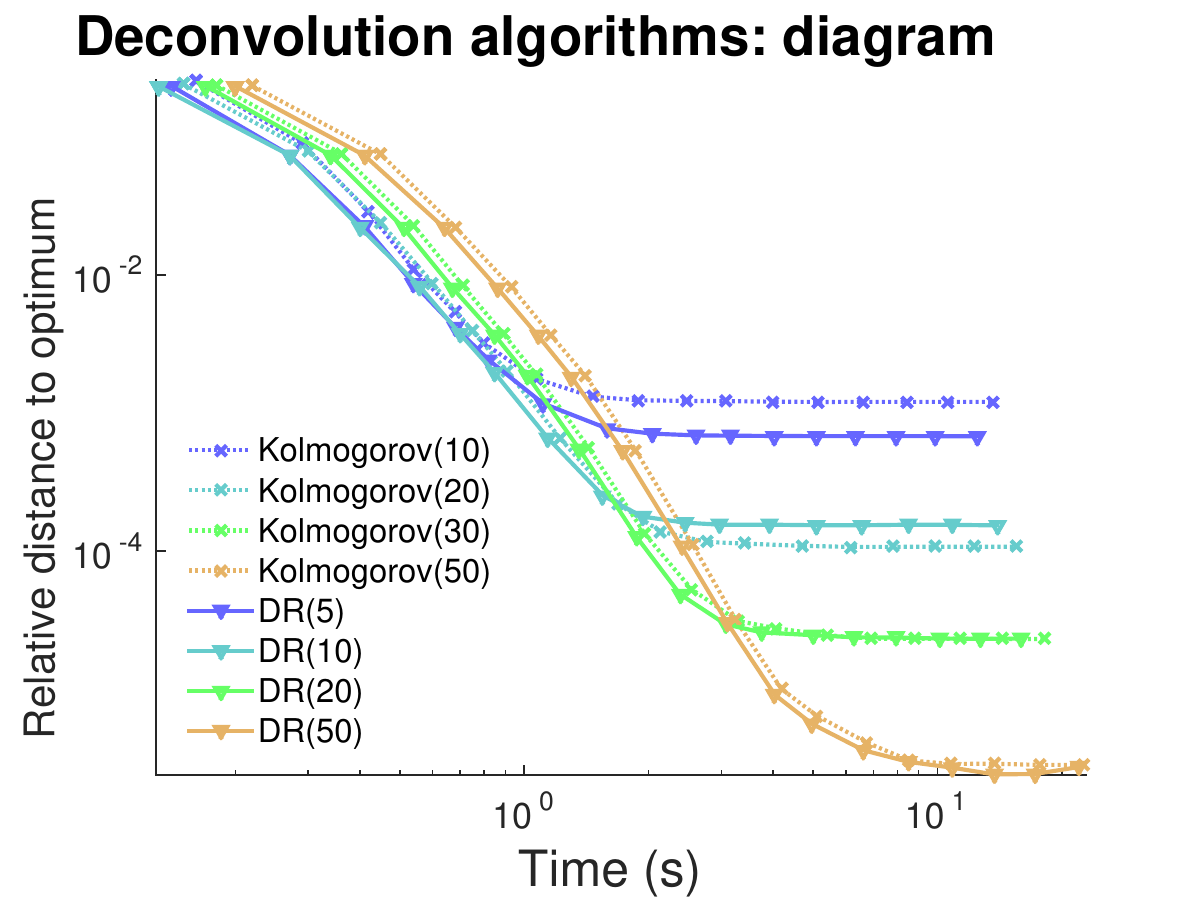}&\hskip-12pt
  \includegraphics[width=.37\linewidth]{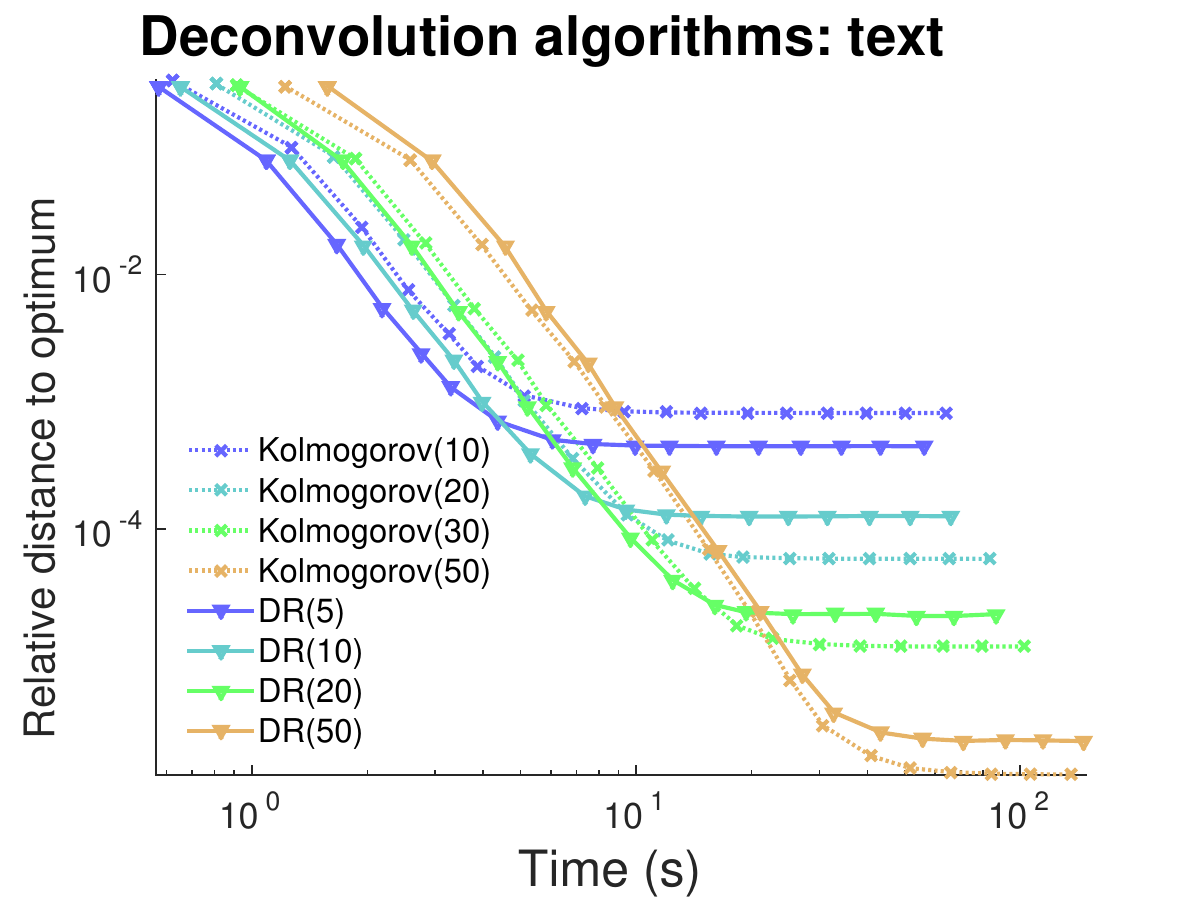}\\
  \hskip-25pt\includegraphics[width=.37\linewidth]{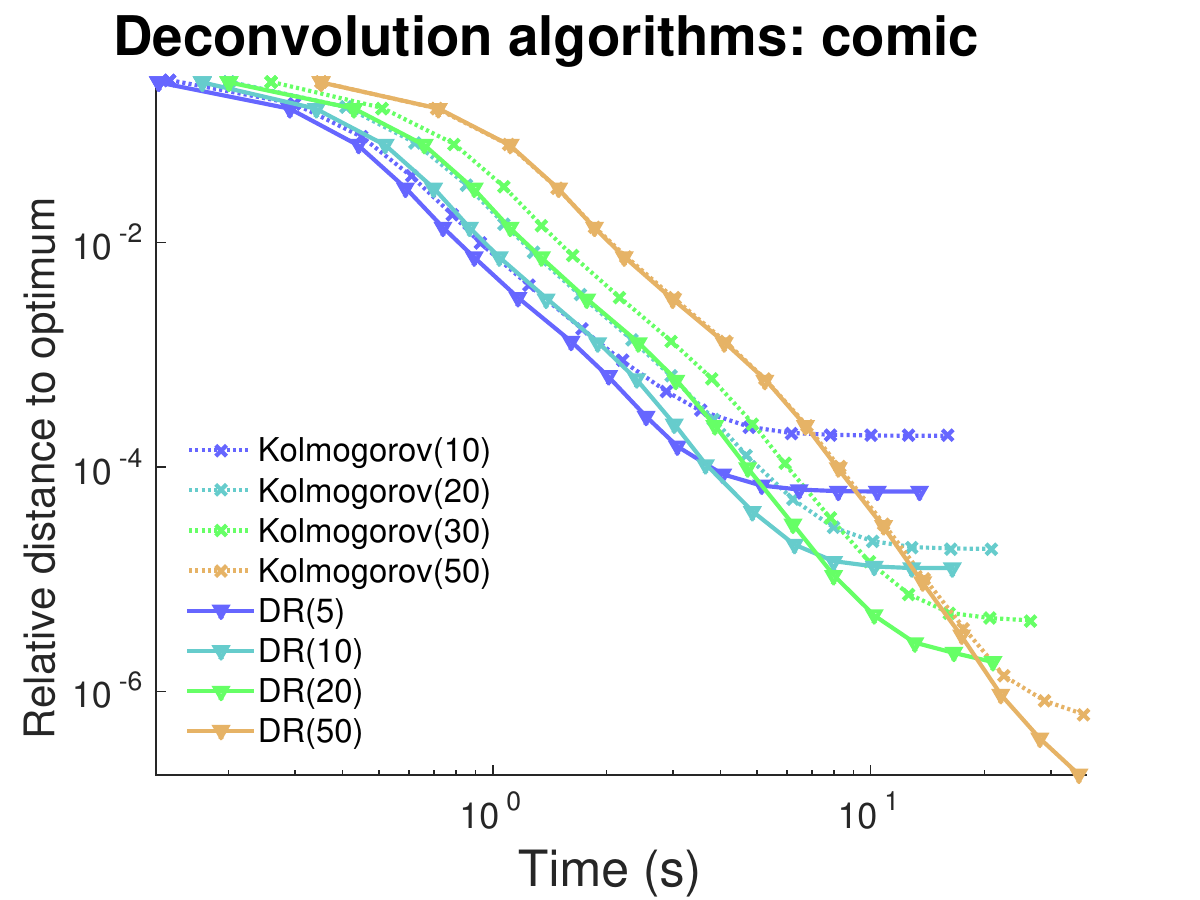}&\hskip-12pt
  \includegraphics[width=.37\linewidth]{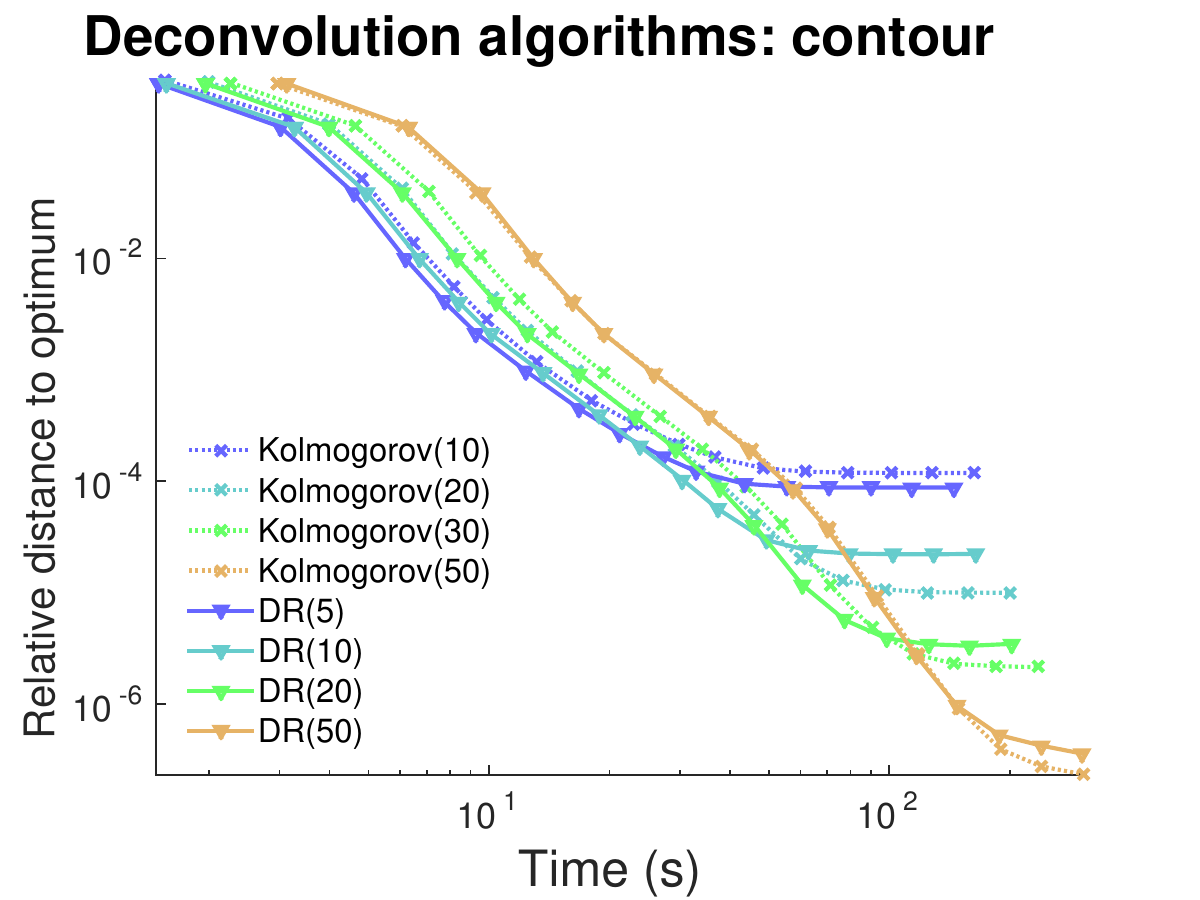}&\hskip-12pt
  \includegraphics[width=.37\linewidth]{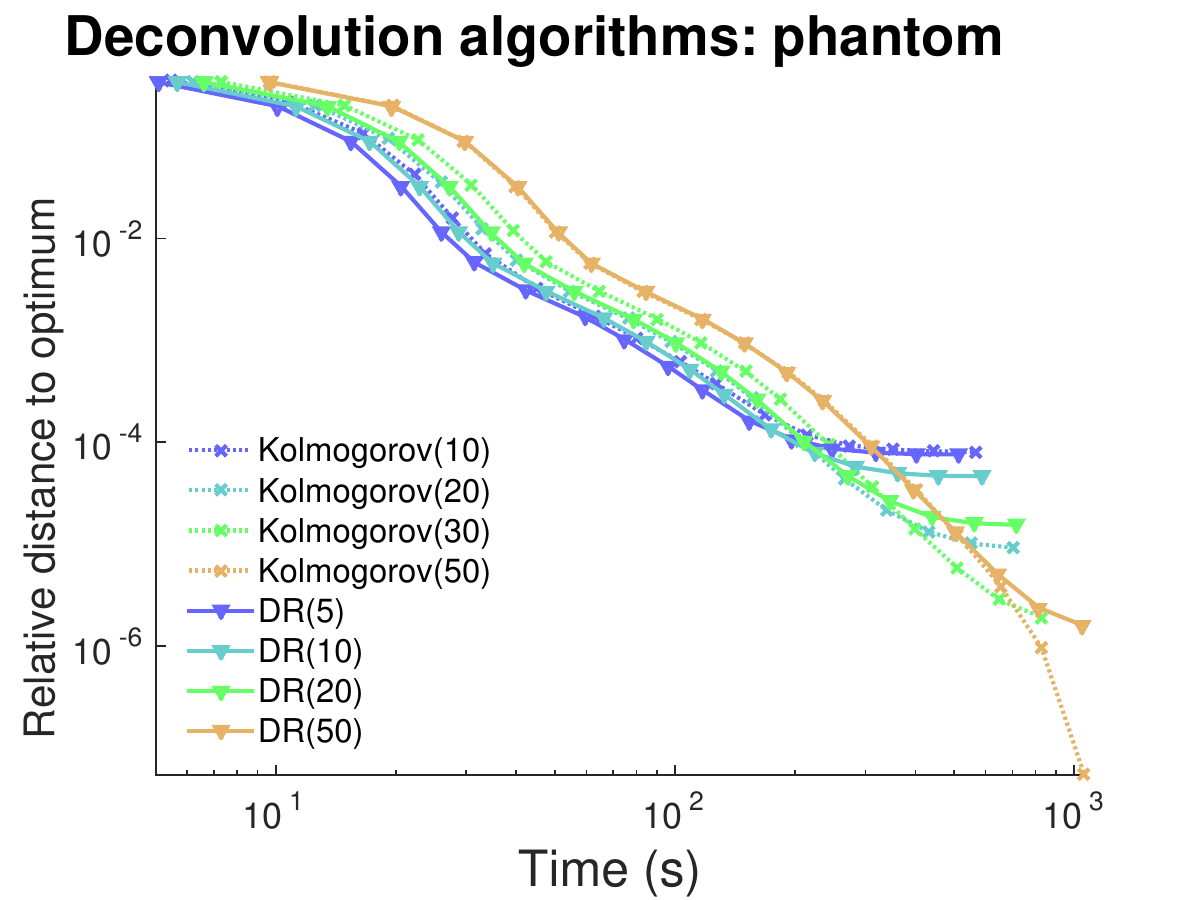}\\
  \end{tabular}
  \caption{Relative distance to optimum vs time of the deconvolution 2D-TV algorithms under comparison, for the different images considered in the experiments.}
  \label{fig:2DTVdeconv-obj}
\end{figure}

\begin{figure}[htbp]
  \centering
  \begin{tabular}{llll}
    \hskip-25pt\includegraphics[width=.37\linewidth]{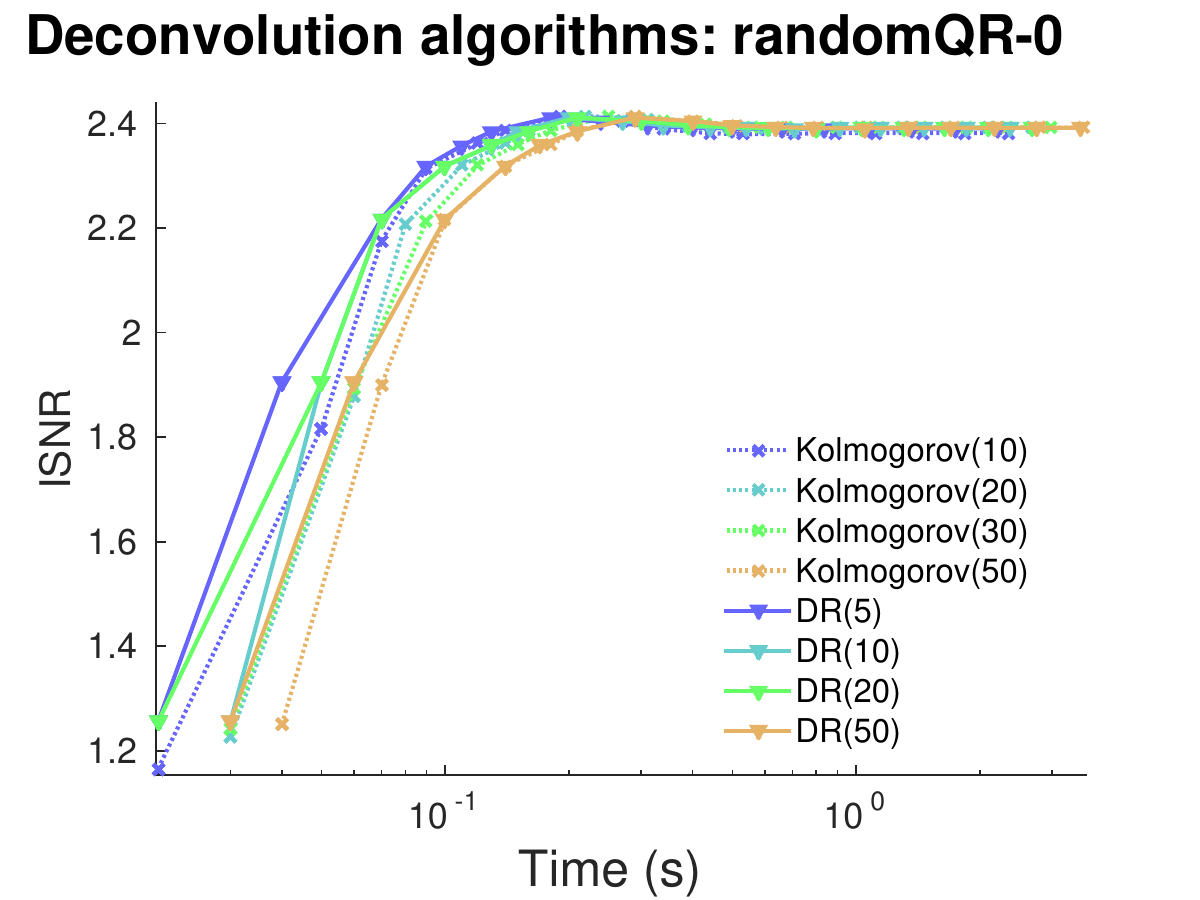}&\hskip-12pt
  \includegraphics[width=.37\linewidth]{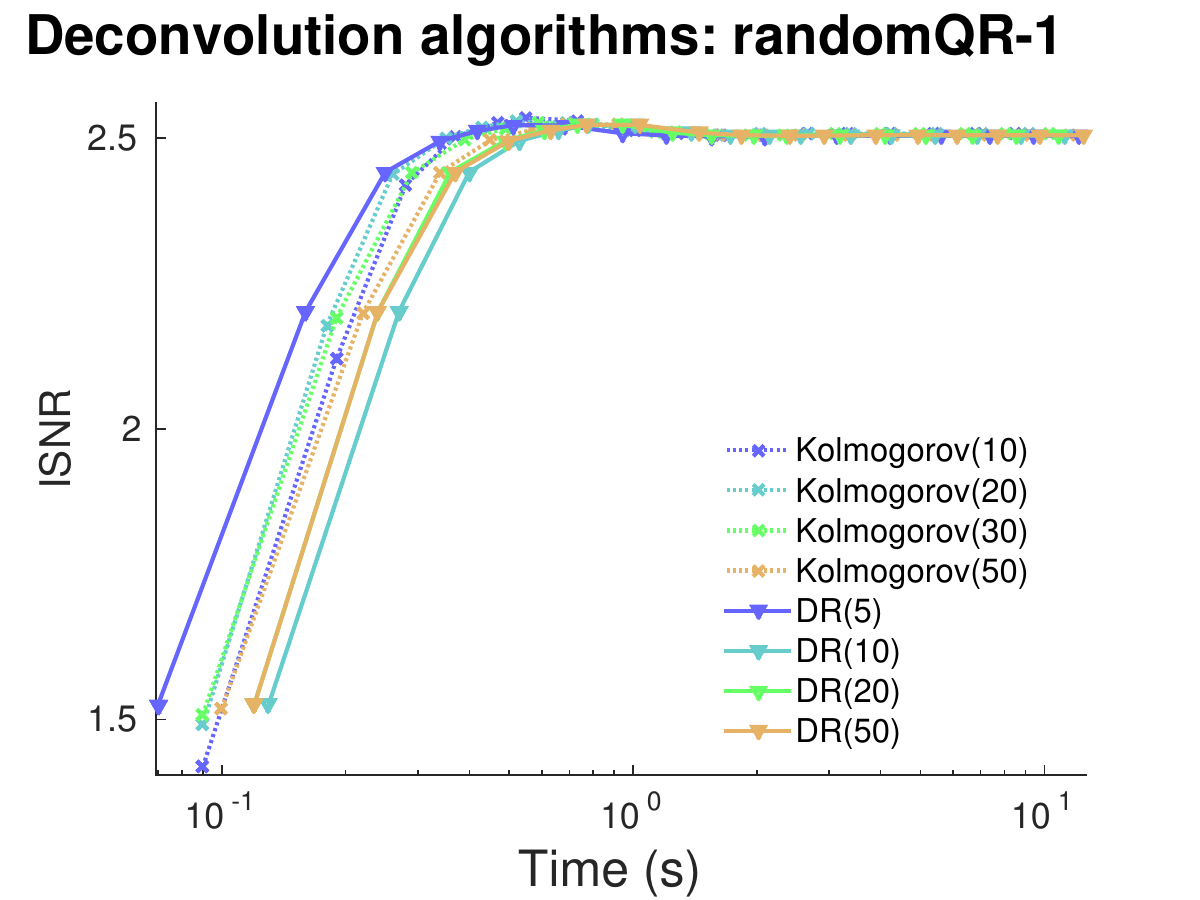}&\hskip-12pt
  \includegraphics[width=.37\linewidth]{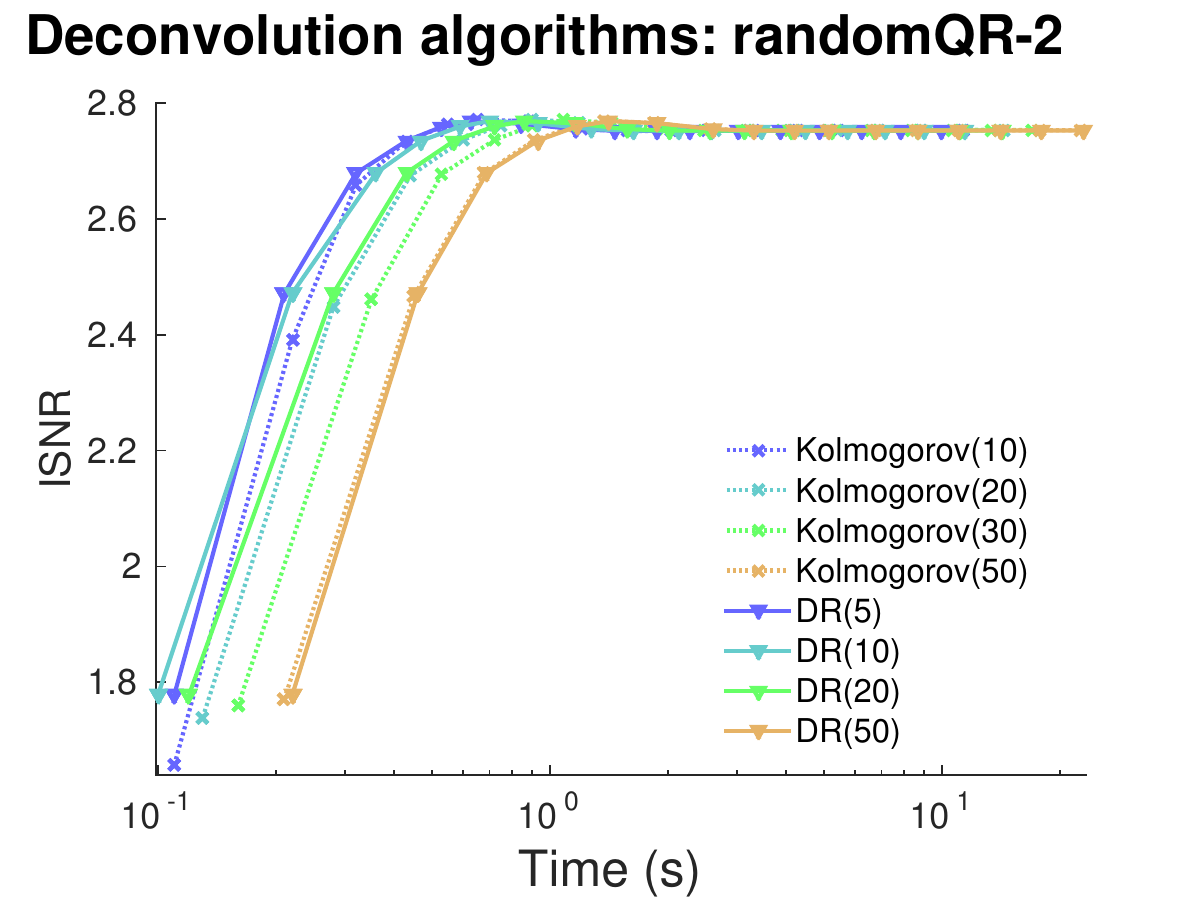}\\
  \hskip-25pt\includegraphics[width=.37\linewidth]{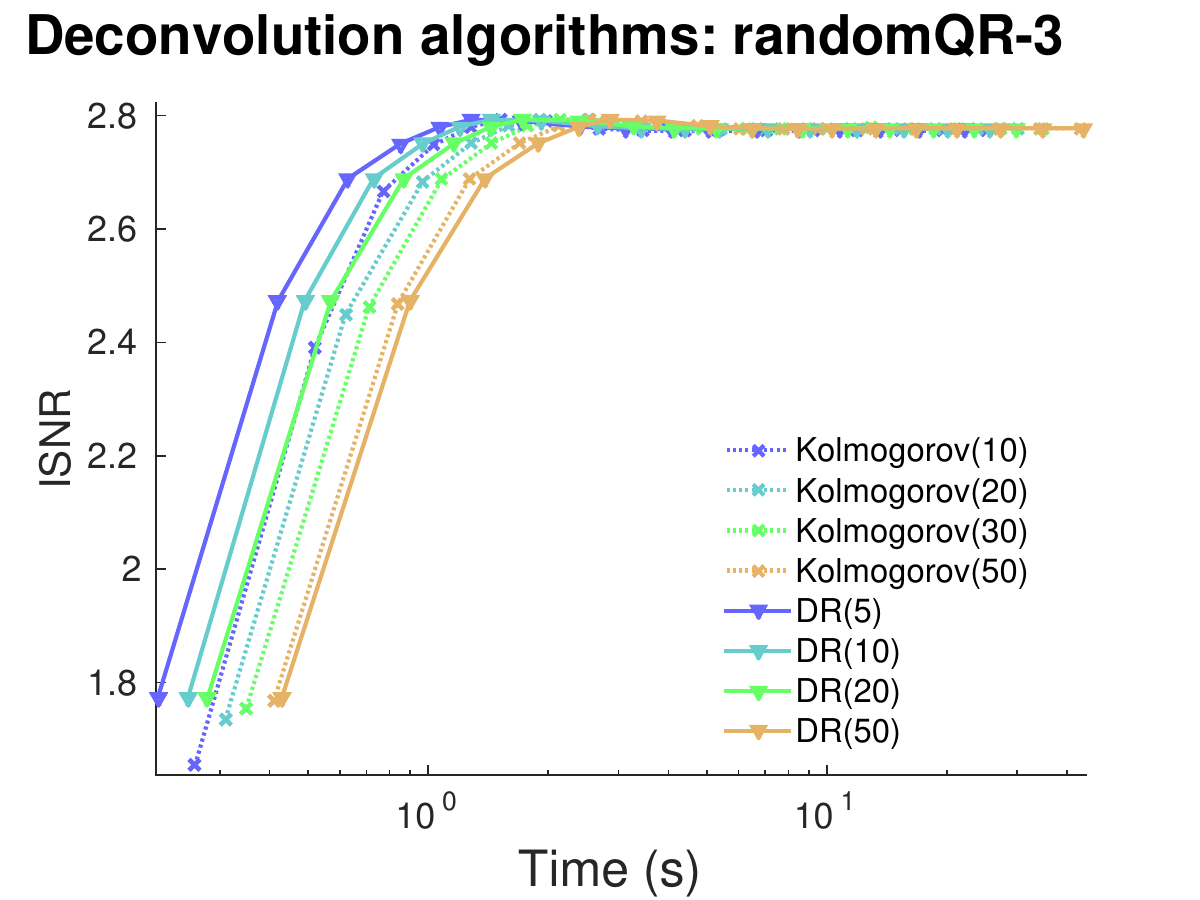}&\hskip-12pt
  \includegraphics[width=.37\linewidth]{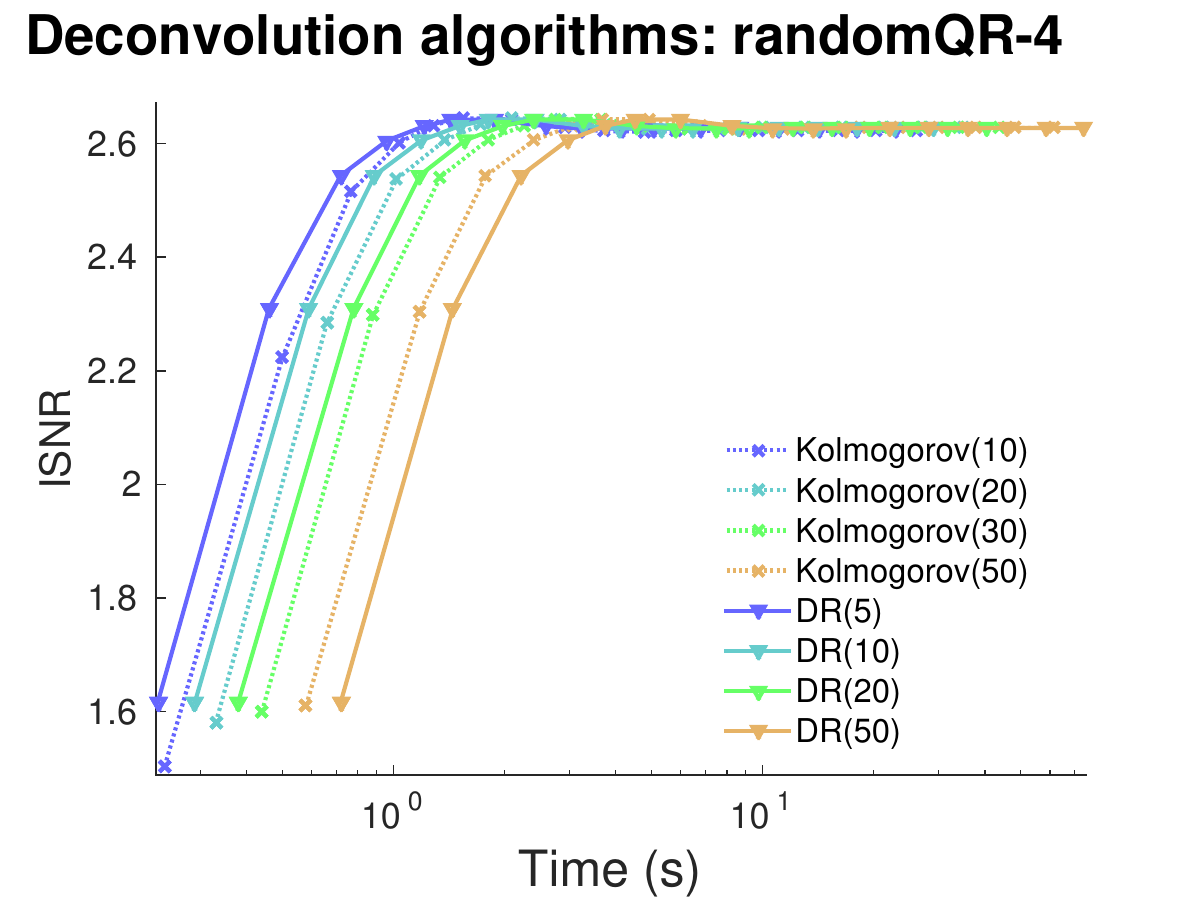}&\hskip-12pt
  \includegraphics[width=.37\linewidth]{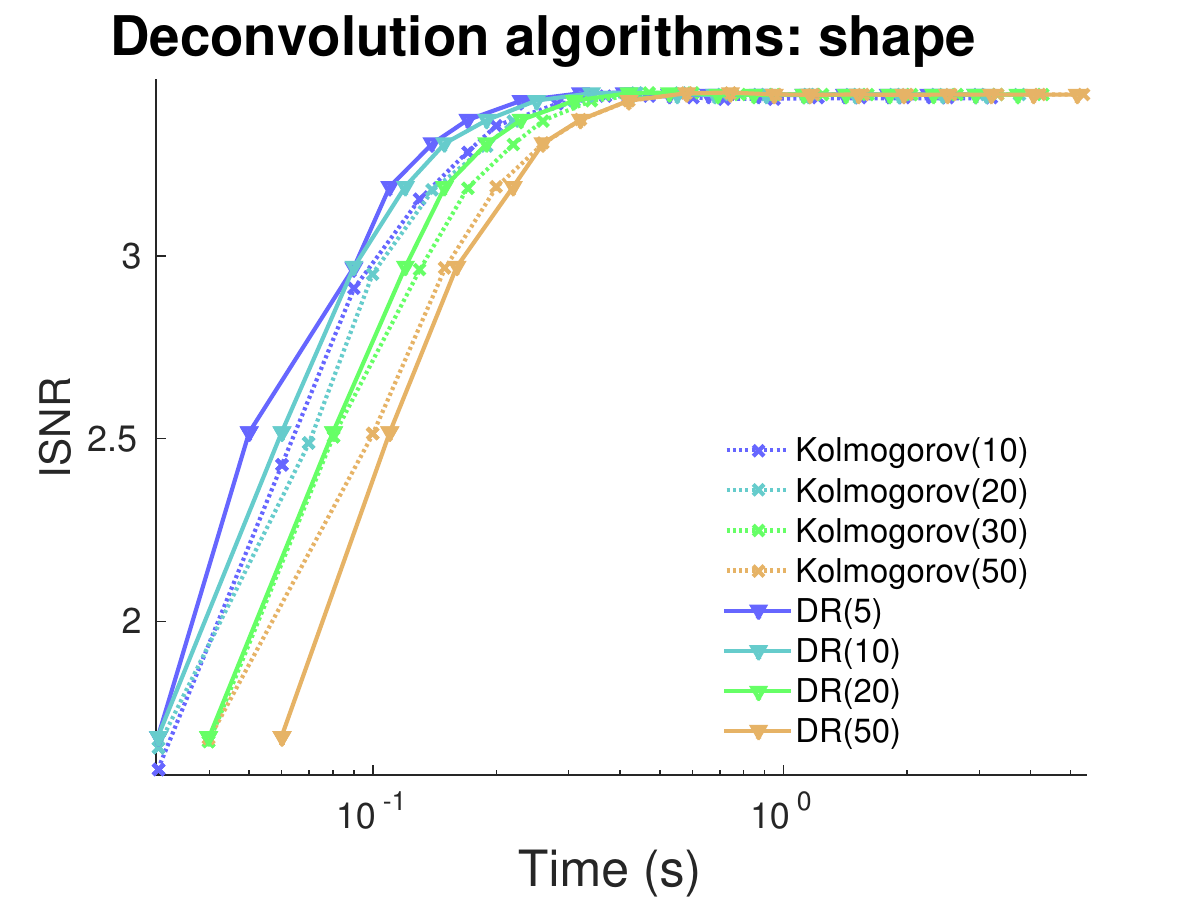}\\
  \hskip-25pt\includegraphics[width=.37\linewidth]{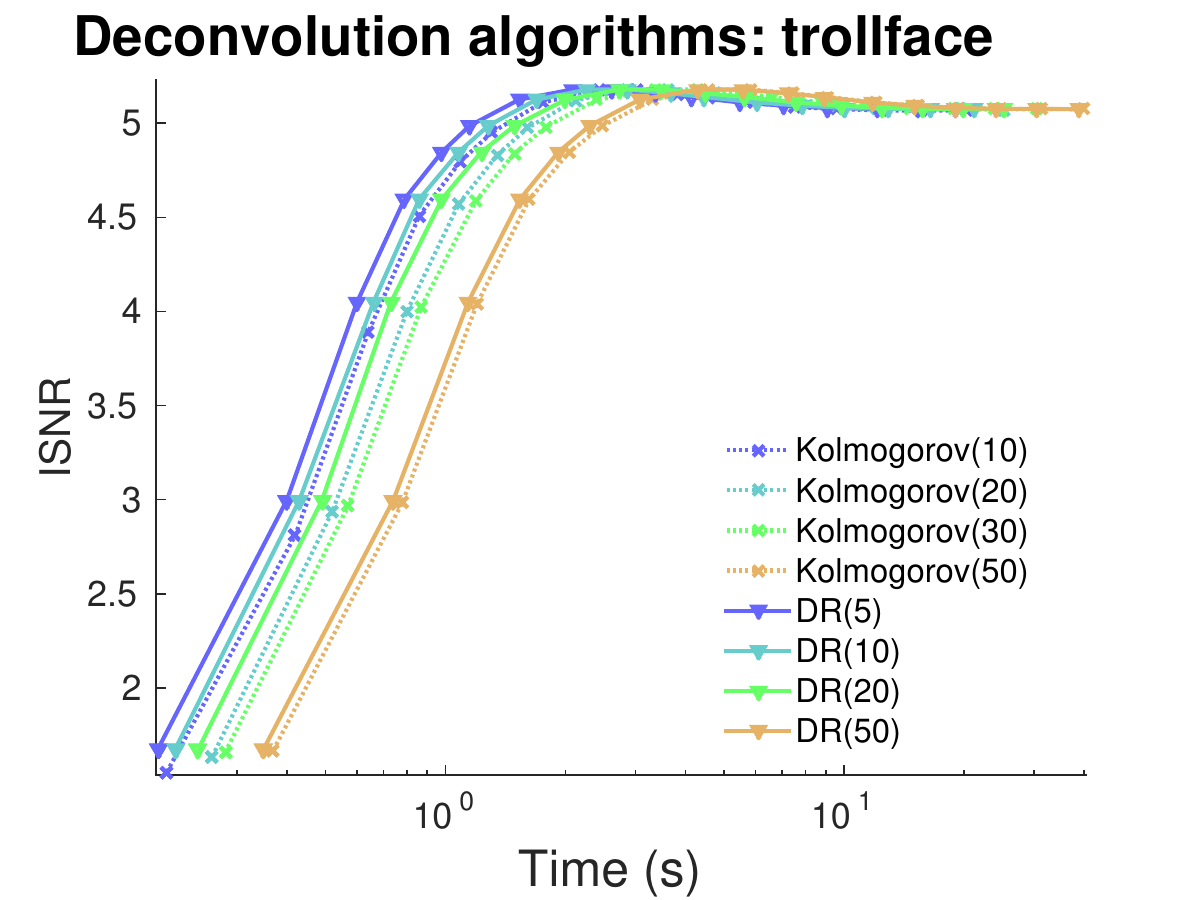}&\hskip-12pt
  \includegraphics[width=.37\linewidth]{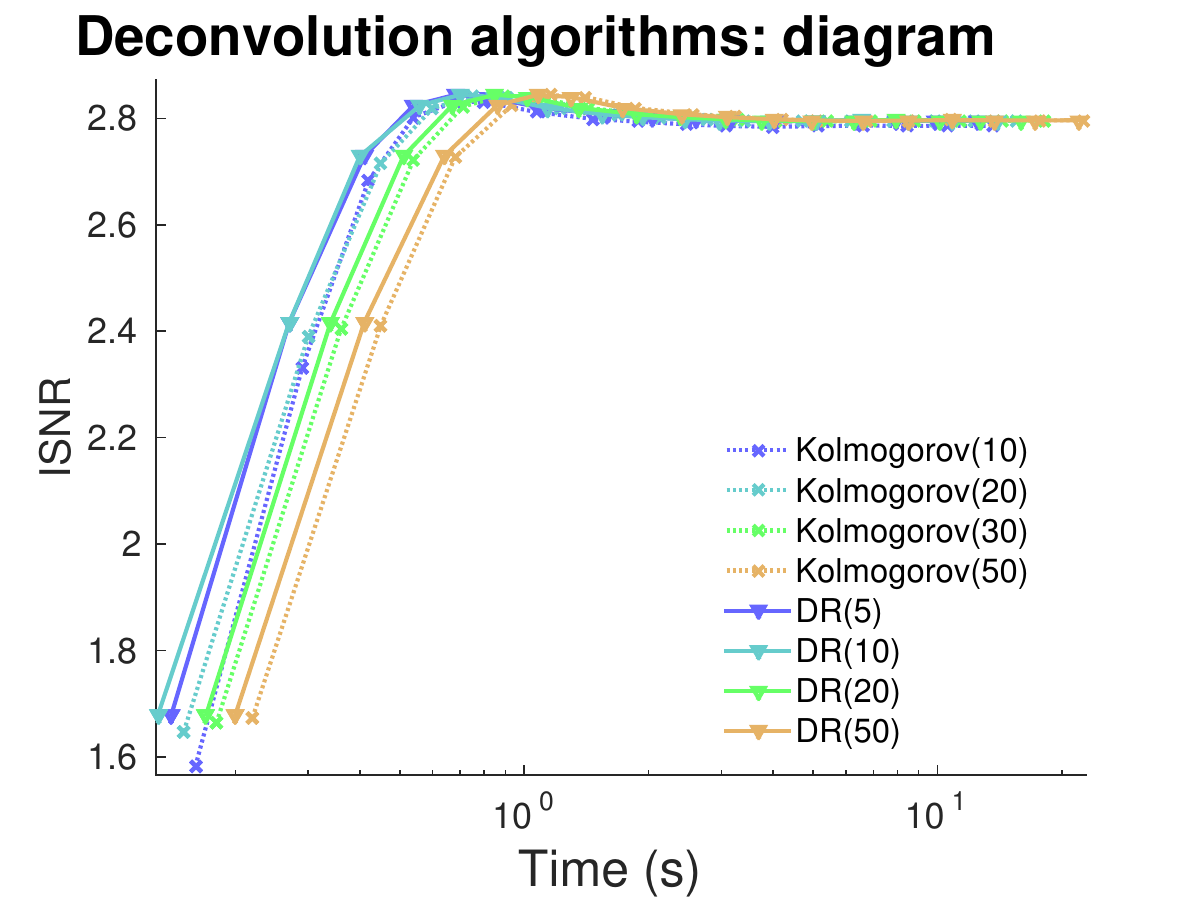}&\hskip-12pt
  \includegraphics[width=.37\linewidth]{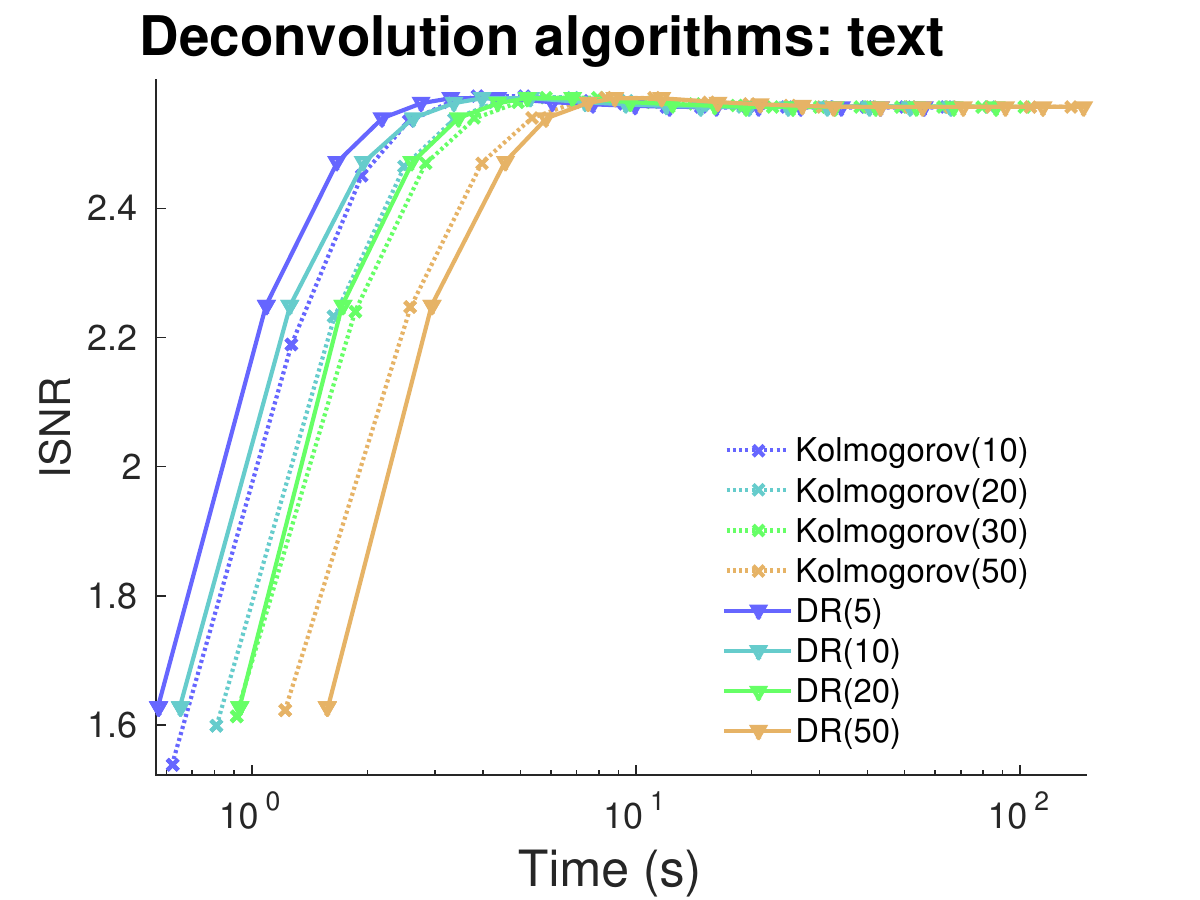}\\
  \hskip-25pt\includegraphics[width=.37\linewidth]{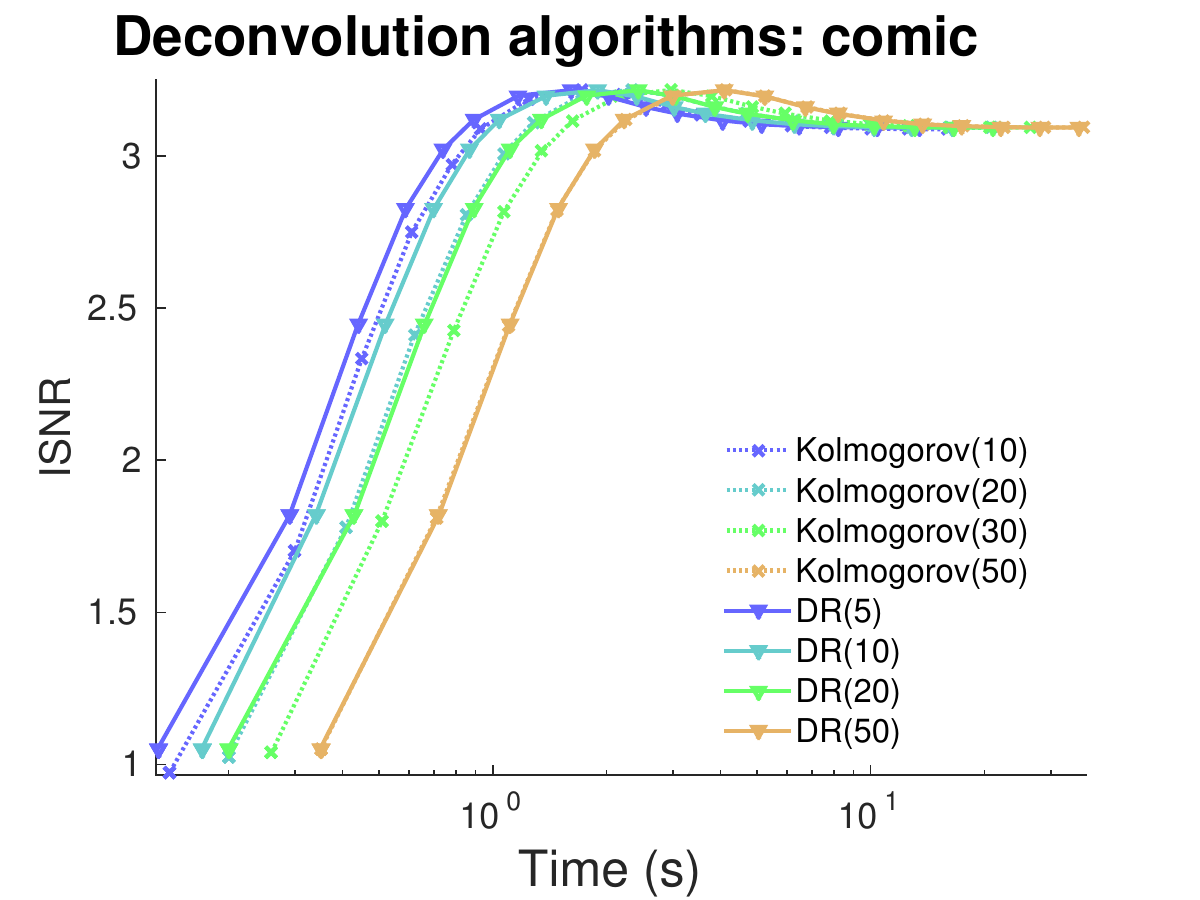}&\hskip-12pt
  \includegraphics[width=.37\linewidth]{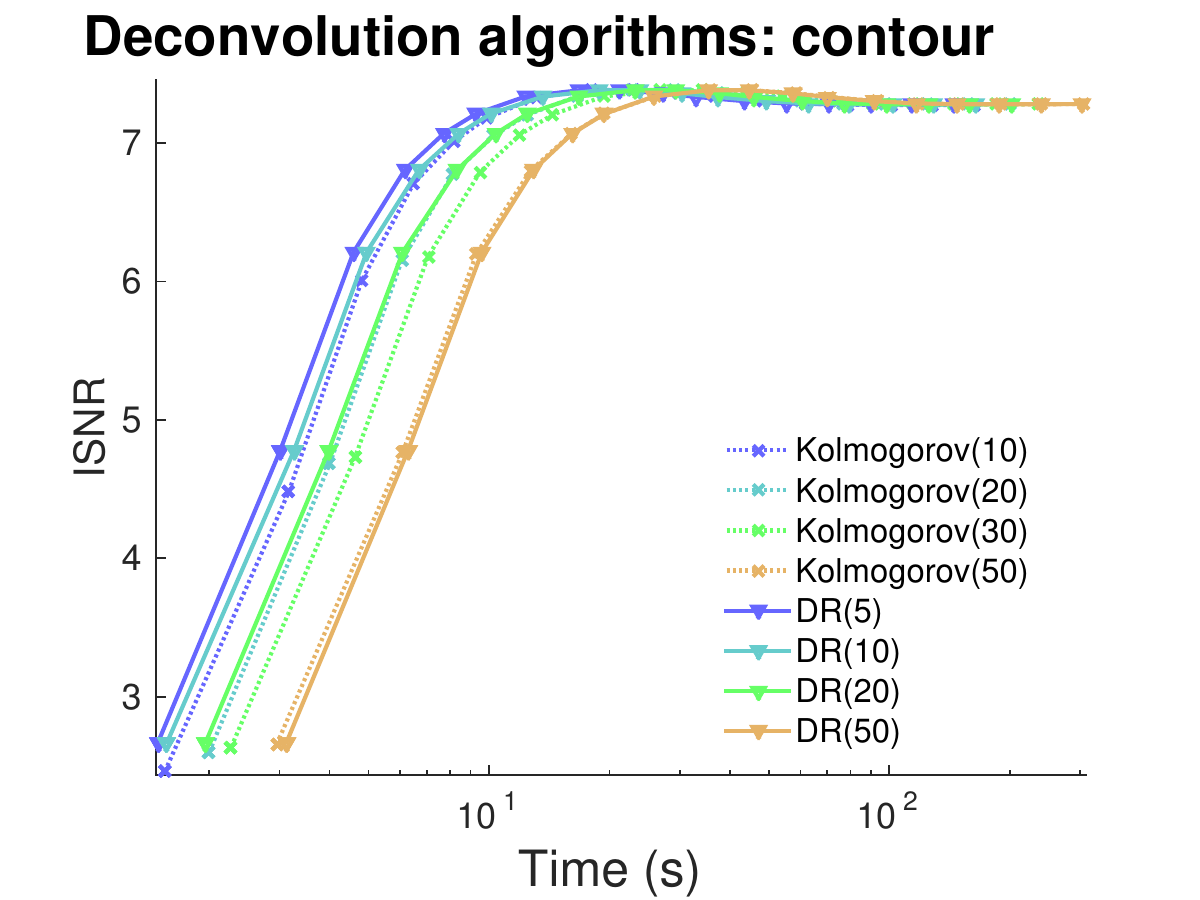}&\hskip-12pt
  \includegraphics[width=.37\linewidth]{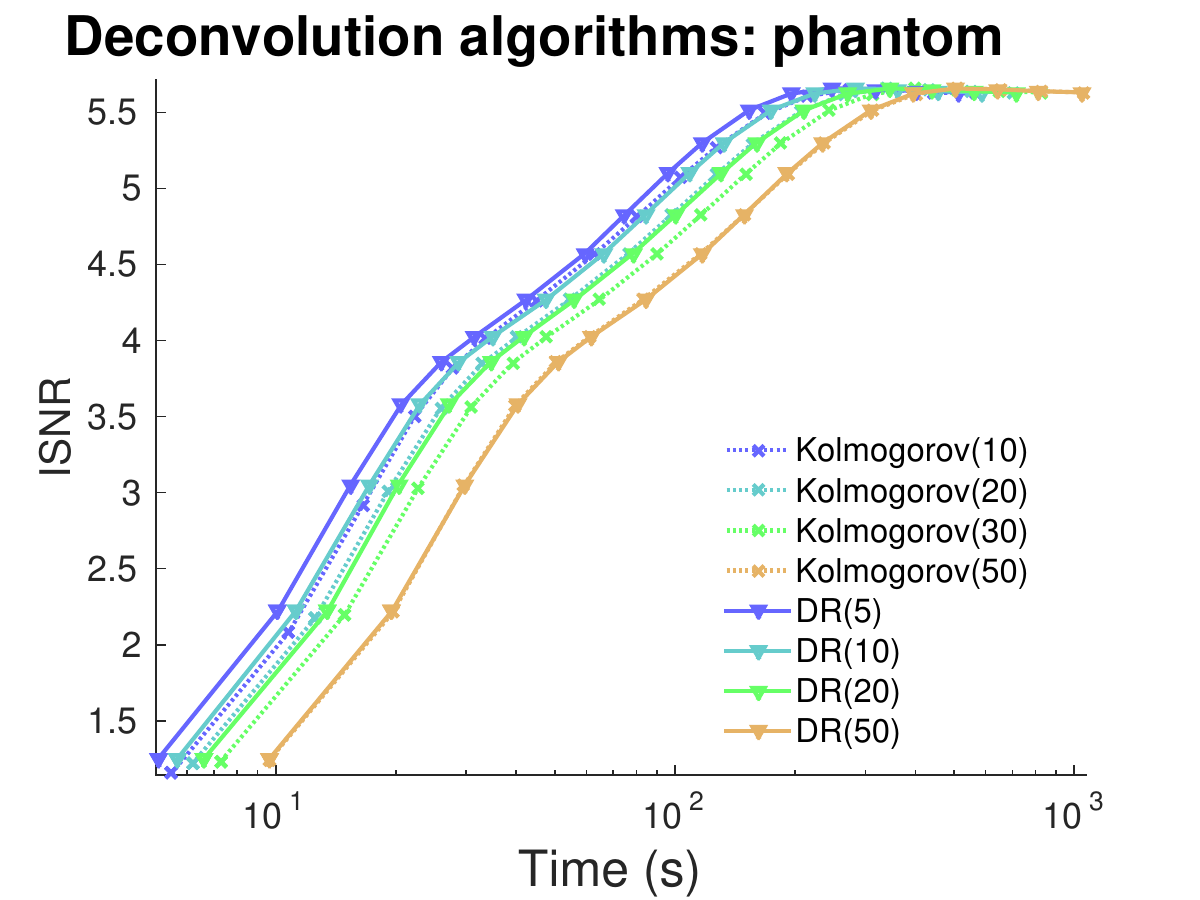}\\
  \end{tabular}
  \caption{Increased Signal to Noise Ratio (ISNR) vs time of the deconvolution 2D-TV algorithms under comparison, for the different images considered in the experiments.}
  \label{fig:2DTVdeconv-ISNR}
\end{figure}

For reference we also provide the resultant deconvoluted images as Figure~\ref{fig:imagesDeblurred}.

\subsubsection{2D Fused-Lasso Signal Approximator}
\label{sec:TV2FLSA}

The {\bf Fused--Lasso Signal Approximator} (FLSA)  \citep{frHaHoTi07} can be regarded as a particular case of Fused-Lasso where the input matrix $\ma$ is the identity matrix $\mi$, i.e.,
\begin{equation*}
   \nlmin_{\vx} \quad \half \enorm{\vx - \vy}^2 + \lambda_1 \norm{\vx}{1} + \lambda_2 \tvell_1^{\oned}(\vx).
\end{equation*}

This problem can be solved immediately using the methods presented in \S\ref{sec:FLasso}. A slightly less trivial problem is the one posed by the 2D variant of FLSA:
\begin{equation}
  \label{eq:29}
 \nlmin_{\mx} \quad \half\frob{\mx - \my}^2 + \lambda_1 \norm{\vect(\mx)}{1} + \lambda_2 \tvell_{1,1}^{\twod}(\mx).
\end{equation}
\citet{frHaHoTi07} used this model for denoising images where a large number of pixels are known to be completely black (intensity $0$), which aligns well with the structure imposed by the $\ell_1$ regularizer.

Akin to the 1D-case, 2D-FLSA~\eqref{eq:29} can also be solved 
by decomposing its computation into two prox operators~\citep{frHaHoTi07}; formally,
\begin{equation*}
 \prox_{\lambda_1 \norm{\cdot}{1} + \lambda_2 \tvell_{1,1}^{\twod}(\cdot)}(\my) = \prox_{\lambda_1 \norm{\cdot}{1}} \bigl( \prox_{\lambda_2 \tvell_{1,1}^{\twod}(\cdot)} (\my)\bigr).
\end{equation*}

Thus, to solve~\eqref{eq:29} we merely invoke one of the presented $\tvell_{1,1}^{\twod}$ prox operators and then apply soft-thresholding to the results. Since soft-thresholding is done in closed form, the performance of a 2D-FLSA solver depends only on its ability to compute $\tvell_{1,1}^{\twod}$-proximity efficiently. We can then safely claim that the results summarized in table \ref{tab:2DmethodsComparison} apply equivalently to 2D-FLSA, and so the proposed Douglas Rachford method performs best when reconstruction ISNR is the primary concern.


\subsection{Application of higher-dimensional TV}
\label{sec:appl.multi}
We now apply the presented multidimensional TV regularizer to anisotropic filtering for {\bf video denoising}. The extension to videos from images is natural. Say a video contains $f$ frames of size $n \times m$ pixels; this video can be viewed as a 3D-tensor $\tx \in \reals^{n \times m \times f}$, on which a 3D-TV based filter can be effected by
\begin{equation}
  \label{eq:30}
 \nlmin_{\tx} \quad \half\frob{\tx - \tu_0}^2 + \lambda \tvell_{p_1,p_2,p_3}^{\threed}(\tx),
\end{equation}
where $\tu_0$ is the observed noisy video, and $\tvell_{p_1,p_2,p_3}^{\threed} = \tvell^3_{\vp}$ with $\vp = [p_1, p_2, p_3]$. Application of the filter~\eqref{eq:30} is nothing but computation of the prox operator, which can be done using the Parallel-Proximal Dykstra (PPD) algorithm presented in Sec.~\ref{sec:proxMulti}.

\begin{table}
\begin{center}
\begin{tabular}{lcccc}
\hline
\abovespace\belowspace
      Sequence & Frame resolution & Number of frames & Total number of pixels \\
\hline
\abovespace 
{\it salesman} & 288 $\times$ 352 & 50 & 5 million \\
{\it coastguard} & 176 $\times$ 144 & 300 & 7.6 million \\
\belowspace
{\it bicycle} & 720 $\times$ 576 & 30 & 12.4 million
\\ \hline
\end{tabular}
\end{center}
\caption{\small Size details of video sequences used in the video denoising experiments.
\label{tab:videoSequences}
}
\end{table}

We apply this idea to the video sequences detailed in Table \ref{tab:videoSequences}. All of the sequences are made of grayscale pixels. Figure \ref{fig:videoClean} in the Appendix shows some of the  frames of the {\it salesman} sequence. We noise every frame of these sequences by applying gaussian noise with zero mean and variance $0.01$, using Matlab's {\it imnoise} function. Then we solve problem \ref{eq:30} for each sequence, adjusting the regularization value so as to maximize ISNR of the reconstructed signal. We test the following algorithms, which have been previously applied in the literature for solving 3D-TV, with the only exception Parallel Proximal Dykstra:
\begin{itemize}
  \setlength{\itemsep}{0pt}
 \item Parallel Proximal Dykstra (§~\ref{sec:PD}).
 \item Yang's method, which is based on ADMM (§~\ref{sec:PD})
 \item The maximum flow approach by~\citet{goldfarb2009parametric}, which features an implementation for 3D grids, thus solving a discrete-valued version of 3D-TV.
\end{itemize}
For both PPD and ADMM we again make use of linearized taut-string 1D TV solver. We must also point out that other image denoising methods seem amenable for extension into the multidimensional setting, such as Condat's and Chambolle-Pock methods. However in the light of our image denoising results we do not deem them as good choices for this problem. A more reasonable choice might be to extend Split-Bregman to multiple dimensions, but such an extension has not been implemented or proposed as far as we know. We would also like to note that we have considered extending the Douglas Rachford method to a multidimensional setting, however such task is complex and thus we decided to focus on Parallel Proximal Dykstra.

\begin{figure}[htbp]
  \centering
  \begin{tabular}{llll}
    \hskip-25pt\includegraphics[width=.37\linewidth]{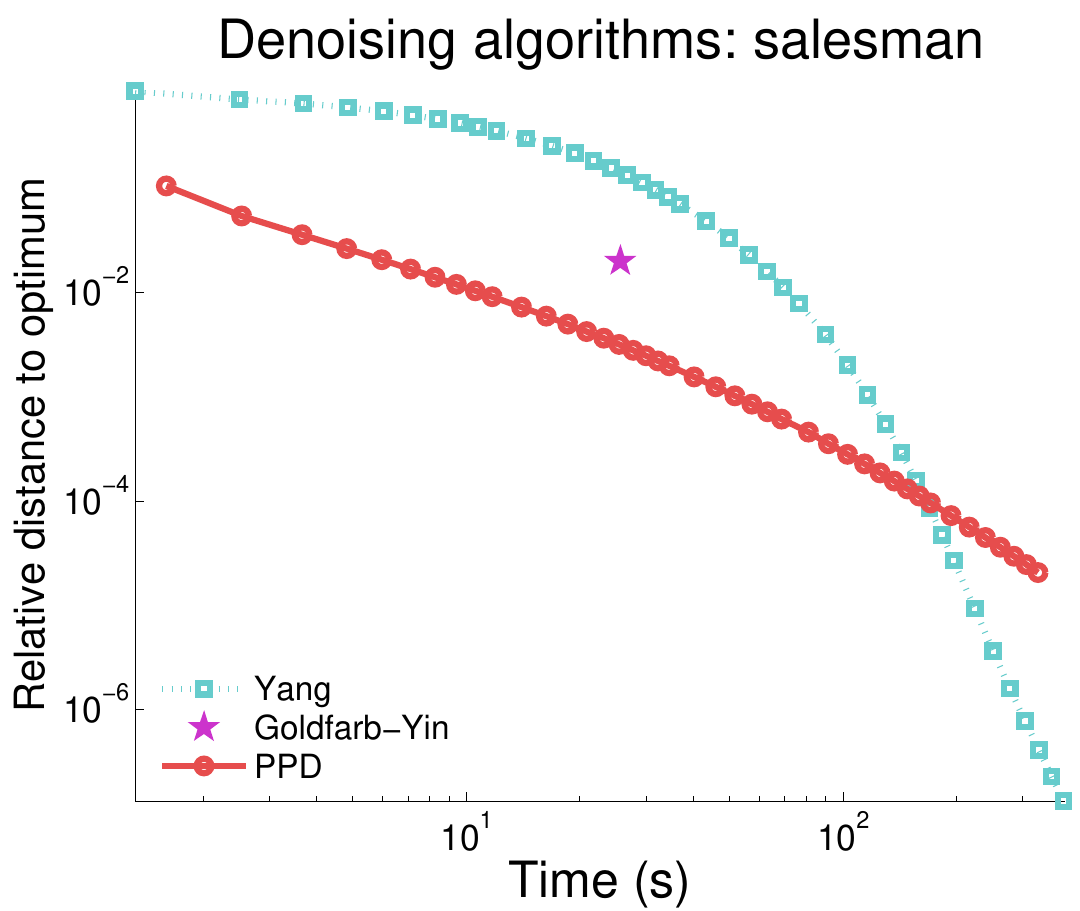}&\hskip-12pt
  \includegraphics[width=.37\linewidth]{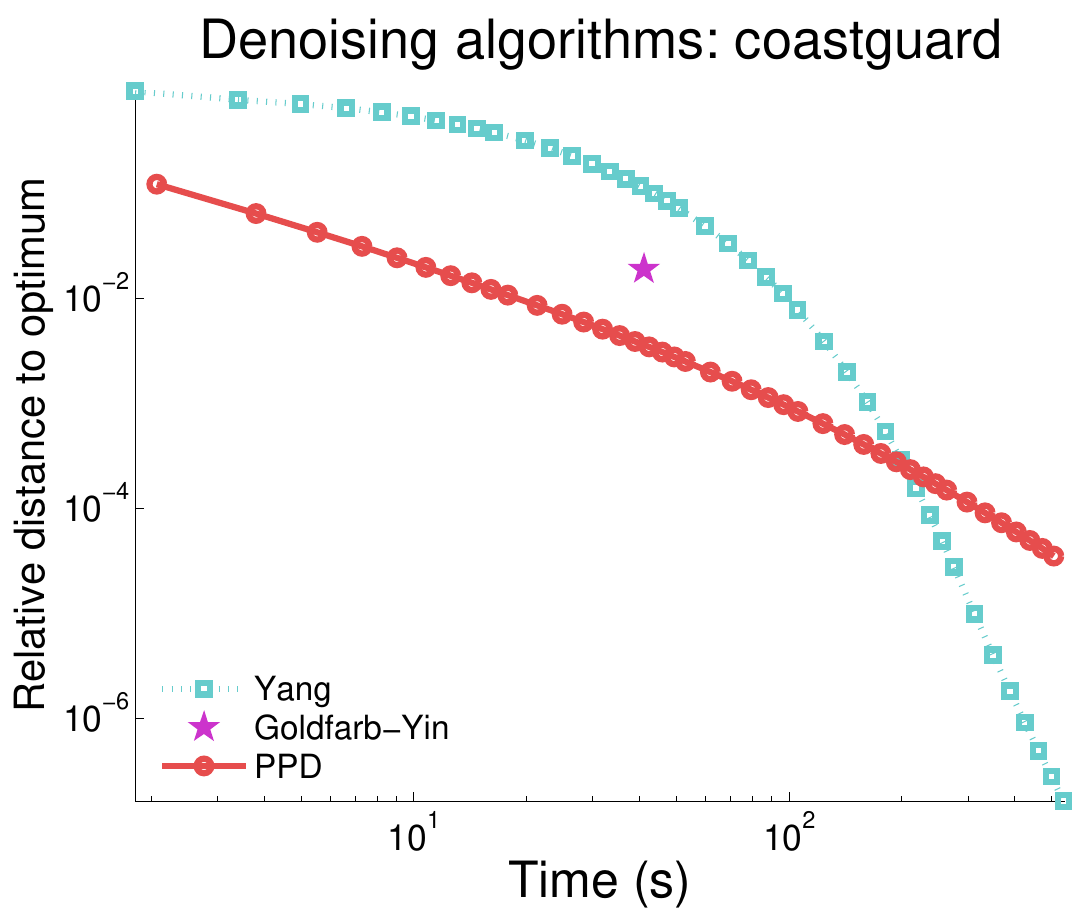}&\hskip-12pt
  \includegraphics[width=.37\linewidth]{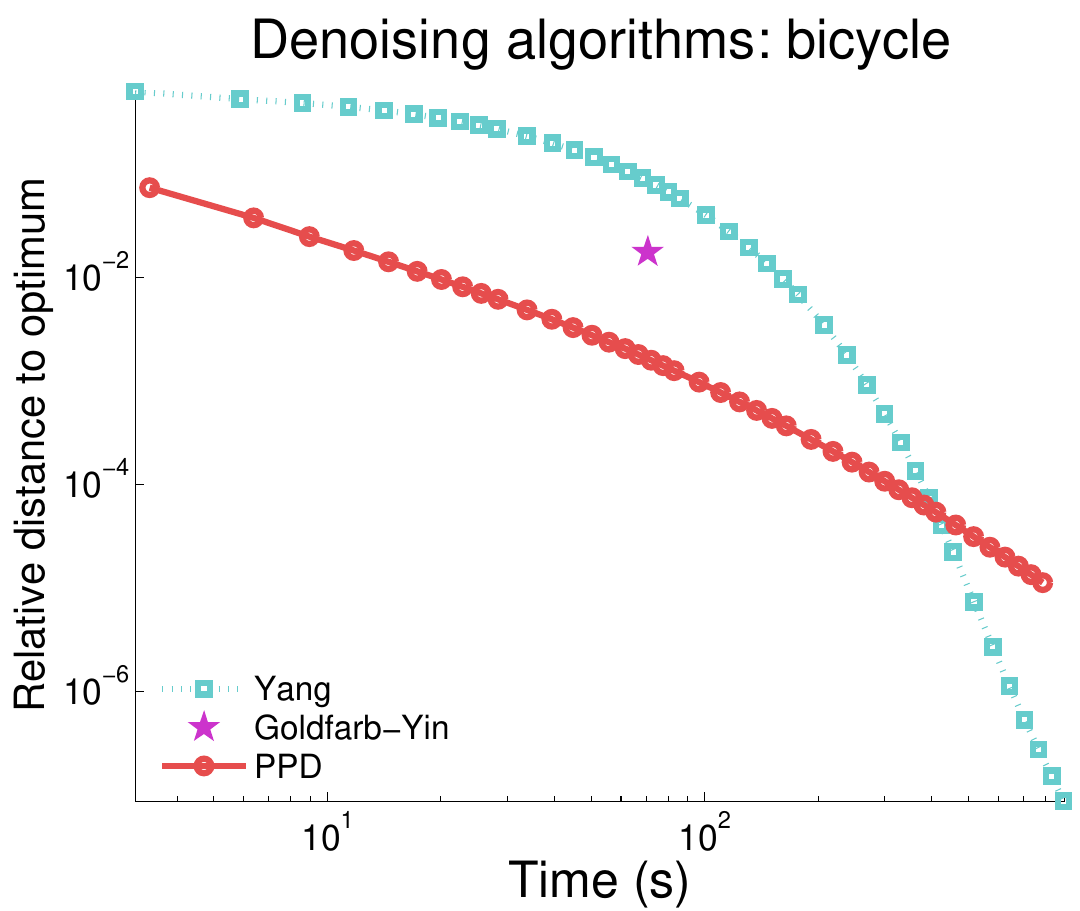}
  \end{tabular}
  \caption{Relative distance to optimum vs time of the denoising 3D-TV algorithms under comparison, for the different video sequences considered in the experiments.}
  \label{fig:3DTV-obj}
\end{figure}

\begin{figure}[htbp]
  \centering
  \begin{tabular}{llll}
    \hskip-25pt\includegraphics[width=.37\linewidth]{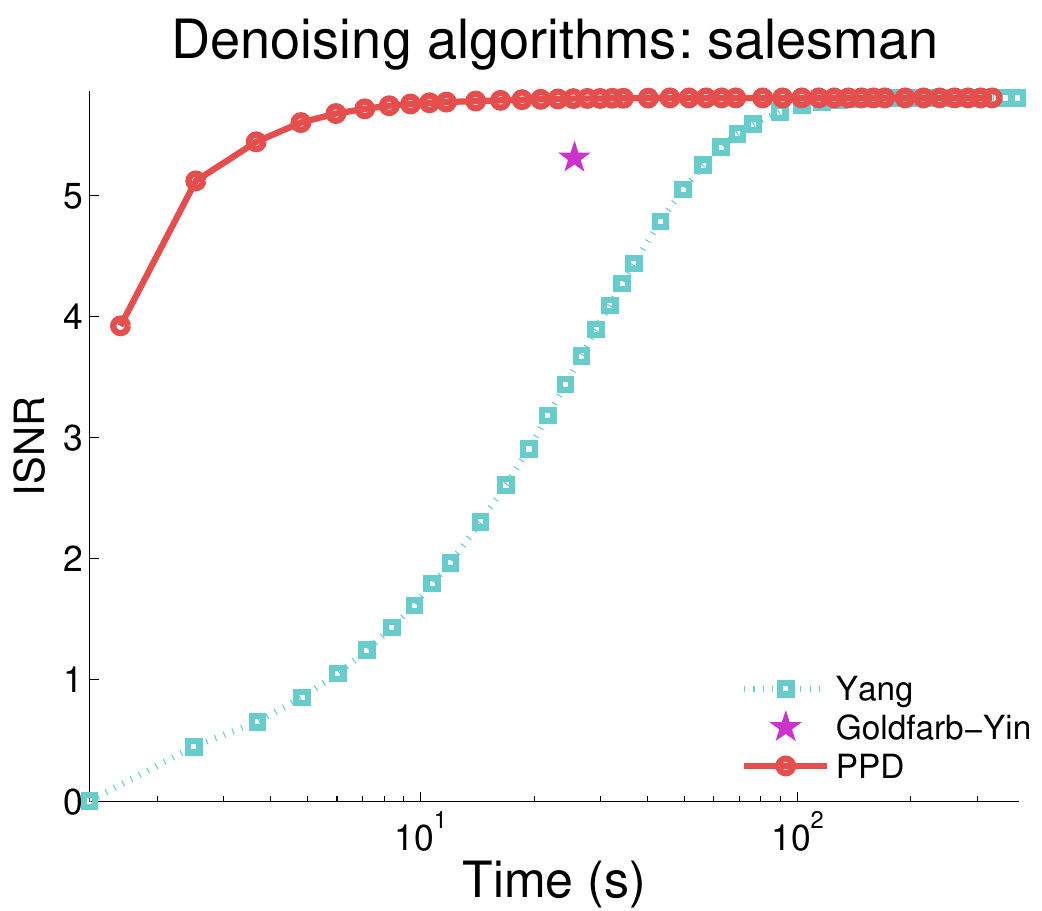}&\hskip-12pt
  \includegraphics[width=.37\linewidth]{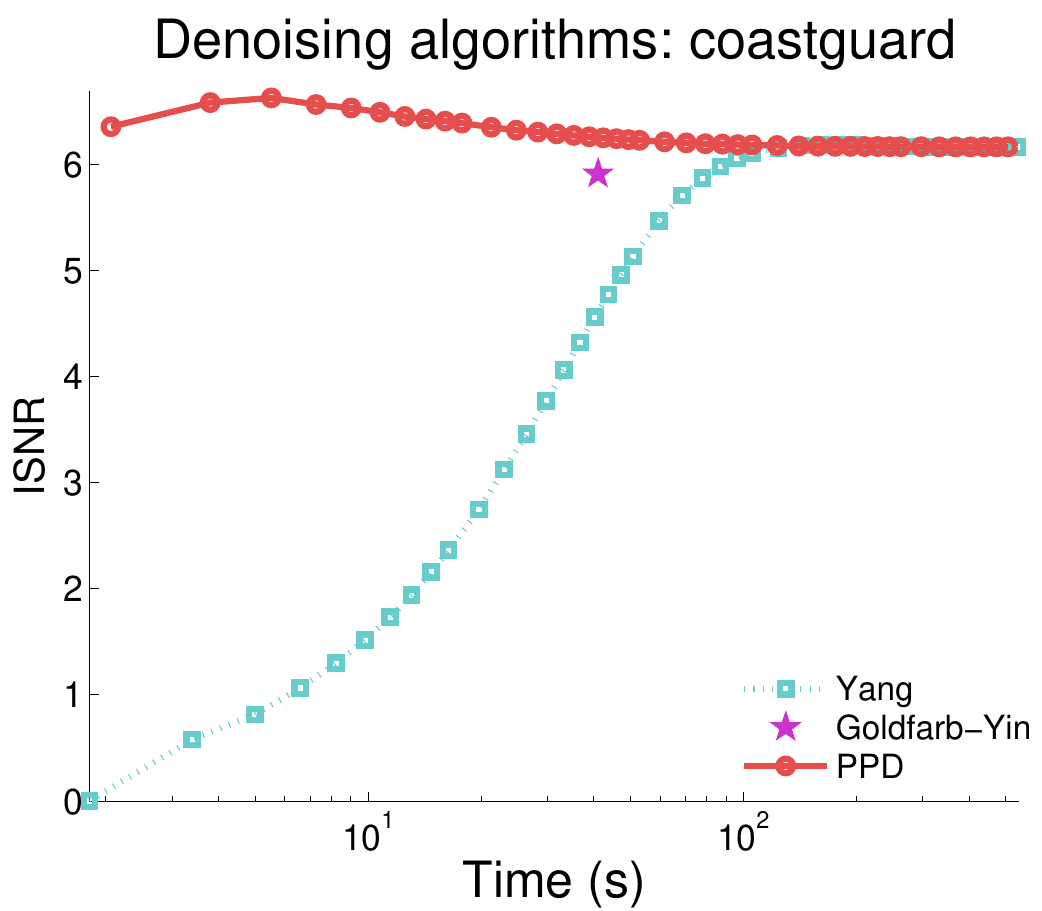}&\hskip-12pt
  \includegraphics[width=.37\linewidth]{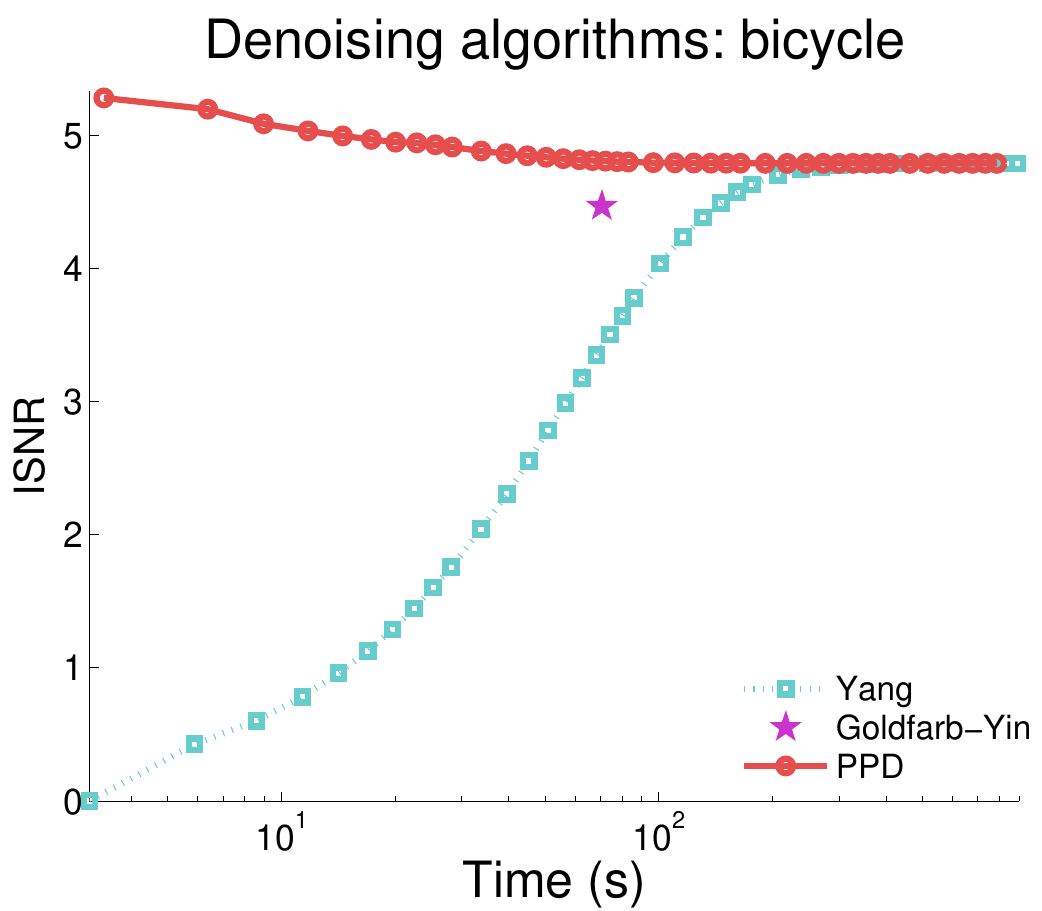}
  \end{tabular}
  \vspace{-0.5cm}
  \caption{Increased Signal to Noise Ratio (ISNR) vs time of the denoising 3D-TV algorithms under comparison, for the different video sequences considered in the experiments.}
  \label{fig:3DTV-ISNR}
\end{figure}

Similarly to our previous image denoising experiments, we ran the algorithms under comparison for each video sequence and measured its ISNR and relative distance to the optimal objective value of the current solution at each iteration through their execution. Again the exception is the Goldfarb-Yin method, which is non--iterative and so we only report the time required for its termination. The optimal objective value was estimated by running all methods for a very large number of iterations and taking the minimum value of them all. This produced the plots shown in Figures \ref{fig:3DTV-obj}--\ref{fig:3DTV-ISNR}. From them the following observations are of relevance:
\begin{itemize}
  \setlength{\itemsep}{1pt}
 \item Following the pattern observed in the image denoising experiments, ADMM (Yang's method) is best suited for finding very accurate solutions.
 \item The method by Goldfarb and Yin again provides suboptimal solutions, due to the discrete approximation it uses.
 \item Parallel Proximal Dykstra is the fastest to achieve a mid-quality solution.
 \item Intermediate solutions prior to convergence of the PPD run result in better ISNR values for the {\it coastguard} and {\it bicycle} datasets. This hints that the denoising model used in this experiment may not be optimal for these kind of signals; indeed, more advanced denoising models abound in the signal processing literature. Hence we do not claim novel results in terms of ISNR quality, but just in solving this classic denoising model more efficiently.
\end{itemize}

The ISNR plots in Figure \ref{fig:3DTV-ISNR} also show how both Parallel Proximal Dykstra and ADMM (Yang's method) converge to equivalent solutions in practice. Therefore, for the purpose of video denoising PPD seems to be the best choice, unless for some reason a high degree of accuracy is required, for which ADMM should be preferred.

%

\subsection*{Acknowledgments}
ÁB acknowledges partial financial support from Spain’s grants TIN2010-21575-C02-01, TIN2013-42351-P, S2013/ICE-2845 CASI-CAM-CM, TIN2016-76406-P and TIN2015-70308-REDT (MINECO/FEDER EU) during the long period leading to the writing of this manuscript. We thank R. Tibshirani for bringing~\citep{dpTV} to our attention, and S. Jegelka for alerting us to the importance of weighted total-variation problems.

\bibliographystyle{abbrvnat} 
\setlength{\bibsep}{1pt}


\appendix

\section{Mathematical background}
\label{app:mathbackground}
We begin by recalling a few basic ideas from convex analysis; we recommend the recent book~\citep{bauCom} for more details.

Let $\Xc \subset \reals^n$ be any set. A function $r : \Xc \to \reals \cup \set{-\infty,+\infty}$ is called \emph{lower semicontinuous} if for every $\vx \in \Xc$ and a sequence $(\vx_k)$ that converges to $\vx$, it holds that 
\begin{equation}
\label{eq.19}
  \vx_k \to \vx \implies r(\vx) \le \lim\inf\nolimits_k r(\vx_k).
\end{equation}
The set of proper lsc convex functions on $\Xc$ is denoted by $\Gamma_0(\Xc)$ (such functions are also called \emph{closed convex functions}). The \emph{indicator function} of a set $C$ is  defined as
\begin{equation}
  \label{eq.3}
  \ind_{C} : \Xc \to [0,\infty] : \vx \mapsto 
  \begin{cases}
    0,& \text{if}\ \vx \in C;\\
    \infty,& \text{if}\ \vx\not\in C,
  \end{cases}
\end{equation}
which is lsc if and only if $C$ is closed. 

The \emph{convex conjugate} of $r$ is given by $r^*(\vz):=\sup_{\vx \in \dom r}\ \ip{\vx}{\vz} - r(\vx)$, and a particularly important example is the Fenchel conjugate of a norm $\norm{\cdot}{}$
\begin{equation}
  \label{eq:7}
  \text{if}\ r=\norm{\cdot}{},\quad\text{then}\ \ r^* = \ind_{\norm{\cdot}{*} \le 1},
\end{equation}
where the norm $\norm{\cdot}{*}$ is dual to $\norm{\cdot}{}$. Let $r$ and $h$ be proper convex functions. The \emph{infimal convolution} of $r$ with $h$ is the convex function given by $(r\infconv h)(\vx) := \inf_{\vy \in \Xc}\bigl(r(\vy)+ h(\vx-\vy))$.
For our purposes, the most important special case is infimal convolution of a convex function with the squared euclidean norm, which yields the \emph{Moreau envelope}~\citep{moreau62}.
\begin{proposition}
  Let $r \in \Gamma_0(\Xc)$ and let $\gamma > 0$. The \emph{Moreau envelope} of $r$ indexed by $\gamma$ is
  \begin{equation}
    \label{eq.5}
    E_r^\gamma(\cdot):= r\infconv (\tfrac{1}{2\gamma}\enorm{\cdot}^2).
  \end{equation}
  The Moreau envelope~\eqref{eq.5} is convex, real-valued, and continuous.
\end{proposition}
\begin{proof}
  See e.g.~\citep[Prop.~12.15]{bauCom}. 
\end{proof}

\noindent Using the Moreau envelope~\eqref{eq.5}, we now formally introduce prox operators.
\begin{defn}[Prox operator]
  Let $r \in \Gamma_0(\Xc)$, and let $\vy \in \Xc$. Then $\prox_r \vy$ is the unique point in $\Xc$ that satisfies $E_r^1(\vy) = \min_{\vx \in \Xc}(r(\vx) + \half\enorm{\vx-\vy}^2)$, i.e.,
  \begin{equation}
    \label{eq.6}
    \prox_r(\vy) := \argmin_{\vx \in \Xc} r(\vx) + \half\enorm{\vx-\vy}^2,
  \end{equation}
  and the nonlinear map $\prox_r : \Xc \to \Xc$ is called the \emph{prox operator} of $r$.
\end{defn}

Sometimes the Fenchel conjugate $r^*$ is easier to use than $r$; similarly, sometimes the operator $\prox_{r^*}$ is easier to compute than $\prox_r$. The result below shows the connection.
\begin{proposition}[Moreau decomposition]
  \label{prop.decomp}
  Let $r \in \Gamma_0(\Xc)$, $\gamma > 0$, and $\vy \in \Xc$. Then,
  \begin{equation}
    \label{eq.10}
    \vy = \prox_{\gamma r}\vy + \gamma\prox_{r^*/\gamma}(\gamma^{-1}\vy).
  \end{equation}
\end{proposition}
\begin{proof}
  A brief exercise; see e.g.,~\citep[Thm.~14.3]{bauCom}.
\end{proof}


\noindent This decomposition provides the necessary tools to exploit useful primal--dual relations. For the sake of clarity we also present an additional result regarding a particular primal-dual relation that plays a key role in our algorithms.
\begin{proposition}
  \label{prop.pd}
  Let $f \in \Gamma_0(\Xc)$ and $r \in \Gamma_0(\Zc)$. The problems below form a primal-dual pair.
  \begin{align}
    \label{eq.12}
    \inf_{\vx \in \Xc}\quad &f(\vx) + r(\bm{Bx})\quad \text{s.t.}\ \ \bm{Bx} \in \Zc\\
    \label{eq.13}
    \inf_{\vu \in \Zc}\quad &f^*(-\bm{B}^T\vu) + r^*(\vu).
  \end{align}
\end{proposition}
\begin{proof}
  Introduce an extra variable $\vz = \mb\vx$, 
  so that the 
  dual function is
  \begin{align*}
    g(\vu) &=\inf_{\vx \in \Xc}\ f(\vx) + \vu^T\mb \vx + \inf_{\vz \in \Zc}\ r(\vz) - \vu^T\vz,
  \end{align*}
  which upon rewriting using Fenchel conjugates yields~\eqref{eq.13}. 
\end{proof}
  
Notions on submodular optimization are also required to introduce some of the decomposition techniques for 2D-TV in this paper. For a more thorough read on this topic we recommend the monograph~\citet{Bach13}.

\begin{defn}[Submodular function]
  A set-function $F: 2^V \rightarrow \mathbb{R}$, for $2^V$ the power set of some set $V$, is submodular if and only if it fulfills the diminishing returns property, that is, for $A \subseteq B \subseteq V$ and $k \in V$, $k \notin B$ we have
  \begin{equation*}
     F(A \cup \left\lbrace k \right\rbrace) - F(A) \geq F(B \cup \left\lbrace k \right\rbrace) - F(B) .
  \end{equation*}
\end{defn}
Intuitively, a set-function is submodular if adding a new element to the set results in less value as the set grows in size.

\begin{defn}[Modular function]
  A set-function $F: 2^V \rightarrow \mathbb{R}$, for $2^V$ the power set of some set $V$, $F(\emptyset) = 0$ is modular (and also submodular) if and only if there exists $\vs \in \mathbb{R}^p$ such that $F(A) = \sum_{k \in A} \vs_k$.
\end{defn}
That is, a function is modular if it always assigns the same value for each element added to the set, regardless of the other elements in the set. A common shorthand for modular functions is $s(A) = \sum_{k \in A} \vs_k$.

Submodular functions can be thought as convex functions in the realm of discrete optimization, in the sense that they feature useful properties that allow for efficient optimization. Similarly, modular functions are connected to linear functions. To make such connections explicit we require of the following geometric concepts.
 
\begin{defn}[Base polytope]
  The base polytope $B_F$ of a submodular function $F$ is the polyhedron given by
  \begin{equation*}
   B_F = \left\lbrace y \in \mathbb{R}^n: y(A) \leq F(A) \,\, \forall A \subseteq V, \quad y(V) = F(V) \right\rbrace .
  \end{equation*}
\end{defn}
That is, the base polytope is a polyhedron defined through linear inequality constraints on the values of $F$ for every one of the $n$ elements of the powerset $2^V$, and an equality constraint for the complete set. This results in a combinatorial number of contraints, but fortunately this polytope will not be used directly.

\begin{defn}[Support function]
  The support function $h_A$ for some non-empty closed convex set $A \in \mathbb{R}^n$ is given by
  \begin{equation*}
   h_A(\vx) = \sup \left\lbrace \vx^T \va : \vx \in A \right\rbrace .
  \end{equation*}
\end{defn}
The support function is useful when connected with the following definition.

\begin{defn}[Lovász extension]
  Suppose a set-function $F$ such that $F(\emptyset) = 0$. 
  Its Lovász extension $f: \mathbb{R}^p \rightarrow \mathbb{R}$ is defined through the following mechanism. 
  Take $\vw \in \mathbb{R}^p$ input to $f$, and order its components in decreasing order $\vw_{j_1} \geq \ldots \geq \vw_{j_p}$, then
  \begin{equation*}
   f(\vw) = \sum_{k=1}^p [F(\left\lbrace j_1, \ldots, j_k \right\rbrace) -  F(\left\lbrace j_1, \ldots, j_{k-1} \right\rbrace)].
  \end{equation*}
\end{defn}
Other equivalent definitions are possible: see~\citet{Bach13} for details. The following result links all the definitions so far.
\begin{proposition}
  For $F$ submodular function such that $F(\emptyset) = 0$ we have
  \begin{itemize}
   \item Its Lovász extension $f$ is a convex function.
   \item The support function of its base polytope is equal to its Lovász extension, that is, $h_{B_F}(\vx) = f(\vx)$.
   \item The problem $\min_{S \subseteq V} F(S)$ is dual to $\min_{\vx} f(\vx) + \half \enorm{\vx}^2$, with $S^* = \left\lbrace k | \vx_k^* \geq 0 \right\rbrace$.
  \end{itemize}
\end{proposition}
For proofs on these points we refer to~\citet{Bach13}. The takeaway from them is that any minimization on a submodular function can be cast into a convex optimization problem. Furthermore, for those convex minimization problems whose objective turns out to be the Lovász extension of some other function, we can trace the steps the other way round, obtaining the minimization of a submodular function.

Consider now a composite problem $\min_{S \subseteq V} \sum_j F_j(S)$. The following results hold
\begin{proposition}
 \label{pro:decompsubmodular}
 The problem $\min_{S \subseteq V} \sum_j F_j(S)$ is equivalent to $\min_x \sum_j f_j(x) + \half \enorm{x}^2$, with $S^* = \left\lbrace k | x_k^* \geq 0 \right\rbrace$. Furthermore it is also equivalent to $\min_{y_j \in B_{F_j} \, \forall j} \half \enorm{\sum_j y_j}^2$, with $x^* = -\sum_j y^*_j$.
\end{proposition}
\begin{proof}
  The first equivalence is a direct result of the properties of Lovász extensions~\citep{Bach13}, in particular that for $F, G$ set-functions with Lovász extensions $f, g$, the Lovász extension of $F + G$ is $f + g$. For the second equivalence we have:
  \begin{align*}
   \min_{\vx} \sum_j f_j(\vx) + \half \enorm{\vx}^2 &= \min_{\vx} \sum_j h_{B_{F_j}} + \half \enorm{\vx}^2, \\
   &= \min_{\vx} \sum_j \max_{y_j \in B_{F_j}} \vy_j^T \vx + \half \enorm{\vx}^2, \\
   &= \max_{\vy_j \in B_{F_j} \, \forall j} \min_{\vx} \left( \sum_j \vy_j^T \vx + \half \enorm{\vx}^2 \right), \\
   &= \min_{\vy_j \in B_{F_j} \, \forall j} \half \enorm{\sum_j \vy_j}^2,
  \end{align*}
  and the dual relationship $\vx^* = -\sum_j \vy^*_j$ comes from solving the inner $\min_{\vx}$ problem for $\vx$.
\end{proof}
Therefore any decomposable submodular minimization, or sum of Lovász extensions plus $\ell_2$ term, can be casted into a geometric problem in terms of the base polytopes. For two functions the resultant problem is of special interest if rewritten as
\begin{align*}
 \min_{\substack{\vy_1 \in B_{F_1} \\ \vy_2 \in B_{F_2}}} \half \enorm{\vy_1 + \vy_2}^2 \quad = \quad
 \min_{\substack{\vy_1 \in B_{F_1} \\ -\vy_2 \in -B_{F_2}}} \half \enorm{\vy_1 - (-\vy_2)}^2 \quad = \quad
 \min_{\substack{\va \in B_{F_1} \\ \vb \in -B_{F_2}}} \half \enorm{\va -\vb}^2
\end{align*}
with $\va = \vy_1$, $\vb = -\vy_2$, as this results in the classic geometric problem of finding the closest points between two convex sets. Many algorithms have been proposed to tackle problems in this form, most of them making use of alternating projection operations onto the two sets. Thus, a legitimate concern is how easy it is to compute such projections for $B_{F_1}$ and $-B_{F_2}$.
\begin{proposition}
 \label{pro:projpolytopes}
 Given a submodular function $F$ and its base polytope $B_F$, the projections $\Pi_{B_F}(\vz)$ and $\Pi_{-B_F}(\vz)$ of a point $\vz$ onto $B_F$ or its negated counterpart can be computed as
 \begin{align*}
  \Pi_{B_F}(\vz) &= \vz - \prox_f(\vz) , \\
  \Pi_{-B_F}(\vz) &= \vz + \prox_f(-\vz) ,
 \end{align*}
 with $\prox$ proximity operator of a function, $f$ the Lovász extension of $F$.
\end{proposition}
\begin{proof}
 We start with the proximity of $f$ and work our way to a relationship with the projection operator,
 \begin{align*}
  \prox_f(\vz) &\equiv \min_{\vx} f(\vx) + \half \enorm{\vx - \vz}^2, \\
   &= \max_{\vy \in B_F} \min_{\vx} \vy^T \vx + \half \enorm{\vx - \vz}^2, \\
   &= \max_{\vy \in B_F} \vy^T (\vz - \vy) + \half \enorm{(\vz - \vy) - \vz}^2, \\
   &= \min_{\vy \in B_F} \half \enorm{\vy}^2 - \vy^T \vz, \\
   &\equiv \min_{\vy \in B_F} \half \enorm{\vy - \vz}^2 = \Pi_{B_F}(\vz),
 \end{align*}
 where solving the inner minimization problem for $\vx$ gives the primal--dual relationship $\vx^* = \vz - \vy^*$. Using this we can obtain the solution for the projection problem from the proximity problem, as $\Pi_{B_F}(\vz) = \vz - \prox_f(\vz)$. Projection onto the negated base polytope follows from the basic geometric argument $\Pi_{-B_F}(\vz) = -\Pi_{B_F}(-\vz)$.
\end{proof}

\section{proxTV toolbox}

All the Total--Variation proximity solvers in this paper have been implemented as the {\bf proxTV} toolbox for C++, Matlab and Python, available at \href{https://github.com/albarji/proxTV}{https://github.com/albarji/proxTV}. The toolbox has been designed to be used out of the box in a user friendly way; for instance, the top--level Matlab function \texttt{TV} solves Total--Variation proximity for a given signal under a variety of settings. For instance
\begin{lstlisting}
>> TV(X,lambda)
\end{lstlisting}
solves $\tvell_1$ proximity for a signal \texttt{X} of any dimension and a regularization value \texttt{lambda}. The weighted version of this problem  is also seamlessly tackled by just providing a vector of weights of the appropriate length as the \texttt{lambda} parameter.

If a third parameter \texttt{p} is provided as
\begin{lstlisting}
>> TV(X,lambda,p)
\end{lstlisting}
the general $\tvell_p$ proximity problem is addressed, whereupon an adequate solver is chosen by the library.

More advanced uses of the library are possible, allowing to specify which norm \texttt{p} and regularizer \texttt{lambda} values to use for each dimension of the signal, and even applying combinations of several different $\tvell_p$ regularizers along the same dimension. Please refer to the documentation within the toolbox for further information.

\section{Proof on the equality of taut-string problems}
\label{app:tautStringProof}

\begin{theorem}[Equality of taut-string problems]
 \label{the:tautStringEq}
 Given the problems 
 \begin{equation}
  \label{eq:tautStringEq1}
  \min_{\vs} \sum_{i=1}^n \left( \vs_i - \vs_{i-1} \right)^2,\ \text{s.t. } \left| \vs_i - \vr_i \right| \le \vw_i \, \forall i=1,\ldots,n-1 \, , \vs_0 = 0, \vs_n = \vr_n ,
 \end{equation}
 and
 \begin{equation}
  \label{eq:tautStringEq2}
  \min_{\hat \vs} \sum_{i=1}^n \sqrt{1 + \left( \hat \vs_i - \hat \vs_{i-1} \right)^2},\ \text{s.t. } \left| \hat \vs_i - \vr_i \right| \le \vw_i \, \forall i=1,\ldots,n-1 \, , \hat \vs_0 = 0, \hat \vs_n = \vr_n ,
 \end{equation}
 for a non-zero vector $\vw$, both problems share the same minimum $\vs^* = \hat \vs^*$.
\end{theorem}

\begin{proof}

The Lagrangian of problem \ref{eq:tautStringEq1} takes the form
\begin{equation*}
 L(\vs,\valpha,\vbeta) = \sum_{i=1}^n \left( \vs_i - \vs_{i-1} \right)^2 + \sum_{i=1}^{n-1} \valpha_i (\vs_i - \vr_i -\vw_i) + \sum_{i=1}^{n-1} \vbeta_i ( -\vw_i - \vs_i + \vr_i )  ,
\end{equation*}
and its Karush-Kuhn-Tucker optimality conditions are given by
\begin{align}
 \label{eq:tautStringEq1KKT1}
 (\vs_{i+1} - \vs_i) - (\vs_i - \vs_{i-1}) &= \valpha_i - \vbeta_i, \\
 \label{eq:tautStringEq1KKT2}
 \left| \vs_i - \vr_i \right| &\le \vw_i, \\
 \label{eq:tautStringEq1KKT3}
 \valpha_i, \vbeta_i &\ge 0, \\
 \label{eq:tautStringEq1KKT4}
 \valpha_i (\vs_i - \vr_i -\vw_i) &= 0, \\
 \label{eq:tautStringEq1KKT5}
 \vbeta_i ( -\vw_i - \vs_i + \vr_i ) &= 0,
\end{align}
$\forall i=1,\ldots,n-1$, and where the first equation comes from the fact that $\frac{\partial L(\vs,\valpha,\vbeta)}{\partial \vs} = 0$ at the minimum.

As the only difference between problems \ref{eq:tautStringEq1} and \ref{eq:tautStringEq2} is in the form of the objective, the KKT conditions for problem \ref{eq:tautStringEq2} take the same form, but for the first one,
\begin{align}
 \label{eq:tautStringEq2KKT1}
 \frac{(\hat \vs_{i+1} - \hat \vs_i)}{\sqrt{1 + (\hat \vs_{i+1} - \hat \vs_i)^2}} - \frac{(\hat \vs_i - \hat \vs_{i-1})}{\sqrt{1 + (\hat \vs_{i} - \hat \vs_{i-1})^2}} &= \hat \valpha_i - \hat \vbeta_i, \\
 \label{eq:tautStringEq2KKT2}
 \left| \hat \vs_i - \vr_i \right| &\le \vw_i, \\
 \label{eq:tautStringEq2KKT3}
 \hat \valpha_i, \hat \vbeta_i &\ge 0, \\
 \label{eq:tautStringEq2KKT4}
 \hat \valpha_i (\hat \vs_i - \vr_i -\vw_i) &= 0, \\
 \label{eq:tautStringEq2KKT5}
 \hat \vbeta_i ( -\vw_i - \hat \vs_i + \vr_i ) &= 0,
\end{align}
$\forall i=1,\ldots,n-1$, and where we use hat notation for the dual coefficients to tell them apart from those of problem \ref{eq:tautStringEq1}.

Suppose $\vs^*$ minimizer to problem \ref{eq:tautStringEq1}, hence fulfilling the conditions \ref{eq:tautStringEq1KKT1}-\ref{eq:tautStringEq1KKT5}. In particular this means that it is feasible to assign values to the dual coefficients $\valpha, \vbeta$ in such a way that the conditions above are met. If we set $\hat \vs = \vs^*$ in the conditions \ref{eq:tautStringEq2KKT1}-\ref{eq:tautStringEq2KKT5} the following observations are of relevance
\begin{itemize}
 \item Condition \ref{eq:tautStringEq2KKT2} becomes the same as condition \ref{eq:tautStringEq1KKT2}, and so it is immediately met.
 \item The operator $f(x) = \frac{x}{\sqrt{1 + x^2}}$ is contractive and monotonous.
 \item The couple $(\valpha_i,\vbeta_i)$ cannot be both non--zero at the same time, since $\valpha_i > 0$ enforces $\vs_i = \vr_i + \vw_i$ and $\vbeta_i > 0$ enforces $\vs_i = \vr_i - \vw_i$, and $\vw_i$ is non--zero.
 \item Hence and because $\valpha_i, \vbeta_i \ge 0$ and condition \ref{eq:tautStringEq1KKT1} holds, when $(\vs_{i+1} - \vs_i) - (\vs_i - \vs_{i-1}) > 0$ then $\valpha_i > 0$, $\vbeta_i = 0$, and when $(\vs_{i+1} - \vs_i) - (\vs_i - \vs_{i-1}) < 0$ then $\valpha_i = 0$, $\vbeta_i > 0$.
 \item $f(\vs_{i+1} - \vs_i) - f(\vs_i - \vs_{i-1})$ has the same sign as $(\vs_{i+1} - \vs_i) - (\vs_i - \vs_{i-1})$, since $f$ is monotonous and as such preserves ordering.
 \item Since $f$ is contractive, condition \ref{eq:tautStringEq2KKT1} can be met by setting $(\hat \valpha_i,\hat \vbeta_i) = (k \valpha_i, k \vbeta_i)$ for some $0 \leq k < 1$. Note that this works because $(\valpha_i,\vbeta_i)$ cannot be both zero at the same time.
 \item Condition \ref{eq:tautStringEq2KKT3} is met for those choices of $\hat \valpha_i, \hat \vbeta_i$, as \ref{eq:tautStringEq1KKT3} was met for $\valpha_i, \vbeta_i$ and $0 \leq k < 1$.
 \item Conditions \ref{eq:tautStringEq2KKT4} and \ref{eq:tautStringEq2KKT5} are also met for those choices of $\hat \valpha_i, \hat \vbeta_i$, as $\hat \valpha_i (\vs_i - \vr_i -\vw_i) = k \valpha_i (\vs_i - \vr_i -\vw_i) = 0$ and $\hat \vbeta_i ( -\vw_i - \vs_i + \vr_i ) = k \vbeta_i ( -\vw_i - \vs_i + \vr_i ) = 0$.
\end{itemize}
Therefore, all of the optimality conditions \ref{eq:tautStringEq2KKT1}-\ref{eq:tautStringEq2KKT5} for problem \ref{eq:tautStringEq2} are met for $\vs^*$ solution of problem \ref{eq:tautStringEq1}, and so a minimum of problem \ref{eq:tautStringEq1} is also a minimum for problem \ref{eq:tautStringEq2}.

The proof can be repeated the other way round by setting $\vs = \hat \vs^*$ optimal for problem \ref{eq:tautStringEq2}, defining the operator $f^{-1}(x) =  \frac{x}{\sqrt{1 - x^2}}$, and observing that this operator is monotonous and expansive, so we can establish $(\valpha_i,\vbeta_i) = (k \hat \valpha_i, k \hat \vbeta_i)$ for some $k \geq 1$ and the optimality conditions \ref{eq:tautStringEq1KKT1}-\ref{eq:tautStringEq1KKT5} for problem \ref{eq:tautStringEq1} are met following a similar reasoning to the one presented above. Thus, a minimum for problem \ref{eq:tautStringEq2} is also a minimum for problem \ref{eq:tautStringEq1}, which joined with the previous result completes the proof.

\end{proof}

\section{Proof on the equivalence of linearized taut-string method}
\label{app:tautStringAlgEq}

\begin{proposition}
 Using affine approximations to the greatest convex minorant and the smallest concave majorant does not change the solution of the taut-string method.
\end{proposition}

\begin{proof}
 Let us note $\cap(f)$ as the smallest concave majorant of some function $f$ taking integer values, $\cup(f)$ as the greatest concave minorant, $\bar a(f)$ as the smallest affine majorant and $\ubar a(f)$ as the greatest affine minorant. By definition we have
 
 \begin{equation*}
  \ubar a(f(i)) \leq \cup(f(i)) \leq f(i) \leq  \cap(f(i)) \leq \bar a(f(i)) \quad \forall i \in \mathbb{Z}
 \end{equation*}
 
 Consider now the nature of the taut-string problem, where a vertically symmetric tube of radius $\lambda_i$ at each section is modelled by following the majorant of the tube bottom ($f-\lambda$) and the minorant of the tube ceiling ($f+\lambda$). We work the inequalities above as:
 
 \begin{align*}
  f(i) - \lambda_i \leq  \cap(f(i) - \lambda_i)  \leq \bar a(f(i) - \lambda_i) \\
  \ubar a(f(i) + \lambda_i) \leq \cup(f(i) + \lambda_i) \leq f(i) + \lambda_i
 \end{align*}
 
 We will show that an overlap of smallest concave majorant / greatest convex minorant takes place iff the same overlap happens when using the affine approximations. We formally define overlap as the setting where for a point $i$ we have $\cup(f_i + \lambda_i) \leq \cap(f_i -\lambda_i)$.
 
 One side of the implication is easy: if $\cup(f(i) + \lambda_i) \leq \cap(f(i) - \lambda)$ for some $i$, then using the relations above we have $\ubar a(f(i) + \lambda_i) \leq \cup(f(i) + \lambda_i) \leq \cap(f(i) - \lambda_i) \leq \bar a(f(i) - \lambda_i)$, and so the affine approximation detects any overlap taking place in the concave/convex counterpart.
 
 The opposite requires the key observation that in the taut-string method both majorant and minorant functions are clamped to the same point of origin: $f(0) = 0$ at the start of the method, or the point where the last segment was fixed after each restart. Let us assume $f(0) = 0$ without loss of generality. Suppose now that an overlap is detected by the affine approximation.
 Because of this affine nature the majorant/minorant slopes are constant, i.e.
 \begin{align*}
  \bar \delta_1 = \bar \delta_2 = \ldots = \bar \delta_n = \bar \delta, \quad 
  \ubar \delta_1 = \ubar \delta_2 = \ldots = \ubar \delta_n = \ubar \delta.
 \end{align*}
 However, if we consider the convex/concave approximations these slopes can increase/decrease as the segment progresses, that is:
 \begin{align*}
  \delta^{\cup}_1 \leq \delta^{\cup}_2 \leq \ldots \leq \delta^{\cup}_n, \quad
  \delta^{\cap}_1 \geq \delta^{\cap}_2 \geq \ldots \geq \delta^{\cap}_n.
 \end{align*}
 Consider now the majorant/minorant values, expressed through the slopes and taking into account the observation above about the starting point.
 \begin{align*}
  \cap(f(i) - \lambda_i) &= \sum_{j=1}^{i} \delta^{\cap}_j, \quad
  \cup(f(i) + \lambda_i) &= \sum_{j=1}^{i} \delta^{\cup}_j, \quad
  \bar a(f(i) - \lambda_i) &= i \bar \delta, \quad
  \ubar a(f(i) + \lambda_i) &= i \ubar \delta.
 \end{align*}
 Since an overlap has been detected in the affine approximation, we have that for some point $i$
 \begin{align*}
  i \ubar \delta = \ubar a(f(i) + \lambda_i) &\leq \bar a(f(i) - \lambda_i) = i \bar \delta, 
 \end{align*}
 so $\ubar \delta \leq \bar \delta$.
 Consider now the values of the affine minorant/majorant at the point immediately after the origin,
 \begin{align*}
  \ubar a(f_1 - \lambda_1) = \ubar \delta, \quad
  \bar a(f_1 + \lambda_1) = \bar \delta.
 \end{align*}
 We will show now that the convex/convex counterpart must take exactly the same values at these points. To do so we take into account the following fact: there must exist points $x$ and $y$, $x, y \leq i$, where
 \begin{align*}
  \ubar a(f_x + \lambda_x) = f_x + \lambda_x = \cup(f_x + \lambda_x), \quad
  \bar a(f_y - \lambda_y) = f_y - \lambda_y = \cap(f_y - \lambda_y),
 \end{align*}
 that is to say, the affine minorant/majorant must touch the tube ceiling/bottom at some point, otherwise we could obtain a greater minorant / smaller majorant by reducing this distance. The equalities to the convex minorant / concave majorant are then obtained by exploiting the inequalities at the beginning of the proof.
 
 By the already presented inequalities $\cup(f_1 + \lambda_1) \geq \ubar a(f_1 + \lambda_1)$, but let us suppose for a moment $\cup(f_1 + \lambda_1) > \ubar a(f_1 + \lambda_1)$. This would imply $\delta_1^{\cup} > \ubar \delta$. We then would have that at the touching point $x$
 \begin{align*}
  f_x + \lambda_x = \ubar a(f_x + \lambda_x) = x \ubar \delta < x \delta_1^{\cup} \leq \cup(f_1 + \lambda_1),
 \end{align*}
 as the slopes in a convex minorant must be monotonically increasing. However, such function would not be a valid convex minorant, as it would grow over $f+\lambda$. Therefore $\cup(f_1 + \lambda_1) = \ubar a(f_1 + \lambda_1)$ must hold. Using a symmetric argument, $\cap(f_1 - \lambda_1) = \bar a(f_1 - \lambda_1)$ can also be shown to hold. Joining this with the previous facts we have that
 \begin{align*}
  \cup(f_1 + \lambda_1) = \ubar a(f_1 + \lambda_1) = \ubar \delta \leq \bar \delta = \bar a(f_1 - \lambda_1) = \cap(f_1 - \lambda_1),
 \end{align*}
 and therefore the overlap detected by the affine approximation is detected through its convex/concave version as well through $\cup(f_1 + \lambda_1) \leq \cap(f_1 - \lambda_1)$.

\end{proof}

\section{Projected-newton for weighted $\tvell_1^{\oned}$}
\label{app:projNewton}

In this appendix we present details of a projected-Newton (PN) approach to solving the weighted-TV problem~\eqref{eq.14}. Although taut-string approaches are empirically superior to this PN approach, the details of this derivation prove to be useful when developing subroutines for handling $\ell_p$-norm TV prox-operators, but perhaps their greatest use lies in presenting a general method that could be applied to other problems that have structures similar to TV, e.g., group total-variation~\citep{alaiz2013group,wytock} and $\ell_1$-trend filtering~\citep{boyd.kim,trendf}.

The weighted-TV dual problem~\eqref{eq.15} is a bound-constrained QP, so it could be solved using a variety of methods such as TRON~\citep{TRON}, L-BFGS-B~\citep{L-BFGS-B}, or projected-Newton (PN)~\citep{ProjNewton}. Obviously, these methods will be inefficient if invoked off-the-shelf; exploitation of problem structure is a must for solving~\eqref{eq.15} efficiently. PN lends itself well to such structure exploitation; we describe the details below. 

PN runs iteratively in three key steps: first it identifies a special subset of \emph{active variables} and uses these to compute a \emph{reduced} Hessian. Then, it uses this Hessian to scale the gradient and move in the direction opposite to it, damping with a  stepsize, if needed. Finally, the next iterate is obtained by projecting onto the constraints, and the cycle repeats. PN can be regarded as an extension of the gradient-projection method (GP,~\cite{Bertsekas}), where the components of the gradient that make the updating direction infeasible are removed; in PN both the gradient and the Hessian are \emph{reduced} to guarantee this feasibility.

At each iteration PN selects the active variables
\begin{equation}
  \label{eq.16}
  I := \set{i \mid (u_i = -w_i \;\text{and}\; [\nabla\phi(\vu)]_i > \epsilon)\quad\text{or}\quad (u_i = w_i \;\text{and}\; [\nabla\phi(\vu)]_i < -\epsilon)},
\end{equation}
where $\epsilon \ge 0$ is small scalar. This corresponds to the set of variables at a bound, and for which the gradient points inside the feasible region; that is, for these variables to further improve the objective function we would have to step out of bounds. It is thus clear that these variables are of no use for this iteration, so we define the complementary set $\bar{I} := \set{1\ldots n} \backslash I$ of indices not in $I$, which are the variables we are interested in updating. From the Hessian $\mH=\nabla^2\phi(u)$ we extract the \emph{reduced Hessian} $\mH_{\bar{I}}$ by selecting rows and columns indexed by $\bar{I}$, and in a similar way the \emph{reduce gradient} $[\nabla\phi(\vu)]_{\bar{I}}$. Using these we perform a Newton--like ``reduced'' update in the form 
\begin{equation}
  \label{eq:6}
  \vu_{\bar{I}} \gets P(\vu_{\bar{I}} - \alpha\mH_{\bar{I}}^{-1}[\nabla\phi(\vu)]_{\bar{I}}), 
\end{equation}
where $\alpha$ is a stepsize, and $P$ denotes projection onto the constraints, which for box--constraints reduces to simple element--wise projection. Note that only the variables in the set $\bar{I}$ are updated in this iterate, leaving the rest unchanged. While such update requires computing the inverse of the reduced Hessian $\mH_{\bar{I}}$, which in the general case can amount to computational costs in the $O(n^3)$ order, we will see now how exploiting the structure of the problem allows us to perform all the steps above efficiently.

First, observe that for~(\ref{eq.15}) the Hessian is  
\begin{equation*}
 \mH  = \md\md^T = \left( \begin{array}{ccccc}
                  2  & -1 &        &        &     \\
                  -1 &  2 & -1     &        &     \\
                     & -1 &  2     & \ddots &     \\
                     &    & \ddots & \ddots & -1  \\
                     &    &        &   -1   &  2
                 \end{array} \right) \in \reals^{(n-1) \times (n-1)} .
\end{equation*}

Next, observe that whatever the active set $I$, the corresponding reduced Hessian $\mH_{\bar{I}}$ remains symmetric tridiagonal. This observation is crucial because then we can quickly compute the updating direction $\vd_{\bar{I}} = \mH_{\bar{I}}^{-1}[\nabla\phi(\vu)]_{\bar{I}}$, which can be done by solving the linear system $\mH_{\bar{I}}\vd_{\bar{I}} = [\nabla \phi(\vu^t)]_{\bar{I}}$ as follows:

\begin{enumerate}
  \setlength{\itemsep}{0pt}
 \item Compute the Cholesky decomposition $\mH_{\bar{I}} = \mr^T\mr$.
 \item Solve the linear system $\mr^T \vv = [\nabla\phi(\vu)]_{\bar{I}}$ to obtain $\vv$.
 \item Solve the linear system $\mr \vd_{\bar{I}} = \vv$ to obtain $\vd_{\bar{I}}$.
\end{enumerate}

Because the reduced Hessian is also tridiagonal, its Cholesky decomposition can be computed in \emph{linear time} to yield a bidiagonal matrix $\mr$, which in turn allows to solve the subsequent linear systems also in linear time. Extremely efficient routines to perform all these tasks are available in the LAPACK libraries~\citep{LAPACK}.

The next crucial ingredient is efficient selection of the stepsize $\alpha$. The original PN algorithm \citet{ProjNewton} recommends Armijo-search along projection arc. However, for our problem this search is inordinately expensive. So we resort to a backtracking strategy using quadratic interpolation~\citep{NumOptNocedal}, which works admirably well. This strategy is as follows: start with an initial stepsize $\alpha_0=1$. If the current stepsize $\alpha_k$ does not provide sufficient decrease in $\phi$, build a quadratic model using $\phi(\vu)$, $\phi(\vu-\alpha_k\vd)$, and $\partial_{\alpha_k} \phi(\vu)$. Then, the stepsize $\alpha_{k+1}$ is set to the value that minimizes this quadratic model. In the event that at some point of the procedure the new $\alpha_{k+1}$ is larger than or too similar to $\alpha_k$, its value is halved. In this fashion, quadratic approximations of $\phi$ are iterated until a good enough $\alpha$ is found. The goodness of a stepsize is measured using the following 
Armijo-like sufficient descent rule
\begin{equation*}
 \phi(\vu) - \phi(P \left[ \vu-\alpha_k\vd \right]) \geq \sigma \cdot \alpha_k \cdot \left( \nabla\phi(\vu) \cdot \vd \right),
\end{equation*}
where a tolerance $\sigma = 0.05$ works well practice.

Note that the gradient $\nabla \phi(\vu)$ might be misleading in the condition above if $\vu$ has components at the boundary and $\vd$ points outside this boundary (because then, due to the subsequent projection no real improvement would be obtained by stepping outside the feasible region). To address this concern, we modify the computation of the gradient $\nabla \phi(\vu)$, zeroing our the entries that relate to direction components pointing outside the feasible set. 

The whole stepsize selection procedure is shown in Algorithm \ref{algStep}. The costliest operation in this procedure is the evaluation of $\phi$, which, nevertheless can be done in linear time. Furthermore, in practice a few iterations more than suffice to obtain a good stepsize.

 \begin{algorithm}[t]
  \caption{\small Stepsize selection for Projected Newton}
  \label{algStep}
  \begin{algorithmic}
   \State {\it Initialize:} $\alpha_0 = 1$, $k=0$, $\vd$, tolerance parameter $\sigma$
   \While{$\phi(\vu) - \phi(P[\vu-\alpha_k \vd]) < \sigma \cdot \alpha_k \cdot \left( \nabla\phi(\vu) \cdot \vd \right) $}
    \State Minimize quadratic model: $\alpha_{k+1} = \frac{\alpha_k^2 \partial_{\alpha_k} \phi(\vu)}{2 (\phi(\vu) - \phi(\vu - \alpha_k) + \alpha_k \partial_{\alpha_k} \phi(\vu))}$. 
    \State \textbf{if} $\alpha_{k+1} > \alpha_k$ \textbf{or} $\alpha_{k+1} \simeq \alpha_k$, \textbf{then} $\alpha_{k+1} = \half \alpha_k$.
    \State $k \gets k + 1$
   \EndWhile
   \State\Return $\alpha_k$
  \end{algorithmic}
 \end{algorithm}

Overall, a full PN iteration as described above runs at $O(n)$ cost. Thus, by exploiting the structure of the problem, we manage to reduce the $O(n^3)$ cost per iteration of a general PN algorithm to a linear-cost method. The pseudocode of the resulting method is shown as  Algorithm~\ref{algTV1}. Note that in the special case when the weights $\bm{W}:=\Diag(w_i)$ are so large that the unconstrained optimum coincides with the constrained one, we can obtain $\vu^*$ directly via solving $\md\md^T\bm{W}\vu^* = \md\vy$ (which can also be done at $O(n)$ cost). The duality gap of the current solution is used as a stopping criterion, where we use a tolerance of $\epsilon = 10^{-5}$ in practice.

{\begin{algorithm}[t]
 \caption{\small PN algorithm for TV-L1-proximity}
 \label{algTV1}
 \begin{algorithmic}
  \State Let $\bm{W}=\Diag(w_i)$; solve $\md\md^T\bm{W}\vu^* = \md\vy$.
  \State \textbf{if} $\infnorm{\bm{W}^{-1}\vu^*} \le 1$, \textbf{return} $\vu^*$.
  \State $\vu^0 = P[\vu^*]$, $t = 0$.
  \While{$\text{gap}(\vu)$ $>$ $\epsilon$}
   \State Identify set of active constraints $I$; let $\bar{I}=\set{1\ldots n} \setminus I$.
   \State Construct reduced Hessian $\mH_{\bar{I}}$.
   \State Solve $\mH_{\bar{I}}\vd_{\bar{I}} = [\nabla \phi(\vu^t)]_{\bar{I}}$.
   \State Compute stepsize $\alpha$ using backtracking + interpolation (Alg.~\ref{algStep}).
   \State Update $\vu^{t+1}_{\bar{I}} = P[\vu^t_{\bar{I}} - \alpha\vd_{\bar{I}}]$.
   \State $t \gets t + 1$.
  \EndWhile
  \State\Return $\vu^t$.
 \end{algorithmic}
\end{algorithm}
}

\section{Testing images and videos, and experimental results}
\label{app:images}

The images used in the experiments are displayed in what follows, along with their noisy/denoised and convoluted/deconvoluted versions for each algorithm tested. QR barcode images were generated by encoding random text using Google chart API\footnote{\url{http://code.google.com/intl/en-EN/apis/chart/}}. Images {\it shape} and {\it phantom} \footnote{Extracted from \url{http://en.wikipedia.org/wiki/File:Shepp_logan.png}} are publicly available and frequently used in image processing. {\it trollface} and {\it comic} \footnote{Author: Francisco Molina. \url{http://www.afrikislife.net/english/}} are also publicly available. {\it gaudi}, used in the multicore experiments, is a high resolution $3197 \times 3361$ photograph of Gaudi's Casa Batll\'o\footnote{Extracted from \url{http://www.flickr.com/photos/jeffschwartz/202423023/}}. The rest of the images were originally created by the authors.

For the video experiments, the {\it salesman}, {\it coastguard} and {\it bicycle}  sequences were used, which are publicly available at \cite{BM3Dvideos}. As an example, frames from the first video are displayed in what follows, along with their noisy/denoised versions.

\begin{figure}[htbp]
  \begin{tabular}{lll}
  \abovespace\belowspace
  \begin{minipage}{0.3\textwidth}
   \framebox{\parbox{\textwidth}{
    \centering
    \includegraphics[width=1\textwidth]{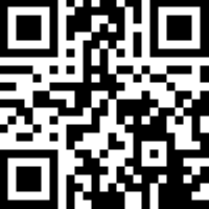} \\
    randomQR-0 ($100 \times 100$)
   }}
  \end{minipage}
  &
  \begin{minipage}{0.3\textwidth}
   \framebox{\parbox{\textwidth}{
    \centering
    \includegraphics[width=1\textwidth]{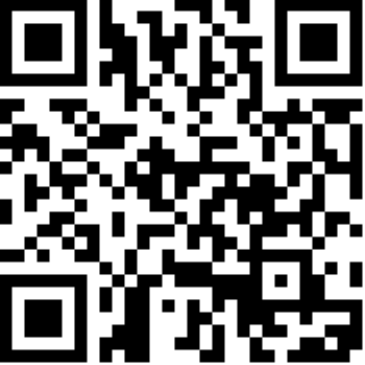} \\
    randomQR-1 ($175 \times 175$)
   }}
  \end{minipage}
  &
  \begin{minipage}{0.3\textwidth}
   \framebox{\parbox{\textwidth}{
    \centering
    \includegraphics[width=1\textwidth]{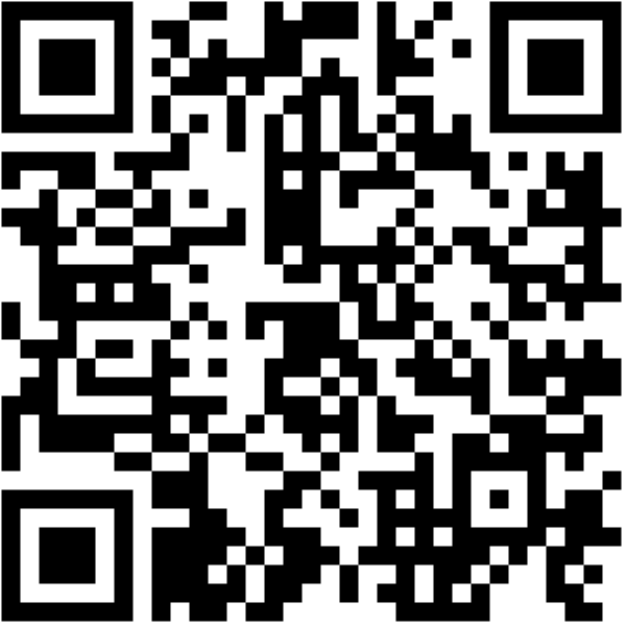} \\
    randomQR-2 ($300 \times 300$)
   }}
  \end{minipage}
  \vspace{0.2cm}
  \\
  \begin{minipage}{0.3\textwidth}
   \framebox{\parbox{\textwidth}{
    \centering
    \includegraphics[width=1\textwidth]{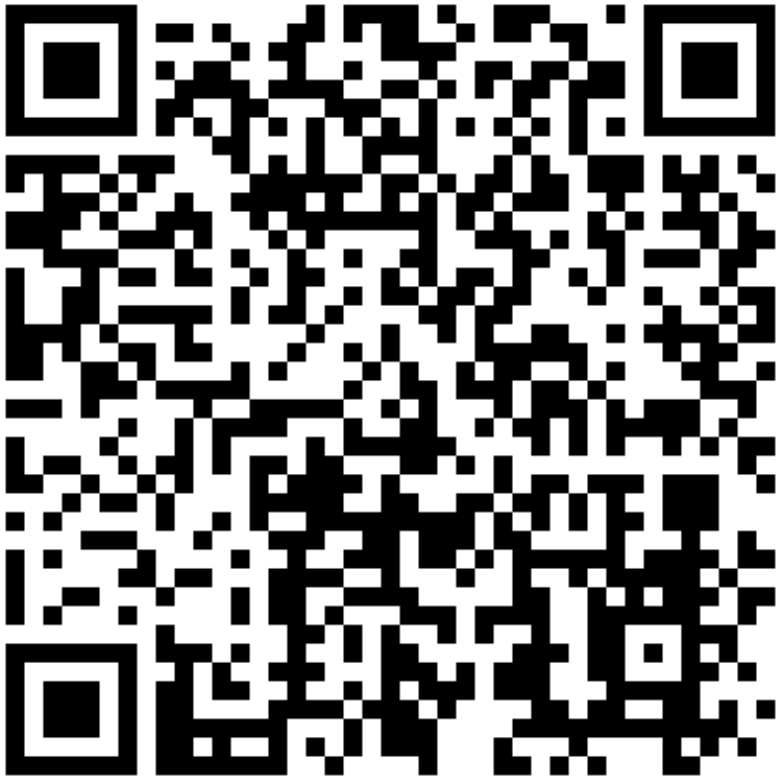} \\
    randomQR-3 ($375 \times 375$)
   }}
  \end{minipage}
  &
  \begin{minipage}{0.3\textwidth}
   \framebox{\parbox{\textwidth}{
    \centering
    \includegraphics[width=1\textwidth]{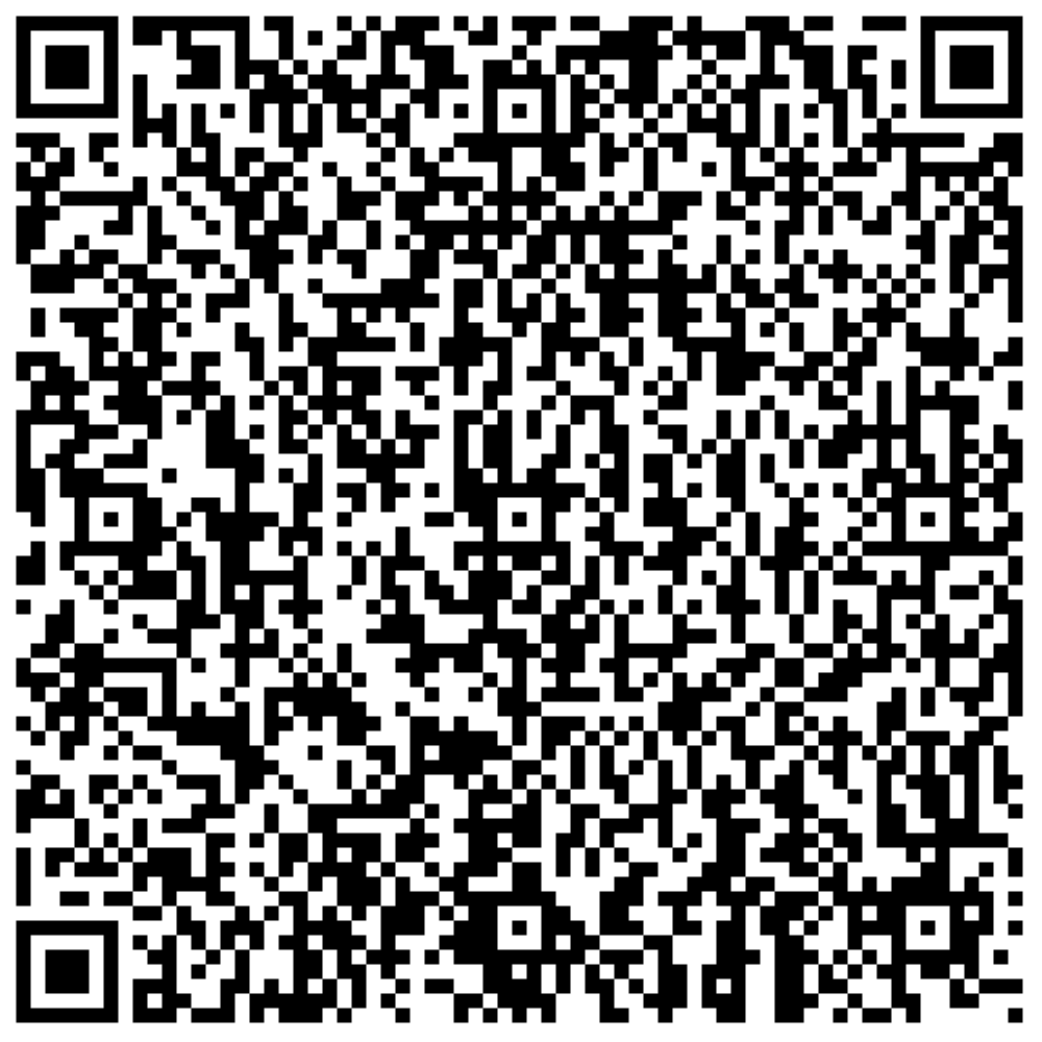} \\
    randomQR-4 ($500 \times 500$)
   }}
  \end{minipage}
  &
  \begin{minipage}{0.3\textwidth}
   \framebox{\parbox{\textwidth}{
    \centering
    \includegraphics[width=1\textwidth]{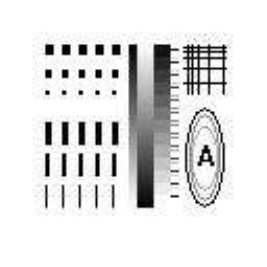} \\
    shape ($128 \times 128$)
   }}
  \end{minipage}
  \vspace{0.2cm}
  \\
  \abovespace\belowspace
  \begin{minipage}{0.3\textwidth}
   \framebox{\parbox{\textwidth}{
    \centering
    \includegraphics[width=1\textwidth]{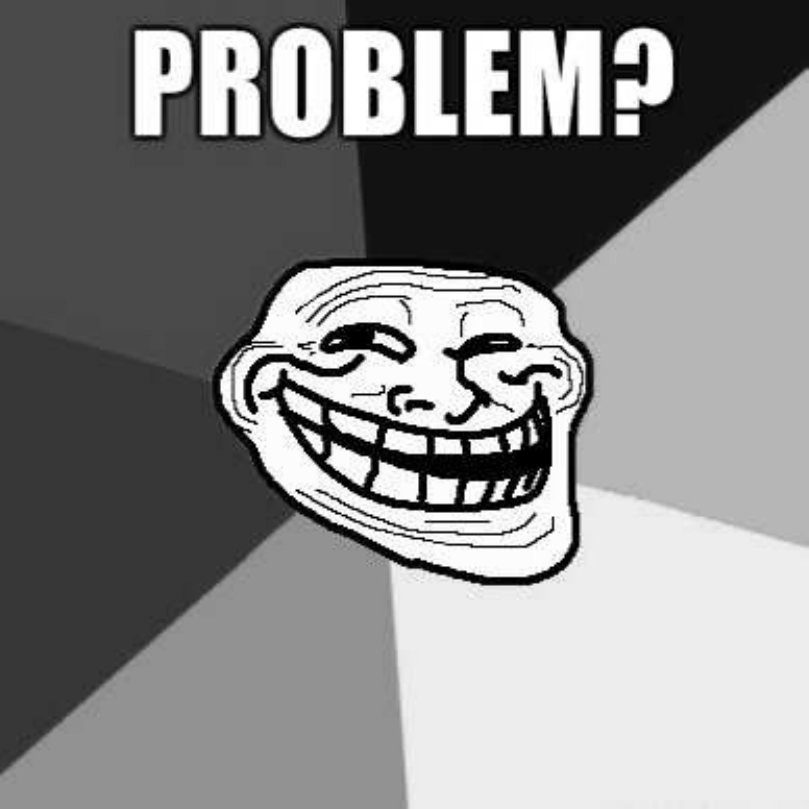} \\
    trollface ($388 \times 388$)
   }}
  \end{minipage}
  &
  \begin{minipage}{0.3\textwidth}
   \framebox{\parbox{\textwidth}{
    \centering
    \includegraphics[width=1\textwidth]{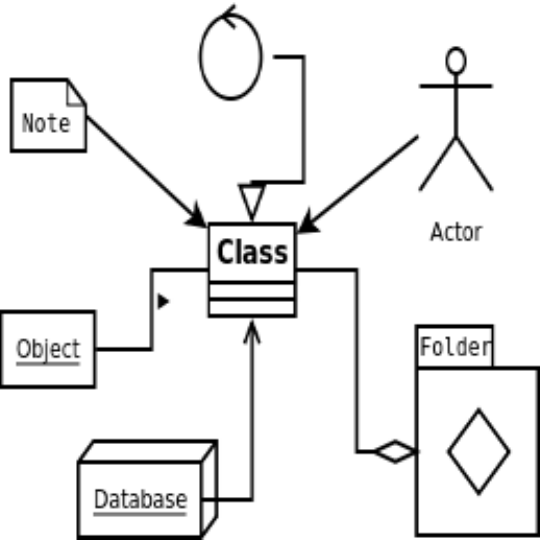} \\
    diagram ($259 \times 259$)
   }}
  \end{minipage}
  &
  \begin{minipage}{0.3\textwidth}
   \framebox{\parbox{\textwidth}{
    \centering
    \includegraphics[width=1\textwidth]{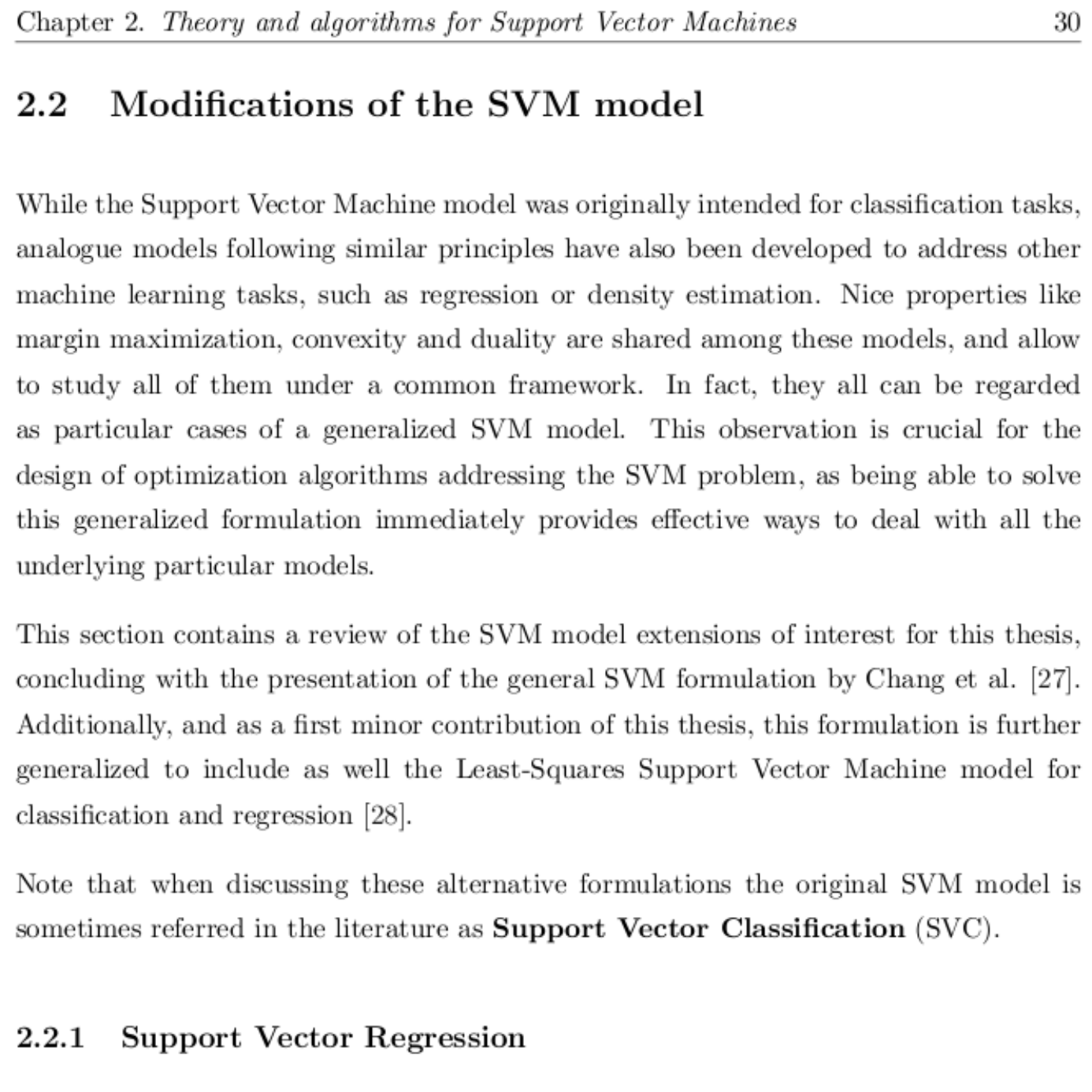} \\
    text ($665 \times 665$)
   }}
  \end{minipage}
  \vspace{0.2cm}
  \\
  \abovespace\belowspace
  \begin{minipage}{0.3\textwidth}
   \framebox{\parbox{\textwidth}{
    \centering
    \includegraphics[width=1\textwidth]{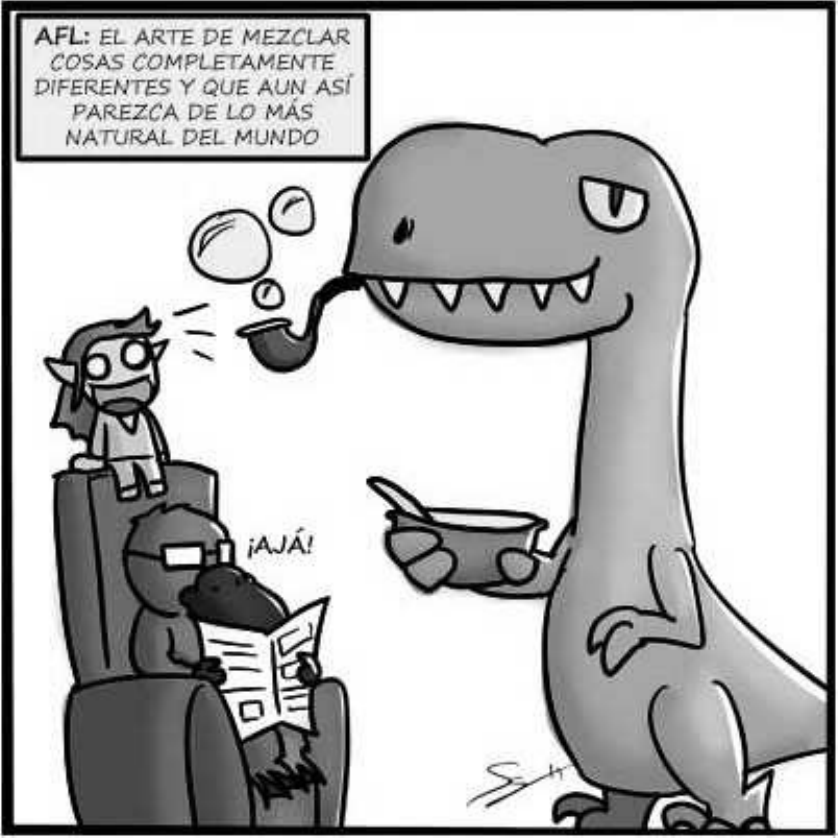} \\
    comic ($402 \times 402$)
   }}
  \end{minipage}
  &
  \begin{minipage}{0.3\textwidth}
   \framebox{\parbox{\textwidth}{
    \centering
    \includegraphics[width=1\textwidth]{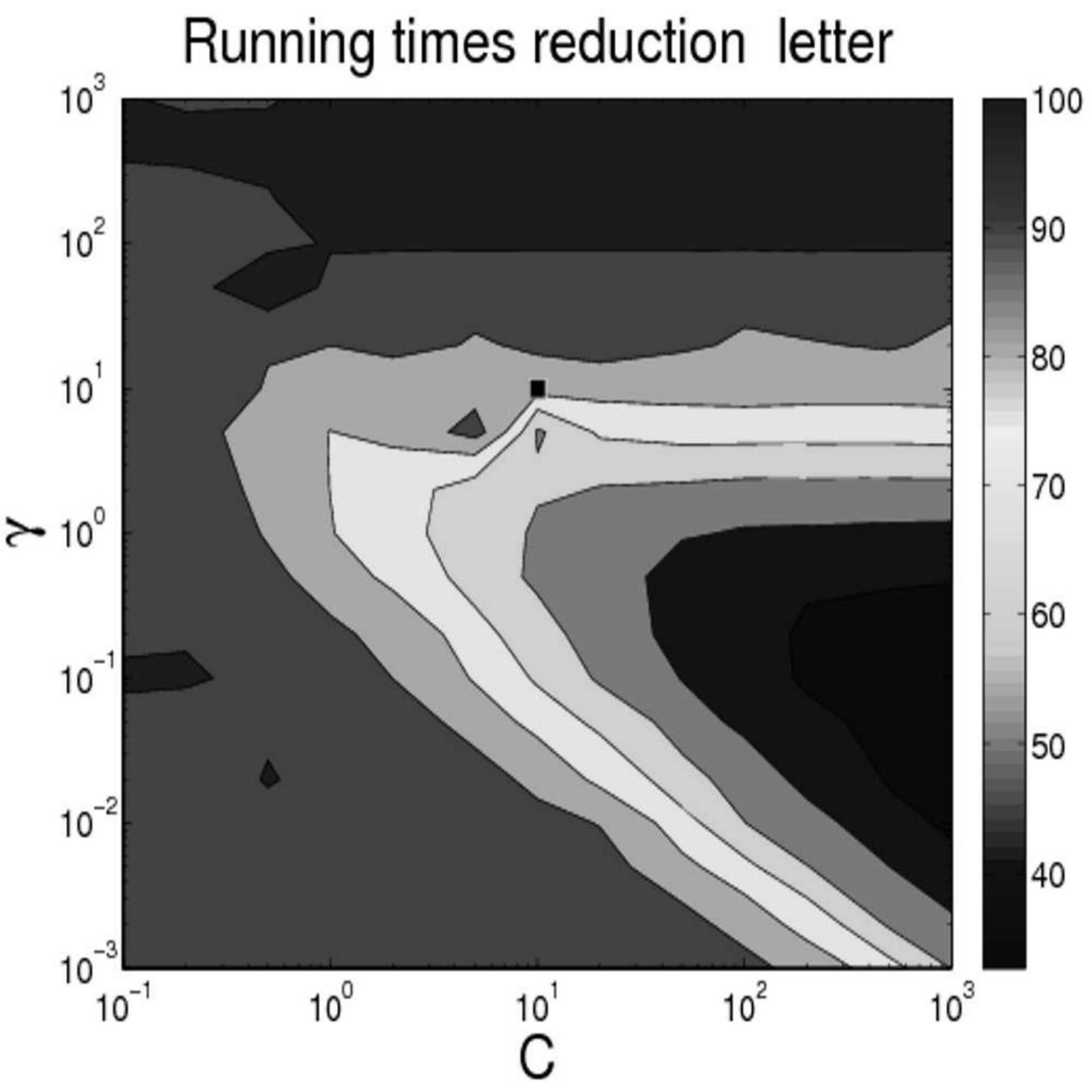} \\
    contour ($1000 \times 1000$)
   }}
  \end{minipage}
  &
  \begin{minipage}{0.3\textwidth}
   \framebox{\parbox{\textwidth}{
    \centering
    \includegraphics[width=1\textwidth]{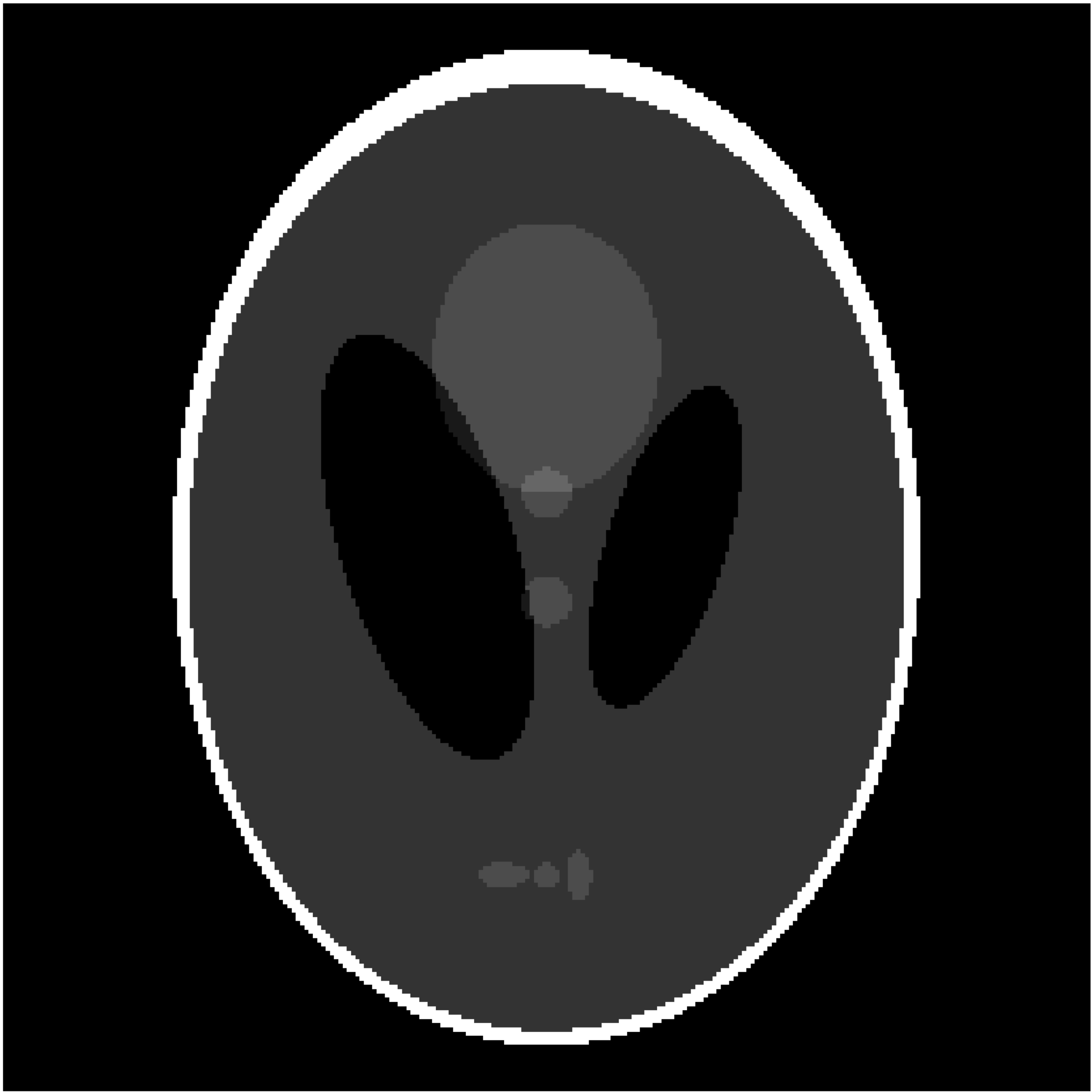} \\
    phantom ($1713 \times 1713$)
   }}
  \end{minipage}
  \end{tabular}

  \caption{
    Test images used in the experiments together with their sizes in pixels. Images displayed have been scaled down to fit in page.
    \label{fig:imagesExp}
  }
\end{figure}

\begin{figure}[htbp]
  \begin{tabular}{lll}
  \abovespace\belowspace
  \begin{minipage}{0.3\textwidth}
   \framebox{\parbox{\textwidth}{
    \centering
    \includegraphics[width=1\textwidth]{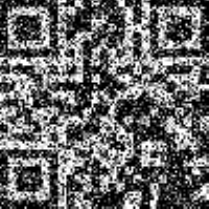} \\
    randomQR-0 ($100 \times 100$)
   }}
  \end{minipage}
  &
  \begin{minipage}{0.3\textwidth}
   \framebox{\parbox{\textwidth}{
    \centering
    \includegraphics[width=1\textwidth]{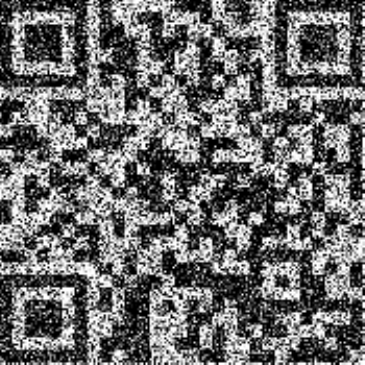} \\
    randomQR-1 ($175 \times 175$)
   }}
  \end{minipage}
  &
  \begin{minipage}{0.3\textwidth}
   \framebox{\parbox{\textwidth}{
    \centering
    \includegraphics[width=1\textwidth]{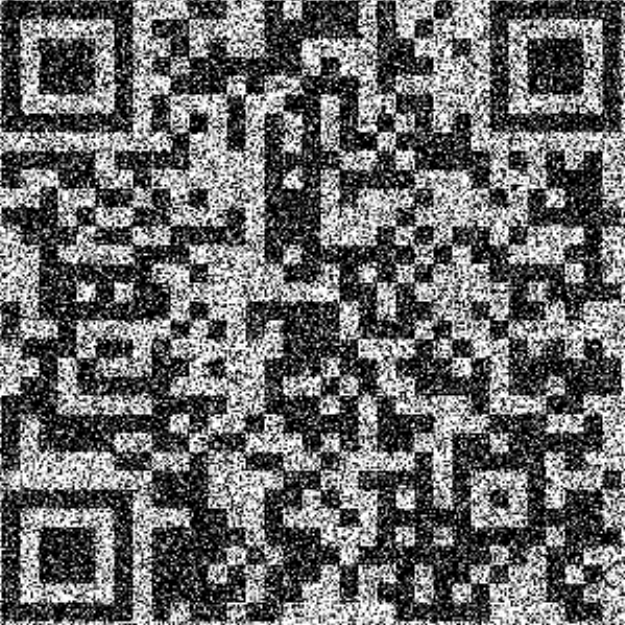} \\
    randomQR-2 ($300 \times 300$)
   }}
  \end{minipage}
  \vspace{0.2cm}
  \\
  \begin{minipage}{0.3\textwidth}
   \framebox{\parbox{\textwidth}{
    \centering
    \includegraphics[width=1\textwidth]{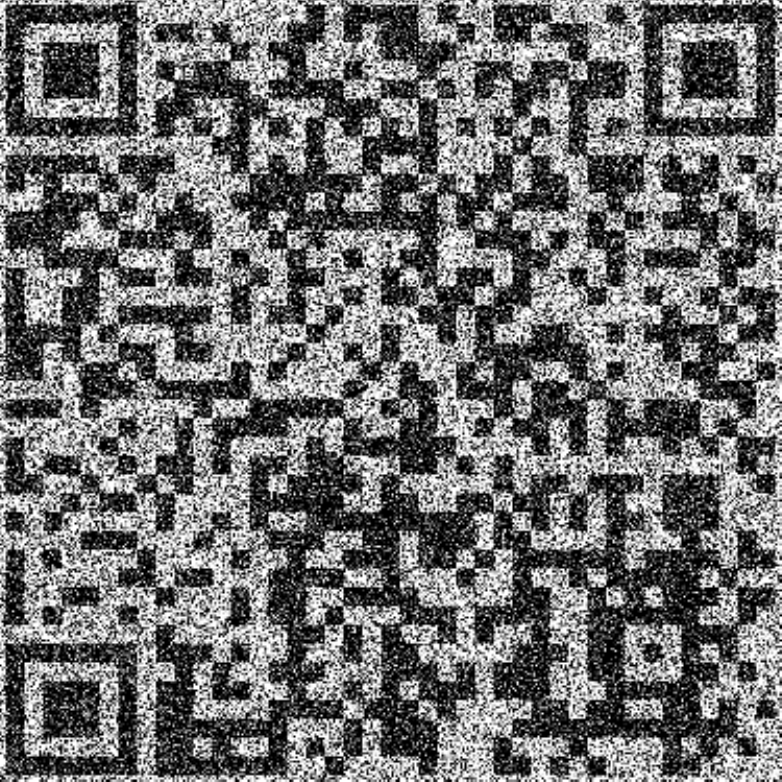} \\
    randomQR-3 ($375 \times 375$)
   }}
  \end{minipage}
  &
  \begin{minipage}{0.3\textwidth}
   \framebox{\parbox{\textwidth}{
    \centering
    \includegraphics[width=1\textwidth]{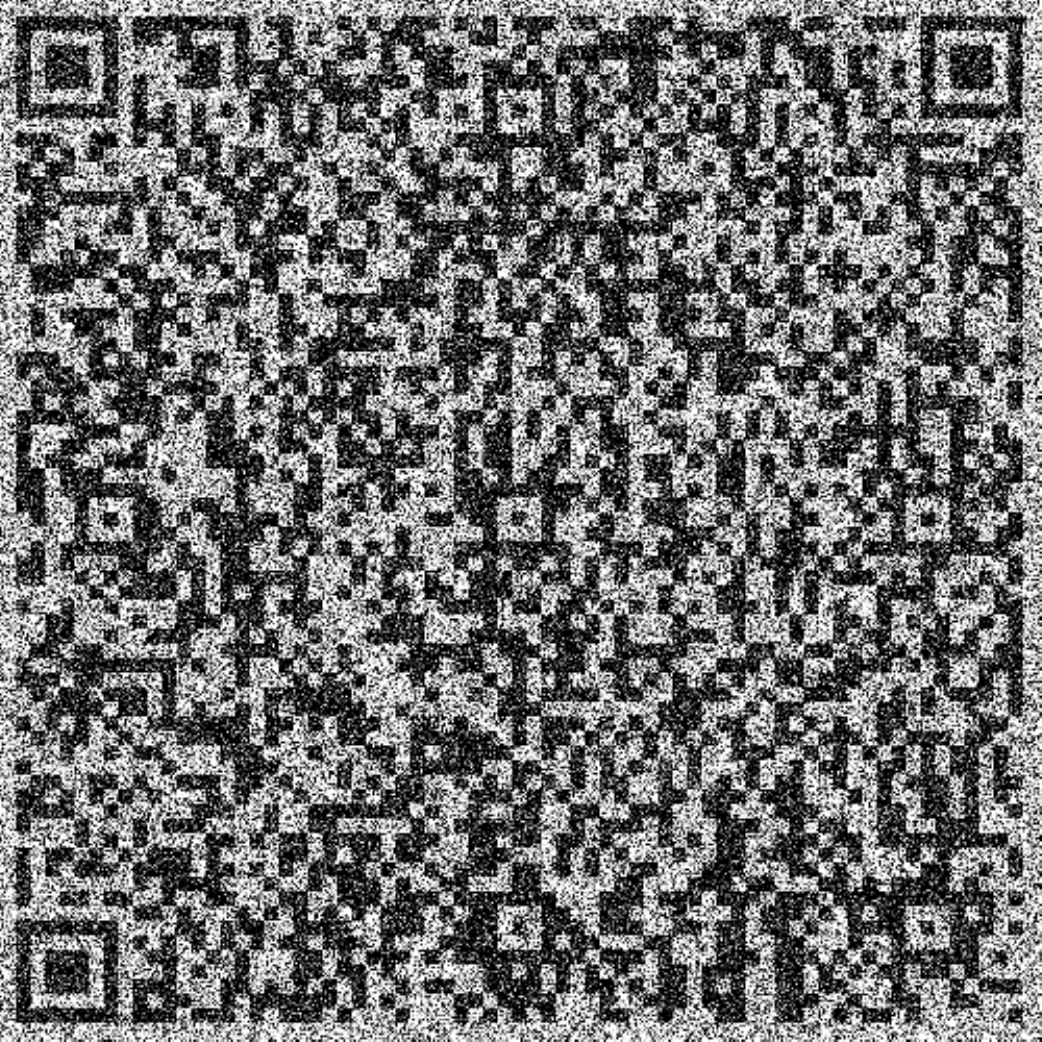} \\
    randomQR-4 ($500 \times 500$)
   }}
  \end{minipage}
  &
  \begin{minipage}{0.3\textwidth}
   \framebox{\parbox{\textwidth}{
    \centering
    \includegraphics[width=1\textwidth]{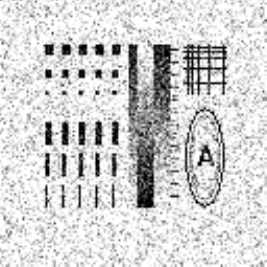} \\
    shape ($128 \times 128$)
   }}
  \end{minipage}
  \vspace{0.2cm}
  \\
  \abovespace\belowspace
  \begin{minipage}{0.3\textwidth}
   \framebox{\parbox{\textwidth}{
    \centering
    \includegraphics[width=1\textwidth]{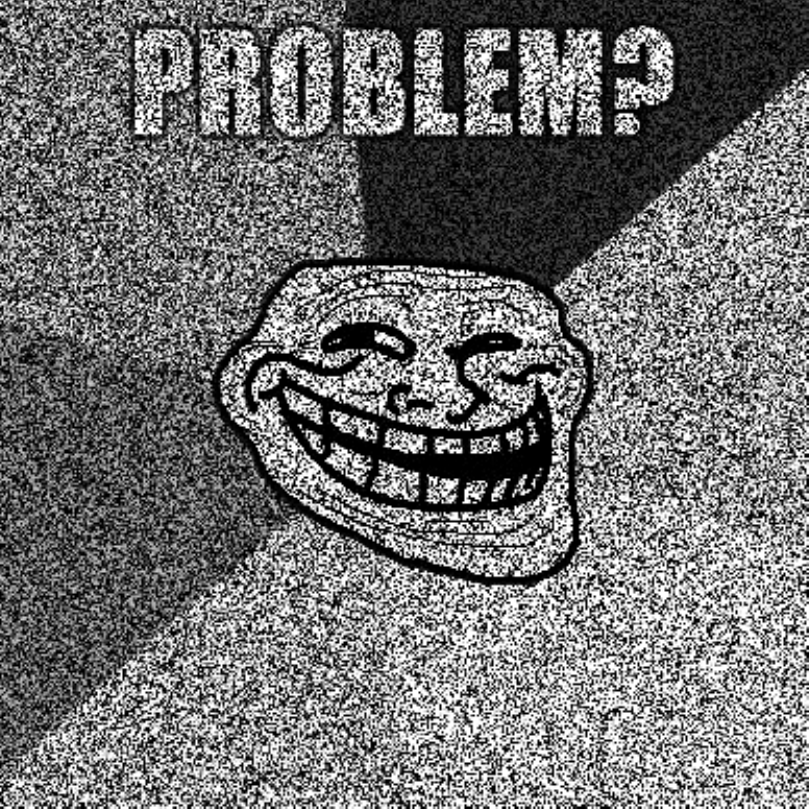} \\
    trollface ($388 \times 388$)
   }}
  \end{minipage}
  &
  \begin{minipage}{0.3\textwidth}
   \framebox{\parbox{\textwidth}{
    \centering
    \includegraphics[width=1\textwidth]{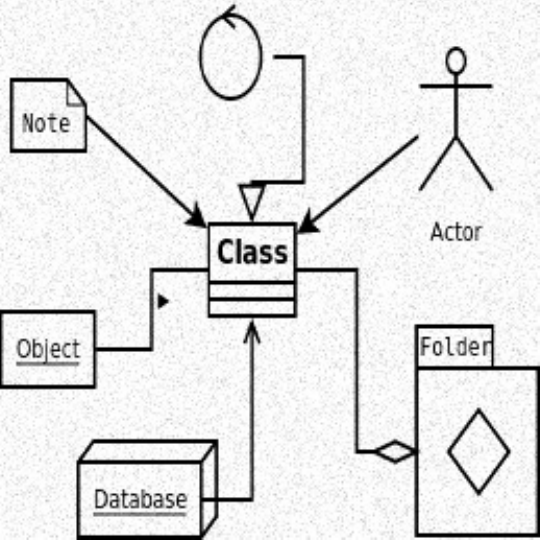} \\
    diagram ($259 \times 259$)
   }}
  \end{minipage}
  &
  \begin{minipage}{0.3\textwidth}
   \framebox{\parbox{\textwidth}{
    \centering
    \includegraphics[width=1\textwidth]{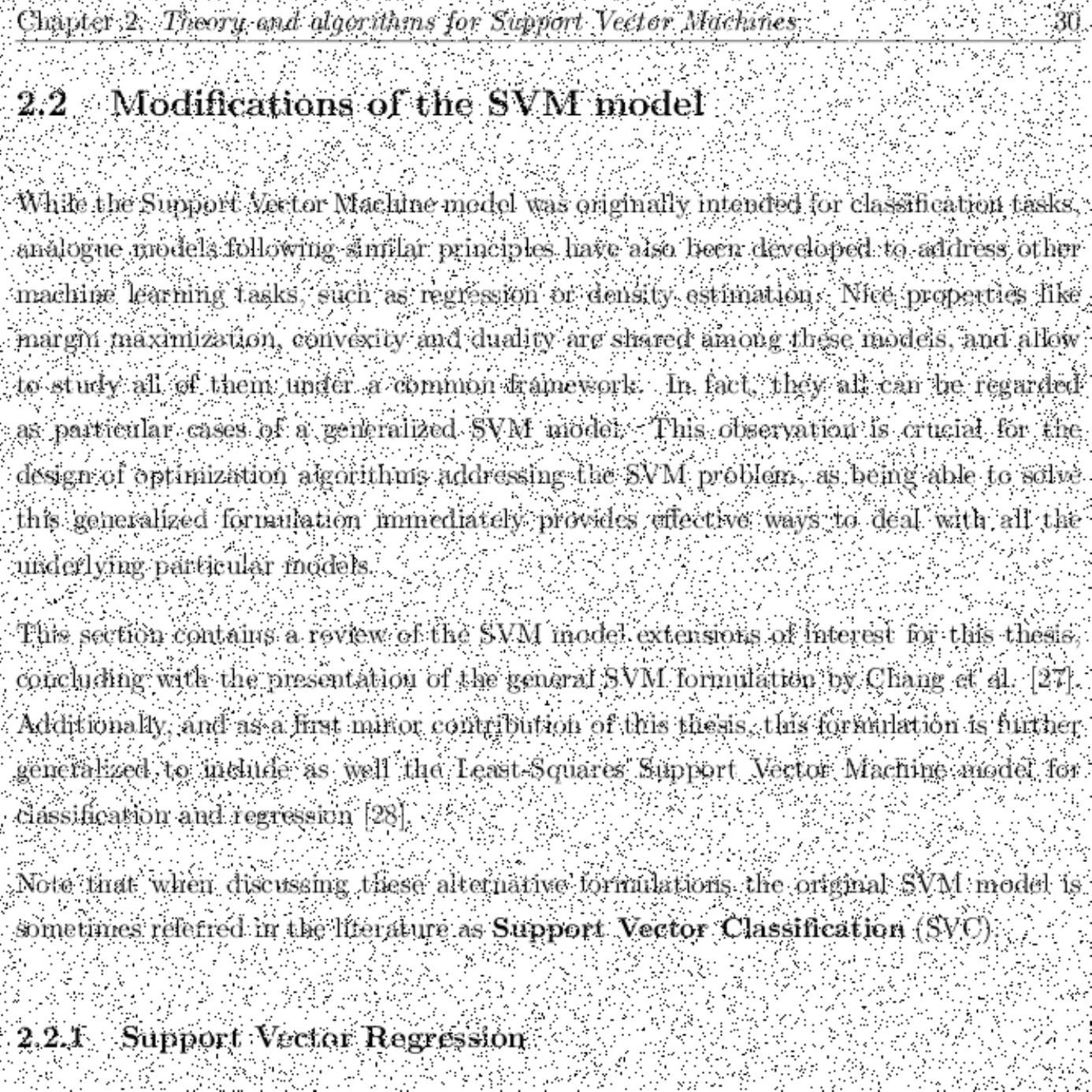} \\
    text ($665 \times 665$)
   }}
  \end{minipage}
  \vspace{0.2cm}
  \\
  \abovespace\belowspace
  \begin{minipage}{0.3\textwidth}
   \framebox{\parbox{\textwidth}{
    \centering
    \includegraphics[width=1\textwidth]{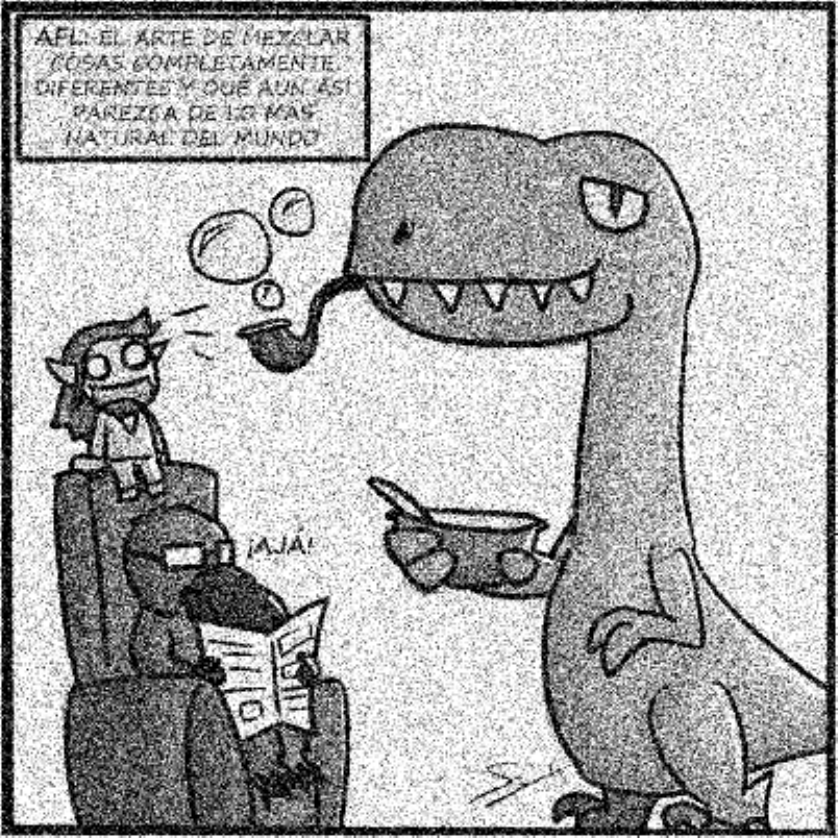} \\
    comic ($402 \times 402$)
   }}
  \end{minipage}
  &
  \begin{minipage}{0.3\textwidth}
   \framebox{\parbox{\textwidth}{
    \centering
    \includegraphics[width=1\textwidth]{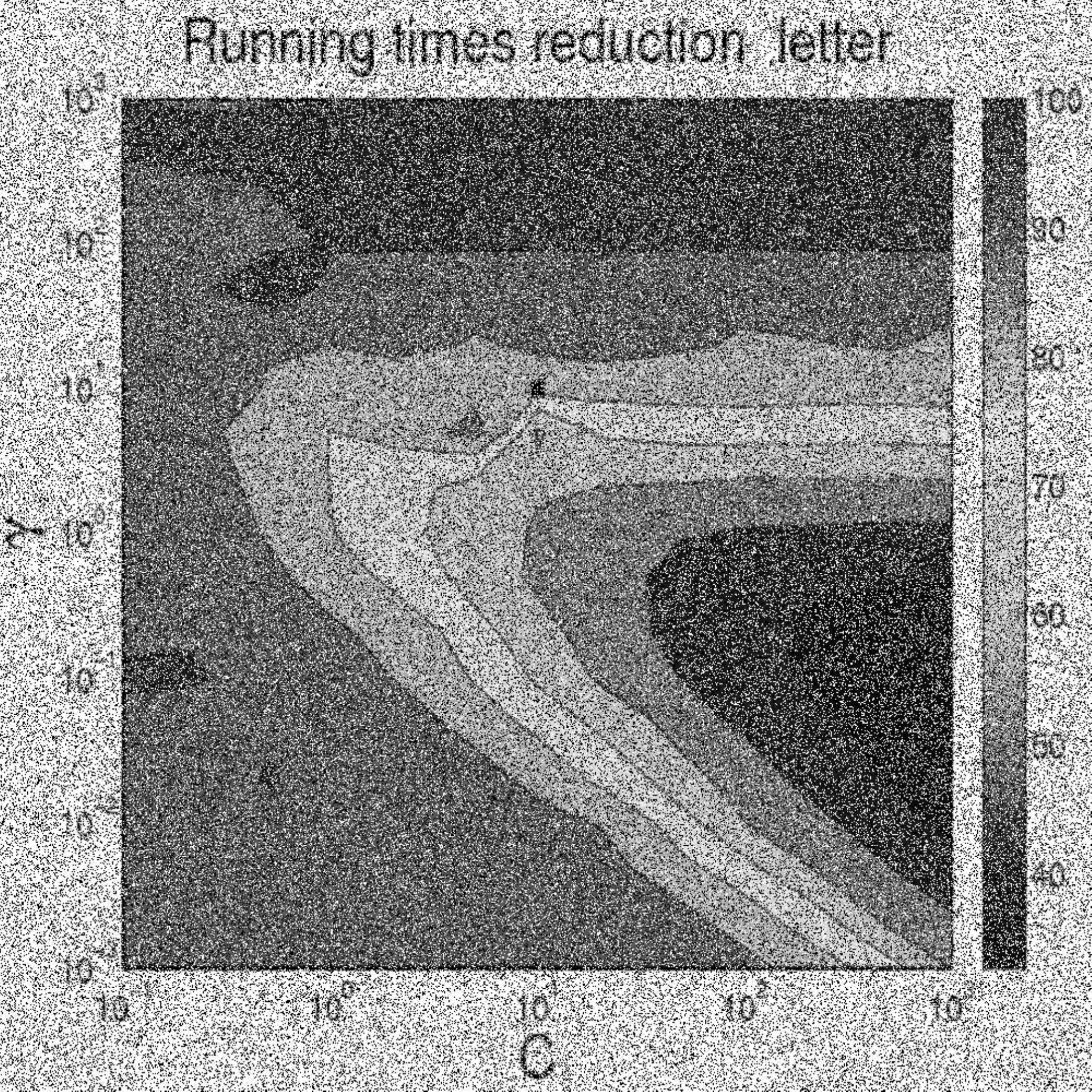} \\
    contour ($1000 \times 1000$)
   }}
  \end{minipage}
  &
  \begin{minipage}{0.3\textwidth}
   \framebox{\parbox{\textwidth}{
    \centering
    \includegraphics[width=1\textwidth]{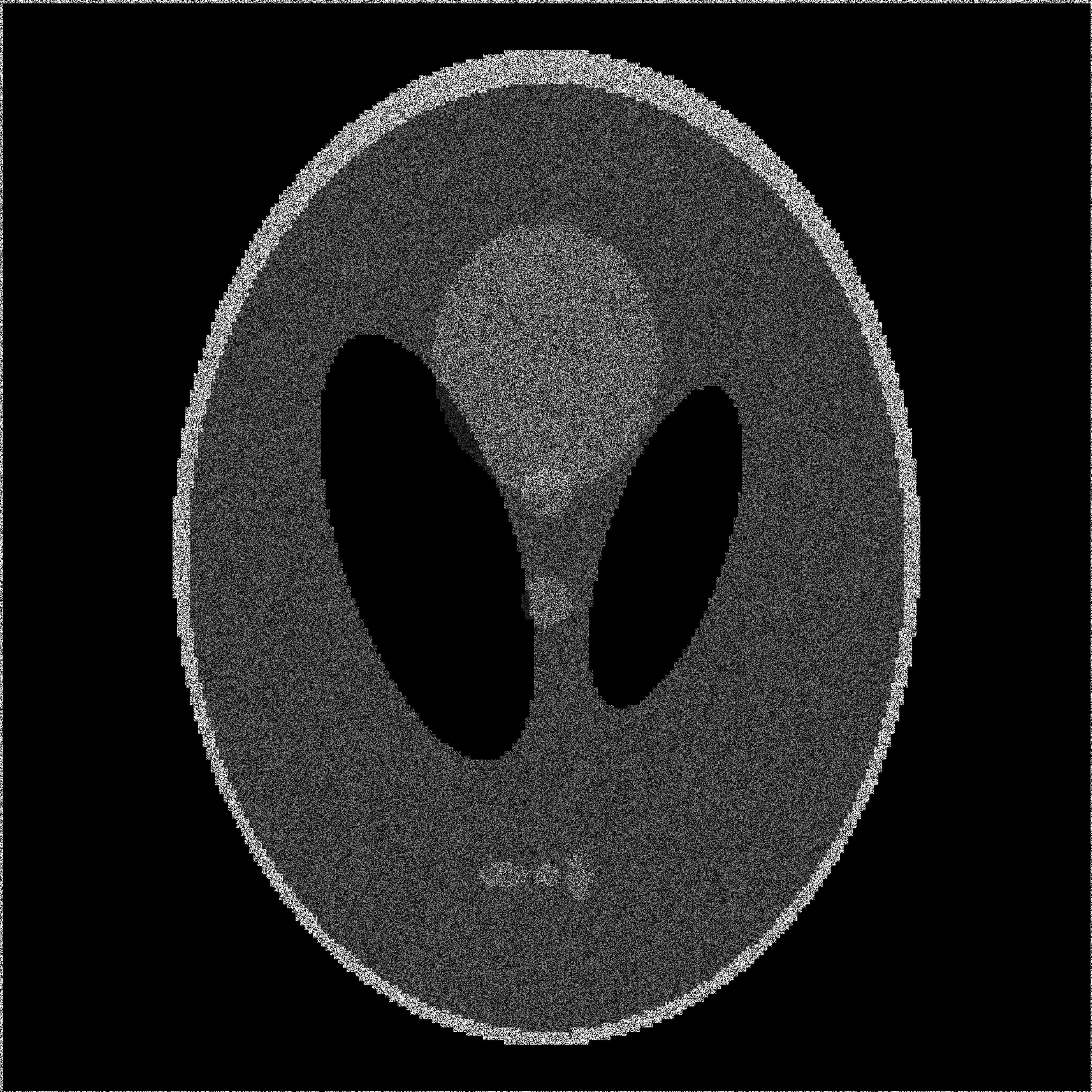} \\
    phantom ($1713 \times 1713$)
   }}
  \end{minipage}
  \end{tabular}

  \caption{
    Noisy versions of images used in the experiments.
    \label{fig:imagesNoised}
  }
\end{figure}

\begin{figure}[htbp]
  \begin{tabular}{lll}
  \abovespace\belowspace
  \begin{minipage}{0.3\textwidth}
   \framebox{\parbox{\textwidth}{
    \centering
    \includegraphics[width=1\textwidth]{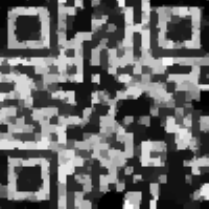} \\
    randomQR-0 ($100 \times 100$)
   }}
  \end{minipage}
  &
  \begin{minipage}{0.3\textwidth}
   \framebox{\parbox{\textwidth}{
    \centering
    \includegraphics[width=1\textwidth]{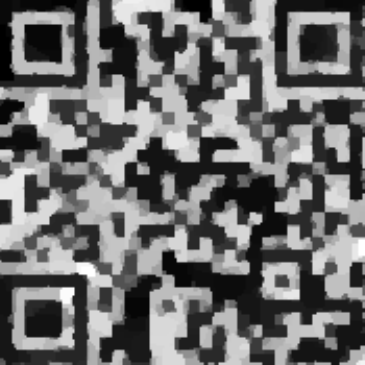} \\
    randomQR-1 ($175 \times 175$)
   }}
  \end{minipage}
  &
  \begin{minipage}{0.3\textwidth}
   \framebox{\parbox{\textwidth}{
    \centering
    \includegraphics[width=1\textwidth]{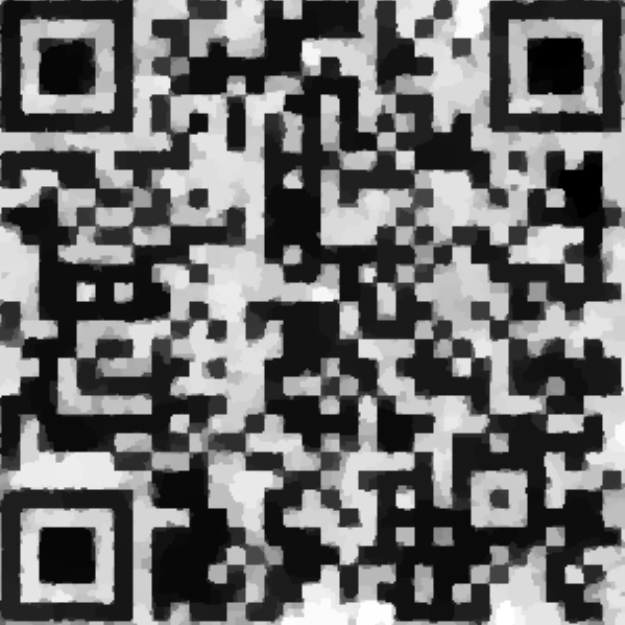} \\
    randomQR-2 ($300 \times 300$)
   }}
  \end{minipage}
  \vspace{0.2cm}
  \\
  \begin{minipage}{0.3\textwidth}
   \framebox{\parbox{\textwidth}{
    \centering
    \includegraphics[width=1\textwidth]{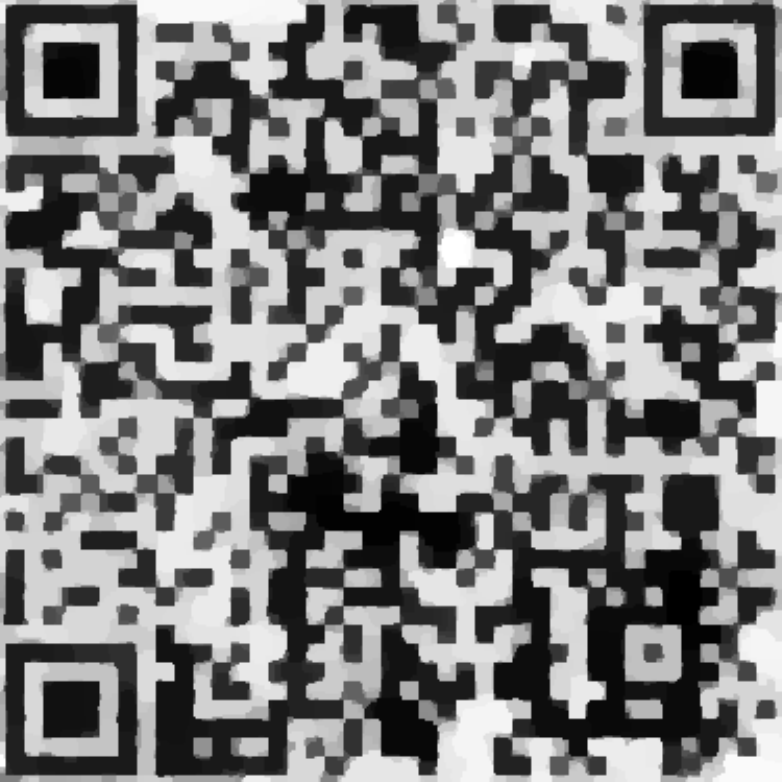} \\
    randomQR-3 ($375 \times 375$)
   }}
  \end{minipage}
  &
  \begin{minipage}{0.3\textwidth}
   \framebox{\parbox{\textwidth}{
    \centering
    \includegraphics[width=1\textwidth]{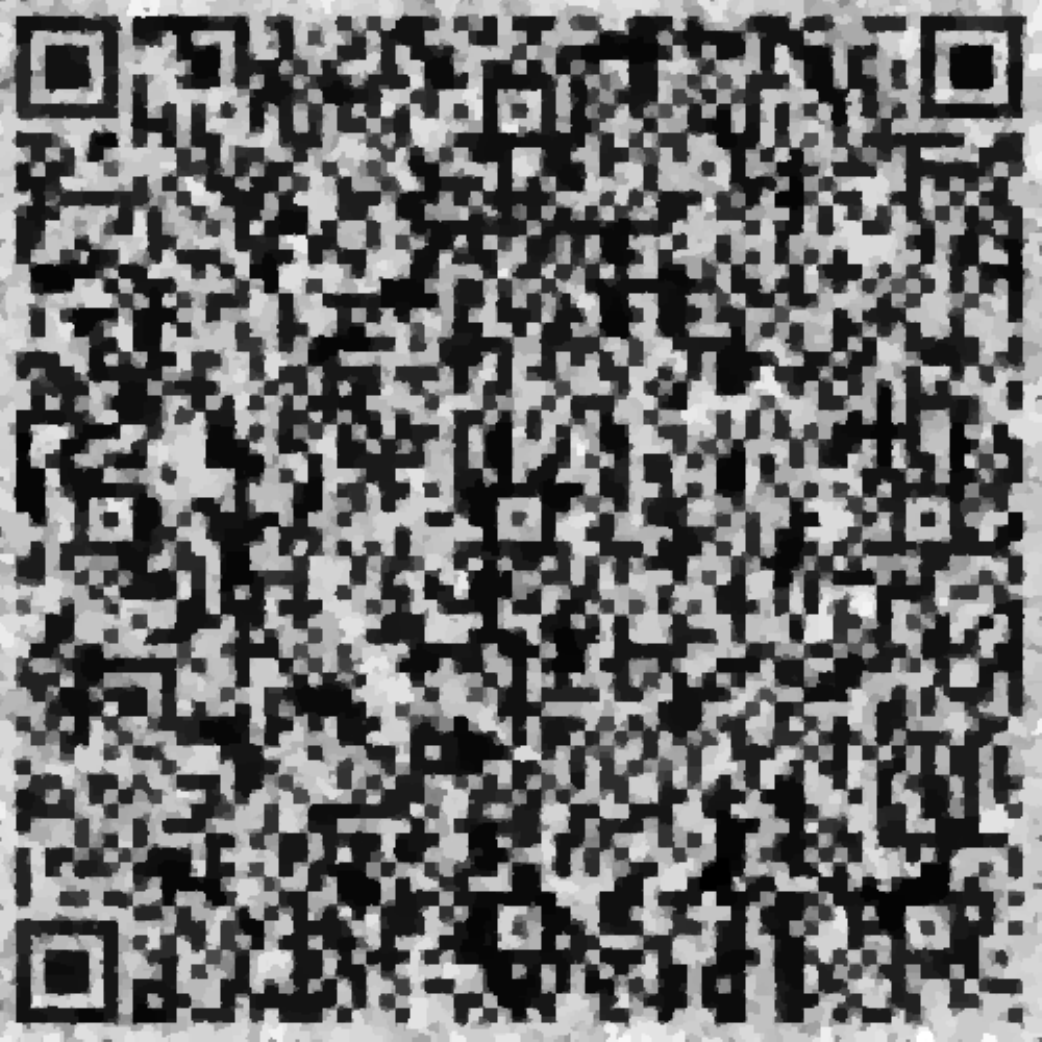} \\
    randomQR-4 ($500 \times 500$)
   }}
  \end{minipage}
  &
  \begin{minipage}{0.3\textwidth}
   \framebox{\parbox{\textwidth}{
    \centering
    \includegraphics[width=1\textwidth]{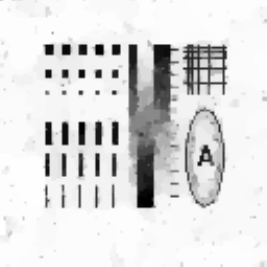} \\
    shape ($128 \times 128$)
   }}
  \end{minipage}
  \vspace{0.2cm}
  \\
  \abovespace\belowspace
  \begin{minipage}{0.3\textwidth}
   \framebox{\parbox{\textwidth}{
    \centering
    \includegraphics[width=1\textwidth]{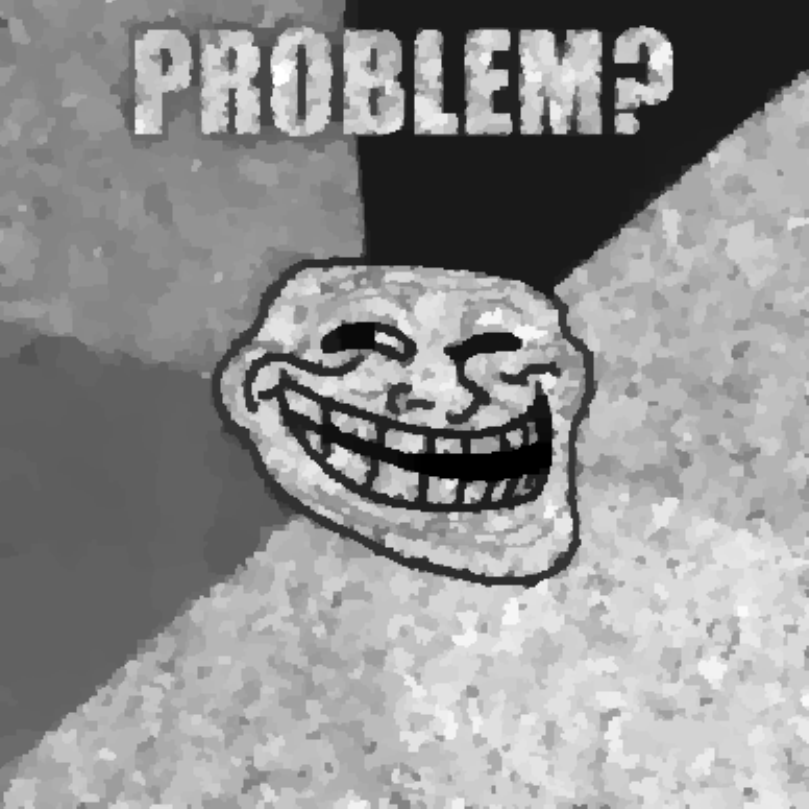} \\
    trollface ($388 \times 388$)
   }}
  \end{minipage}
  &
  \begin{minipage}{0.3\textwidth}
   \framebox{\parbox{\textwidth}{
    \centering
    \includegraphics[width=1\textwidth]{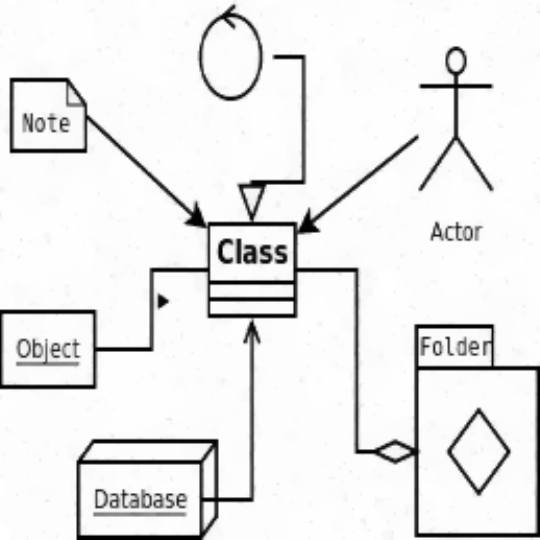} \\
    diagram ($259 \times 259$)
   }}
  \end{minipage}
  &
  \begin{minipage}{0.3\textwidth}
   \framebox{\parbox{\textwidth}{
    \centering
    \includegraphics[width=1\textwidth]{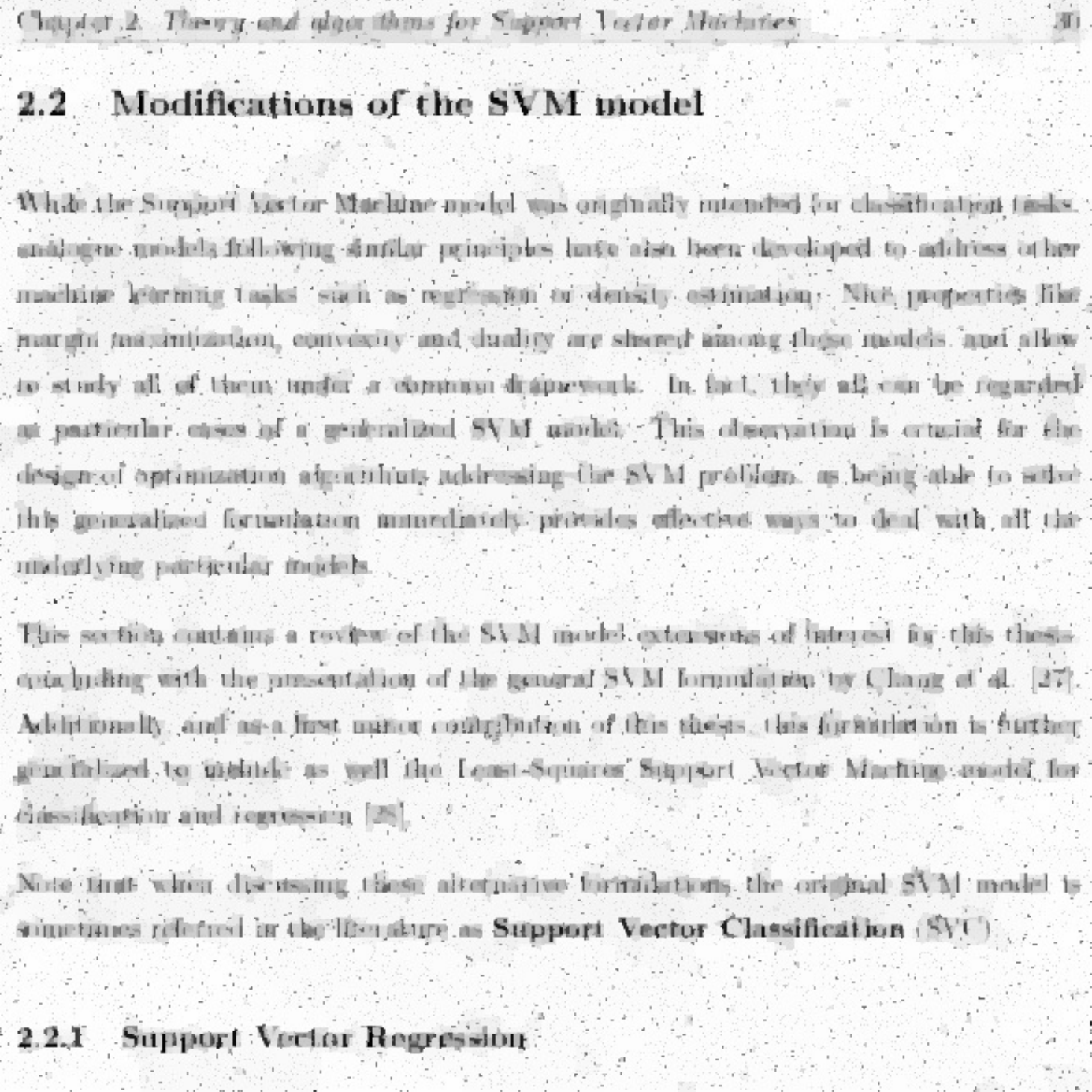} \\
    text ($665 \times 665$)
   }}
  \end{minipage}
  \vspace{0.2cm}
  \\
  \abovespace\belowspace
  \begin{minipage}{0.3\textwidth}
   \framebox{\parbox{\textwidth}{
    \centering
    \includegraphics[width=1\textwidth]{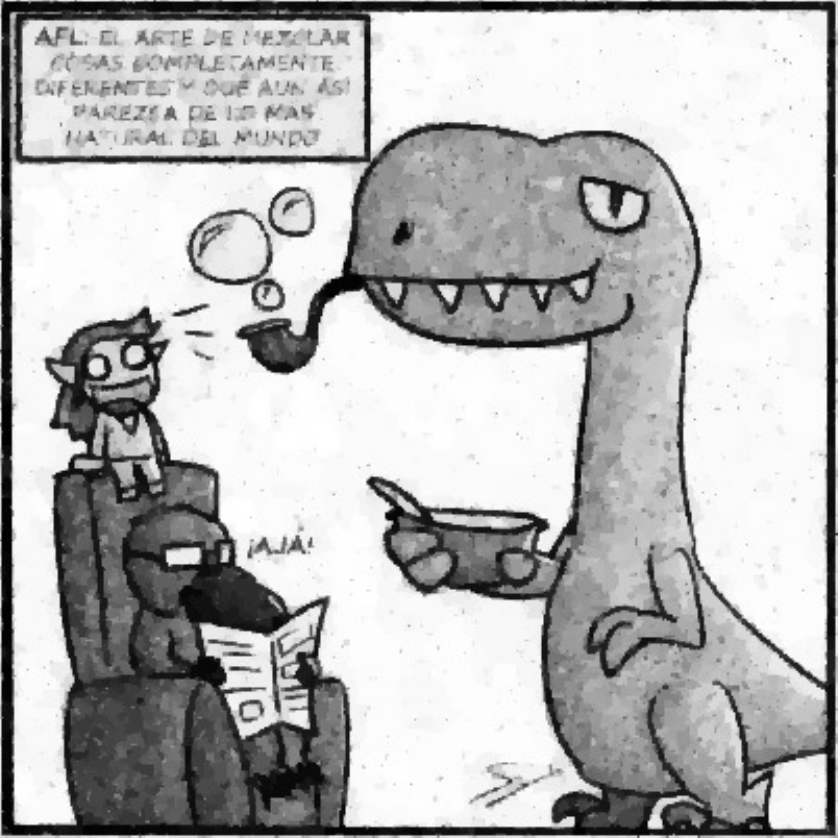} \\
    comic ($402 \times 402$)
   }}
  \end{minipage}
  &
  \begin{minipage}{0.3\textwidth}
   \framebox{\parbox{\textwidth}{
    \centering
    \includegraphics[width=1\textwidth]{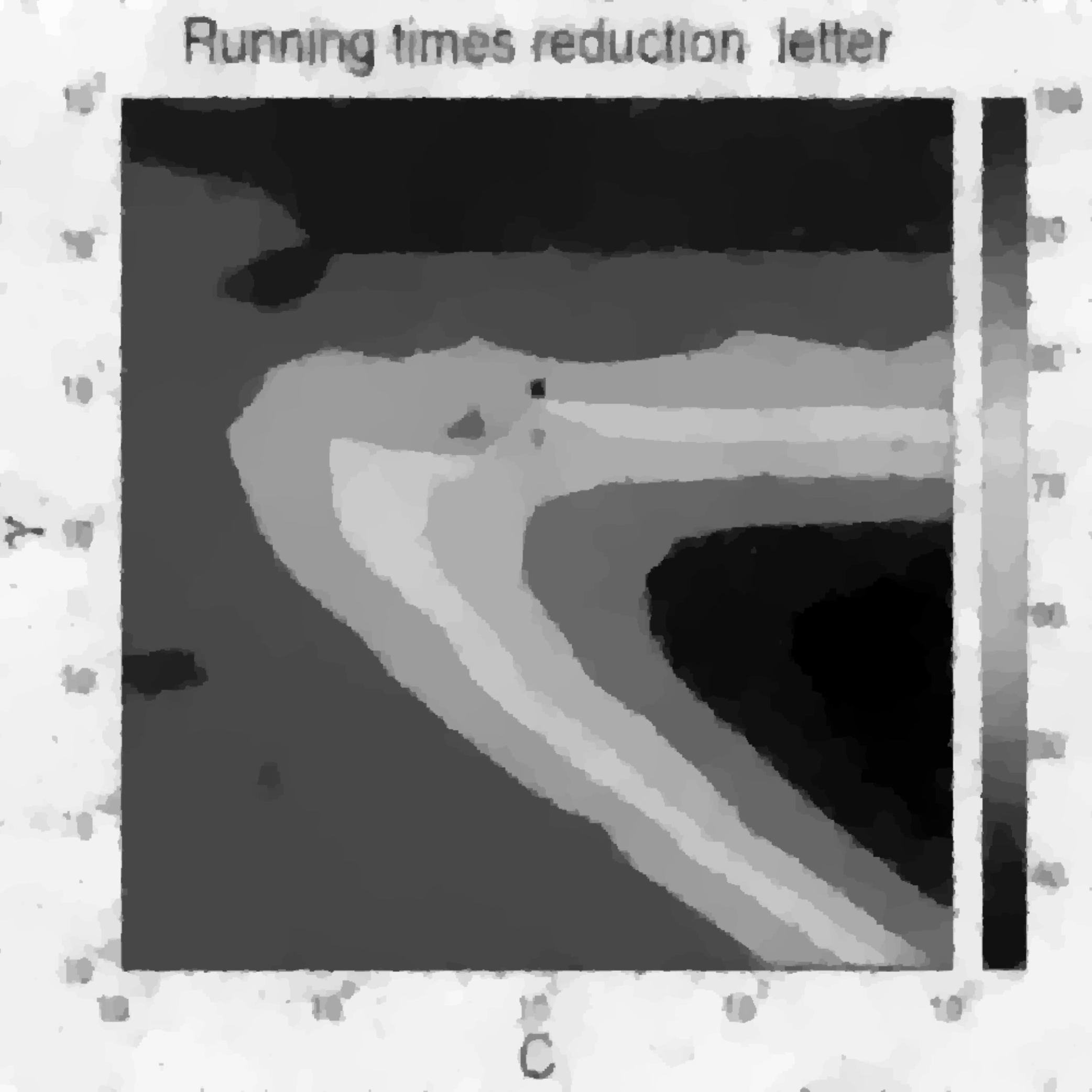} \\
    contour ($1000 \times 1000$)
   }}
  \end{minipage}
  &
  \begin{minipage}{0.3\textwidth}
   \framebox{\parbox{\textwidth}{
    \centering
    \includegraphics[width=1\textwidth]{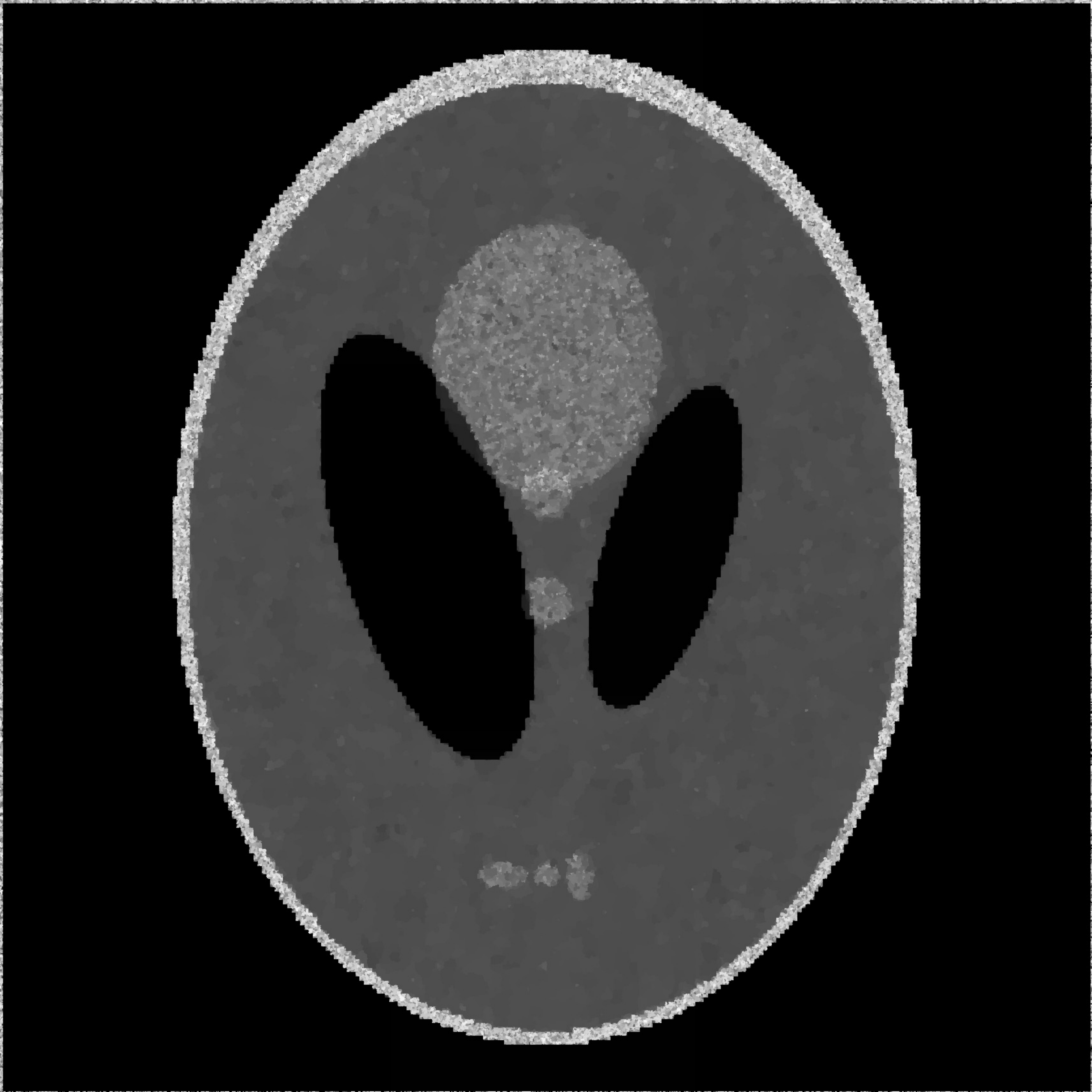} \\
    phantom ($1713 \times 1713$)
   }}
  \end{minipage}
  \end{tabular}

  \caption{
    Denoising results for the test images.
    \label{fig:imagesDenoised}
  }
\end{figure}


\begin{figure}[htbp]
  \begin{tabular}{lll}
  \abovespace\belowspace
  \begin{minipage}{0.3\textwidth}
   \framebox{\parbox{\textwidth}{
    \centering
    \includegraphics[width=1\textwidth]{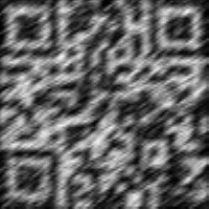} \\
    randomQR-0 ($100 \times 100$)
   }}
  \end{minipage}
  &
  \begin{minipage}{0.3\textwidth}
   \framebox{\parbox{\textwidth}{
    \centering
    \includegraphics[width=1\textwidth]{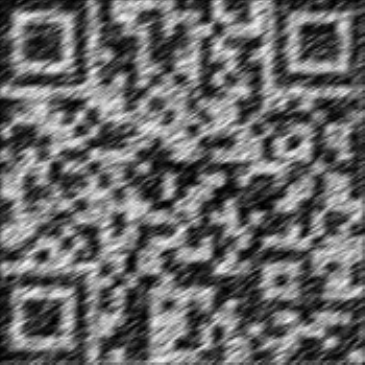} \\
    randomQR-1 ($175 \times 175$)
   }}
  \end{minipage}
  &
  \begin{minipage}{0.3\textwidth}
   \framebox{\parbox{\textwidth}{
    \centering
    \includegraphics[width=1\textwidth]{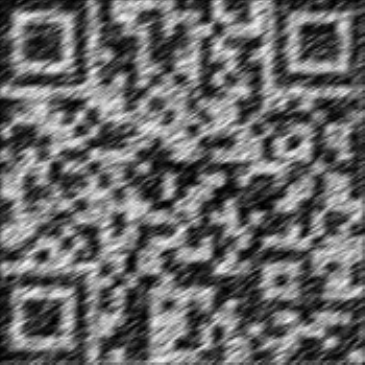} \\
    randomQR-2 ($300 \times 300$)
   }}
  \end{minipage}
  \vspace{0.2cm}
  \\
  \begin{minipage}{0.3\textwidth}
   \framebox{\parbox{\textwidth}{
    \centering
    \includegraphics[width=1\textwidth]{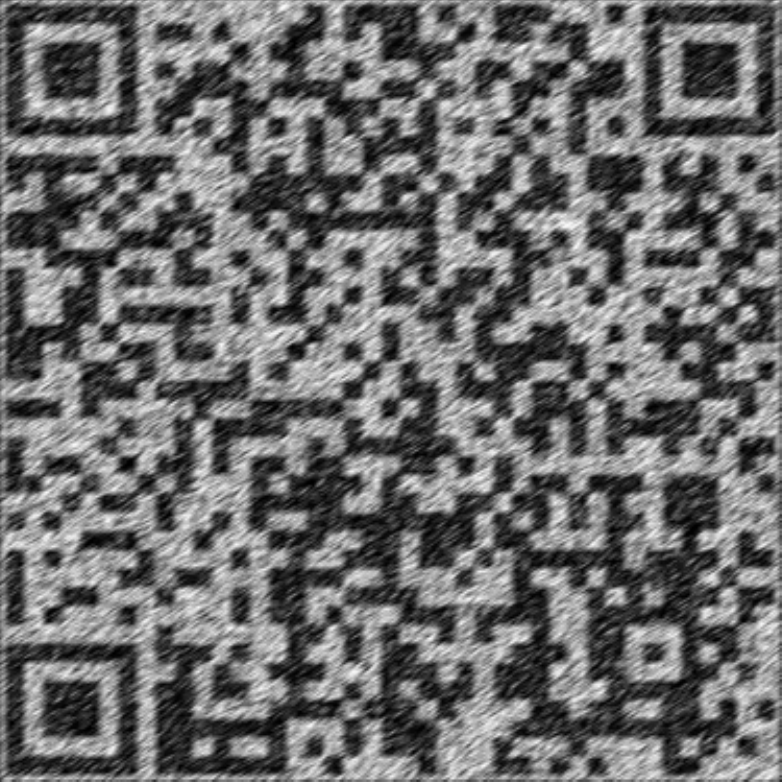} \\
    randomQR-3 ($375 \times 375$)
   }}
  \end{minipage}
  &
  \begin{minipage}{0.3\textwidth}
   \framebox{\parbox{\textwidth}{
    \centering
    \includegraphics[width=1\textwidth]{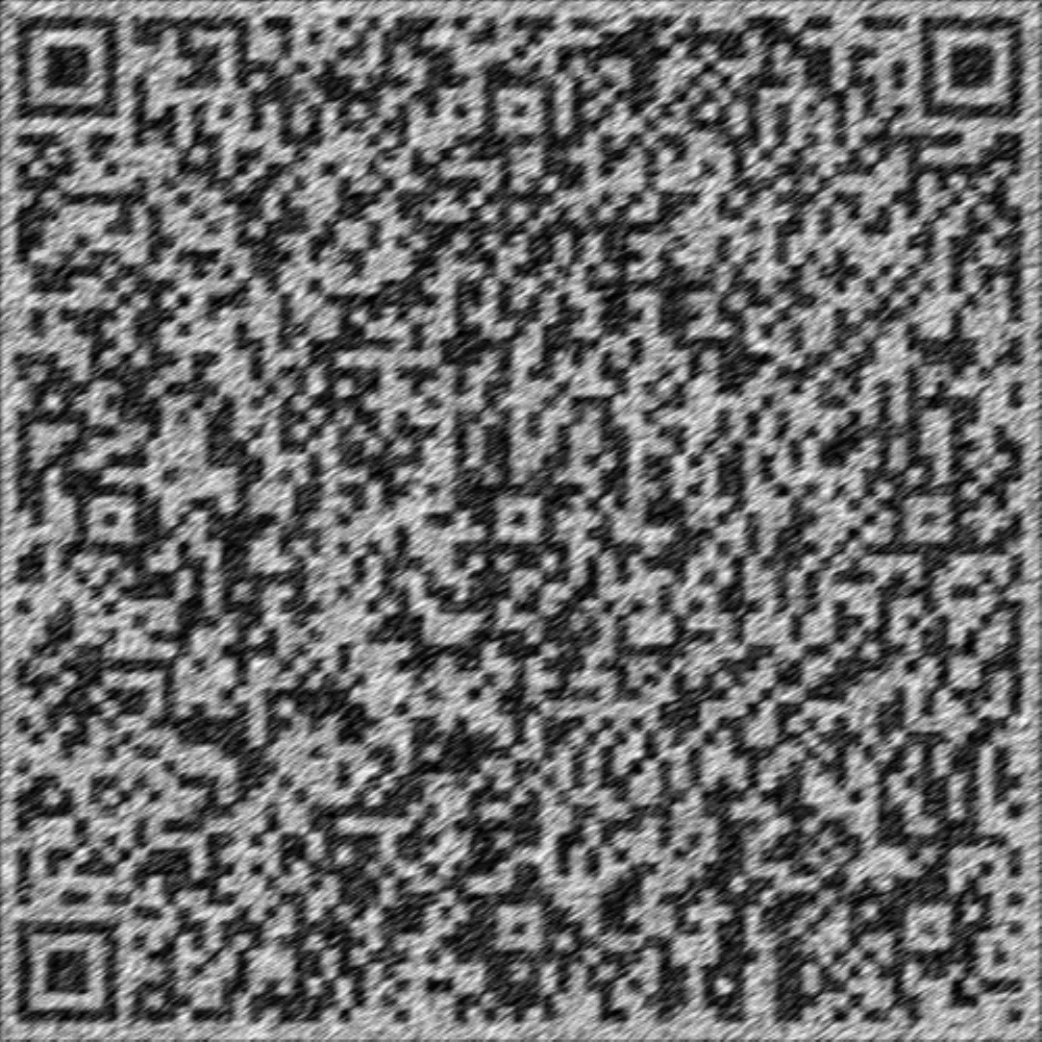} \\
    randomQR-4 ($500 \times 500$)
   }}
  \end{minipage}
  &
  \begin{minipage}{0.3\textwidth}
   \framebox{\parbox{\textwidth}{
    \centering
    \includegraphics[width=1\textwidth]{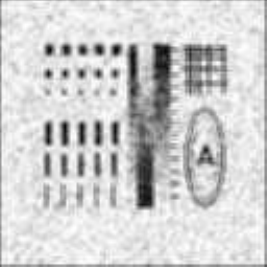} \\
    shape ($128 \times 128$)
   }}
  \end{minipage}
  \vspace{0.2cm}
  \\
  \abovespace\belowspace
  \begin{minipage}{0.3\textwidth}
   \framebox{\parbox{\textwidth}{
    \centering
    \includegraphics[width=1\textwidth]{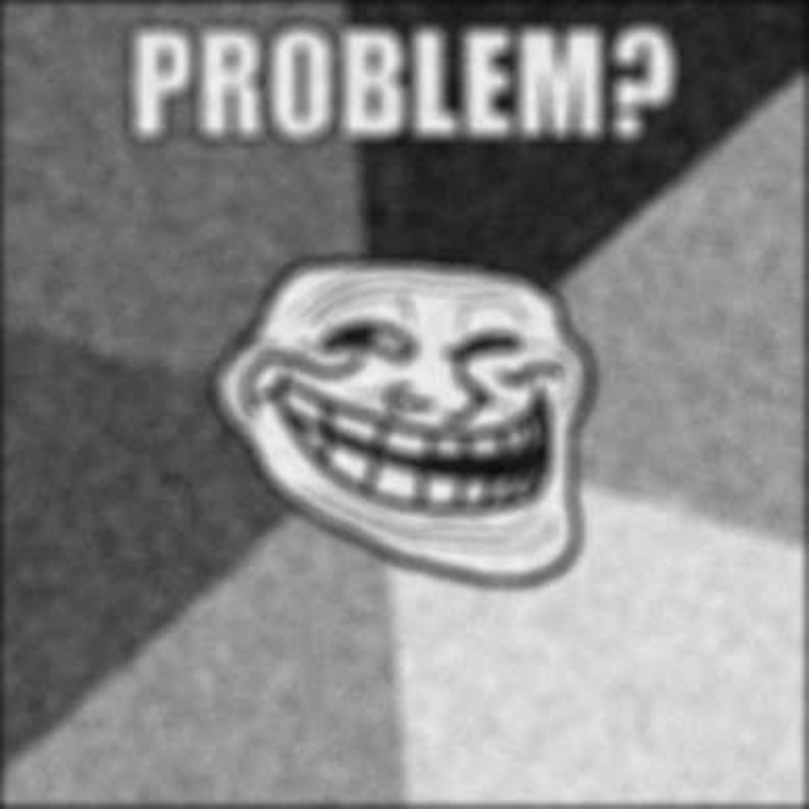} \\
    trollface ($388 \times 388$)
   }}
  \end{minipage}
  &
  \begin{minipage}{0.3\textwidth}
   \framebox{\parbox{\textwidth}{
    \centering
    \includegraphics[width=1\textwidth]{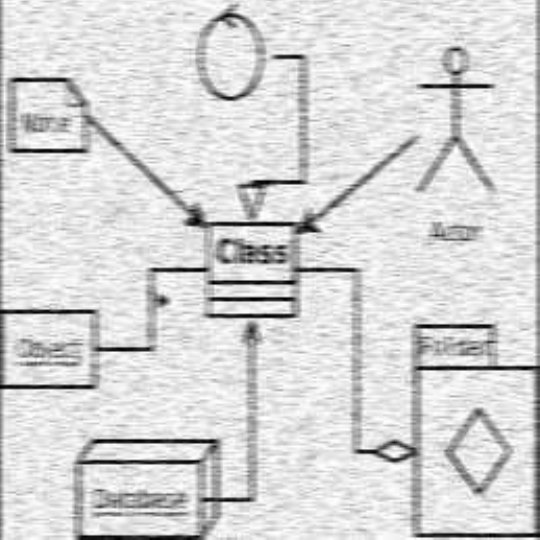} \\
    diagram ($259 \times 259$)
   }}
  \end{minipage}
  &
  \begin{minipage}{0.3\textwidth}
   \framebox{\parbox{\textwidth}{
    \centering
    \includegraphics[width=1\textwidth]{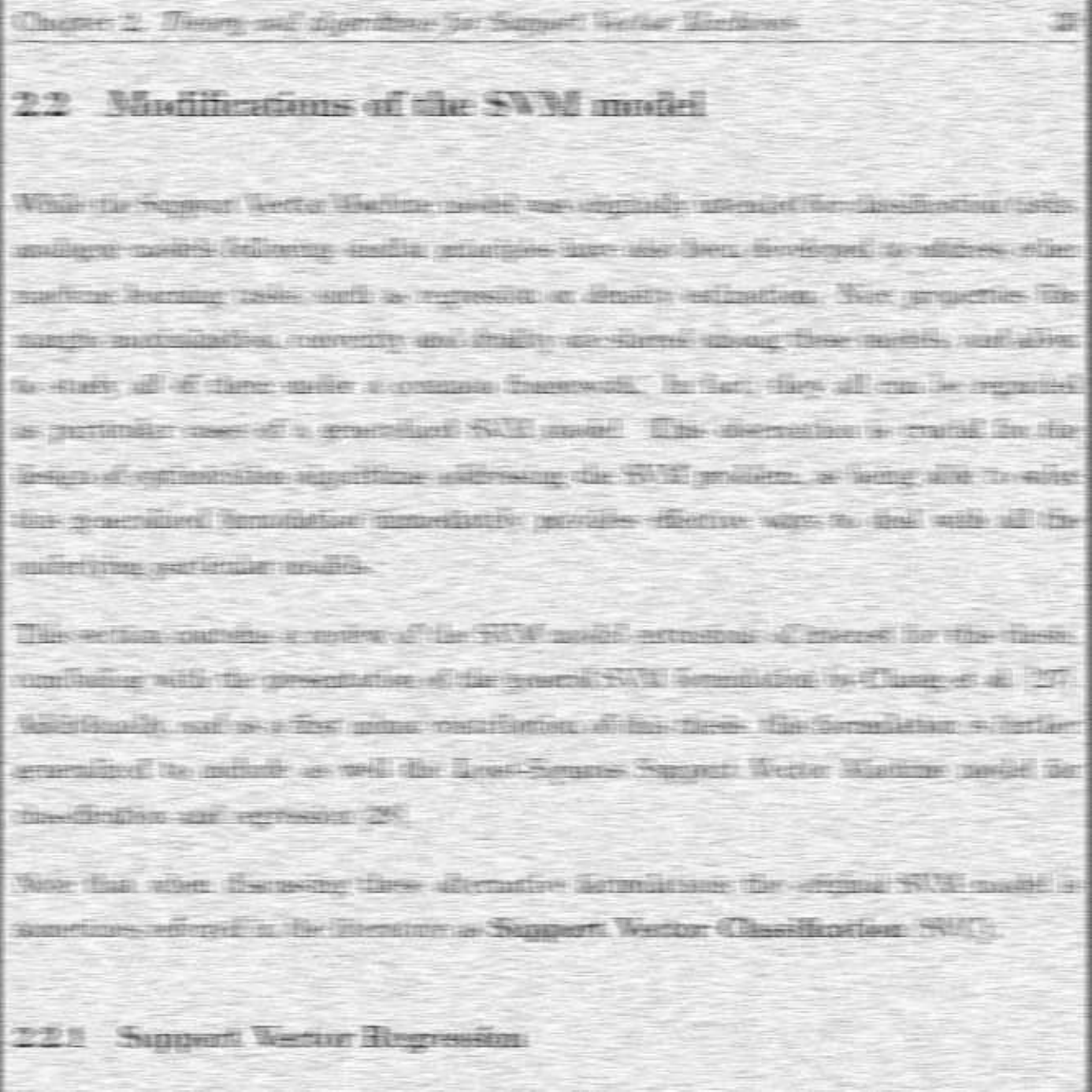} \\
    text ($665 \times 665$)
   }}
  \end{minipage}
  \vspace{0.2cm}
  \\
  \abovespace\belowspace
  \begin{minipage}{0.3\textwidth}
   \framebox{\parbox{\textwidth}{
    \centering
    \includegraphics[width=1\textwidth]{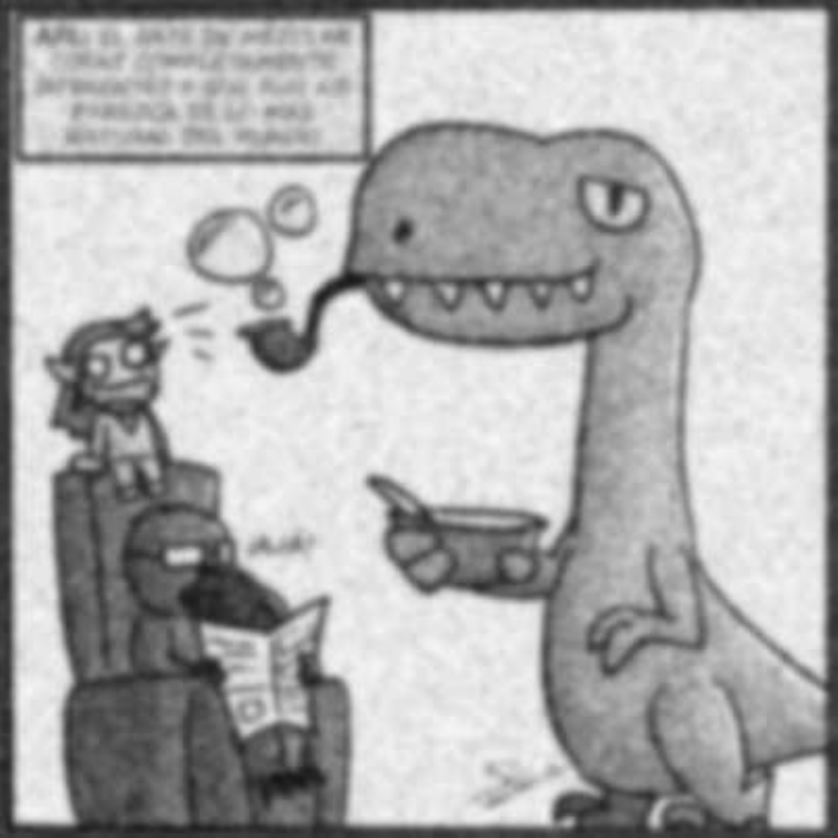} \\
    comic ($402 \times 402$)
   }}
  \end{minipage}
  &
  \begin{minipage}{0.3\textwidth}
   \framebox{\parbox{\textwidth}{
    \centering
    \includegraphics[width=1\textwidth]{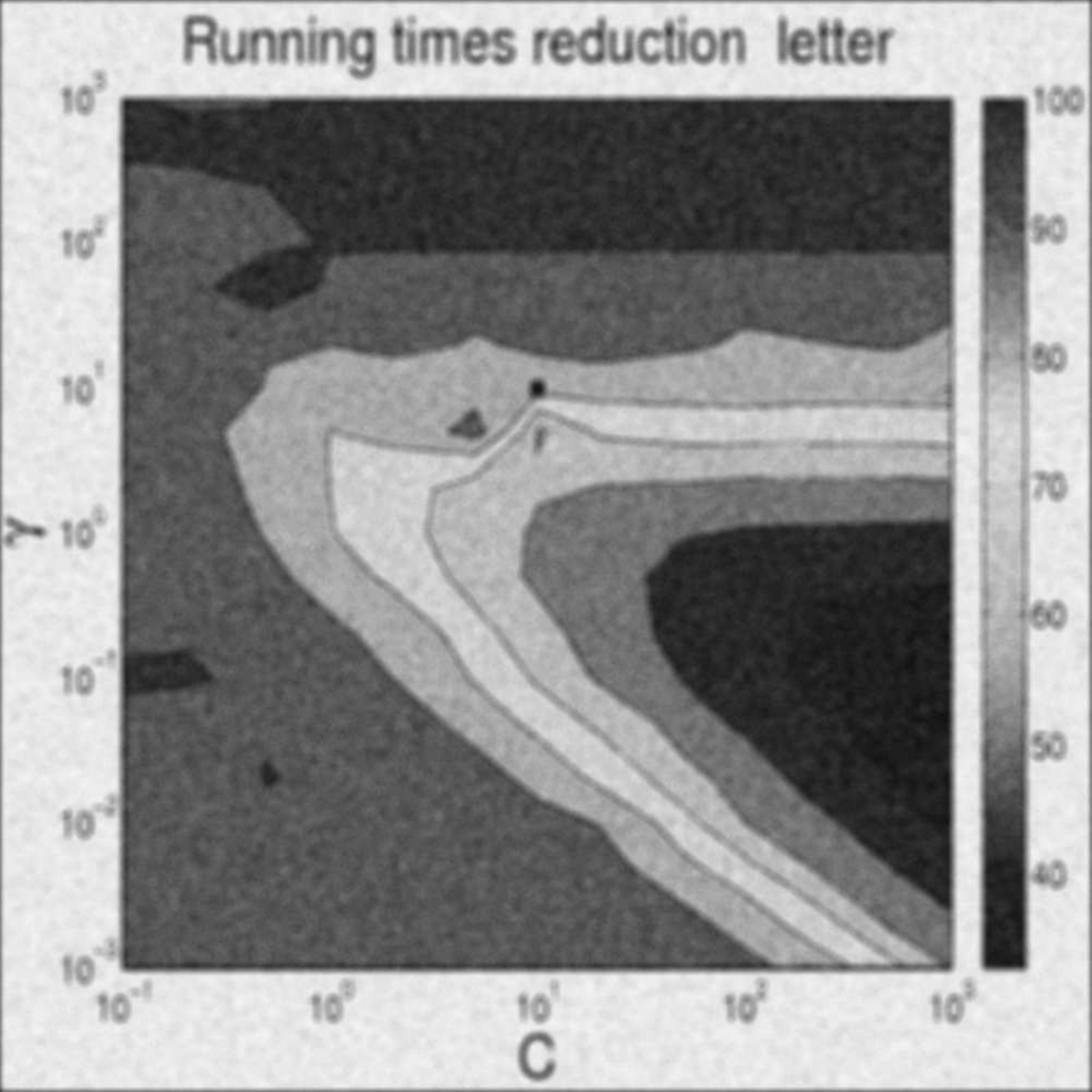} \\
    contour ($1000 \times 1000$)
   }}
  \end{minipage}
  &
  \begin{minipage}{0.3\textwidth}
   \framebox{\parbox{\textwidth}{
    \centering
    \includegraphics[width=1\textwidth]{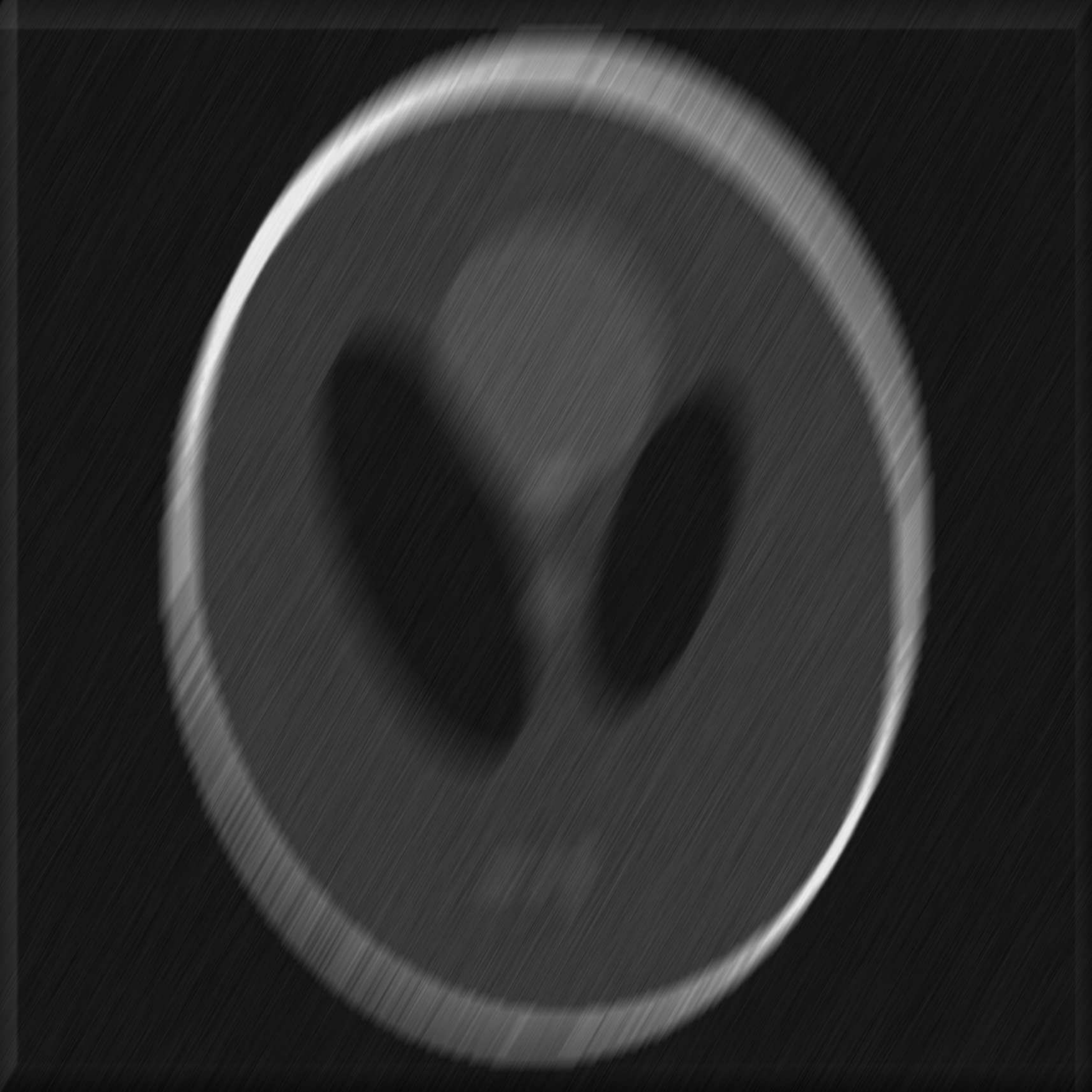} \\
    phantom ($1713 \times 1713$)
   }}
  \end{minipage}
  \end{tabular}

  \caption{
    Noisy and convoluted versions of images used in the experiments.
    \label{fig:imagesBlurred}
  }
\end{figure}

\begin{figure}[htbp]
  \begin{tabular}{lll}
  \abovespace\belowspace
  \begin{minipage}{0.3\textwidth}
   \framebox{\parbox{\textwidth}{
    \centering
    \includegraphics[width=1\textwidth]{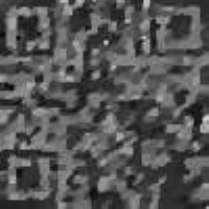} \\
    randomQR-0 ($100 \times 100$)
   }}
  \end{minipage}
  &
  \begin{minipage}{0.3\textwidth}
   \framebox{\parbox{\textwidth}{
    \centering
    \includegraphics[width=1\textwidth]{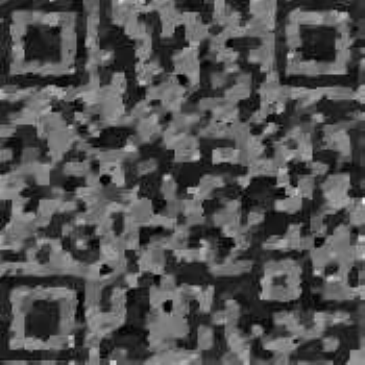} \\
    randomQR-1 ($175 \times 175$)
   }}
  \end{minipage}
  &
  \begin{minipage}{0.3\textwidth}
   \framebox{\parbox{\textwidth}{
    \centering
    \includegraphics[width=1\textwidth]{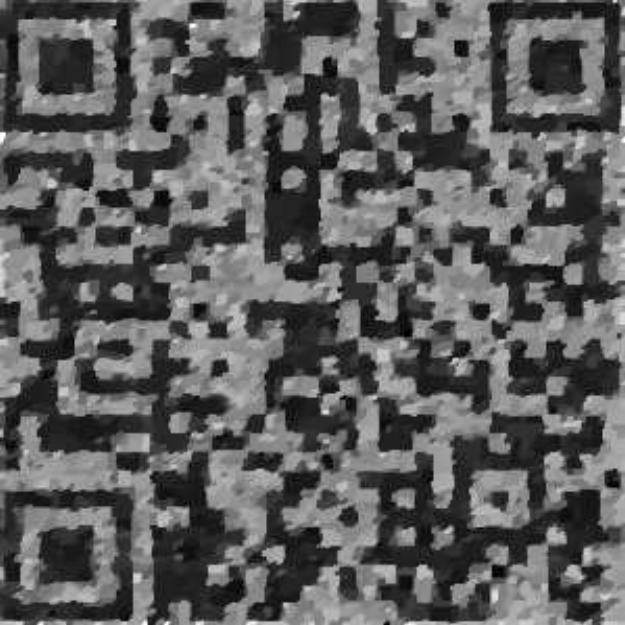} \\
    randomQR-2 ($300 \times 300$)
   }}
  \end{minipage}
  \vspace{0.2cm}
  \\
  \begin{minipage}{0.3\textwidth}
   \framebox{\parbox{\textwidth}{
    \centering
    \includegraphics[width=1\textwidth]{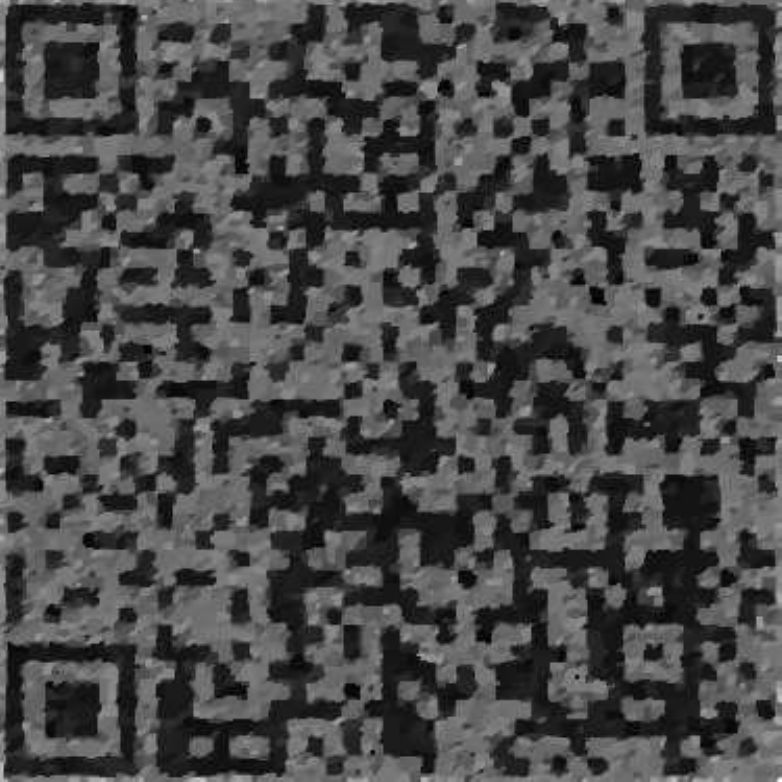} \\
    randomQR-3 ($375 \times 375$)
   }}
  \end{minipage}
  &
  \begin{minipage}{0.3\textwidth}
   \framebox{\parbox{\textwidth}{
    \centering
    \includegraphics[width=1\textwidth]{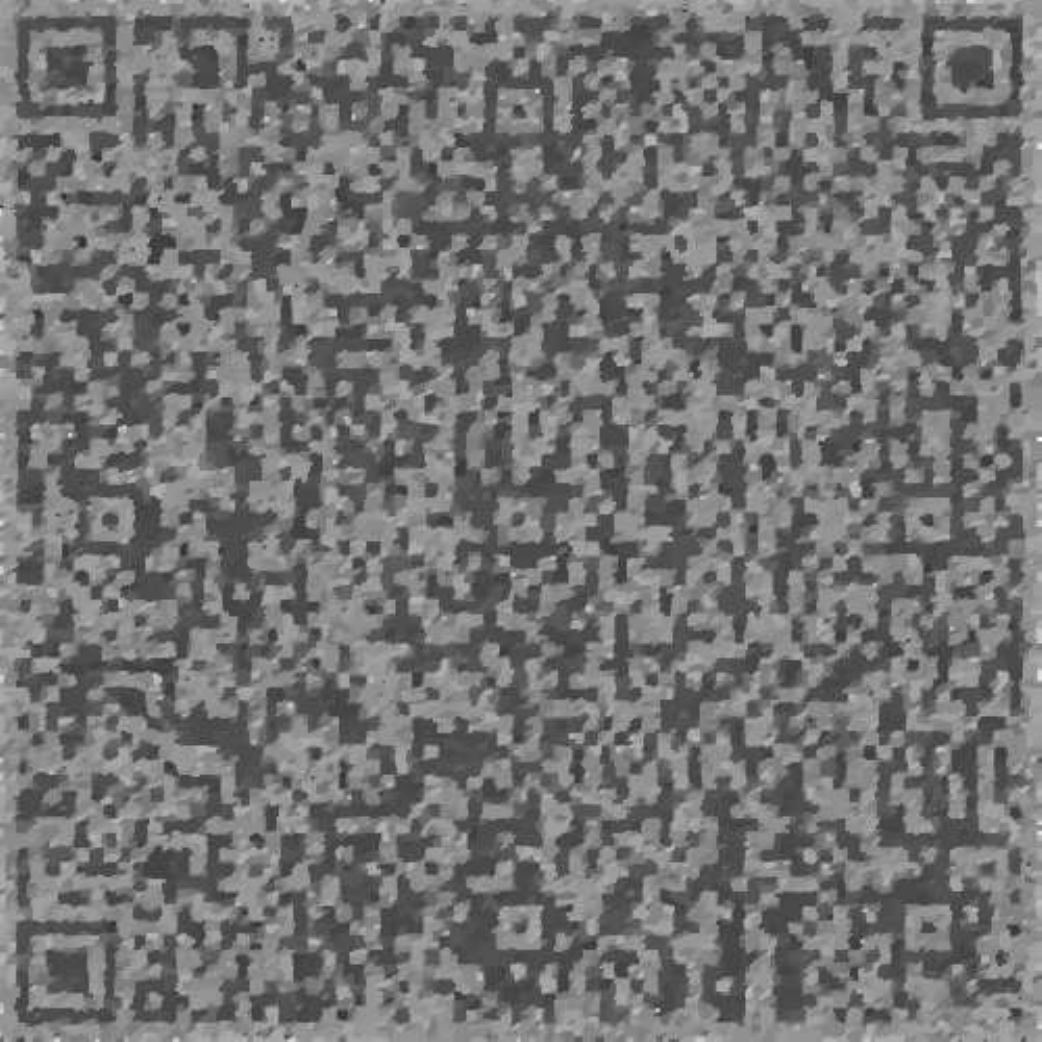} \\
    randomQR-4 ($500 \times 500$)
   }}
  \end{minipage}
  &
  \begin{minipage}{0.3\textwidth}
   \framebox{\parbox{\textwidth}{
    \centering
    \includegraphics[width=1\textwidth]{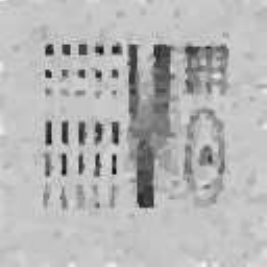} \\
    shape ($128 \times 128$)
   }}
  \end{minipage}
  \vspace{0.2cm}
  \\
  \abovespace\belowspace
  \begin{minipage}{0.3\textwidth}
   \framebox{\parbox{\textwidth}{
    \centering
    \includegraphics[width=1\textwidth]{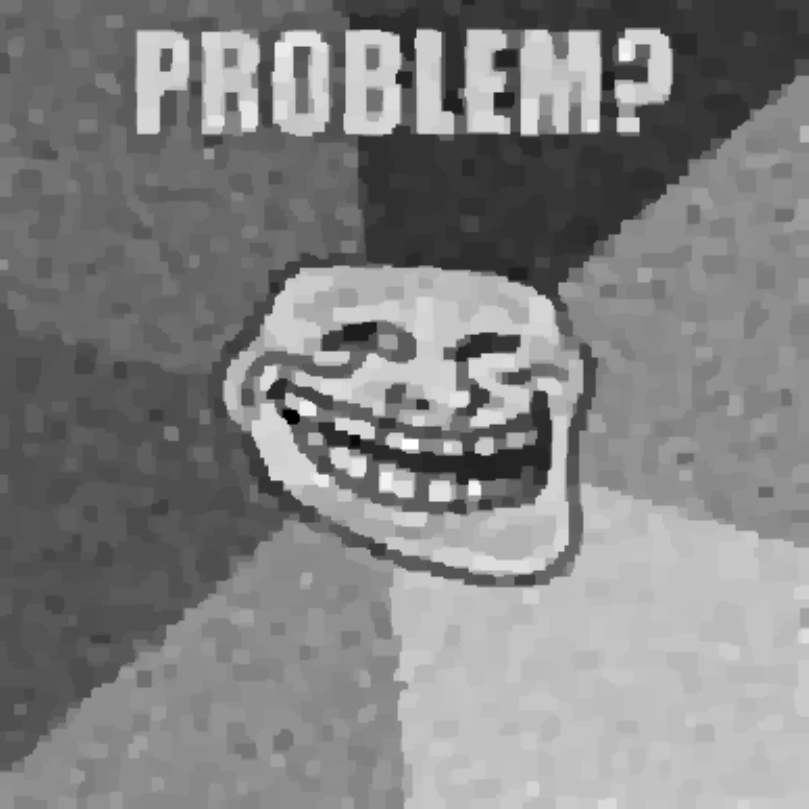} \\
    trollface ($388 \times 388$)
   }}
  \end{minipage}
  &
  \begin{minipage}{0.3\textwidth}
   \framebox{\parbox{\textwidth}{
    \centering
    \includegraphics[width=1\textwidth]{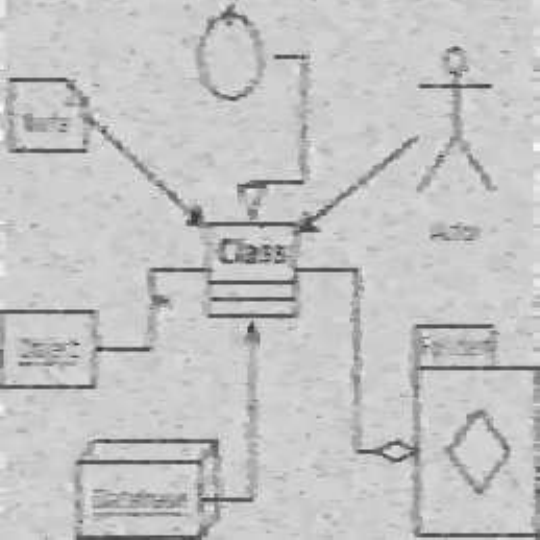} \\
    diagram ($259 \times 259$)
   }}
  \end{minipage}
  &
  \begin{minipage}{0.3\textwidth}
   \framebox{\parbox{\textwidth}{
    \centering
    \includegraphics[width=1\textwidth]{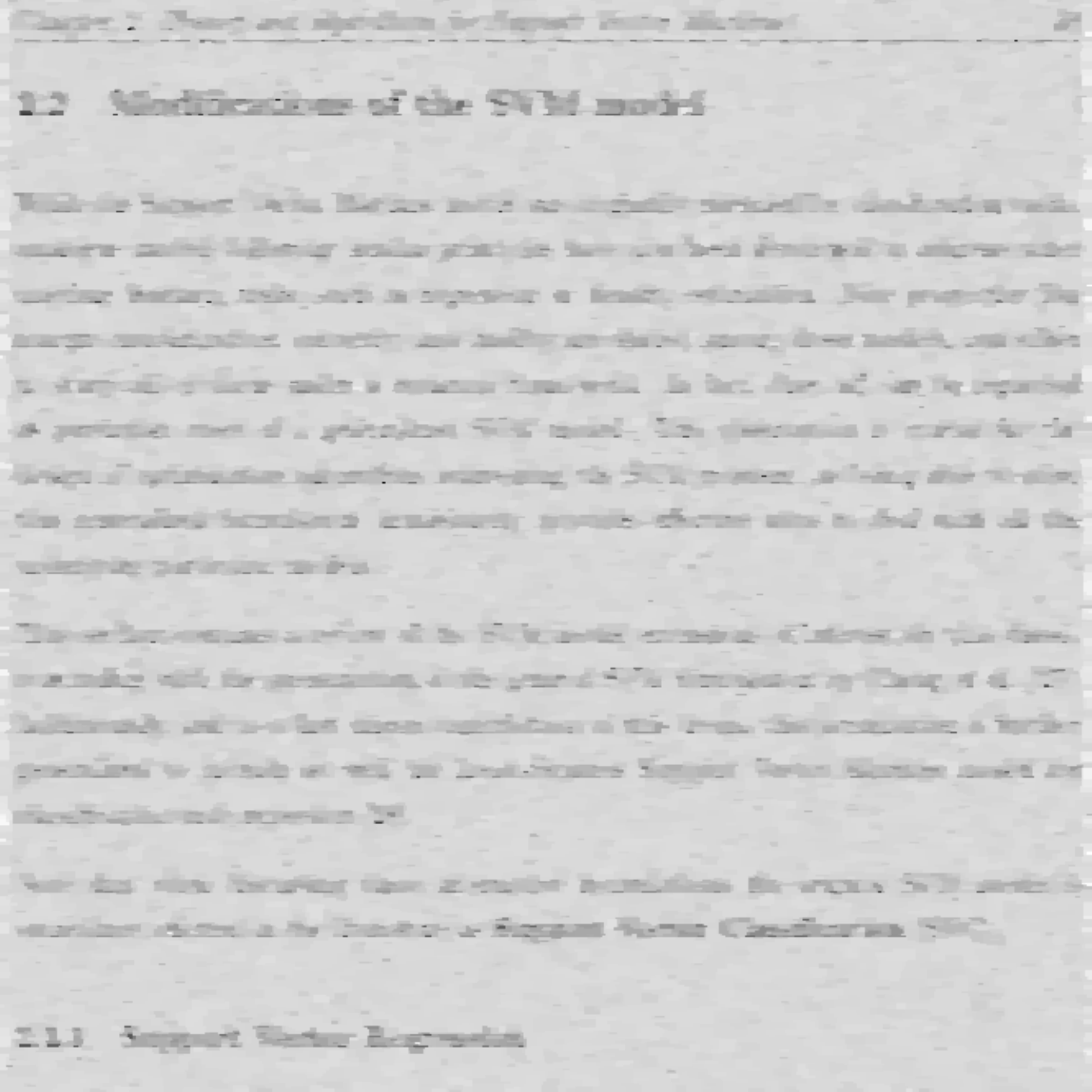} \\
    text ($665 \times 665$)
   }}
  \end{minipage}
  \vspace{0.2cm}
  \\
  \abovespace\belowspace
  \begin{minipage}{0.3\textwidth}
   \framebox{\parbox{\textwidth}{
    \centering
    \includegraphics[width=1\textwidth]{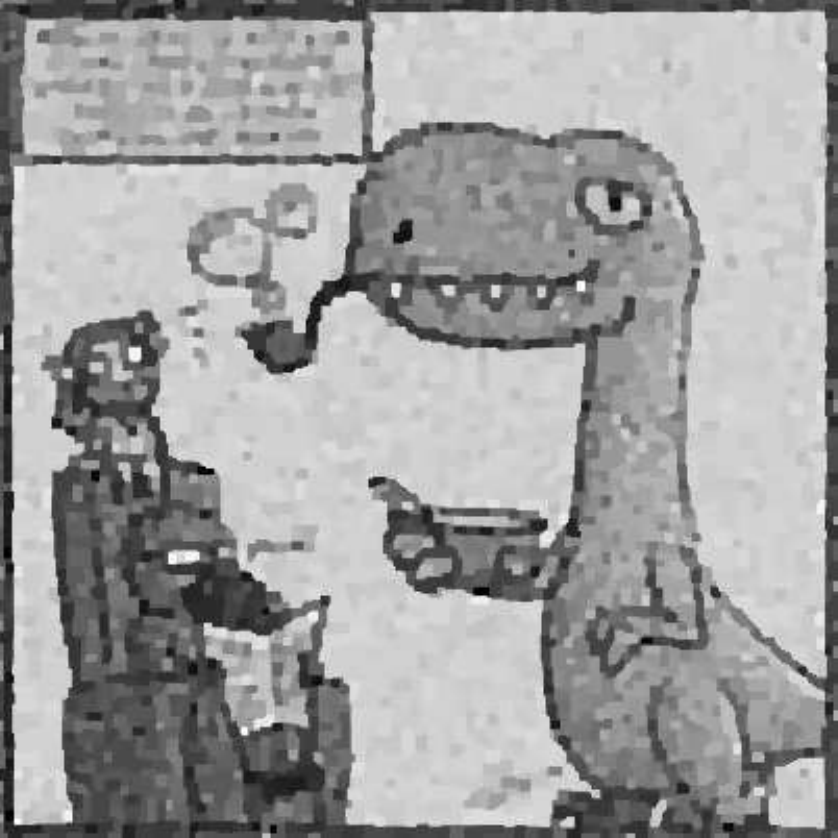} \\
    comic ($402 \times 402$)
   }}
  \end{minipage}
  &
  \begin{minipage}{0.3\textwidth}
   \framebox{\parbox{\textwidth}{
    \centering
    \includegraphics[width=1\textwidth]{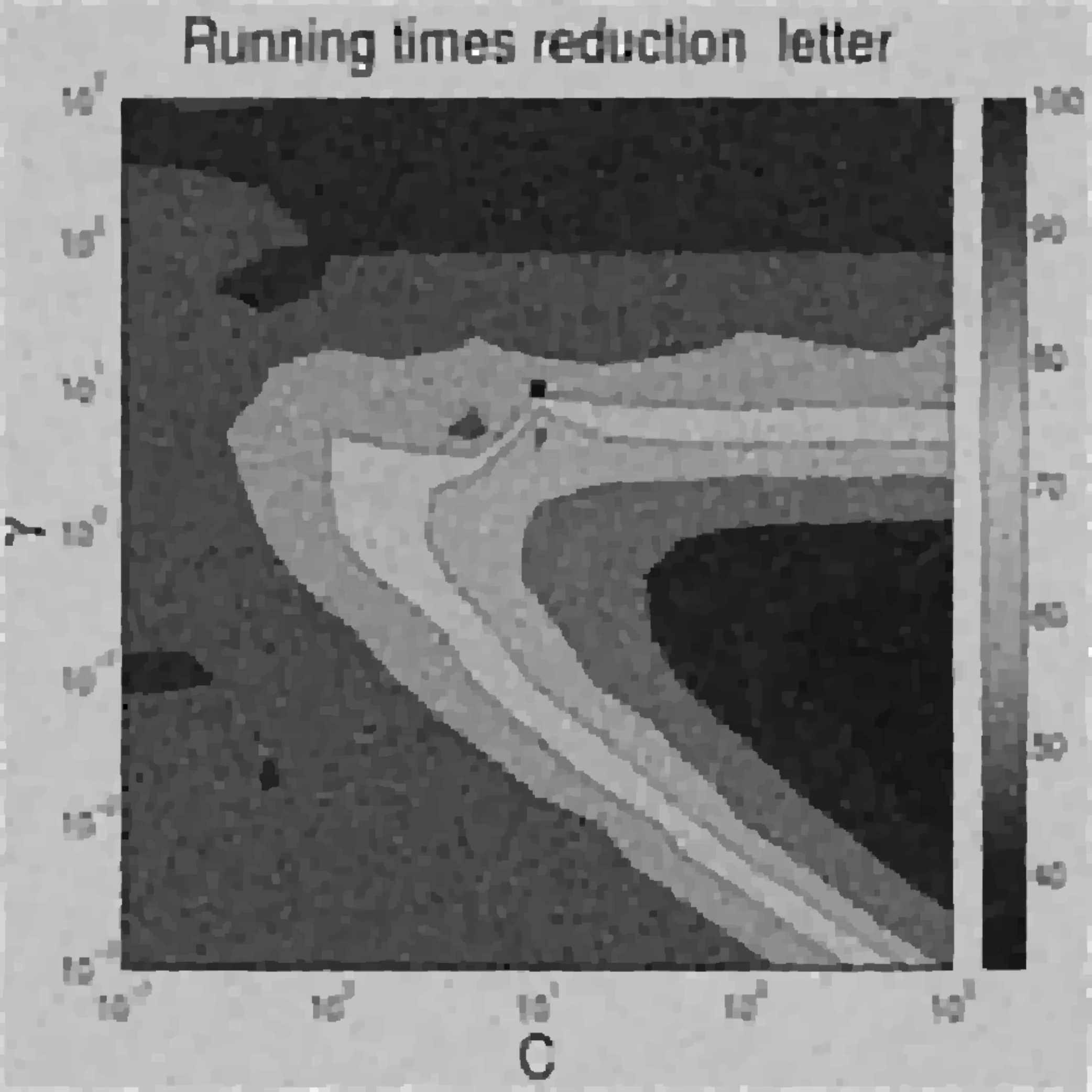} \\
    contour ($1000 \times 1000$)
   }}
  \end{minipage}
  &
  \begin{minipage}{0.3\textwidth}
   \framebox{\parbox{\textwidth}{
    \centering
    \includegraphics[width=1\textwidth]{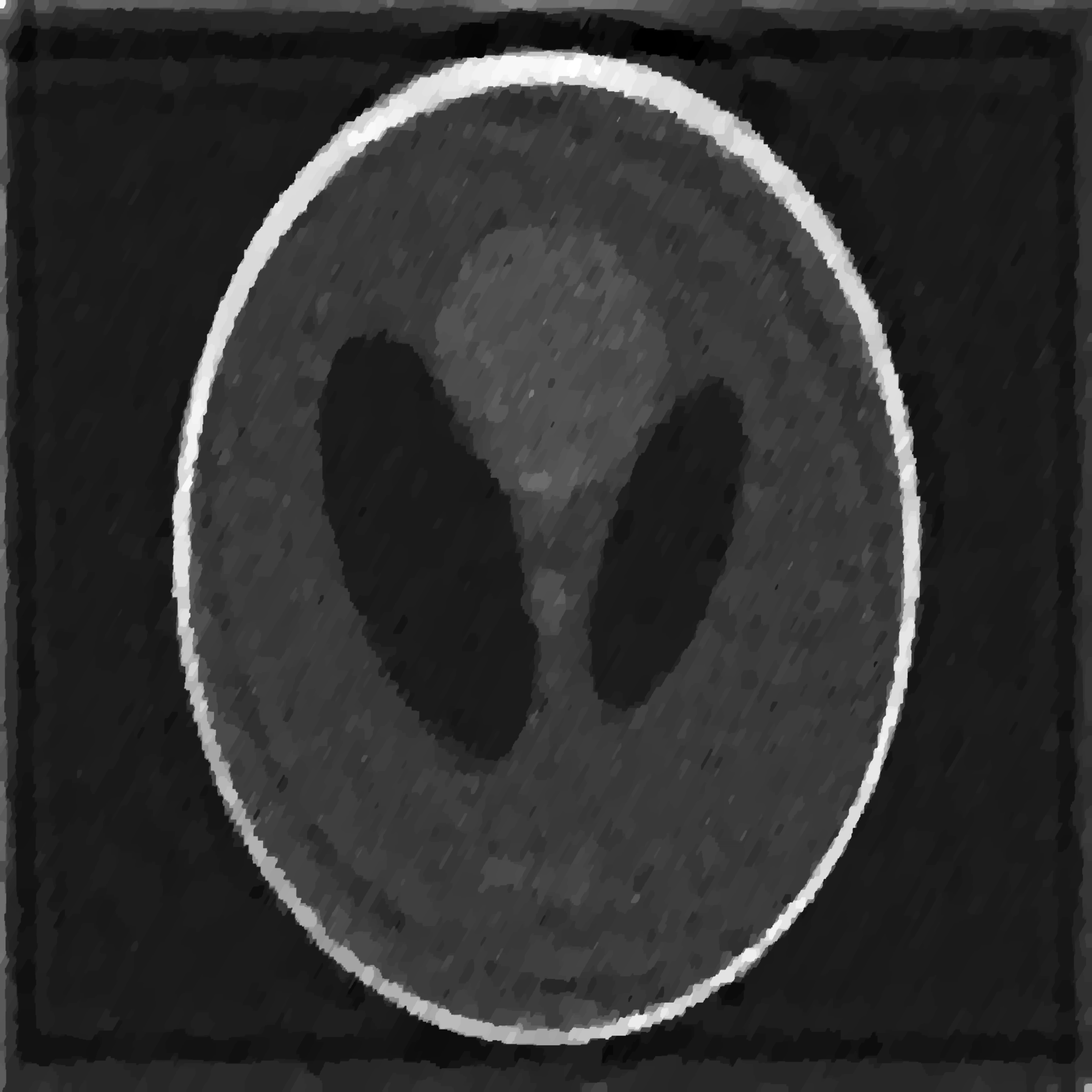} \\
    phantom ($1713 \times 1713$)
   }}
  \end{minipage}
  \end{tabular}

  \caption{
    Deconvolution results for the test images.
    \label{fig:imagesDeblurred}
  }
\end{figure}
\pagestyle{empty}
\begin{figure}[htbp]
 \centering
 \includegraphics[width = 0.98\textwidth]{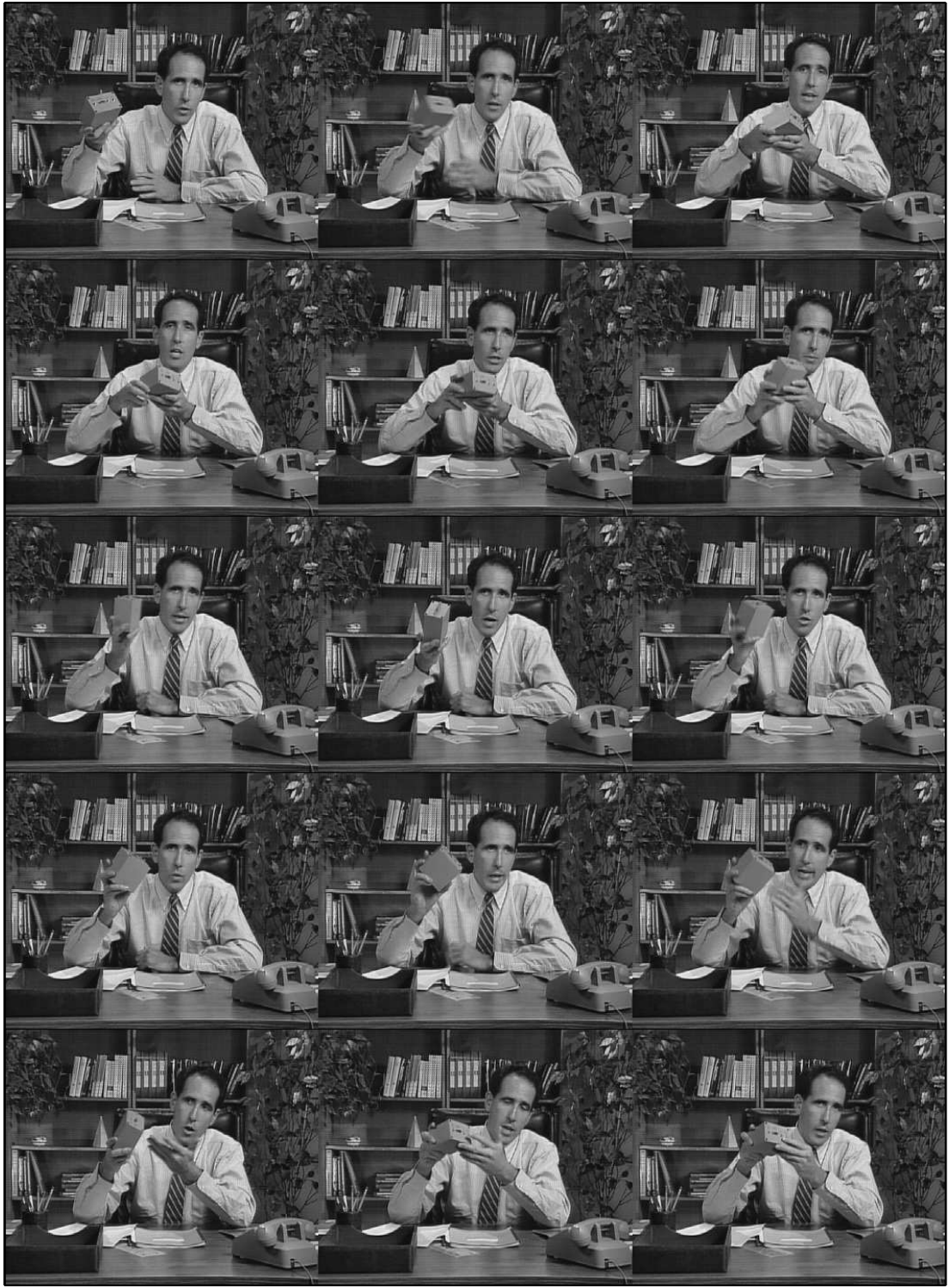}
 \caption{
  \label{fig:videoClean}
  A selection of frames from the {\it salesman} video sequence.
 }
\end{figure}

\begin{figure}[htbp]
 \centering
 \includegraphics[width = 0.98\textwidth]{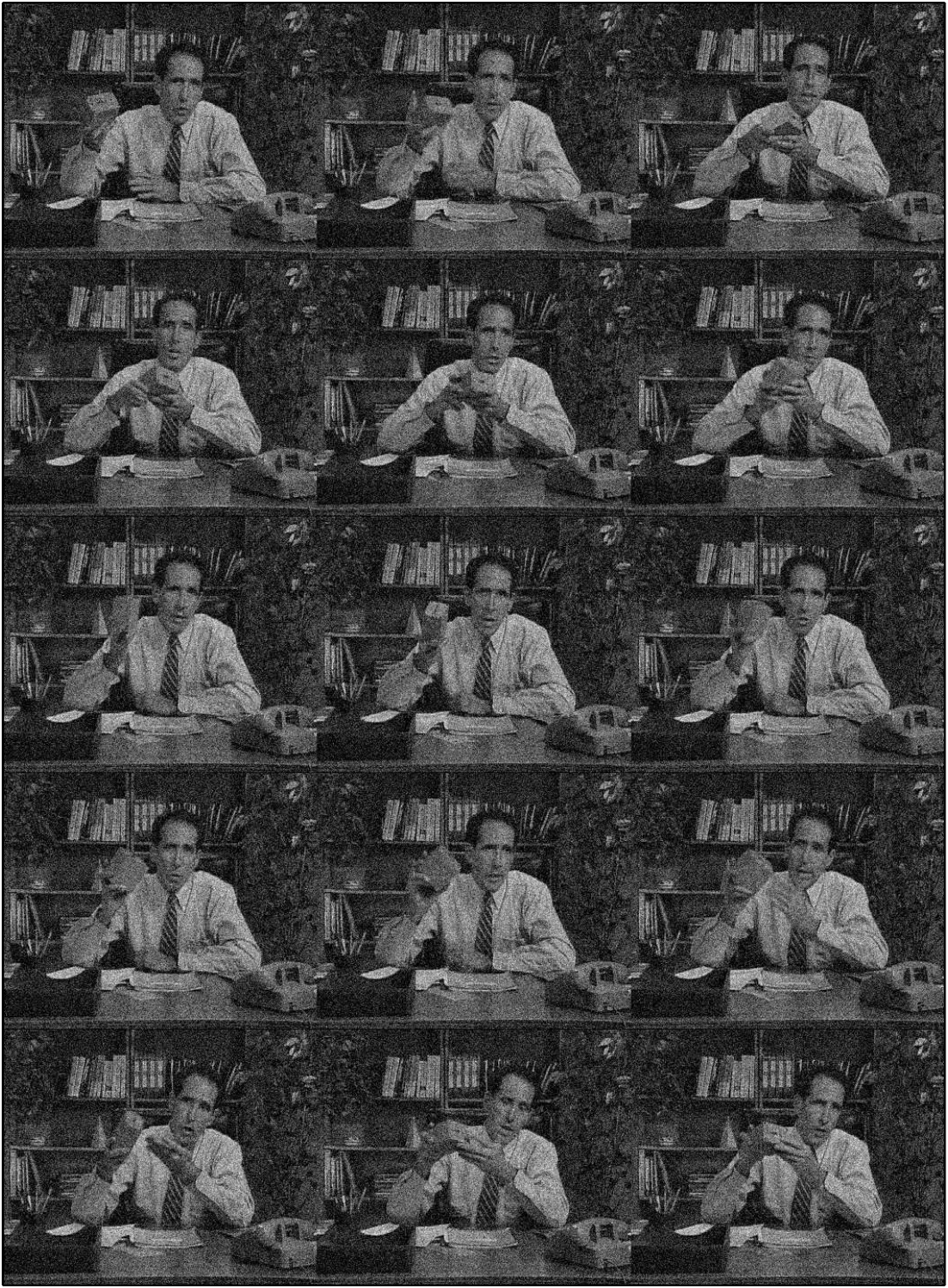}
 \caption{
  \label{fig:videoNoisy}
  Noisy frames from the {\it salesman} video sequence.
 }
\end{figure}

\begin{figure}[htbp]
 \centering
 \includegraphics[width = 0.98\textwidth]{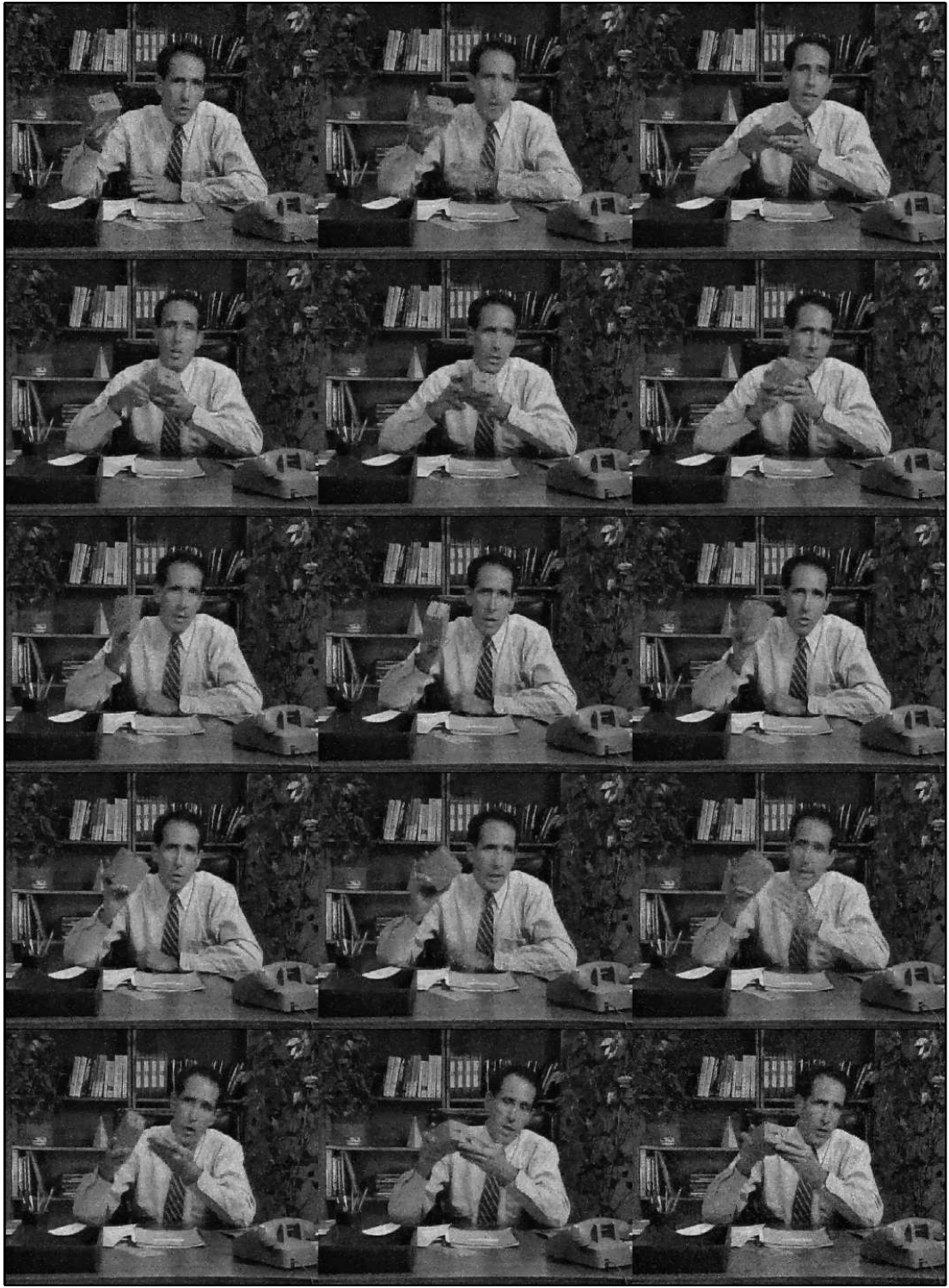}
 \caption{
  \label{fig:videoAniso}
  Denoised frames from the {\it salesman} video sequence.
 }
\end{figure}

\end{document}